\documentclass[twoside,11pt]{article}

\usepackage{jmlr2e}
\usepackage{graphicx}
\usepackage{epsfig,epstopdf,subfigure}
\usepackage{MnSymbol,bm}
\usepackage{tabularx,arydshln}
\usepackage{amsmath,amsfonts}
\usepackage{cite}
\usepackage{url}
\usepackage{comment}
\usepackage{appendix}
\usepackage{algorithm}
\usepackage{algpseudocode}
\usepackage{bigints}
\usepackage{xcolor}

\usepackage{thm-restate}

\newtheorem{assumption}[theorem]{Assumption}

\firstpageno{1}

\begin{document}
\title{Sparsity Induced Identifiability in  Matrix Tri-Factorisation }

\author{\name Tingting  Mu \email tingting.mu@manchester.ac.uk\\
       \addr Department of Computer Science\\
       University of Manchester\\
       Manchester M13 9PL, UK 
        }

\editor{ }

\maketitle

\begin{abstract}

Matrix factorisation is a fundamental tool for exploiting low-dimensional structure in high-dimensional data, with applications  such as data compression, denoising, structure discovery, interpretable representation learning, and dimensionality reduction. 
Compared to conventional two-factor models, matrix tri-factorisation provides greater modelling flexibility, while sparsity constraints often improve both interpretability and recovery performance. 
Although the role of sparsity has been extensively studied for two-factor matrix factorisation, rigorous theoretical guarantees for general real-valued matrix tri-factorisation remain largely unexplored. 
To address this gap, we establish, to the best of our knowledge, the first rigorous theoretical study for sparsity-induced identifiability in general real-valued matrix tri-factorisation. 
Our analysis is enabled by a novel decomposition strategy that transforms the original problem into two coupled auxiliary factorisation problems, while preserving the structural information necessary to the recovery of the original factor matrices from the observations.
Building upon this decomposition, we derive   recovery guarantees and structural consistency results that  characterise how coefficient sparsity influences the sufficient recovery conditions, convergence behaviour, spectral approximation error,  high-probability bounds, and structure preservation.
Comprehensive Monte Carlo experiments validate the proposed theory and demonstrate close agreement between the theoretical results and empirical observations.

\end{abstract}

\section{Introduction}

In machine learning,  data are  commonly represented as a matrix  $\textbf{X}\in \mathbb{R}^{n\times m}$, where the rows correspond to  $n$  objects described in an $m$-dimensional space, or  more generally, where the entries encode relationships between $n$ row objects and $m$ column objects.
In the real world, however, the key information contained in data is often governed by a small number of degrees of freedom, implying an underlying low-dimensional structure.
A classical  approach to uncover such   structure  is matrix  factorisation (also referred to as matrix decomposition), which seeks a low-rank approximation  to the data matrix \citep{Ste00, Kol09, Ste18,Hae20}.
As an effective means of exploiting low-dimensional structure, factorisation-based methods have become fundamental tools for data  compression, denoising, structure discovery, interpretable representation learning and dimensionality reduction  \citep{Bok15, Ste18, Cic15, Hua24}.
Their success has led to  broad applications  in signal and image processing \citep{Hae20}, computational biology \citep{Ste18, Hui23, Liu25},  chemistry \citep{Flo22,Gu24},   drug discovery \citep{Li25}, and, more recently, neural network compression and pruning \citep{Boz25}. 

The simplest form of matrix factorisation, known as the two-factor model,   approximates the data matrix as   $\textbf{X} \approx \textbf{A}\textbf{B}^T$, where the   factor matrices   $\textbf{A}\in \mathbb{R}^{n\times k}$ and $\textbf{B}\in \mathbb{R}^{m\times k}$  provide  latent representations of the row and column objects, respectively.
A more expressive alternative is the tri-factorisation  model   $\textbf{X} \approx \textbf{A}\textbf{S}\textbf{B}^T$, which introduces an additional latent relation matrix  $\textbf{S}\in \mathbb{R}^{k\times k}$ to explicitly model  the interactions between the latent representations. 
Depending on the application,  structural constraints, such as sparsity, non-negativity  and orthogonality, can be imposed on the  factor matrices to improve interpretability, identifiability, predictive performance, and the implicit disentanglement of latent information  (see more details in Section \ref{intr_mf}).

In this work, we focus on sparsity.  Among these constraints,  sparsity has consistently been shown to provide substantial gains in interpretability, recovery accuracy, predictive performance, and computational efficiency across a wide range of matrix factorisation models \citep{Zha15,Jos16,Ber24,Zul23}.
From a theoretical perspective, the role of sparsity in two-factor matrix factorisation has been  studied extensively, particularly in relation to solution uniqueness and model identifiability \citep{Tro04,Can_Tao06,Gri10,Aga14,Aga17,Coh19}.
For instance, these studies investigate how sparsity influences the extent to which the factor matrices $\mathbf{A}$ and $\mathbf{B}$ (or their underlying salient structures) can be recovered from the observed matrix $\mathbf{X}$. 
In contrast, research on   matrix tri-factorisation  under sparsity constraints has focused primarily on algorithmic development and empirical validation \citep{Kim12,Bil16,Zul23}, with comparatively little attention devoted to its theoretical foundations.
Only recently have theoretical identifiability results emerged for specialised matrix tri-factorisation models, such as those under the Boolean and non-negative settings \citep{Kol25,Saha25}, highlighting the importance of structural constraints in establishing identifiability.
However, for general real-valued matrix tri-factorisation, theoretical guarantees for solution uniqueness, model identifiability, and algorithmic recovery remain largely unexplored.

To address this fundamental theoretical gap, we   provide, to the best of our knowledge, the first rigorous analysis of sparsity-induced identifiability for general real-valued matrix tri-factorisation.
Compared with the classical two-factor model $\textbf{X} \approx \textbf{A}\textbf{B}^T$,  introducing the latent relation matrix $\textbf{S}$ in the tri-factorisation $\textbf{X} \approx \textbf{A}\textbf{S}\textbf{B}^T$ creates additional indeterminacies that make existing identifiability analyses inapplicable.
To overcome this challenge, we develop a   decomposition strategy that transforms the original tri-factorisation into two coupled auxiliary two-factor matrix factorisation problems.
Building upon this decomposition and an accompanying algorithmic framework, we establish a new recovery theory that characterises how coefficient sparsity governs solution uniqueness and (partial) model identifiability in general real-valued matrix tri-factorisation.
Our main contributions are summarised as follows:
\begin{itemize}
    \item A novel decomposition strategy that transforms the   tri-factorisation into two  auxiliary two-factor  factorisation problems, enabling a tractable theoretical analysis while preserving the structural information required for recovering the original factor matrices.
    \item  Rigorous theoretical recovery guarantees underpinning sparsity-induced identifiability for the auxiliary  problems,  obtained through complementary analyses of an alternating optimisation procedure (Theorem  \ref{main_res}) and a spectral dictionary approximation  (Theorem \ref{main_res2}).
    Our analysis  characterises how coefficient sparsity affects the sufficient recovery conditions, convergence behaviour, spectral approximation error, and high-probability recovery guarantees.
    \item A rigorous proof of structural consistency between the original and auxiliary tri-factorisations (Theorem \ref{main_res3}), showing that the auxiliary factorisation faithfully preserves the support patterns of $\textbf{A}$ and $\textbf{B}$, while introducing only bounded changes to the row norms   and  pairwise row angles  of   $\textbf{A}$ and $\textbf{B}$, as well as  to  the conditioning of $\textbf{S}$. 
    This provides a rigorous justification for analysing the auxiliary factorisation as a structurally consistent surrogate for the original tri-factorisation.
    \item A collection of  structural properties of the matrices arising from the original and auxiliary problems, providing the essential ingredients for proving the recovery guarantees.
    \item Comprehensive empirical  validation of all principal theoretical bounds and recovery results through Monte Carlo experiments,   demonstrating close agreement between the theoretical results and  empirical observations.
\end{itemize}
 
\section{Background and Related Work}

\subsection{Matrix Factorisation}
\label{intr_mf}

Many real-world datasets can be naturally represented as matrices or higher-order tensors. 
Examples include high-dimensional biological data, where entries represent measurements such as   expression counts, methylation levels, and protein concentrations across samples \citep{Ste18}, and user-item interaction matrices in e-commerce recommender systems \citep{Bok15}. 
More generally, knowledge graphs can be represented as third-order tensors comprising binary matrix slices, each encoding a particular relation between entities \citep{Nic11}. 
Factorisation-based methods are well suited to analysing such structured data. In the remainder of this section, we review representative matrix factorisation models and the constraints commonly imposed on their factor matrices.

\subsubsection{Two-factor  Factorisation Models} 

The two-factor matrix factorisation  $\textbf{X} \approx \textbf{A}\textbf{B}^T$ is one of the most widely used formulations for analysing high-dimensional data. 
The columns of $\textbf{B}$  define $k$ latent factors  that reside in the same feature space as the row objects of $\textbf{X}$.
These latent factors capture the dominant structural patterns in the data.
Each row object  of $\textbf{X}$  is subsequently represented as a weighted sum of these latent factors, with the corresponding weights stored in the associated row of $\textbf{A}$,   commonly referred to as  a mixing vector  or a coefficient vector.
Choosing $k\ll \min(n,m) $ explicitly induces a low-dimensional structure.
Many classical methods, such as principal component analysis (PCA), independent component analysis (ICA), and non-negative matrix factorisation (NMF), can be expressed within this framework.

The same factorisation   $\textbf{X} \approx \textbf{A}\textbf{B}^T$ also underlies dictionary learning and sparse coding in  signal and image processing \citep{Aha06,Mai10,Aga14,Men16,Bao16}.
In this setting, the columns of $\textbf{B}$ are interpreted as dictionary atoms.   
Unlike PCA, ICA and NMF, which typically employ a  small number of latent factors, dictionary learning constructs an over-complete dictionary by using a large number of atoms.
Sparse coding complements this formulation by imposing row-wise sparsity on the coefficient matrix $\mathbf{A}$, ensuring that each row object of $\mathbf{X}$ is represented using only a small subset of the available atoms.
From this perspective,  PCA, ICA and NMF may be regarded as learning   under-complete dictionaries, whereas  dictionary learning typically adopts an over-complete one.
Additional low-dimensional structure may be  encouraged by imposing low-rank constraints  on the dictionary matrix $\textbf{B}$ \citep{Par22}, or, in some formulations, on the coefficient matrix $\textbf{A}$ \citep{Wan22}.
More recently, inspired by the success of deep learning, deep matrix factorisation  extends the two-factor model by recursively factorising one or both factor matrices, thereby enabling hierarchical representation learning with increased modelling capacity \citep{DEH21,Hua24}.
Overall, the two-factor model provides the fundamental building block for many matrix factorisation methods, which differ primarily in the structural assumptions imposed on their factor matrices.

\subsubsection{Tri-factorisation Models}  
While two-factor models employ a single latent space shared by the row and column objects, many applications benefit from learning separate latent representations together with explicit modelling of their interactions.
To achieve this, matrix tri-factorisation $\mathbf{X}\approx\mathbf{A}\mathbf{S}\mathbf{B}^T$ extends the two-factor model by introducing distinct latent spaces for the row and column objects of $\mathbf{X}$, together with an interaction matrix between them \citep{Cop17}.
Specifically, the columns of $\mathbf{A}$ and $\mathbf{B}$ define latent factors for representing the row and column objects, respectively.
Each entry of the interaction matrix $\mathbf{S}$ characterises the  strength of interaction between a pair of latent factors.
Low-dimensional structure   can be achieved by either selecting  a small number of latent factors, or  imposing an explicit low-rank constraint on   $\textbf{S}$.
For directed relational data, the asymmetry of $\mathbf{X}$ can be captured through an asymmetric   $\mathbf{S}$, whereas symmetric $\mathbf{X}$ naturally corresponds to a symmetric $\mathbf{S}$.
For multi-relational data represented as a third-order tensor   $\left\{\textbf{X}_i\in \mathbb{R}^{n\times n}\right\}_{i=1}^r$ between entities, a joint tri-factorisation  can be  written as  $\textbf{X}_i \approx \textbf{A}\textbf{S}_i\textbf{A}^T $, where a common    entity representation $\mathbf{A}\in\mathbb{R}^{n\times k}$ is shared across all relation slices, while each relation is associated with its own latent relation matrix $\mathbf{S}_i\in\mathbb{R}^{k\times k}$.
This   formulation underlies the RESCAL algorithm for multi-relational learning \citep{Nic11}. 

\subsubsection{Constraints on Factor Matrices} 

Matrix factorisation methods frequently impose structural constraints  on the factor matrices,  such as sparsity, non-negativity, orthogonality, stochasticity,   low-rankness, and their combinations.
Such constraints have demonstrated considerable effectiveness in improving both representation quality and interpretability.  
For instance, in the  two-factor model $\textbf{X} \approx \textbf{A}\textbf{B}^T$,  sparse coding enforces row-wise sparsity on the coefficient matrix $\textbf{A}$, whereas NMF constrains both $\textbf{A}$ and $\textbf{B}$ to be element-wise non-negative.  
Sparse NMF further combines these requirements by imposing simultaneous non-negativity and sparsity on one or both factor matrices \citep{Hoy04}. 
Another example arises from the formulation  of  k-means clustering based on matrix factorisation,  where $\textbf{A}$ is constrained  either  to  have  its rows lying on the probability simplex, or to  satisfy non-negativity together with orthogonality \citep{Aro13}.
The more recent work on triple component matrix factorisation extends the two-factor model to recover both common and unique latent factors from noisy observations, with orthogonality constraints enforcing independence between the common and unique factors \citep{Shi04}.
In matrix tri-factorisation,  non-negativity has been imposed on $\textbf{S}$  together with  stochastic constraints on the rows of $\textbf{A}$  to improve interpretability \citep{Sif13}.
Alternatively,   the two-way DEDICOM has imposed column-wise orthogonality   on $\textbf{A}$   \citep{Bad06}.
Working with the more general model $\textbf{X} \approx \textbf{A}\textbf{S}\textbf{B}^T$, the scalable non-negative matrix tri-factorisation (NMTF)  imposes element-wise non-negativity  on all the three factor matrices  \citep{Cop17}.
Recently, \citet{Hou25} has introduced Frobenius-norm constraints on the left and right factor matrices and a non-negative diagonal constraint on the latent relation matrix to improve low-rank approximation accuracy.

Overall, the choice of constraints  is largely application dependent, and should be consistent with the intended downstream use of the learned latent representations.
Sparsity, orthogonality, and non-negativity are among the most commonly employed constraints, as they  encourage the latent factors to capture distinct and less redundant information.
Although such constraints do not explicitly enforce  statistical independence or group-theoretic decomposition typically associated with modern notions of disentangled representations, they may be viewed as weak  or indirect forms of disentanglement because they promote factor separation   \citep{Emi18,HLI24}. 
These constraints also frequently improve  the interpretability of the learned latent factors.  
The development of matrix factorisation algorithms  under such constraints draws primarily on techniques from linear algebra, constrained optimisation, and  Bayesian probabilistic modelling \citep{Mni07,Nak11}.

\subsection{Sparsity in Factor Matrices }
\label{intr_fp}
In matrix factorisation, sparsity refers to the presence of many zero entries in one or more factor matrices. 
It is one of the most widely adopted structural constraints because it not only improves practical performance and interpretability, but also plays a fundamental role in model identifiability and recovery guarantees.
Over the past decades, considerable progress has been made in both empirical studies and theoretical analyses of matrix factorisation under sparsity constraint.
The former demonstrates the practical benefits of sparsity across diverse applications, whereas the latter explains why sparsity gives rise to desirable   theoretical properties.

\subsubsection{Empirical Development} 

For the two-factor model $\textbf{X} \approx \textbf{A}\textbf{B}^T$, sparsity   is typically imposed on the coefficient matrix $\textbf{A}$.
Its practical advantages  have been demonstrated across a broad range of applications, including  signal and image processing, biological data analysis, and recommendation systems  \citep{Pey09, Ela10, Zha15, Ste18, Dai18}.
Representative approaches include sparse PCA \citep{Cad95,Jour10,Ber24}, sparse NMF \citep{Hoy04,Hei06}, and sparse coding  \citep{Fie94, Mai10,Aga16}.
Collectively, empirical studies have reported improvements in interpretability, recovery accuracy, predictive performance,  computational efficiency, and model compression \citep{Zha15,Jos16,Ber24, Boz25}.

Research on sparsity in matrix tri-factorisation $\textbf{X} \approx \textbf{A}\textbf{S}\textbf{B}^T$ has predominantly focused on developing practical algorithms and evaluating their empirical performance.
For instance, sparse NMTF (SNMTF) introduces sparsity together with non-negativity into all the factor matrices to improve the  identification of patient subgroups for therapeutic strategy development  in cancer genomics  \citep{Kim12}.  
Subsequent work has refined  SNMTF   through  improved problem formulations  and more efficient optimisation algorithms, such as the projected variant of the fast iterative shrinkage-thresholding algorithm (FISTA), yielding enhanced generalisation ability and factorisation accuracy  \citep{Bil16}.
More recently,  sparse RESCAL   imposes sparsity   on  the shared entity   representation   $\textbf{A}$ to improve  interpretability  in   knowledge graph analysis \citep{Zul23}. 

\subsubsection{Theoretical Development} 

Beyond its empirical success, sparsity is also of central theoretical importance in matrix factorisation.
For the two-factor model $\mathbf{X}\approx\mathbf{A}\mathbf{B}^T$, strong connections between coefficient sparsity and solution uniqueness (or model identifiability) have been rigorously established under suitable conditions.
These theoretical developments are supported by extensive research spanning compressed sensing, dictionary identification, identifiability and uniqueness analysis, sample complexity theory, and convergence analysis for dictionary learning algorithms \citep{Tro04,Can06,Can_Tao06,Cand07,Can08, Gri10, Vai11,  Gri15, Gri_Jen15,Aga14, Aga16, Aga17, Coh19}. 
In contrast to the well-established theories for two-factor models, analogous sparsity-based results for matrix tri-factorisation remain limited.
Recently, identifiability results have been established for Boolean matrix tri-factorisation \citep{Kol25}.
Another recent study extends the identifiability theory of NMF to nonnegative Tucker decomposition (nTD), from which identifiability results for NMTF follow by treating NMTF as the order-2 case of nTD \citep{Saha25}.
These works establish identifiability under specialised structural assumptions, in which sparsity is one contributing ingredient.
%
%
To the best of our knowledge, a significant theoretical gap remains in understanding how sparsity governs identifiability and recovery in general real-valued matrix tri-factorisation.

\section{Problem Description}

\subsection{Notations}

Given a vector $\bm x$, we use $x_i$ to denote its $i$-th entry. 
Unless otherwise specified, all vectors are assumed to be column vectors.
Given a matrix $\textbf{X}$, we use $\textbf{X}^{(i)}$, $\textbf{X}_i$ or $(\textbf{X})_i$,  and $x_{ij}$ or $(X)_{ij}$  to denote its $i$-th row, $i$-th column, and $ij$-th entry, respectively, unless otherwise specified.  
Accordingly, for a matrix product $\textbf{X}_1\textbf{X}_2\cdots\textbf{X}_n$, its $i$-th row, $i$-th column, and $ij$-th entry  are denoted by $(\textbf{X}_1\textbf{X}_2\cdots\textbf{X}_n)^{(i)}$, $(\textbf{X}_1\textbf{X}_2\cdots\textbf{X}_n)_i$,  and $(\textbf{X}_1\textbf{X}_2\cdots\textbf{X}_n)_{ij}$, respectively. 
Given a square matrix $\textbf{X}$, we use $\textbf{X}^{(i)}_{\setminus i}$ and  $(\textbf{X})^{(i)}_{\setminus i}$ interchangeably to denote the vector consisting of all the off-diagonal entries in the $i$-th row of $\textbf{X}$.
The index set of the nonzero entries  of a vector (or matrix) is denoted by $supp(\cdot)$.
The cardinality of a set $I$ is denoted  by $|I|$.
Given an index set $I$, we use $\textbf{X}^{I}$ (and $\textbf{X}_{I}$) to denote the sub-matrix formed by  selecting rows (and columns) of $\textbf{X}$   indexed by $I$. 
Suppose that a  set $I$ is partitioned into two disjoint subsets $I_1$ and $I_2$, i.e., $I_1 \cap I_2 =\emptyset $.
Then $I_2$ is the relative  complement of  $I_1$ in  $I$, denoted by $I_2=I \setminus I_1$.
The index set obtained by excluding the $i$-th row  (or column)  in a matrix   is dentoed by $I = \setminus i$.
For a vector, we use $\|\cdot\|_1$, $\|\cdot\|_2$,  and $\|\cdot\|_{\infty}$ to denote its  $l_1$-norm,  $l_2$-norm, and maximum norm, respectively.
Given two column vectors $\textbf{x}, \textbf{y} \in \mathbb{R}^d$, their cosine similarity  is defined by $\textmd{cos}(\textbf{x}, \textbf{y} ) = \frac{\textbf{x}^T\textbf{y}}{\|\textbf{x}\|_2\|\textbf{y}\|_2}$.
For a matrix, we use $\|\cdot\|_2$ to denote  its spectral norm (i.e., its largest singular value),   $\|\cdot \|_{\infty}$ to denote its maximum absolute entry,  and $\|\cdot\|_F$ the Frobenius norm.
We use $\circ$ to denote the Hadamard product.
Given a matrix $\textbf{X}$, its singular values  are denoted by $\sigma(\textbf{X})$, with subscripts used   to distinguish particular singular values, e.g., $\sigma_{\max}(\textbf{X})$  and $\sigma_{\min}(\textbf{X})$. 
For a matrix $\textbf{X}$,   its condition number is denoted   by $\kappa(\textbf{X})=\frac{\sigma_{\max}(\textbf{X})}{\sigma_{\min}(\textbf{X})}$, and its Moore-Penrose pseudo-inverse by $ \textbf{X}^{\dagger}$.
For a vector input,   $\text{diag}(\cdot)$ constructs a diagonal matrix whose diagonal entries are given by  this vector, while, for a matrix input,   $\text{diag}(\cdot)$ extracts the diagonal entries and returns these as a column vector. 
The $n\times n$ identity matrix   is denoted by $\textbf{I}_n$. 
The $n$-dimensional column vector with unit entries is denoted by $\textbf{1}_n$. 
We use $[n]$ to denote the index set $\{1,2,  \ldots n\}$. 
The ceiling function $\lceil \cdot \rceil $ rounds its argument up to the nearest integer, while the floor function  $\lfloor \cdot \rfloor $ rounds its argument down   to the nearest integer. 
%

\subsection{Matrix Tri-Factorisation and Coefficient Sparsity}

Throughout this work, we interpret a data matrix as a pairwise relation matrix  encoding associations between the row and column objects, and denote it by $\textbf{R}=[r_{ij}]\in \mathbb{R}^{n\times m}$.
The objective is to recover the three factor matrices $\textbf{A}=[\alpha_{ij}]\in \mathbb{R}^{n\times k}$, $\textbf{B}=[\beta_{ij}]\in \mathbb{R}^{m\times k}$, and $\textbf{S}=[s_{ij}]\in \mathbb{R}^{k\times k}$  such that $\textbf{R} \approx \textbf{A}\textbf{S}\textbf{B}^T$.
Under this model, each  entry $r_{ij}$ is  approximated by the bilinear form  $r_{ij} \approx   \sum_{t,h=1}^k \alpha_{it}\beta_{jh} s_{th}$.
Without additional assumptions,  matrix tri-factorisation  is generally non-identifiable, since multiple sets of factor matrices can produce the same data matrix.
Meaningful and identifiable solutions require appropriate structural constraints on  the  factor matrices. 
Classical examples include the orthogonality constraints,  underlying the singular value decomposition (SVD).
In this paper, we investigate  how row-wise sparsity imposed on  the  coefficient matrices  $\textbf{A}$ and $\textbf{B}$ influences the  recovery   of the three factor matrices $\textbf{A}$, $\textbf{B}$ and $\textbf{S}$ from the observed matrix $\textbf{R}$.

\paragraph{Embedding  Interpretation:} 

The tri-factorisation model $\textbf{R} \approx \textbf{A}\textbf{S}\textbf{B}^T$ admits a natural latent embedding interpretation that reveals the geometric structure underlying the observed relations.
To expose this structure, we consider the truncated  SVD  of the latent relation matrix $\textbf{S}$  retaining the largest  $d$ singular values, i.e., $\textbf{S}\approx \textbf{U}_s\bm\Sigma_s\textbf{V}_s^T =\left(\textbf{U}_s\bm\Sigma_s^{\frac{1}{2}}\right)\left(\textbf{V}_s\bm\Sigma_s^{\frac{1}{2}}\right)^T$, where $\textbf{U}_s, \textbf{V}_s \in \mathbb{R}^{k\times d}$ have orthogonal columns,  and $\bm\Sigma_s \in\mathbb{R}^{d\times d} $ is diagonal with nonnegative singular values as its diagonal entries.
Defining  $\textbf{U}=\textbf{U}_s\bm\Sigma_s^{\frac{1}{2}} $ and $\textbf{V}= \textbf{V}_s\bm\Sigma_s^{\frac{1}{2}} $  and letting  $\{\bm u_t\}_{t=1}^k $ and $ \{\bm v_h\}_{h=1}^k$ denote their respective rows  in $\mathbb{R}^{  d}$, each entry of $\textbf{S}$ is approximated as $s_{th} \approx \bm u_t ^T\bm v_h$.
Thus, the latent relation matrix itself admits an embedding representation, in which every latent interaction is approximated by the inner product of two low-dimensional latent embedding vectors.
Substituting this representation into the tri-factorisation yields
\begin{equation} 
 \label{eq:sim}
r_{ij} \approx     \sum_{t,h=1}^k \alpha_{it}\beta_{jh} \bm u_t^T  \bm v_h = \left(\sum_{t=1}^k \alpha_{it}\bm u_t\right)^T \left(\sum_{h=1}^k \beta_{jh}\bm v_h\right) = \bm x_i^T   \textbf{y}_j,
\end{equation}
where  $\bm x_i =\sum_{t=1}^k \alpha_{it}\bm u_t $ and $ \bm y_j  = \sum_{h=1}^k \beta_{jh}\bm v_h$.
Eq. (\ref{eq:sim}) shows that matrix tri-factorisation is equivalent to approximating each observed relation by the inner product of two embeddings $\bm x_i$ and $\bm y_j$.
These   embeddings $\{\bm x_i\}_{i=1}^n$ and  $\{\bm y_j\}_{j=1}^m$ therefore provide low-dimensional representations of the row and column objects, respectively, while the vectors   $\{\bm u_t\}_{t=1}^k$ and $\{\bm v_h\}_{h=1}^k$ constitute latent basis vectors  from which the object embeddings are constructed through the coefficient matrices   $\mathbf{A}$ and $\mathbf{B}$.

This embedding formulation also admits a natural interpretation from the perspective of dictionary learning and sparse coding.
The  latent basis vectors $\{\bm u_{t} \}_{t=1}^{k}$ and $\{\bm v_{h} \}_{h=1}^{k}$  can be viewed as the atoms of two dictionaries for representing the row and column objects, respectively.
The rows of $\textbf{A}$ and $\textbf{B}$ provide the corresponding coding coefficients that linearly combine these atoms to construct the object embeddings.
The  latent relation entries $\{s_{th} \}_{t,h=1}^{k}$  characterise  the interactions between pairs of dictionary atoms, each approximated by the  inner product  $s_{th}\approx \bm u_t ^T\bm v_h$.
When the rows of $\mathbf{A}$ and $\mathbf{B}$ are sparse, each object embedding is constructed only from a small subset of dictionary atoms determined by its coding coefficients, thereby highlighting the principal latent components responsible for the observed relations.
This naturally yields compact and interpretable representations of both the row and column objects.
Such sparse coding perspective provides the conceptual basis for our theoretical analysis developed in the subsequent sections.

\section{A Decomposition  Strategy for Sparsity Analysis }
\label{sec:strategy}

A standard approach for studying model identifiability and solution uniqueness is to investigate whether the exact data-generating model can be recovered from the observed data (up to the expected ambiguity), or, more generally, whether certain invariant properties of the underlying model can be reliably recovered.
For matrix tri-factorisation, it is well known that the model is non-identifiable in the absence of additional structural constraints. 
The problem therefore becomes one of identifying what properties of the factor matrices  are   recoverable, and  establishing the conditions under which such recovery is possible.
In particular, we focus on sparsity-induced identifiability in matrix tri-factorisation from a recovery-theoretic perspective.
Our objective is to determine how sparsity constraints restrict the set of admissible  solutions, and to what extent they enable the recovery of structurally salient properties of the true factor matrices.
The resulting analysis formalises the role of sparsity in inducing partial identifiability and  reducing intrinsic  ambiguity of the tri-factorisation solution space.
To enable this analysis, we introduce a decomposition strategy  to derive   auxiliary problems that admit an efficient algorithmic solution.

\subsection{Relation Generative Model}
\label{sec:gen}

To facilitate the theoretical analysis, we adopt a probabilistic generative model for relation matrices.
Specifically, the relation matrix is generated according to $\textbf{R} = \textbf{A}\textbf{S}\textbf{B}^T$, where the sparse coefficient matrices $\mathbf{A}$ and $\mathbf{B}$, together with the latent relation matrix $\mathbf{S}$, are generated under the assumptions below.   
The model incorporates sparsity, boundedness, and incoherence conditions that are widely used in  theoretical analyses of problems such as sparse coding and matrix completion.
For the coefficient matrices $\mathbf{A}$ and $\mathbf{B}$, we assume that the support of each row is sampled uniformly at random (Assumption  \ref{SC}), while the corresponding nonzero coefficients are independently drawn from zero-mean distributions with bounded support and fixed variance (Assumption \ref{NZC}).
For the latent relation matrix  $\textbf{S}$,  we assume that   its rows and columns  have bounded $l_2$-norms (Assumption  \ref{LS}), and satisfy an incoherence condition (Assumption \ref{LI}).
\begin{assumption}[Sparse Support Generation] \label{SC}
For each  row of  $\textbf{A}$  and $\textbf{B}$, assume that its support set  $supp\left(\textbf{A}^{(i)}\right)$   and  $supp\left(\textbf{B}^{(i)}\right)$  is sampled uniformly at random from all subsets of  $[k]$ that have   cardinality $s_A$ and $s_B$, respectively.  
Define indicator variables  $\chi_{ij}^{(A)} = 1$ if $j\in \textmd{supp}\left(\textbf{A}^{(i)}\right)$ and $\chi_{ij}^{(B)} = 1$ if $j\in \textmd{supp}\left(\textbf{B}^{(i)}\right)$, while $\chi_{ij}^{(A)} = 0$ and $\chi_{ij}^{(B)} = 0$ otherwise. 
Assume entries of  $\textbf{A}$ and $\textbf{B}$ are generated by 
\begin{equation}
    a_{ij} = \alpha_{ij}\chi_{ij}^{(A)}, \;\; b_{ij} = \beta_{ij}\chi_{ij}^{(B)},
\end{equation}
where $\alpha_{ij}, \beta_{ij}\in \mathbb{R}$ are generated by following the mechanism defined in Assumption \ref{NZC}.
\end{assumption}
\begin{assumption}[Distribution of Nonzero Coefficients] \label{NZC}
Assume that all nonzero entries of the coefficient matrices $\textbf{A}=[\alpha_{ij}]\in \mathbb{R}^{n\times k}$ and $\textbf{B}=[\beta_{ij}]\in \mathbb{R}^{m\times k}$ are independently drawn   from  distributions with zero mean, fixed variance, and bounded  support:
\begin{align}
E\left [\alpha_{ij}\right]= 0, \; E\left [\alpha_{ij}^2\right ] = \sigma_A^2, \;  M^{(A)}_{\min} \leq |\alpha_{ij}| \leq M^{(A)}_{\max};\\
E\left [\beta_{ij}\right]= 0, \; E\left [\beta_{ij}^2\right ] =  \sigma_B^2, \; M^{(B)}_{\min} \leq   | \beta_{ij}| \leq M^{(B)}_{\max}.
\end{align}
Assume the random variables $ \alpha_{ij}$ and   $\beta_{ij}$ are independent of the support selection mechanism defined in Assumption \ref{SC}.
\end{assumption}
\begin{assumption}[Bounded Latent Relations]\label{LS}
Assume that there exist positive constants $0<l_s \leq u_s  < +\infty$ such that  $\forall i\in [k]$, it has
\begin{equation}
l_s\leq \|\textbf{S}_i\|_2  \leq u_s, \;\; l_s\leq  \|\textbf{S}^{(i)}\|_2 \leq u_s.
\end{equation}
\end{assumption}
\begin{assumption}[Incoherent Latent Relations]\label{LI}
Assume that the latent relation matrix $\textbf{S} $ has incoherent rows and columns. Specifically,   $\forall i \neq j$, it has
\begin{equation}
\left| \cos\left(\bm S^{(i)}, \bm S^{(j)} \right) \right| <   \frac{\mu_s}{\sqrt{d}},  \; \left| \cos\left(\bm S_{i}, \bm S_{j} \right ) \right| <  \frac{\mu_s}{\sqrt{d}}, 
\end{equation}
where $\mu_s>0$  is a coherence constant  and  $d = \textmd{rank}(\textbf{S})$.
\end{assumption}

\paragraph{Remarks:}  
Assumption~\ref{SC} specifies the row-wise sparsity levels $s_A$ and $s_B$, helping isolate the effect of sparsity in the analysis.
In practice, norm-based regularisation is commonly employed to control the scale of the solutions, and to avoid degenerate (or ill-conditioned) solutions.
Likewise, the bounded amplitudes of the nonzero coefficients in Assumption~\ref{NZC}, together with the bounded row and column $l_2$ norms of $\mathbf{S}$ in Assumption~\ref{LS}, serve as non-degeneracy conditions that prevent pathological scaling behaviour.
The zero-mean assumption  on the nonzero coefficients in Assumption~\ref{NZC} is  standard in probabilistic models of sparse representations and is widely adopted in the dictionary learning and sparse coding literature \citep{Aga16,Aga17,Wu18,Sul22}.
Besides simplifying moment-based recovery analyses, the zero-mean assumption  provides a natural modeling abstraction in which sparse coefficients are treated as unbiased fluctuations around zero.
Assumption~\ref{LI} requires the pairwise cosine similarities between distinct rows and between distinct columns of $\textbf{S}$ to be uniformly bounded. 
Equivalently,  the corresponding row-wise and column-wise cosine similarity matrices are approximately diagonal.
This assumption is closely related to the mutual coherence conditions commonly employed in sparse representation, dictionary learning, and low-rank matrix recovery \citep{Tro04,Can09,Bar13}.
It prevents strong alignment among the rows and among the columns of $\textbf{S}$, thereby promoting sufficiently diverse latent components and facilitating the theoretical recovery analysis.

\subsection{Auxiliary  Problem Construction}
\label{sec:decomposition}

\subsubsection{SVD-induced Decomposition}
\label{sec:aux_svd}

We begin by  considering the compact SVD of the observed relation matrix, i.e., $\textbf{R} = \textbf{U}_R\bm\Sigma_R \textbf{V}_R^T $, where  $\textbf{U}_R \in \mathbb{R}^{n\times d}$ and $\textbf{V}_R \in \mathbb{R}^{m\times d}$ are the left and right singular vector matrices, respectively,  $\bm\Sigma_R \in \mathbb{R}^{d\times d}$  is   the diagonal matrix of positive singular values, and  $d= \textmd{rank}(\textbf{R})$. 
We construct two matrices $\textbf{Y}_U\in \mathbb{R}^{n\times d}$ and $ \textbf{Y}_V \in \mathbb{R}^{m\times d} $  by 
\begin{align}
\label{Y_uv1}
\textbf{Y}_U = & \frac{1}{\sqrt{nm}} \textbf{R}\textbf{V}_R = \frac{1}{\sqrt{nm}}\textbf{U}_R\bm\Sigma_R, \\ 
\label{Y_uv2}
\textbf{Y}_V = & \frac{1}{\sqrt{nm}} \textbf{R}^T\textbf{U}_R = \frac{1}{\sqrt{nm}}\textbf{V}_R\bm\Sigma_R.
\end{align}
The scaling factor $\frac{1}{\sqrt{nm}}$ is introduced so that the effect of sparsity is not confounded by the problem dimensions.
Rather than analysing the original tri-factorisation of $\textbf{R}$, we   study  the  two-factor factorisations of the new matrices $\textbf{Y}_U$ and $\textbf{Y}_V$, separately. 
Below, we  establish the relationship between the  new two-factor models and the original tri-factorisation  model.

Define two matrices $\textbf{F}_A, \textbf{F}_B \in \mathbb{R}^{k\times d}$  as 
\begin{equation}
\label{F_ab}
    \textbf{F}_A  = \textbf{S} ^T\textbf{A} ^T\textbf{U}_R, \; \;  \textbf{F}_B  = \textbf{S} \textbf{B} ^T \textbf{V}_R,
\end{equation}
and two diagonal matrices $\textbf{L}_A, \textbf{L}_B \in \mathbb{R}^{k\times k}$  as
\begin{align}
\label{matrix_La}
\textbf{L}_A = & \textmd{diag}\left( \left\|\textbf{F}_A^{(1)} \right\|_2^{-1}, \left\|\textbf{F}_A^{(2)} \right\|_2^{-1}, \dots, \left\|\textbf{F}_A^{(k)} \right\|_2^{-1}\right),\\
\label{matrix_Lb}
\textbf{L}_B =& \textmd{diag}\left( \left\|\textbf{F}_B^{(1)} \right\|_2^{-1}, \left\|\textbf{F}_B^{(2)} \right\|_2^{-1}, \dots, \left\|\textbf{F}_B^{(k)} \right\|_2^{-1}\right).
\end{align}
Construct two  dictionary matrices $\textbf{D}_A^*, \textbf{D}_B^*  \in \mathbb{R}^{k\times d}$ with each row having unit $l_2$-norm,  by   
\begin{equation}
\label{D_ab}
    \textbf{D}_A^* = \textbf{L}_A \textbf{F}_A,\; \; \textbf{D}_B^* = \textbf{L}_B \textbf{F}_B, 
\end{equation}
and two new coefficient matrices $\textbf{X}_A^* \in \mathbb{R}^{n\times k}$ and $\textbf{X}_B^* \in \mathbb{R}^{m\times k}$ by 
\begin{equation}
\label{X_ab}
\textbf{X}_A^* = \frac{1}{\sqrt{nm}}\textbf{A} \textbf{L}_B^{-1},\;  \;\textbf{X}_B^* = \frac{1}{\sqrt{nm}}\textbf{B} \textbf{L}_A^{-1}.
\end{equation}
%
%
Since $\textbf{A}$, $  \textbf{B}$, and $\textbf{S}$  are the ground-truth factor matrices  used to generate the observed matrix $\textbf{R}$, it has $\textbf{A} \textbf{S} \textbf{B} ^T = \textbf{R} = \textbf{U}_R\bm\Sigma_R\textbf{V}_R^T$, 
which, together with Eqs. (\ref{Y_uv1}), (\ref{Y_uv2}), (\ref{D_ab}) and (\ref{X_ab}),   verifies the following
\begin{equation}
    \label{two_F1}
    \textbf{Y}_U = \textbf{X}_A^* \textbf{D}_B^* ,    \;\; \textbf{Y}_V =  \textbf{X}_B^* \textbf{D}_A^*.
\end{equation}
Eq. (\ref{two_F1}) corresponds to the standard dictionary learning representations of  $\textbf{Y}_U$ and  $\textbf{Y}_V$.
The  dictionary matrices $\textbf{D}_A^*$ and $ \textbf{D}_B^*$ together with the coefficient matrices $\textbf{X}_A^*$ and $ \textbf{X}_B^*$  constitute the ground-truth factor matrices to recover.
This motivates the following two auxiliary two-factor factorisation problems:
\begin{equation}
\label{two_F}
    \textbf{Y}_U \approx \textbf{X}_A\textbf{D}_B,    \;\; \textbf{Y}_V \approx  \textbf{X}_B\textbf{D}_A,
\end{equation}
which can be analysed within established theoretical and algorithmic frameworks for dictionary learning and sparse coding.

\subsubsection{Auxiliary Generative Model}
\label{sec:gen_aus}
We formally define the auxiliary generative model below based on the connection between the tri-factorisation of $\textbf{R}$ and the two-factor factorisations of $\textbf{Y}_U$ and $\textbf{Y}_V$.
\begin{definition}[Auxiliary Generative Model]
\label{gen_aux}
Suppose that the relation matrix $\textbf{R} $ is generated according to $\textbf{R} = \textbf{A}\textbf{S}\textbf{B}^T$. 
Let $\textbf{R} = \textbf{U}_R\bm\Sigma_R \textbf{V}_R^T $ be the compact SVD of $\textbf{R}$.
Construct the observation matrices $\textbf{Y}_U $ and $ \textbf{Y}_V $ according to  Eqs. (\ref{Y_uv1}) and (\ref{Y_uv2}).
The  auxiliary generative model consists of the following two generating processes:
\begin{itemize}
    \item $\textbf{Y}_U = \textbf{X}_A\textbf{D}_B$,  where  $\textbf{X}_A = \frac{1}{\sqrt{nm}}\textbf{A}\textbf{L}_B^{-1}$, $\textbf{D}_B = \textbf{L}_B \textbf{F}_B$,   $\textbf{F}_B  = \textbf{S}\textbf{B}^T \textbf{V}_R$,    Eq. (\ref{matrix_Lb}) for $\textbf{L}_B$.
    \item $\textbf{Y}_V =  \textbf{X}_B\textbf{D}_A$,  where  $\textbf{X}_B  = \frac{1}{\sqrt{nm}}\textbf{B} \textbf{L}_A^{-1}$, $\textbf{D}_A = \textbf{L}_A \textbf{F}_A$,     $\textbf{F}_A  = \textbf{S}^T\textbf{A}^T\textbf{U}_R $,     Eq. (\ref{matrix_La}) for $\textbf{L}_A $.
\end{itemize}
\end{definition} 

\noindent
Instead of directly analysing the original tri-factorisation problem, we study the recovery of the auxiliary dictionaries ($\textbf{D}_B$, $\textbf{D}_A$) and  coefficients ($\textbf{X}_A$, $\textbf{X}_B$) from the auxiliary observations ($\textbf{Y}_U$, $\textbf{Y}_V$).
Here we omit the superscript $^{*}$ for indicating the ground truth to simplify notations.
Under this formulation, the factorisation is subject to  standard sign and permutation ambiguities.
Specifically, simultaneously flipping the sign of a dictionary atom  and the sign of its corresponding coefficient vector   does not alter the observation matrix. 
Likewise, simultaneously  permuting  the atoms and their associated coefficient vectors leave the observation matrix unchanged.
Therefore, we establish recovery conditions  up to sign and permutation.

\subsubsection{Auxiliary Tri-factorisation}  
\label{sec:AuxTri}

Combining the two auxiliary factorisations in Eq. (\ref{two_F}) yields the induced tri-factorisation $ \textbf{R} \approx \textbf{X}_A \tilde{\textbf{S}}\textbf{X}_B^T$, where the auxiliary latent relation matrix satisfies $\tilde{\textbf{S}}  =  nm\textbf{L}_B \textbf{S}\textbf{L}_A$. 
Specifically, it has
\begin{equation}
\label{eq:factorS}
     \textbf{R}  =  \textbf{A}  \textbf{S} \textbf{B}^T   =   \left( \frac{1}{\sqrt{nm}} \textbf{A}\textbf{L}_B^{-1} \right)(  nm\textbf{L}_B \textbf{S}\textbf{L}_A  ) \left( \frac{1}{\sqrt{nm}}\textbf{B}\textbf{L}_A^{-1}\right)^T =  \textbf{X}_A (\underbrace{  nm\textbf{L}_B \textbf{S}\textbf{L}_A}_{\tilde{\textbf{S}}})\textbf{X}_B^T. 
\end{equation}
Furthermore, Eqs. (\ref{Y_uv1}) and (\ref{Y_uv2}) enable  the auxiliary latent relation matrix to be computed from either the auxiliary coefficient matrices  or the auxiliary dictionary matrices as follows
\begin{equation}
\label{eq:scaled_S_aux}
    \tilde{\textbf{S}}  =      \textbf{X}_A^{\dagger} \textbf{R}\left(\textbf{X}_B^T\right)^{\dagger}  =  nm\textbf{X}_A^{\dagger} \textbf{Y}_U \bm\Sigma_\mathbb{R}^{-1}     \textbf{Y}_V^T \left(\textbf{X}_B^T\right)^{\dagger} = nm\textbf{D}_B\bm\Sigma_\mathbb{R}^{-1}\textbf{D}_A^T. 
\end{equation}
The induced tri-factorisation $ \textbf{R} = \textbf{X}_A \tilde{\textbf{S}}\textbf{X}_B^T$ is closely related to the original  tri-factorisation  $\textbf{R} =  \textbf{A}  \textbf{S} \textbf{B}^T$ as shown below.

First, the auxiliary and original coefficient matrices share identical sparsity structures.
As seen in Eq. (\ref{eq:factorS}), the auxiliary coefficient matrices are obtained from the original coefficient matrices through invertible diagonal scalings:
 $\textbf{X}_A = \frac{1}{\sqrt{nm}}\textbf{A}\textbf{L}_B^{-1}$ and $\textbf{X}_B = \frac{1}{\sqrt{nm}}\textbf{B}\textbf{L}_A^{-1}$.
Since diagonal scaling preserves the support of a vector, the auxiliary coefficient matrices $\mathbf{X}_A$ and $\mathbf{X}_B$ have exactly the same support patterns as the original $\mathbf{A}$ and $\mathbf{B}$. 
Furthermore,  Lemma \ref{RowLengthF} establishes that the diagonal entries of  $\textbf{L}_A$, $\textbf{L}_B$, and their inverses are uniformly bounded. 
Consequently, the transformation from  $\textbf{A}$, $\textbf{B}$, and $ \textbf{S} $ to  $\textbf{X}_A$, $\textbf{X}_B$, and $\tilde{\textbf{S}} $ amounts only to bounded diagonal rescalings.
As a result, the norms and pairwise angles of the coefficient vectors, as well as the conditioning of the latent relation matrix, undergo only bounded distortions, as established  in Theorem \ref{main_res3}. 
Hence, the induced tri-factorisation preserves the essential structural, geometric, and spectral properties of the original tri-factorisation. 
Therefore,  recovering the auxiliary factor matrices yields an equivalent representation of the underlying latent structure.
Together with the exact preservation of the sparsity structure, this establishes the mathematical justification of the auxiliary formulation adopted in the subsequent recovery analysis.

\section{ Tri-factorisation Algorithmic Framework}
\label{sec:algorithm}

Our theoretical analysis builds upon an algorithmic framework for solving the auxiliary tri-factorisation problem.
To facilitate the presentation of the used algorithms, we first recall several essential concepts from the existing literature.

\begin{definition}[$\rho$-Correlation Graph]
\label{def:rho_corr_graph} 
Let $Y=\{\bm y_i\}_{i=1}^k$ be a collection of row vectors and let  $\rho>0$ be a threshold.
The  $\rho$-correlation  graph associated with $Y$, denoted by $G_{\rho}(Y) =(V, E)$, is   an  undirected graph with the vertex set $V = \{v_1, \dots, v_k\}$, where, for $i \neq j$,  $(v_i, v_j) \in E$  if and only if $|\bm y_i \bm y_j^T| > \rho$.
\end{definition}

\begin{definition}[Shared $\rho$-Neighbour Set]
\label{def:share_neighbour}
Let $Y = \{\bm y_i\}_{i=1}^k $ and $\rho > 0$. 
For $i \neq j$, the shared $\rho$-neighbour set of $(\bm y_i, \bm y_j)$ is defined as
\begin{equation}
N_{\rho}(\bm y_i, \bm y_j, Y) = \left\{ \bm y_t\left|  t\in[k],  t \notin \{i,j\}, \left|\bm y_i\bm y_t^T\right| > \rho,  \left|\bm y_j\bm y_t^T\right|> \rho   \right.\right\}.
\end{equation}
\end{definition}

\begin{definition}[Smallest $\rho$-Neighbour Size]
\label{def:smallest_neighbour}
Let $Y = \{\bm y_i\}_{i=1}^k $ and $\rho > 0$.
Let $G_\rho(Y) = (V,E)$ be the $\rho$-correlation graph defined in Definition \ref{def:rho_corr_graph}, and  let  $N_\rho(\bm y_i, \bm y_j, Y)$ be the shared $\rho$-neighbour set of  $(\bm y_i, \bm y_j)$ defined in Definition \ref{def:share_neighbour}.
The smallest $\rho$-neighbour size of $Y$ is   defined as
\begin{equation}
    \bar{N}_{\rho}(Y) = \min_{ (v_i,v_j)\in E } |N_{\rho}(\bm y_i, \bm y_j, Y)|.
\end{equation}
\end{definition}

\subsection{Factorisation Algorithms}

Our tri-factorisation framework is presented in Algorithm  \ref{alg:rl}, which forms the basis of the subsequent identifiability analysis.
It independently solves the two auxiliary two-factor matrix factorisation problems and combines their results to estimate  $\textbf{X}_A$,  $\textbf{X}_B$ and $\tilde{\textbf{S}}$. 
Each auxiliary problem is solved using a classical  alternating minimisation algorithm for dictionary learning and sparse coding \citep{Aga16}, detailed in Algorithm \ref{alg:dl}.
We  further analyse a spectral method for obtaining a coarse estimate of the auxiliary dictionary matrices  \citep{Aga17},  described in Algorithms \ref{alg:ini} and \ref{alg:unique}.
This estimate can  be used to initialise Algorithm  \ref{alg:dl} at Step \ref{step:dic_ini}.

\subsubsection{Relation Tri-factorisation}

\begin{algorithm}[t]
	\caption{ Auxiliary Tri-factorisation}
	\label{alg:rl}
	\begin{algorithmic}[1]
\State \textbf{Input}: Observed relation matrix $\textbf{R} \in \mathbb{R}^{n \times m} $;    atom separation parameter $\xi$; hyperparameter sets $\bm h_U$ and $\bm h_V$.

\State \textbf{Output}: Matrix estimates $\hat{\textbf{X}}_A$, $\hat{\textbf{X}}_B$, $\hat{\textbf{S}}$.

\State  Compute the compact SVD  $\textbf{R} = \textbf{U}_R\bm\Sigma_R \textbf{V}_R^T $, construct $\textbf{Y}_U$ and $\textbf{Y}_V$ by Definition \ref{gen_aux}.

\State  Apply Algorithm \ref{alg:dl} to   $\textbf{Y}_U$ with hyper-parameters     $\xi$ and  $\bm h_U$, to obtain $\hat{\textbf{X}}_B$ and $\hat{\textbf{D}}_A$. \label{alg3_apply}

\State  Apply Algorithm \ref{alg:dl} to  $\textbf{Y}_V$ with hyper-parameters     $\xi$ and  $\bm h_V$, to obtain $\hat{\textbf{X}}_A$ and $\hat{\textbf{D}}_B$ \label{alg3_apply2}

\State Compute $\hat{\textbf{S}} =  \hat{\textbf{X}}_A^{\dagger} \textbf{R}\left(\hat{\textbf{X}}_B^T\right)^{\dagger} $.
\label{eq:scaled_S_est}

\State Return  $\hat{\textbf{X}}_A$, $\hat{\textbf{X}}_B$, $\hat{\textbf{S}}$.
	\end{algorithmic}
\end{algorithm}

Algorithm \ref{alg:rl} takes the observed relation matrix $\textbf{R}$ as input and constructs two auxiliary problems according to Definition~\ref{gen_aux}.
The two problems are solved in Steps \ref{alg3_apply} and \ref{alg3_apply2}, respectively, by applying Algorithm \ref{alg:dl} to the auxiliary observation matrices $\textbf{Y}_U$ and $\textbf{Y}_V$, producing estimates of the corresponding auxiliary coefficient and dictionary matrices, denoted by $\hat{\textbf{X}}_A$, $\hat{\textbf{X}}_B$, $\hat{\textbf{D}}_A$, and $\hat{\textbf{D}}_B$.
 Algorithm \ref{alg:dl}  requires several hyper-parameters, including the iteration number $T$, Lasso accuracy parameters $\left\{\epsilon_t\right\}_{t=1}^{T}$ and sparsity control parameters $\left\{\rho^{(s)}_t\right\}_{t=1}^{T}$.
 When Algorithm \ref{alg:ini}  is used to initialise the dictionary estimation process, additional hyper-parameters are required, including the correlation threshold $\rho$  and unique intersection threshold $\rho_p$.
 We denote the complete set of hyper-parameters for Algorithm \ref{alg:rl} by $\bm h $, and distinguish the   hyperparameter settings used for factorising $\textbf{Y}_U$ and $\textbf{Y}_V$  by $\bm h_U$ and $\bm h_V$, respectively.
The estimate of the auxiliary latent relation matrix, denoted by $\hat{\textbf{S}}$, is then computed according to   Eq. (\ref{eq:scaled_S_aux}).
Finally, the algorithm returns the estimated tri-factorisation  $ \textbf{R} = \hat{\textbf{X}}_A\hat{\textbf{S}} \hat{\textbf{X}}_B^T$.

\begin{algorithm}[t]
	\caption{ An Alternating Minimisation Algorithm  for Solving $\textbf{Y} \approx \textbf{X}\textbf{D}$ \citep{Aga16}}
	\label{alg:dl}
	\begin{algorithmic}[1]
\State \textbf{Input}: Observed data $\textbf{Y} \in \mathbb{R}^{n \times d} $; iteration number $T$; Lasso accuracy parameters $\left\{\epsilon_t\right\}_{t=1}^{T}$; sparsity control parameters $\left\{\rho^{(s)}_t\right\}_{t=1}^{T}$; correlation threshold $\rho$; atom separation parameter $\xi$; unique intersection threshold $\rho_p$.

\State \textbf{Output}: Coefficient matrix estimate $\hat{\textbf{X}}$; dictionary matrix estimate $\hat{\textbf{D}}$.

\State Initialise the  dictionary estimate $\hat{\textbf{D}}_0$ either randomly or   by  
$\hat{\textbf{D}}_0 = \textmd{Dic}(\textbf{Y}, \rho, \xi, \rho_p)$ using Algorithm \ref{alg:ini}.
\label{step:dic_ini}

\For{iterations $t=1,2,\ldots, T$}
\For{sample index $i=1,2,\ldots, n$}
 
\State Solve the following constrained optimisation problem:
\begin{equation}
\nonumber
 \bm x^* = \arg \min_{ \left\|\textbf{Y}^{(i)} - \bm x \hat{\textbf{D}}_{t-1} \right\|_2 \leq \epsilon_{t} }\;  \|\bm x\|_1.
\end{equation}
\label{alg2:step:Lasso}

 \State Obtain a sparse vector $\hat{\bm x} \in \mathbb{R}^{k}$ by thresholding the optimised coefficients: 
 \label{step:sparse_thresholding} 
 $$\hat{x}_j= \left\{\begin{array}{ll}
         x^*_j,  &    \textmd{if }   |x^*_j| \geq \rho^{(s)}_t, \\
         0, &   \textmd{otherwise}.
        \end{array}
        \right.
  $$ 
   
\State Update the $i$-th row of the coefficient estimate by $\hat{\textbf{X}}_{t}^{(i)} = \hat{\bm x}$.
 \EndFor
 
 \State Obtain the dictionary estimate by $ \hat{\textbf{D}}_t = \hat{\textbf{X}}_{t}^{\dagger}\textbf{Y}$.
 \label{step:dic_comp}

 \State Normalise each row of the dictionary estimate to unit length by  $\hat{\textbf{D}}_{t}^{(i)} \leftarrow \frac{\hat{\textbf{D}}_t^{(i)}}{\left\|\hat{\textbf{D}}_t^{(i)}\right\|_2}$.

\EndFor
 
\State Return  $\hat{\textbf{X}}_T$ and $\hat{\textbf{D}}_{T}$.
	\end{algorithmic}
\end{algorithm}

\subsubsection{Dictionary Learning and Sparse Coding by Alternating Optimisation}  
Algorithm \ref{alg:dl} alternates between estimating the coefficient matrix and updating the dictionary matrix from the observed data matrix.
Starting from an initial  guess  of the dictionary,  it computes the coefficient vector for each observation by solving  a constrained Lasso problem.
The resulting coefficient matrix is then used to update the dictionary matrix  by computing  its pseudo-inverse.
The fitting error and sparsity of the Lasso solution at each iteration $t$ are controlled by  the Lasso accuracy parameter $\epsilon_t$ and the sparsity  control parameters $ \rho^{(s)}_t$, respectively.
This procedure is repeated for   $T$ iterations.

When applied within Algorithm \ref{alg:rl}, Algorithm \ref{alg:dl}    updates   the  auxiliary dictionary matrices  according to   $\hat{\textbf{D}} _{A} =  \hat{\textbf{X}} _{B}^{\dagger}\textbf{Y}_V = \hat{\textbf{X}} _{B}^{\dagger} \textbf{X}_B\textbf{D}_A$ and $\hat{\textbf{D}} _{B} = \hat{\textbf{X}} _{A}^{\dagger}\textbf{Y}_U = \hat{\textbf{X}} _{A}^{\dagger} \textbf{X}_A\textbf{D}_B$.
Step \ref{eq:scaled_S_est} of Algorithm \ref{alg:rl}  estimates the auxiliary  latent relation matrix as $\hat{\textbf{S}} =   \hat{\textbf{X}}_A^{\dagger} \textbf{R}\left(\hat{\textbf{X}}_B^{\dagger} \right)^T$.
These expressions show that the overall performance of the tri-factorisation algorithm  depends critically on the accuracy of the estiamted coefficient matrices, which in turn is largely determined by the  quality of the   Lasso solution obtained in Step \ref{alg2:step:Lasso} of Algorithm \ref{alg:dl}.
The classical compressed sensing result  of \citet{Can08}, stated as Theorem \ref{Candes} in the appendix, characterises  the recovery accuracy of the Lasso estimator given an approximate dictionary $\hat{\textbf{D}}$.
We adapt this result to our setting and present an error bound on the coefficient estimated by Lasso in Lemma \ref{coeff_error} in appendix, which provides the key theoretical foundation for our analysis.
Lemma \ref{coeff_error} also  motivates the choice   $\epsilon_{t} = \sqrt{s}M\epsilon$ for the Lasso accuracy parameter via  Eq. (\ref{eq:Lasso_accuracy})   and the choice $\rho^{(s)}_t = 8.6\sqrt{s}M\epsilon$ for the sparsity control parameter via Eq. (\ref{eq:thresholding_function}).

\begin{algorithm}[t]
	\caption{Spectral Dictionary Approximation $\textmd{Dic}(\textbf{Y}, \rho, \xi, \rho_p)$ \citep{Aga17}}
	\label{alg:ini}
	\begin{algorithmic}[1]
		\State \textbf{Input}: Observed data $\textbf{Y} \in \mathbb{R}^{n \times d}  $;  correlation threshold $\rho$;  atom separation parameter $\xi$;  unique intersection threshold $\rho_p$.
  \State \textbf{Output}: Dictionary estimate  $\hat{\textbf{D}}_0$.
  
  \State Construct the $\rho$-correlation graph   $G_{\rho}\left(\left\{\textbf{Y}^{(i)}\right\}_{i=1}^n \right)$ for the  rows  of $\textbf{Y}$.

  \State Set $t= 0$ and initialise the atom set $D_0^{(0)} = \emptyset$.
  \For{Each pair of  row vectors $\left(\textbf{Y}^{(i)}, \textbf{Y}^{(j)}\right)$ connected  by  an edge}
  
  \State Set $D_0^{(t+1 )}=D_0^{(t)}$.

  \State Identify the  shared $\rho$-neighbour  set $S = N_{\rho}\left(\textbf{Y}^{(i)}, \textbf{Y}^{(j)}, \left\{\textbf{Y}^{(h)}\right\}_{h=1}^n\right)$.   \label{nei_size}

  \State Compute the unique intersection indicator $I = u(S, \rho, \rho_p)$ using Algorithm \ref{alg:unique}.
  
 \If{$I=1$}
    
    \State Compute $\textbf{L} = \sum_{\bm y\in S}\bm y^T\bm y$.\label{step:atom1}

    \State Compute the leading singular vector $\bm d$ of $\textbf{L}$. \label{step:atom2}
    
    \If{$D_0^{(t)} = \emptyset$ or $  \min_{\bm d_0 \in D_0^{(t)}}\min\left( \|\bm d-\bm d_0\|_2, \|\bm d+\bm d_0\|_2 \right)> 2\xi$}
        \State Update the atom set by  $D_0^{(t+1)} =  D_0^{(t)} \cup \{\bm d\}$.  \label{step:atom_duplicate}
   \EndIf
        
 \EndIf
  \State  Update $t \leftarrow t+1$.
 \EndFor
 
 \State Store the atoms in $D_0^{(t)}$ as the rows of  the dictionary estimate $\hat{\textbf{D}}_0$.
 
\State Return $\hat{\textbf{D}}_0$.
	\end{algorithmic}
\end{algorithm}

\begin{algorithm}[t]
	\caption{Unique Intersection Function $u(S, \rho, \rho_p)$ \citep{Aga17}}
	\label{alg:unique}
	\begin{algorithmic}[1]
 
\State \textbf{Input}: A set of vectors $S=\{\bm y_i\}_{i=1}^N$ ($N\geq 2$); correlation threshold $\rho$; unique intersection threshold $\rho_p$.

\State \textbf{Output}: A binary indicator $I\in\{0,1\}$.

\State Construct the $\rho$-correlation graph $G_{\rho}(S)$.
\label{alg4:step_correlation}

\State Count the number of  edges in $G_{\rho}(S)$ and denote it  by  $N_e$.

\State Compute the unique-intersection score $r =\frac{N_e}{\frac{1}{2}N(N-1)}$.
\label{alg3:unique_score}

\State Set the indicator as 
        $I = \left\{\begin{array}{ll}
         1,  &    \textmd{if }   r > \rho_p, \\
         0, &   \textmd{otherwise}.
        \end{array}
        \right.
 $
\State Return $I$.
	\end{algorithmic}
\end{algorithm}

\subsubsection{Spectral Dictionary Approximation} 

Algorithm \ref{alg:ini} provides a spectral approximation of the auxiliary dictionary matrices. 
Instead of solving an optimisation problem, it analyses the spectral structure of the observed data and directly estimates the dictionary atoms by identifying the leading singular vectors.
Specifically, Algorithm \ref{alg:ini}  first identifies the shared $\rho$-neighbour set  defined in Definition \ref{def:share_neighbour}  for each connected pair of data examples, where connectivity is determined by the $\rho$-correlation graph in Definition \ref{def:rho_corr_graph}.
It then applies Algorithm \ref{alg:unique} to identify example pairs that are likely to share exactly one common atom  using the unique intersection threshold $\rho_p$.
Each such pair, referred to as a   unique intersection pair, is subsequently used to estimate the corresponding shared atom through  SVD in Steps \ref{step:atom1} and \ref{step:atom2} of Algorithm \ref{alg:ini}.
Finally, only sufficiently distinct atoms are retained in the dictionary, as determined by the atom separation parameter $\xi$.
Further properties of unique intersection pairs and the corresponding shared atoms are provided in Appendices \ref{sec:existing_ini} and \ref{app:main_proof_alg3}.

\subsection{Recovery Errors}
\label{sec:error_metrics}

The tri-factorisation framework in Algorithm  \ref{alg:rl} returns estimates  of the auxiliary coefficient and latent relation matrices,  including $\hat{\textbf{X}}_A$, $\hat{\textbf{X}}_B$, and $\hat{\textbf{S}}$.
Algorithm  \ref{alg:dl} additionally produces estimates of the auxiliary dictionary matrices as a byproduct, including $\hat{\textbf{D}}_A$ and $\hat{\textbf{D}}_B$.
We establish conditions under which the  estimation errors of these quantities converge to zero.
These errors are measured by comparing the estimates with their corresponding ground truths, including  $\textbf{X}_A$, $\textbf{X}_B$, $\tilde{\textbf{S}}$, $\textbf{D}_A$, and $\textbf{D}_B$.
In this section, we introduce the  error metrics used throughout  our theoretical analysis.
As discussed earlier,  simultaneous sign flips  and  permutations of the dictionary atoms  and their corresponding coefficient vectors leave the generated matrices unchanged, giving rise to the well-known sign and permutation ambiguities.
Our error metrics accommodate these ambiguities following standard practice.
To help explanation, we  reintroduce the superscript $^*$ to denote the ground-truth quantities.

\paragraph{Sign and Permutation Ambiguity:}
Consider  the auxiliary  factorisation problem   $\textbf{Y}  \approx \textbf{X} \textbf{D}$ (without distinguishing between $\textbf{Y}_U$ and $\textbf{Y}_V$ for simplicity), together with the ground-truth factor matrices  $\textbf{X}^* \in \mathbb{R}^{n\times k}$ and $ \textbf{D}^* \in \mathbb{R}^{k\times d}$  and  their corresponding estimates $\hat{\textbf{X}} \in \mathbb{R}^{n\times k}$ and $ \hat{\textbf{D}}\in \mathbb{R}^{k\times d}$.
To account for permutation ambiguity, we permute the estimated atoms  to align with the ground-truth  atoms.
Letting $\bm \Pi  \in \{0,1\}^{k\times k}$ denote the corresponding permutation matrix, the permuted   dictionary estimate $\hat{\textbf{D}}_{\pi}   = \bm\Pi \hat{\textbf{D}}  $ shares  the same atom ordering as the ground truth  $ \textbf{D}^*$.
Accordingly, it has
\begin{equation}
\label{ground_truth_permute}
\textbf{Y} =\hat{\textbf{X}}\hat{\textbf{D}}  =\left(\hat{\textbf{X}}  \bm\Pi^{-1} \right)\left(\bm\Pi\hat{\textbf{D}}\right) = \hat{\textbf{X}}_{\pi}  \hat{\textbf{D}}_{\pi} .  
\end{equation}
where $\hat{\textbf{X}}_{\pi} = \hat{\textbf{X}}  \bm\Pi^{-1}  $ corresponds to  the permuted   coefficient estimate.
To further account for sign ambiguity,   a binary sign variable is introduced for each estimated atom, given as
\begin{equation}
    z_i = \arg\min_{z\in\{-1, +1\}}\left\|z\hat{\textbf{D}}^{(i)}_{\pi}- {\textbf{D}^*}^{(i)}  \right\|_2,
\end{equation}
which determines whether the corresponding atom should be sigh-flipped.
Letting $\textbf{Z} = \textmd{diag}\left([z_1, z_2, \ldots, z_k]\right)$, the permutation   and  sign matrices   then define the aligned estimates 
\begin{equation}
    \hat{\textbf{X}}_{\pi,z}  = \hat{\textbf{X}}_{\pi}  \textbf{Z}^{-1} = \hat{\textbf{X}} \bm\Pi^{-1} \textbf{Z}^{-1}  ,\;\; \hat{\textbf{D}}_{\pi,z}  = \textbf{Z}\hat{\textbf{D}}_{\pi}  = \textbf{Z}\bm\Pi \hat{\textbf{D}},
\end{equation}
satisfying $\textbf{Y} =\hat{\textbf{X}}\hat{\textbf{D}}   =  \hat{\textbf{X}}_{\pi,z}   \hat{\textbf{D}}_{\pi,z} $.
All subsequent error measures are defined by comparing the ground-truth matrices with these aligned  estimates rather than the original estimates.

\paragraph{Auxiliary Dictionary Error:}
We adopt the following error metric to assess each estimated atom:
\begin{equation} 
\label{revised_dic_error}
  \epsilon_D\left(\hat{\textbf{D}}^{(i)}, \textbf{D}^*{}^{(i)} \right) = \left\|\hat{\textbf{D}}_{\pi,z}^{(i)}- {\textbf{D}^*}^{(i)}  \right\|_2  = \min_{z\in\{-1, +1\}}\left\|z\hat{\textbf{D}}_{\pi}^{(i)}- {\textbf{D}^*}^{(i)}  \right\|_2.
\end{equation}
Our analysis considers both the initial estimate of the auxiliary dictionary matrices  $\hat{\textbf{D}}_{A_{0}}$ and  $\hat{\textbf{D}}_{B_{0}}$ (e.g.,  obtained either by random or from Algorithm \ref{alg:ini}) and the dictionary  estimates $\hat{\textbf{D}}_{A_{t}}$ and $\hat{\textbf{D}}_{B_{t}}$ produced at each iteration of Algorithm \ref{alg:dl}.
Based on Eq. (\ref{revised_dic_error}),  the following dictionary error is analysed:
\begin{equation}
\label{eq:dic_errorAB}
    \epsilon_{A_t}^D = \max_{i=1}^k  \epsilon_D\left(\hat{\textbf{D}}_{A_{t}}^{(i)}, \textbf{D}_{A}^{(i)} \right), \;
    \epsilon_{B_t}^D   = \max_{i=1}^k  \epsilon_D\left(\hat{\textbf{D}}_{B_t}^{(i)}, \textbf{D}_{B}^{(i)} \right). 
\end{equation}
Here $t=0$ corresponds to the  initial  estimate, while $t\in[T]$ denotes the iteration index.
These error measures are invariant under sign and permutation ambiguities.
 
\paragraph{Auxiliary Coefficient Error:}
We  measure the error of the estimated auxiliary coefficient  matrix by
\begin{equation}
\label{revised_coe_error}
  \epsilon_X\left(\hat{\textbf{X}} , \textbf{X}^*\right) = \left\|\hat{\textbf{X}}_{\pi,z}- \textbf{X}^*\right\|_{\infty}. 
\end{equation}
Subsequently, we analyse the following error measures for the coefficient estimates computed at each iteration $t\in[T]$ of Algorithm~\ref{alg:dl}:
\begin{equation}
    \label{eq:error_coefAB}
    \epsilon_{A_t}^X =    \epsilon_X\left(\hat{\textbf{X}}_{A_{t}} , \textbf{X}_A \right),\;\; \epsilon_{B_t}^X   = \epsilon_X\left(\hat{\textbf{X}}_{B_{t}} , \textbf{X}_B \right).
\end{equation}
The final auxiliary coefficients estimated by Algorithm \ref{alg:rl} correspond to the case $t=T$.
\paragraph{Auxiliary Latent Relation Error:} 
Regarding   the auxiliary latent relation matrix, both the ground-truth matrix $\tilde{\textbf{S}}$  defined in Eq. (\ref{eq:scaled_S_aux}) and its estimate computed in   Step \ref{eq:scaled_S_est} of Algorithm \ref{alg:rl} share a common scalar factor of $nm$. 
To remove this dependence on problem size, we define
\begin{equation}
      \epsilon_S = \epsilon_{S}\left(\hat{\textbf{S}},  \tilde{\textbf{S}}\right) = \frac{1}{nm}\left\|\hat{\textbf{S}}_{\pi,z}- \tilde{\textbf{S}} \right\|_{\infty},
\end{equation}
where  $ \hat{\textbf{S}}_{\pi,z}$ denotes the estimate after applying the corresponding permutation and sign corrections to the rows and columns of   $\hat{\textbf{S}}$.
Specifically,  since $ \hat{\textbf{X}}_{A,\pi,z} =\hat{\textbf{X}}_A \bm\Pi_A^{-1} \textbf{Z}_A^{-1}  $ and $ \hat{\textbf{X}}_{B,\pi,z} =\hat{\textbf{X}}_B \bm\Pi_B^{-1} \textbf{Z}_B^{-1}  $ following the earlier convention with the subscript ``$_A$'' and ``$_B$'' introduced to distinguish the left and right cases, it has $ \hat{\textbf{S}}_{\pi,z} = \textbf{Z}_A\bm\Pi_A \hat{\textbf{S}} \bm\Pi_B^{-1}\textbf{Z}_B^{-1}$ so that $\hat{\textbf{X}}_{A,\pi,z}\hat{\textbf{S}}_{\pi,z}\hat{\textbf{X}}_{B,\pi,z}^T =\hat{\textbf{X}}_{A}\hat{\textbf{S}}\hat{\textbf{X}}_{B}^T$.
This transformation to $ \hat{\textbf{S}}_{\pi,z}$ ensures that the permutation and sign corrections used for $\hat{\textbf{S}}$  are  consistent with those  adopted in computing the  coefficient estimation errors $\epsilon_{A_T}^X $ and $ \epsilon_{B_T}^X$.

\section{Properties of Problem Matrices}
\label{sec:matrix_res}

Our strategy for studying sparsity-induced identifiability in matrix tri-factorisation is to establish sufficient conditions under which the auxiliary coefficient and dictionary matrices can be recovered from the observed data matrix, up to permutation and sign flips, using Algorithms \ref{alg:rl}-\ref{alg:ini}, with estimation errors converging to zero.
We then analyse how coefficient sparsity ($s_A$ and $s_B$) influences these recovery conditions and error convergence.
This strategy is justified by the proven structural consistency between the  auxiliary and original  factor matrices.
Our analysis directly depends on the structure of both the original tri-factorisation problem and its auxiliary counterpart, as defined by the relation generative model   in Section \ref{sec:gen} and  the auxiliary generative model  in Definition \ref{gen_aux}.
We therefore  investigate the key properties of the matrices arising from these two generative models, including the original factor matrices ($\textbf{A}$, $\textbf{B}$ and $\textbf{S}$),  auxiliary coefficient matrices ($\textbf{X}_A$ and $\textbf{X}_B$),  auxiliary dictionary matrices ($\textbf{D}_A$ and $\textbf{D}_B$), and the auxiliary observation matrices ($\textbf{Y}_U$ and $\textbf{Y}_V$).
This section presents the resulting  properties that form the foundation of Theorems \ref{main_res}-\ref{main_res3}, with proofs  deferred to Appendix \ref{app_proof_matrix}.
In addition, we provide empirical demonstrations alongside the corresponding lemmas to illustrate the tightness of the derived bounds.
For convenience,  Table \ref{tab:constant-definitions} summarises all the problem-specific constants  used throughout the lemmas and theorems.
These constants are determined by the generating parameters of the generative models  independent of the problem dimensions $n$, $m$, and $k$, and by the probabilities $ 0<\Delta, \Delta_X<1$  required for the theoretical results to hold.

\begin{table}[t]
\centering
\resizebox{0.9\textwidth}{!}{%
\begin{tabular}{ l}
\hline
\hline
   Definition \\
\hline
  $\rho_{AB} = \frac{\min(s_A, s_B)}{\max(s_A, s_B)}$ \\  [0.8em]
 
  $ \mu_D = \frac{2\Delta}{1-\Delta} + \frac{\mu_s}{\sqrt{d}}$\\  [0.8em]
 
  $M_{A}=   M^{(A)}_{\max}   u_s\sigma_B\sqrt{\frac{(1+ \Delta )  s_B }{nk }} ,\;  M_{B} = M^{(B)}_{\max}   u_s\sigma_A\sqrt{\frac{(1+ \Delta )  s_A }{mk }}, \; m_{A} =  M^{(A)}_{\min} l_s\sigma_B\sqrt{\frac{(1- \Delta )  s_B }{nk }}  , \; m_B = M^{(B)}_{\min} l_s\sigma_A\sqrt{\frac{(1- \Delta )  s_A }{mk }}$ \\  [0.8em]
  
    $\gamma  = (1+\Delta_Y)   u_s  \sigma_A \sigma_B    \sqrt{  ( 1 +\Delta)\left(1+ \mu_D(d-1)\right)  }$ \\  [0.8em]

  $\hat{\gamma}  = (1+\Delta_Y)    \sqrt{  ( 1 +\Delta)\left(1+ \mu_D(d-1)\right) }$ \\  [0.8em]

  $\sigma_X =   \frac{u_s\sigma_A\sigma_B\sqrt{(1+\Delta)(1+\Delta_X)} \max(s_A, s_B)}{k}$ \\  [0.8em]

 $\theta_0 = \rho_{AB}^2l_s^2\sigma_A^2\sigma_B^2 (1-\Delta)
\left( \rho_{AB}-\dfrac{(1+\Delta)u_s\Delta_X}{(1-\Delta)l_s}\right)$ \\  [0.8em]

$\theta_1 =   \frac{ u_s\sqrt{ 1+\Delta_X }  }{\rho_{AB}l_s \left(\rho_{AB}-\frac{(1+\Delta)u_s\Delta_X}{(1-\Delta)l_s}\right)^{\frac{1}{2}} } \sqrt{\frac{ 1+ \Delta     }{ 1- \Delta    }}$ \\  [0.8em]

$\theta_2^{(A)} =   \frac{8.5 M^{(A)}_{\max}   u_s   }{\rho_{AB}l_s\sigma_A  \left(\rho_{AB}-\frac{(1+\Delta)u_s\Delta_X}{(1-\Delta)l_s}\right)^{\frac{1}{2}}}\sqrt{\frac{ 1+ \Delta     }{ 1- \Delta    }}, \; \theta_2^{(B)} =   \frac{8.5 M^{(B)}_{\max}   u_s   }{\rho_{AB}l_s\sigma_B  \left(\rho_{AB}-\frac{(1+\Delta)u_s\Delta_X}{(1-\Delta)l_s}\right)^{\frac{1}{2}}}\sqrt{\frac{ 1+ \Delta     }{ 1- \Delta    }}$ \\  [0.8em]
 
   $ \theta_3^{(A)} =  21(u_s  \sigma_A \sigma_B  \sqrt{1+1.1\Delta}+2) M^{(A)}_{\max}   u_s\sigma_B\sqrt{1+ \Delta}, \; \theta_3^{(B)} =  21(u_s  \sigma_A \sigma_B  \sqrt{1+1.1\Delta}+2) M^{(B)}_{\max}   u_s\sigma_A\sqrt{1+ \Delta}$ \\  [0.8em]
   
    $ \theta_4 = 1.3(u_s  \sigma_A \sigma_B  \sqrt{1+1.1\Delta}+2)\left(2 + u_s\sigma_A\sigma_B\sqrt{(1+\Delta)(1+\Delta_X)} \right)$  \\  [0.8em]
    
  $  \theta_5^{(A)} =  17 \left(  u_s\sigma_A\sigma_B\sqrt{(1+\Delta)(1+\Delta_X)} +  2   \right)\sqrt{ 1+\Delta}  M^{(A)}_{\max}   u_s\sigma_B, \; \theta_5^{(B)} =  17 \left(  u_s\sigma_A\sigma_B\sqrt{(1+\Delta)(1+\Delta_X)} +  2   \right)\sqrt{ 1+\Delta}  M^{(B)}_{\max}   u_s\sigma_A$ \\  [0.8em]
    
   $  \theta_6^{(A)} = \frac{2\theta_5^{(A)} }{ \theta_0  },\; \theta_6^{(B)} = \frac{2\theta_5^{(B)}  }{ \theta_0  },\; \theta_7^{(A)} = \frac{2\theta_3^{(A)} }{ \theta_0  }, \; \theta_7^{(B)} = \frac{2\theta_3^{(B)} }{ \theta_0  },\; \theta_8^{(A)} = \frac{4\theta_4\theta_5^{(A)} }{ \theta_0^2  }, \theta_8^{(B)} = \frac{4\theta_4\theta_5^{(B)} }{ \theta_0^2 }$\\  [0.8em]
   
   $\theta_9^{(A)} = \sqrt{1+\mu_D(d-1)}\left(\theta_7^{(A)}+\theta_8^{(A)}\right),\; \theta_9^{(B)} = \sqrt{1+\mu_D(d-1)}\left(\theta_7^{(B)}+\theta_8^{(B)}\right)$  \\
   
   $\theta_{10}^{(A)} =  \frac{M_{\min}^{(A)}l_s}{17.2M_{\max}^{(A)}u_s}\sqrt{\frac{1-\Delta}{1+\Delta}}, \; \theta_{10}^{(B)} =  \frac{M_{\min}^{(B)}l_s}{17.2M_{\max}^{(B)}u_s}\sqrt{\frac{1-\Delta}{1+\Delta}}$ \\  [0.8em]
   
  $\theta_{11}^{(A)} =  8.5M^{(A)}_{\max}u_s\sigma_B\sqrt{1+ \Delta},\; \theta_{11}^{(B)} =   8.5M^{(B)}_{\max}u_s\sigma_A\sqrt{1+ \Delta}$ \\  [0.8em]
  
  $\theta_{12}^{(A)} =\min\left(\frac{1}{4\theta_2^{(A)}}, \frac{1}{16\theta_1\theta_2^{(A)}}, \frac{1}{5\theta_6^{(A)}}\right), \; \theta_{12}^{(B)} = \min\left(\frac{1}{4\theta_2^{(B)}}, \frac{1}{16\theta_1\theta_2^{(B)}}, \frac{1}{5\theta_6^{(B)} }\right) $ \\  [0.8em]

   $\theta_{13}^{(A)} = \frac{0.18}{ \left(\theta_7^{(A)} + \theta_8^{(A) }\right)^2}$,\;  $\theta_{13}^{(B)} = \frac{0.18}{ \left(\theta_7^{(B)} + \theta_8^{(B) }\right)^2}$\\  [0.8em]

    $\theta_{14}^{(A)}=\frac{12M^{(A)}_{\max}u_s^2\sigma_A^2\hat{\gamma}(1+4\hat{\gamma})}{ {M^{(A)}_{\min}}^2l_s^2(1- \Delta )  }  $, $\theta_{14}^{(B)}=\frac{12M^{(B)}_{\max}u_s^2\sigma_B^2\hat{\gamma}(1+4\hat{\gamma})}{ {M^{(B)}_{\min}}^2l_s^2(1- \Delta )  }  $ \\  [0.8em]

    $D_A= \frac{\left(u_s^2 -l_s^2 +\left(u_s^2 + l_s^2\right)\Delta\right) {M_{\max}^{(A)}}^2 }{l_s^2(1-\Delta){M_{\min}^{(A)}}^2}, \; D_B = \frac{\left(u_s^2 -l_s^2 +\left(u_s^2 + l_s^2\right)\Delta\right){M_{\max}^{(B)}}^2 }{l_s^2(1-\Delta){M_{\min}^{(B)}}^2}$ \\  [0.8em]
  
\hline
\hline
\end{tabular}%
}
\caption{Definitions of relevant problem constants.}
\label{tab:constant-definitions}
\end{table}

\subsection{Original Coefficient Matrices: $\textbf{A}$ and $\textbf{B}$}
\label{sec:prop_coef_rel}

We first derive bounds on the singular values of  $\textbf{A}$ and $\textbf{B}$, presented in Lemma \ref{singularAB}.
Compared with the   singular-value result of  \citet{Aga16}, our result is more general.
It relaxes the unit-second-moment assumption and allows both the  tightness of the bounds and the probability with which the results hold to be controlled through a user-specified parameter  $0<\Delta<1$.
\begin{lemma}[Singular Values, $\textbf{A}$ and $\textbf{B}$]    \label{singularAB}
Suppose Assumptions   \ref{SC} and \ref{NZC}   hold. 
For any $0<\Delta<1$, there exist  universal constants $C_A, C_B>0$ so that the following holds
\begin{align}
\label{eq:sigular_value_A}
&\sigma_A \sqrt{\frac{(1- \Delta ) n s_A }{k } } \leq \sigma_{\min}(\textbf{A}) \leq 
\sigma_{\max}(\textbf{A})  \leq     \sigma_A \sqrt{\frac{(1+ \Delta ) n s_A }{k } }, \\
\label{eq:sigular_value_B}
&\sigma_B \sqrt{\frac{(1- \Delta ) m s_B }{k } } \leq \sigma_{\min}(\textbf{B}) \leq 
\sigma_{\max}(\textbf{B})  \leq     \sigma_B \sqrt{\frac{(1+ \Delta ) m s_B }{k } },
\end{align}
with  probabilities at least $1-ke^{-\frac{C_A\Delta^2  \sigma_A^2n}{k {M_{\max}^{(A)}}^2}}$ and $1-ke^{-\frac{C_B\Delta^2  \sigma_B^2m}{k {M_{\max}^{(B)}}^2}}$, respectively.
\end{lemma}

\noindent
 Lemma \ref{singularAB} implies that  the smallest singular values of $\textbf{A}$ and $\textbf{B}$  simultaneously exceed their respective lower bounds with probability  
\begin{equation}
    p=\left(1-ke^{-\frac{C_A\Delta^2  \sigma_A^2n}{k  {M_{\max}^{(A)}}^2}}\right) \left(1-ke^{-\frac{C_B\Delta^2  \sigma_B^2m}{k  {M_{\max}^{(B)}}^2}}\right) \geq  1-ke^{-\frac{C_A\Delta^2  \sigma_A^2n}{k  {M_{\max}^{(A)}}^2}} - ke^{-\frac{C_B\Delta^2  \sigma_B^2m}{k  {M_{\max}^{(B)}}^2}}.
\end{equation}
Since both lower bounds are strictly positive,  $\textbf{A}$ and $\textbf{B}$ both have full column rank with at least this probability, yielding the following corollary.
\begin{corollary}[Rank of $\textbf{A}$, $\textbf{B}$ and $\textbf{R}$]    \label{rankAB}
Suppose Assumptions \ref{SC} and \ref{NZC} hold. 
For any $0<\Delta<1$,  with a  probability  at least $1-ke^{-\frac{C_A\Delta^2  \sigma_A^2n}{k  {M_{\max}^{(A)}}^2}} - ke^{-\frac{C_B\Delta^2  \sigma_B^2m}{k  {M_{\max}^{(B)}}^2}}$, $\textbf{A}$ and $\textbf{B}$ have full column rank at the same time, and subsequently  $\textmd{rank}(\textbf{R})  = \textmd{rank}\left(\textbf{A}\textbf{S}\textbf{B}^T\right)=\textmd{rank}\left(\textbf{S}\right)=d $.
\end{corollary}

\begin{figure}[t]
    \centering
\subfigure[Lemma \ref{singularAB}, $\sigma(\mathbf{A})$]{
    \includegraphics[width=0.45\linewidth]{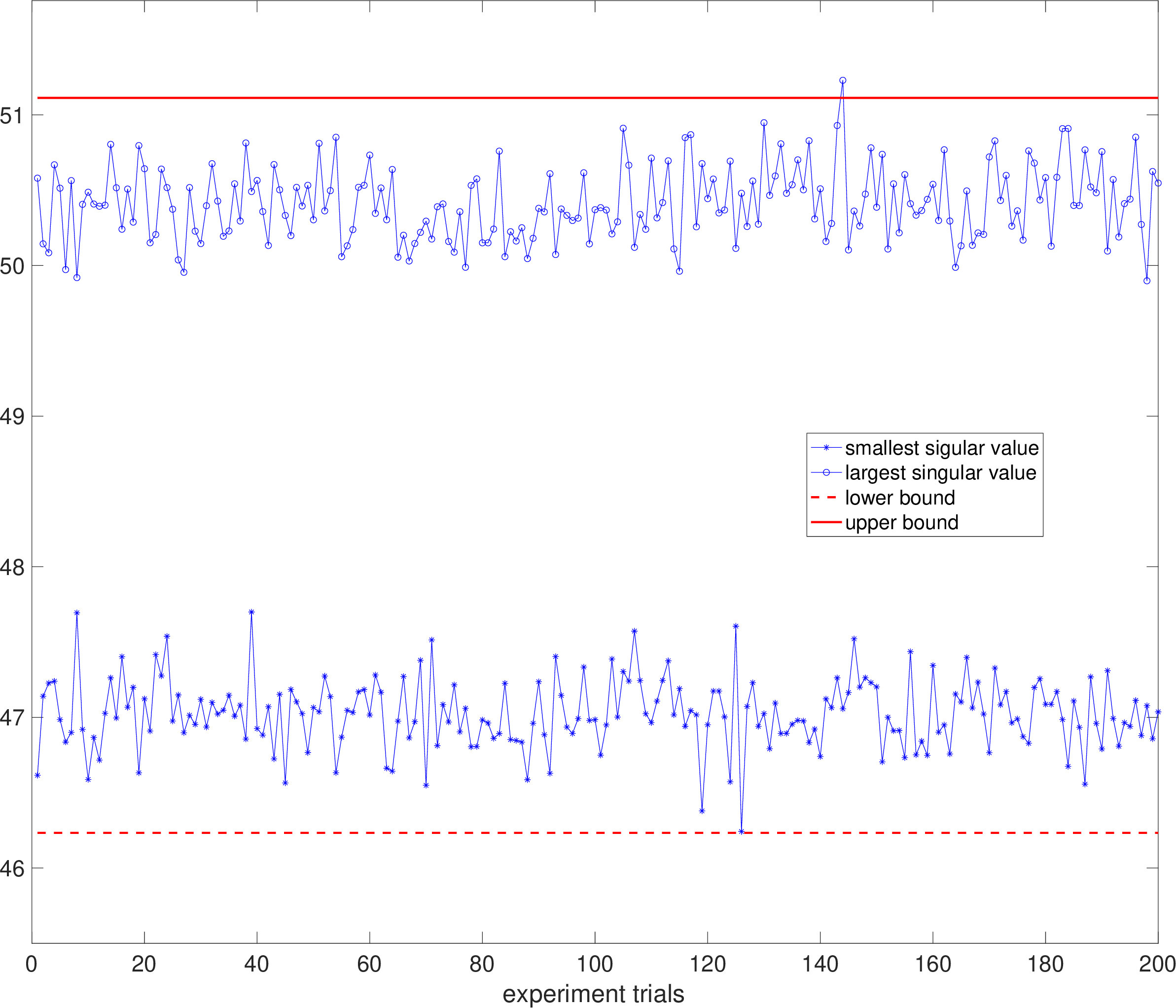}
    \label{fig:lemma9}    
  }
   \hfill
\subfigure[Corollary \ref{RowLengthFhat}, $\left\| \hat{\textbf{F}}_A^{(i)}  \right \|_2 $]{
    \includegraphics[width=0.45\linewidth]{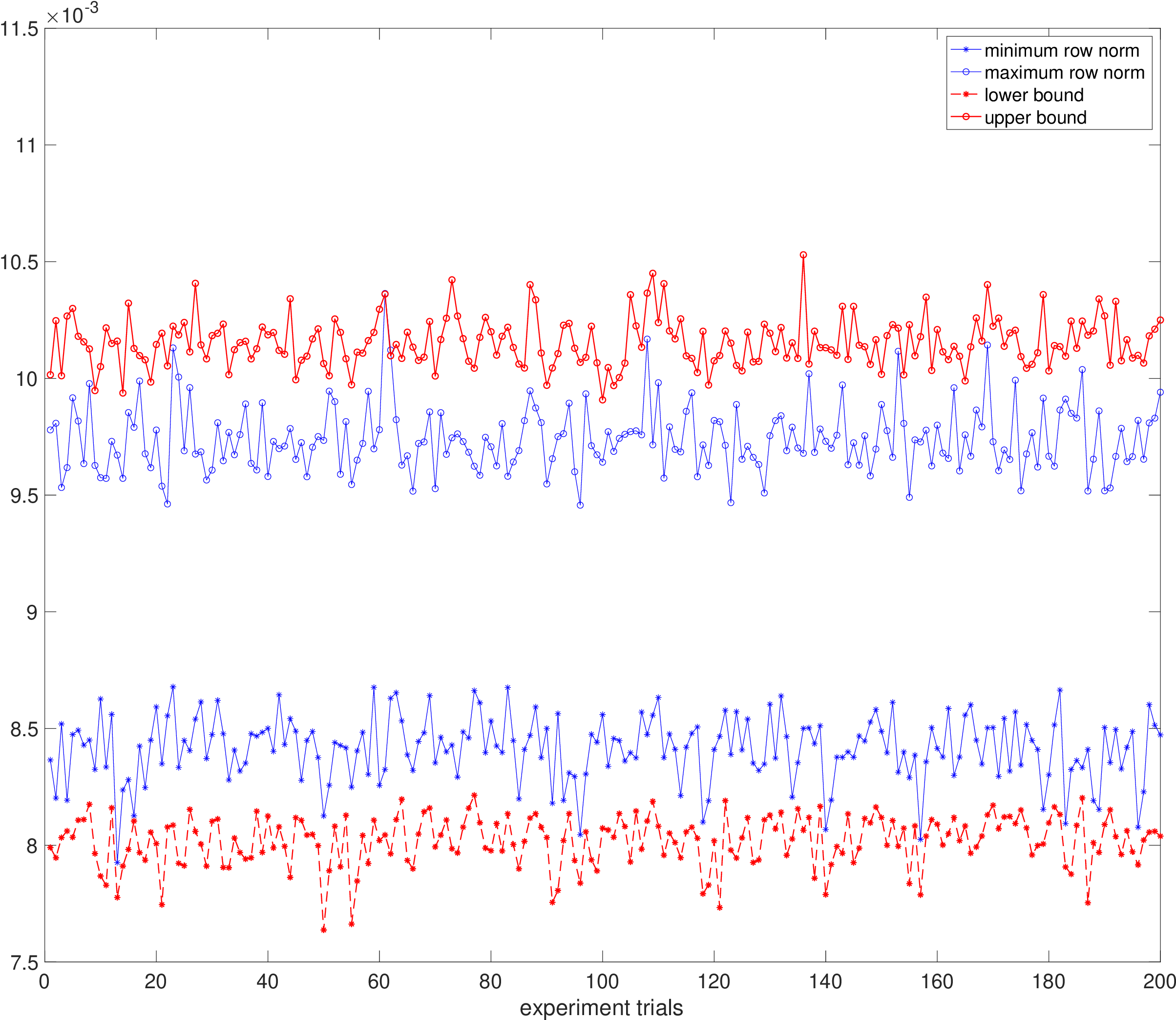}
    \label{fig:corollary12}    
  }
 \caption{(a) Comparison of the smallest and largest singular values of $\textbf{A}$ with  their theoretical bounds in Eq. (\ref{eq:sigular_value_A}) over 200  trials. (b) Comparison of the maximum and minimum values of the row $l_2$-norms of  $\hat{\textbf{F}}_A$ with  their theorectical bounds in Eq. (\ref{eq:l2norm_hatFa}) over 200 trials.}     \label{fig:AF_bound} 
\end{figure}

\paragraph{Empirical Illustration:}
We illustrate the bounds established in Lemma \ref{singularAB} by comparing the theoretical bounds from Eq. (\ref{eq:sigular_value_A}) with the empirical singular values of randomly generated sparse matrices $\mathbf{A}\in \mathbb{R}^{5000 \times 10}$,  where  each row contains at most $s_A=3$ nonzero entries. 
For each row of  $\mathbf{A}$, the positions of the nonzero entries are sampled uniformly at random without replacements.
The nonzero values are drawn independently from the uniform distribution on the interval $[-1.5, -1]\cup [1,1.5]$, yielding ${m_{\max}^{(A)}}=1$ and  ${M_{\max}^{(A)}}=1.5$.  
Consequently,
$\sigma_A^2 = \frac{ \left(M_{\max}^{(A)}\right)^3 - \left(m_{\max}^{(A)}\right)^3 }{3\left({M_{\max}^{(A)}}-{m_{\max}^{(A)}}\right)} = 1.58$, where the expression follows from the second  moment formula for the uniform distribution.
According to Lemma \ref{singularAB}, $\Delta $ controls the trade-off between   bound tightness and the probability with which the bounds hold. 
Increasing $\Delta $ produces looser bounds that hold with higher probability.
We experiment with $\Delta=0.1$,  resulting in the lower bound $\sigma_A \sqrt{\frac{(1- \Delta ) n s_A }{k } } =46.3$ for the smallest singular value  and the upper bound $ \sigma_A \sqrt{\frac{(1+ \Delta ) n s_A }{k } } = 51.1$ for the largest singular value.
We independently generate 200 random sparse matrices and compute  the  singular values of each realisation.
Figure \ref{fig:lemma9} compares the empirical singular values with the corresponding theoretical bounds over all 200   trials.
The lower bound condition $\sigma_{\min}(\textbf{A})\geq 46.3$ is satisfied in all 200 trials, while the upper-bound condition $\sigma_{\max}(\textbf{A})\leq 51.1$ is satisfied in 199 trials, demonstrating effectiveness of the derived bounds  at characterising  singular values in practice.

\subsection{Auxiliary Coefficient Matrices: $\textbf{X}_A $ and $\textbf{X}_B$}
\label{sec:prop_coef_aux}

Recalling Definition \ref{gen_aux},  the auxiliary coefficient matrices $\textbf{X}_A $ and $\textbf{X}_B$ are scaled versions of $\textbf{A}$ and $\textbf{B}$, specifically,  $
{(X_A)}_{ji}  = \textbf{A}_{ji}\left\|\hat{\textbf{F}}_B^{(i)} \right\|_2 $ and $ {(X_B)}_{ji}  =  \textbf{B}_{ji}\left\|\hat{\textbf{F}}_A^{(i)} \right\|_2$ where $\hat{\textbf{F}}_A =\frac{1}{\sqrt{nm}}\textbf{F}_A $ and $ \hat{\textbf{F}}_B= \frac{1}{\sqrt{nm}} \textbf{F}_B $.
Such scaling preserves the sparsity pattern  but  changes the distributions of the  nonzero coefficients,  which in turn affects the spectral properties of $\textbf{X}_A $ and $\textbf{X}_B$.
To facilitate the analysis, we  first bound the row norms of $\textbf{F}_A$ and $\textbf{F}_B$ in Lemma \ref{RowLengthF}, which yields the corresponding bounds for the row norms of $\hat{\textbf{F}}_A $ and $\hat{\textbf{F}}_B$ in Corollary \ref{RowLengthFhat}.

\begin{lemma}[Bounded Row Length, $\textbf{F}_A$ and $\textbf{F}_B$] 
\label{RowLengthF}
Suppose Assumptions   \ref{SC}-\ref{LS}  hold. For any $0<\Delta<1$, the $l_2$-norm of every row of $\textbf{F}_A$ and $\textbf{F}_B$ is bounded by
\begin{align}
& l_s\sigma_A\sqrt{\frac{(1- \Delta ) n s_A }{k }}   \leq   \left\| {\textbf{F}}_A^{(i)}  \right \|_2   \leq    u_s\sigma_A\sqrt{\frac{(1+ \Delta ) n s_A }{k }}, \\
&  
l_s\sigma_B\sqrt{\frac{(1- \Delta ) m s_B }{k }}   \leq   \left\| {\textbf{F}}_B^{(i)}  \right \|_2   \leq    u_s\sigma_B\sqrt{\frac{(1+ \Delta ) m s_B }{k }},
\end{align}
with probabilities  at least $1-ke^{-\frac{C_A\Delta^2  \sigma_A^2n}{k  {M_{\max}^{(A)}}^2}}$ and $1-ke^{-\frac{C_B\Delta^2  \sigma_B^2m}{k  {M_{\max}^{(B)}}^2}}$, respectively.
\end{lemma}
The following corollary is an immediate consequence of Lemma \ref{RowLengthF} after scaling by $\frac{1}{\sqrt{nm}}$.
 \begin{corollary}[Bounded Row Length, $\hat{\textbf{F}}_A$ and $\hat{\textbf{F}}_B$] \label{RowLengthFhat}
Suppose Assumptions   \ref{SC}-\ref{LS}  hold. 
For any $0<\Delta<1$, the $l_2$-norm of every row of  $\hat{\textbf{F}}_A$ and $\hat{\textbf{F}}_B$ is bounded by
\begin{align}
\label{eq:l2norm_hatFa}
&  l_s\sigma_A\sqrt{\frac{(1- \Delta )  s_A }{mk }}     \leq   \left\| \hat{\textbf{F}}_A^{(i)}  \right \|_2   \leq    u_s\sigma_A\sqrt{\frac{(1+ \Delta )  s_A }{mk }}, \\
\label{eq:l2norm_hatFb}
&  
l_s\sigma_B\sqrt{\frac{(1- \Delta )   s_B }{nk }}     \leq   \left\| \hat{\textbf{F}}_B^{(i)}  \right \|_2   \leq   u_s\sigma_B\sqrt{\frac{(1+ \Delta )   s_B }{nk }},
\end{align}
with the same probabilities as in Lemma \ref{RowLengthF}.
\end{corollary}
Corollary \ref{RowLengthFhat} immidiately yields bounds on the magnitudes of the entries of $\textbf{X}_A$ and $\textbf{X}_B$, namely $m_{A} \leq |\left(X_A\right)_{ij} | \leq M_{A} $ and $m_{B} \leq |\left(X_B\right)_{ij} | \leq M_{B} $, where
 \begin{align}
 \label{bound_maxA}
|\left(X_A\right)_{ij} | \leq &  \max_{i\in[k]} M^{(A)}_{\max}\left\|\hat{\textbf{F}}_B^{(i)} \right\|_2 \leq     M^{(A)}_{\max}   u_s\sigma_B\sqrt{\frac{(1+ \Delta )  s_B }{nk }} = M_{A} ,  \\
 \label{bound_minA}
|\left(X_A\right)_{ij} | \geq   &  \min_{i\in [k]}  M^{(A)}_{\min}\left\|\hat{\textbf{F}}_B^{(i)} \right\|_2  \geq M^{(A)}_{\min} l_s\sigma_B\sqrt{\frac{(1- \Delta )  s_B }{nk }} = m_{A} ,   \\
 \label{bound_maxB}
  |\left(X_B\right)_{ij} | \leq & \max_{i\in[k]} M^{(B)}_{\max}\left\|\hat{\textbf{F}}_A^{(i)} \right\|_2 \leq     M^{(B)}_{\max}   u_s\sigma_A\sqrt{\frac{(1+ \Delta )  s_A }{mk }} =M_{B},  \\
 \label{bound_minB}
|\left(X_B\right)_{ij} | \geq   &  \min_{i\in [k]}  M^{(B)}_{\min}\left\|\hat{\textbf{F}}_A^{(i)} \right\|_2  \geq M^{(B)}_{\min} l_s\sigma_A\sqrt{\frac{(1- \Delta )  s_A }{mk }} =m_B.   
\end{align}

Next,  we present bounds on the singular values and  column norms of $\textbf{X}_A$ and $\textbf{X}_B$  in  Lemmas \ref{singularXAXB} and \ref{columnXAXB}, respectively.
Although the full proofs are deferred to appendix, we briefly outline the main idea here.
To bound the singular values of $\textbf{X}_A$ and $\textbf{X}_B$, we first analyse    the singular values of their second moment matrices $\bm\Sigma_A$ and $\bm\Sigma_B$, using Lemma \ref{RowLengthF} together with a supporting Lemma  \ref{cov2}  provided in Appendix \ref{app:spectrum_proof}. 
We then transfer these spectra bounds   to the singular values of $\textbf{X}_A $ and $\textbf{X}_B$, through an existing theorem \citep{Ver12}  stated as Theorem \ref{SpecRand} in Appendix \ref{app:existing_theo}.

\begin{lemma}[Singular Values, $\textbf{X}_A$ and $\textbf{X}_B$] \label{singularXAXB}
Suppose Assumptions  \ref{SC}-\ref{LS} hold. 
For any $0<\Delta<1$ and $0<\Delta_X <  \frac{l_s(1-\Delta)\min(s_A, s_B)}{u_s(1+\Delta)\max(s_A, s_B)} $, the largest and smallest singular values of $ \textbf{X}_A$ and $ \textbf{X}_B$ satisfy
\begin{align}
&\sigma _{\max}\left(  \textbf{X}_A\right), \sigma _{\max}\left(  \textbf{X}_B\right) \leq \frac{u_s\sigma_A\sigma_B\sqrt{(1+\Delta)(1+\Delta_X)} \max(s_A, s_B)}{k} , \\
&\sigma_{\min}\left(  \textbf{X}_A\right), \sigma_{\min}\left(  \textbf{X}_B\right)  \geq l_s\sigma_A\sigma_B\sqrt{1-\Delta}\left(\frac{\min(s_A, s_B)}{\max(s_A,s_B)}-\frac{(1+\Delta)u_s\Delta_X}{(1-\Delta)l_s}\right)^{\frac{1}{2}} \frac{ \min(s_A, s_B)}{k }.
\end{align}
These bounds hold for $\textbf{X}_A$ with probability at least   $1-   ke^{-\frac{C_B\Delta^2  \sigma_B^2m}{k  {M_{\max}^{(B)}}^2}} - ke^{-\frac{ \hat{C}_Al_s^2 \sigma_A^2 (1- \Delta ) \Delta_X^2 ns_B   }{u_s^2 {M^{(A)}_{\max}}^2   (1+ \Delta ) ks_A }}$,  and for $\textbf{X}_B$ with probability at least   $1 - ke^{-\frac{C_A\Delta^2  \sigma_A^2n}{k  {M_{\max}^{(A)}}^2}} -ke^{-\frac{ \hat{C}_Bl_s^2 \sigma_B^2 (1- \Delta ) \Delta_X^2 ms_A   }{u_s^2 {M^{(B)}_{\max}}^2   (1+ \Delta ) k s_B}}$, where  $\hat{C}_A, \hat{C}_B>0$ are additional universal constants. 
\end{lemma}
\begin{lemma}[Column Norms, $\textbf{X}_A$ and $\textbf{X}_B$] \label{columnXAXB}
Suppose Assumptions  \ref{SC}-\ref{LS}   hold.  
For any $ \delta>0$ and $0<\Delta<1$, the $l_2$-norm of every column of the auxiliary coefficient matrices $\textbf{X}_A$ and $\textbf{X}_B$ is bounded by
\begin{equation}
\label{bound_lenX_AB}
   \frac{l_s^2 \sigma_A^2\sigma_B^2(1- \Delta )  s_As_B }{k^2 } -\delta \leq \left\|({\textbf{X}_A})_j\right\|_2^2 , \left\|({\textbf{X}_B})_j\right\|_2^2  \leq  \frac{u_s^2 \sigma_A^2\sigma_B^2(1+ \Delta )  s_As_B }{k^2 } +\delta.
\end{equation}
These bounds hold for  $\|({\textbf{X}_A})_j\|_2^2$ with probability at least  $p_A$  and for  $\|({\textbf{X}_B})_j\|_2^2$ with probability at least  $p_B$, where
\begin{align}
p_A =\; &  1 - ke^{-\frac{C_B\Delta^2  \sigma_B^2m}{k  {M_{\max}^{(B)}}^2}} - 2e^{- \frac{\frac{1}{2}nk^3\delta^2}{s_A\sigma_A^4 u_s^4\sigma_B^4(1+ \Delta )^2s_B^2 +\frac{1}{3}\left( {M^{(A)}_{\max}}^2k +s_A\sigma_A^2\right)  u_s^2\sigma_B^2(1+ \Delta )  s_B k \delta }},  \\
p_B =\; & 1 - ke^{-\frac{C_A\Delta^2  \sigma_A^2n}{k  {M_{\max}^{(A)}}^2}} - 2e^{- \frac{\frac{1}{2}mk^3\delta^2}{ s_B\sigma_B^4 u_s^4\sigma_A^4(1+ \Delta )^2s_A^2  + \frac{1}{3} \left( {M^{(B)}_{\max}}^2 k +s_B\sigma_B^2\right)  u_s^2\sigma_A^2(1+ \Delta )  s_A k \delta}}.
\end{align}
\end{lemma}

\begin{figure}[t]
    \centering
\subfigure[Lemma \ref{singularXAXB}, $\sigma(\textbf{X}_A)$]{
    \includegraphics[width=0.45\linewidth]{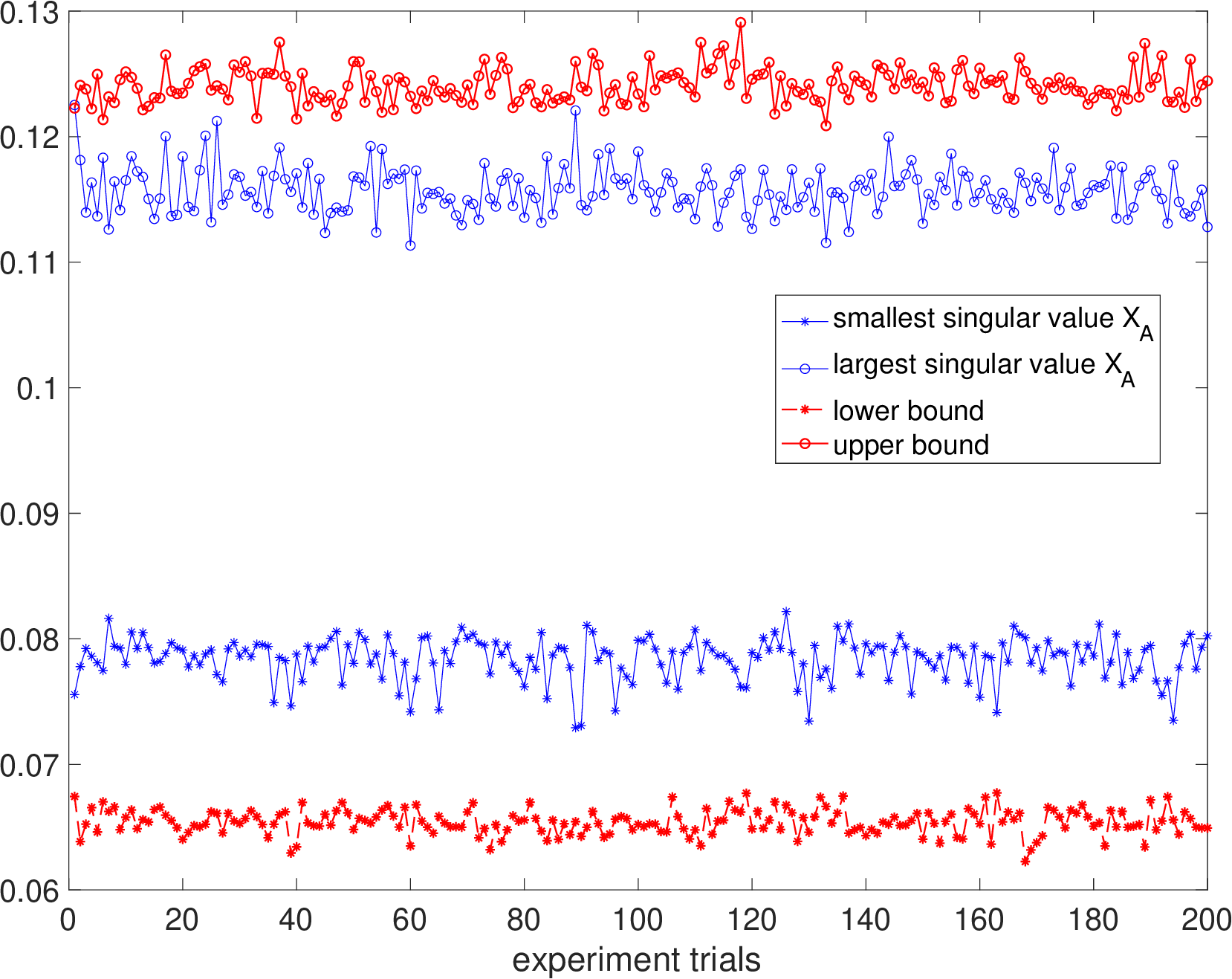}
    \label{fig:lemma13A}    
  }
  \hfill
\subfigure[Lemma \ref{singularXAXB}, $\sigma(\textbf{X}_B)$]{
    \includegraphics[width=0.45\linewidth]{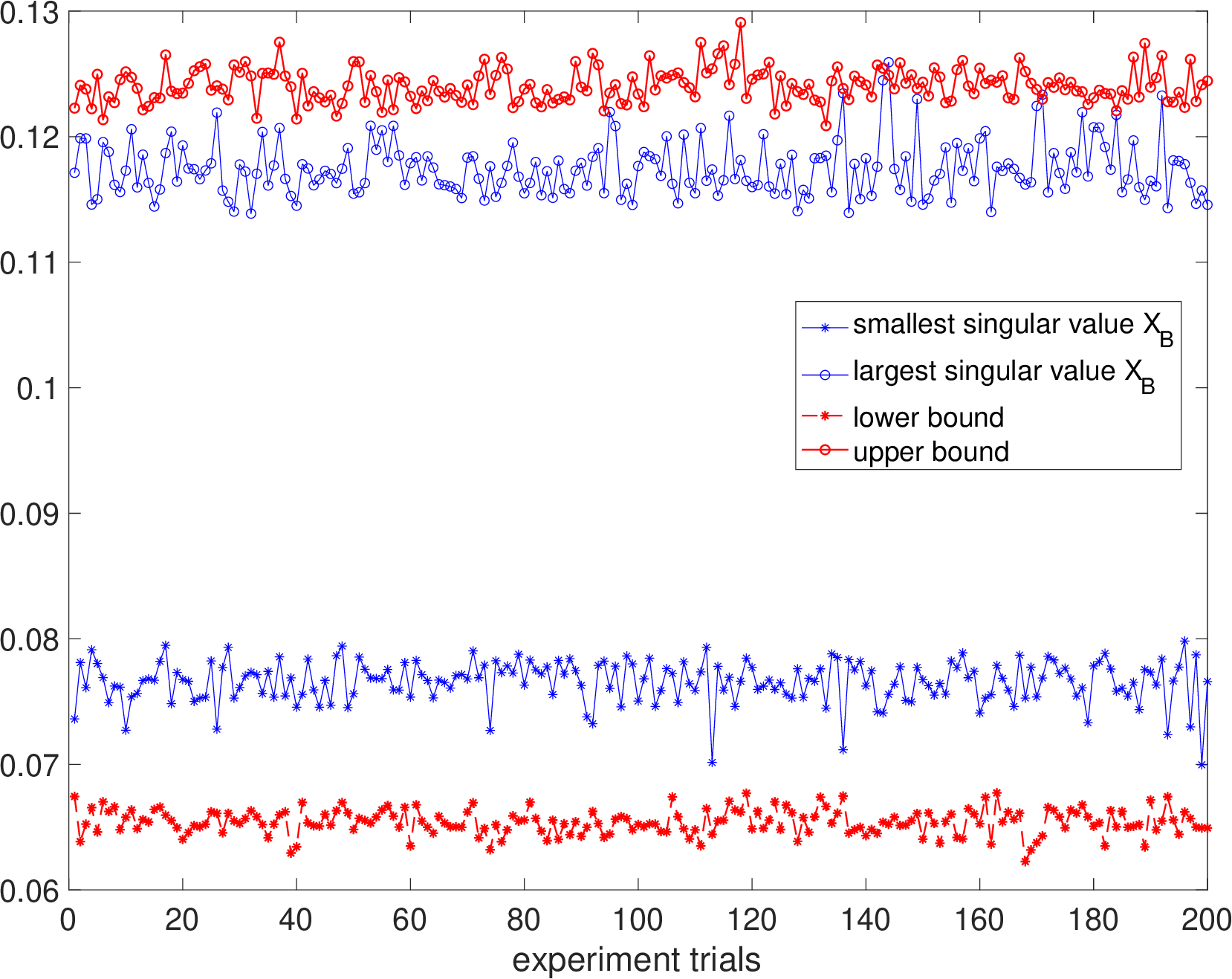}
    \label{fig:lemma13B}    
  }

\subfigure[Lemma \ref{columnXAXB},  $\|({\textbf{X}_A})_j\|_2^2$ ]{
    \includegraphics[width=0.45\linewidth]{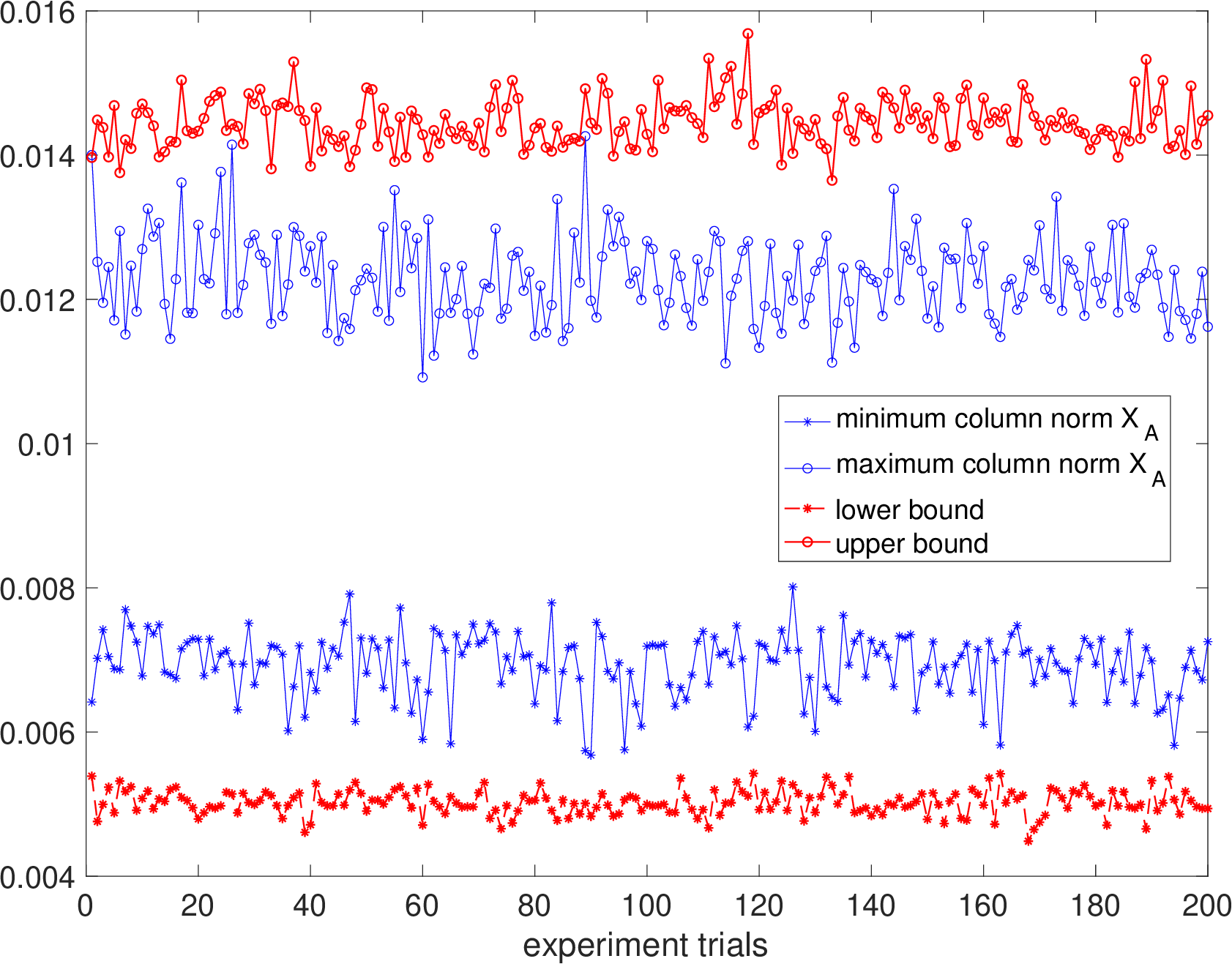}
    \label{fig:lemma14A}    
  }
\hfill
\subfigure[Lemma \ref{columnXAXB},  $\|({\textbf{X}_B})_j\|_2^2$]{
    \includegraphics[width=0.45\linewidth]{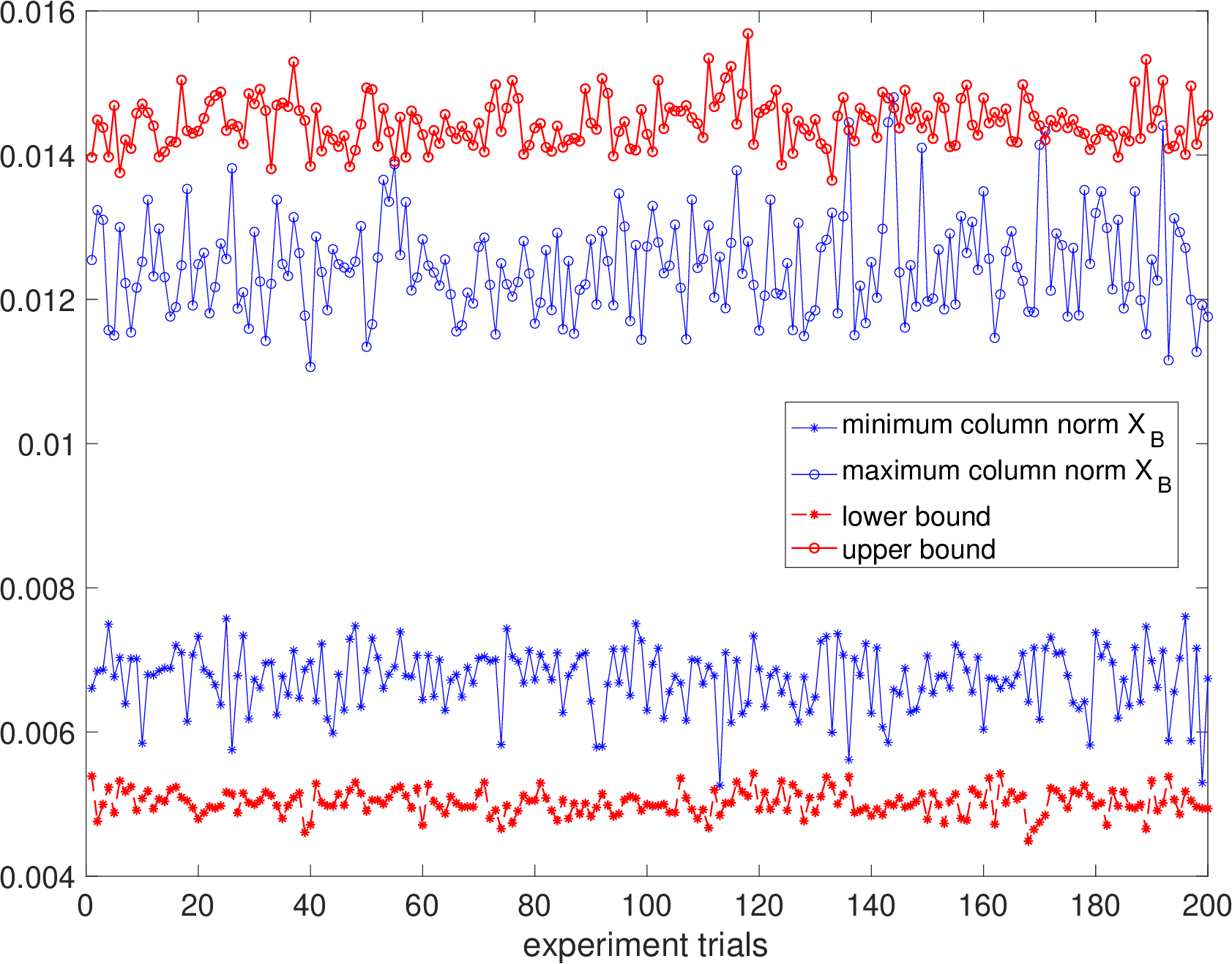}
    \label{fig:lemma14B}    
  }
 \caption{Comparison of the smallest/largest singular values  and the minimum/maximum column $l_2$-norms   with  their theoretical bounds over 200 trials for  $\textbf{X}_A$ and $\textbf{X}_B$.}    \label{fig:aux_coe_bound} 
\end{figure}

\paragraph{Empirical Illustration:}
We illustrate the bounds  derived in Corollary \ref{RowLengthFhat},  Lemma    \ref{singularXAXB}, and Lemma \ref{columnXAXB}.
We generate  two independent sparse coefficient matrices $\mathbf{A}\in \mathbb{R}^{1500 \times 50}$ and $\mathbf{B}\in \mathbb{R}^{1200 \times 50}$,  with each row containing at most $s_A=s_B=3$ nonzero entries. 
The  coefficients  of $\mathbf{A}$ and  $\mathbf{B}$ are generated using the same procedure as in the empirical illustration in Section \ref{sec:prop_coef_rel}, which   yields   ${m_{\max}^{(A)}}={m_{\max}^{(B)}}=1$, ${M_{\max}^{(A)}}= {M_{\max}^{(B)}}=1.5$, and $\sigma_A^2 = \sigma_B^2  = 1.58$.
The latent relation matrix $\textbf{S} \in \mathbb{R}^{50 \times 50}$ is generated as $\textbf{S} =  \textbf{O}  + 0.2\textbf{L}$, where $\textbf{O}\in \mathbb{R}^{50 \times 50}$  is a random orthogonal matrix, and $\textbf{L}\in \mathbb{R}^{50 \times 50}$ is a random matrix whose entries are independently sampled from the standard normal distribution and   normalised to have unit  $l_2$-norm for  its rows. 
The resulting $\textbf{S}$ satisfies  $\textbf{S}^T\textbf{S}$ and  $\textbf{S}\textbf{S}^T$ being approximately diagonal.

A total of  200 realisations of the triple $(\mathbf{A},\mathbf{B},\mathbf{S})$ are independently generated.
For each realisation, we compute  $\hat{\textbf{F}}_A$ and $\textbf{X}_A$, and estimate $l_s$ and $u_s$ by  the minimum and maximum   $l_2$-norms of the rows and columns of $\textbf{S}$.
To verify Corollary \ref{RowLengthFhat},  we compare the minimum and maximum  of the row norms $\left\{ \left\| \hat{\textbf{F}}_A^{(i)}  \right \|_2\right\}_{i=1}^{50}$ with the theoretical bounds in Eq. (\ref{eq:l2norm_hatFa}) computed using $\Delta = 0.1$.
As shown in Figure \ref{fig:corollary12}, both   bounds are satisfied in all 200 trials.
To verify  Lemma   \ref{singularXAXB}, we compare the smallest and largest singular values with their theoretical bounds computed using  $\Delta = 0.2$, in Figures \ref{fig:lemma13A}  and \ref{fig:lemma13B} for   $\textbf{X}_A$ and $\textbf{X}_B$, respectively.
Across the 200 trials, the lower bound on the smallest singular value is satisfied in every trial for both $\mathbf{X}_A$ and $\mathbf{X}_B$, while the upper bound on the largest singular value is satisfied in 199 and 198 trials for $\textbf{X}_A$ and   $\textbf{X}_B$, respectively.
To verify  Lemma \ref{columnXAXB},  we compare the minimum and maximum   of the column norms  with their theoretical bounds computed using $\Delta = 0.25$ in   Figures \ref{fig:lemma14A}  and \ref{fig:lemma14B}, for  $\textbf{X}_A$ and $\textbf{X}_B$, respectively.
The lower and upper bounds are satisfied in 200 and 199 trials for $\mathbf{X}_A$, and in 200 and 197 trials for $\mathbf{X}_B$, respectively.
As discussed earlier, $\Delta $  controls the bound tightness. 
In this illustration, we choose values of $\Delta$ that provide reasonably tight bounds while maintaining a high empirical probability of validity.

\subsection{Auxiliary Dictionary Matrices: $\textbf{D}_A$ and $\textbf{D}_B$}
\label{sec:prop_dic}

The rows of the two auxiliary dictionary matrices $\textbf{D}_A$ and $\textbf{D}_B$ have unit $l_2$-norm by definition. They satisfy the following properties.

\begin{lemma}[Mutual Incoherence, $\textbf{D}_A$ and $\textbf{D}_B$] \label{MC_DaDb}
 Suppose Assumptions  \ref{SC} , \ref{NZC}, and \ref{LI} hold.
For any $0<\Delta<1$,    the following holds
\begin{equation}
\label{Coherence_DADB}
    \left |   \left \langle \textbf{D}_A^{(i)}, \textbf{D}_A^{(j)}\right \rangle  \right |\leq \mu_D, \; \left |   \left \langle \textbf{D}_B^{(i)}, \textbf{D}_B^{(j)}\right \rangle  \right | \leq \mu_D.
\end{equation} 
with probabilities at least $1-ke^{-\frac{C_A\Delta^2  \sigma_A^2n}{k {M_{\max}^{(A)}}^2}}$  and $1-ke^{-\frac{C_B\Delta^2  \sigma_B^2m}{k {M_{\max}^{(B)}}^2}}$, respectively.

\end{lemma}

\begin{lemma}[Bounded Spectral Norm, $\textbf{D}_A$ and $\textbf{D}_B$] \label{BS_DaDb}
 Suppose Assumptions \ref{SC} , \ref{NZC}, and \ref{LI}  hold.
For any $0<\Delta<1$,   the spectral norm of the auxiliary dictionaries is bounded by
\begin{equation} 
\label{SpectralNorm_DAB}
  \| \textbf{D}_A\|_2 \leq  \sqrt{ 1+ \mu_D(d-1)},\;  \| \textbf{D}_B\|_2  \leq  \sqrt{ 1+ \mu_D(d-1)},
 \end{equation}
with probabilities at least $1-ke^{-\frac{C_A\Delta^2  \sigma_A^2n}{k {M_{\max}^{(A)}}^2}}$   and $1-ke^{-\frac{C_B\Delta^2  \sigma_B^2m}{k {M_{\max}^{(B)}}^2}}$, respectively.
\end{lemma}

\paragraph{Remarks:} The   upper bounds established in the two preceding lemmas are expressed in terms of the quantities $\mu_D$ and $\sqrt{ 1+ \mu_D(d-1)}$, both of which are governed by $ \mu_D = \frac{2\Delta}{1-\Delta} + \frac{\mu_s}{\sqrt{d}}$ as defined in Table \ref{tab:constant-definitions}.
These bounds become tighter as $\Delta$ decreases, although this comes at the cost of a  lower  probability for the bounds to hold.
From the expression of the auxiliary dictionary matrices, i.e.,  $\textbf{D}_A  = \textbf{L}_A  \textbf{S}^T\textbf{A}^T\textbf{U}_R $ and  $\textbf{D}_B  = \textbf{L}_B  \textbf{S}\textbf{B}^T\textbf{V}_R $, we see that their spectral properties are primarily determined by the corresponding coefficient matrix  ($\textbf{A}$ or $\textbf{B}$)  and the latent relation matrix  $\textbf{S}$.
The quantity $ \mu_D$ clearly separates the contributions of these two factors.
Specifically,  the first term $\frac{2\Delta}{1-\Delta} $ reflects the conditioning of the coefficient matrices through the ratio between their largest and smallest singular values (see Eqs. (\ref{express_D2}) and (\ref{express_D3}) in the proof of Lemma \ref{MC_DaDb}), whereas the second term $\frac{\mu_s}{\sqrt{d}}$ captures the incoherence of the latent relation matrix  as characterised by Assumption \ref{LI}.

\paragraph{Empirical Illustration:}

We demonstrate the bounds derived in Lemmas \ref{MC_DaDb} and \ref{BS_DaDb}, using exactly the same experiment setting as the empirical illustration in  Section \ref{sec:prop_coef_aux}.
The only difference is that we substantially reduce the sample size of one of the two  coefficient matrices, for which we keep $n=1500$ for $\textbf{A}$  while reducing the sample size for $\textbf{B}$ to $m=500$.
Following Assumption \ref{LI}, we estimate $\frac{\mu_s}{\sqrt{d}}$ by taking the maximum   cosine similarities among both the rows and columns of the generated $\textbf{S}$ in each trial.
All bounds are computed  with $\Delta = 0.05$, and the results are presented in Figure \ref{fig:dic_bound}.
For   $\textbf{D}_A $, which is constructed from   $\textbf{A}$ with the larger sample size $n=1500$, the    upper bound   on  the atom inner-product magnitudes given by Lemma  \ref{MC_DaDb}  is satisfied in all 200 trials, as shown in Figure \ref{fig:lemma15}.
In contrast, for $\textbf{D}_B $, which is constructed from   $\textbf{B}$ with a much smaller sample size $m=500$, the same upper bound holds in 187  out of 200 trials, corresponding to a success rate of $93.5\%$.
This empirical observation is consistent with the sample-complexity behavior implied by the  probabilities under which the upper bound holds, i.e., $1-ke^{-\frac{C_A\Delta^2  \sigma_A^2n}{k {M_{\max}^{(A)}}^2}}$   and $1-ke^{-\frac{C_B\Delta^2  \sigma_B^2m}{k {M_{\max}^{(B)}}^2}}$.
These probability expressions increase with the sample sizes $n$ and $m$, indicating that the bound is more likely to hold when more samples are available.
As shown in Figure \ref{fig:lemma16}, the upper bound on the spectral norm of the auxiliary dictionary matrix is satisfied in all trials, but is less tight than the bound on the atom inner products.

\begin{figure}[t]
    \centering
    \subfigure[Lemma \ref{MC_DaDb}, $\left \langle \textbf{D}_A^{(i)}, \textbf{D}_A^{(j)}\right \rangle$ and $\left \langle \textbf{D}_B^{(i)}, \textbf{D}_B^{(j)}\right \rangle$]{
        \includegraphics[width=0.45\linewidth]{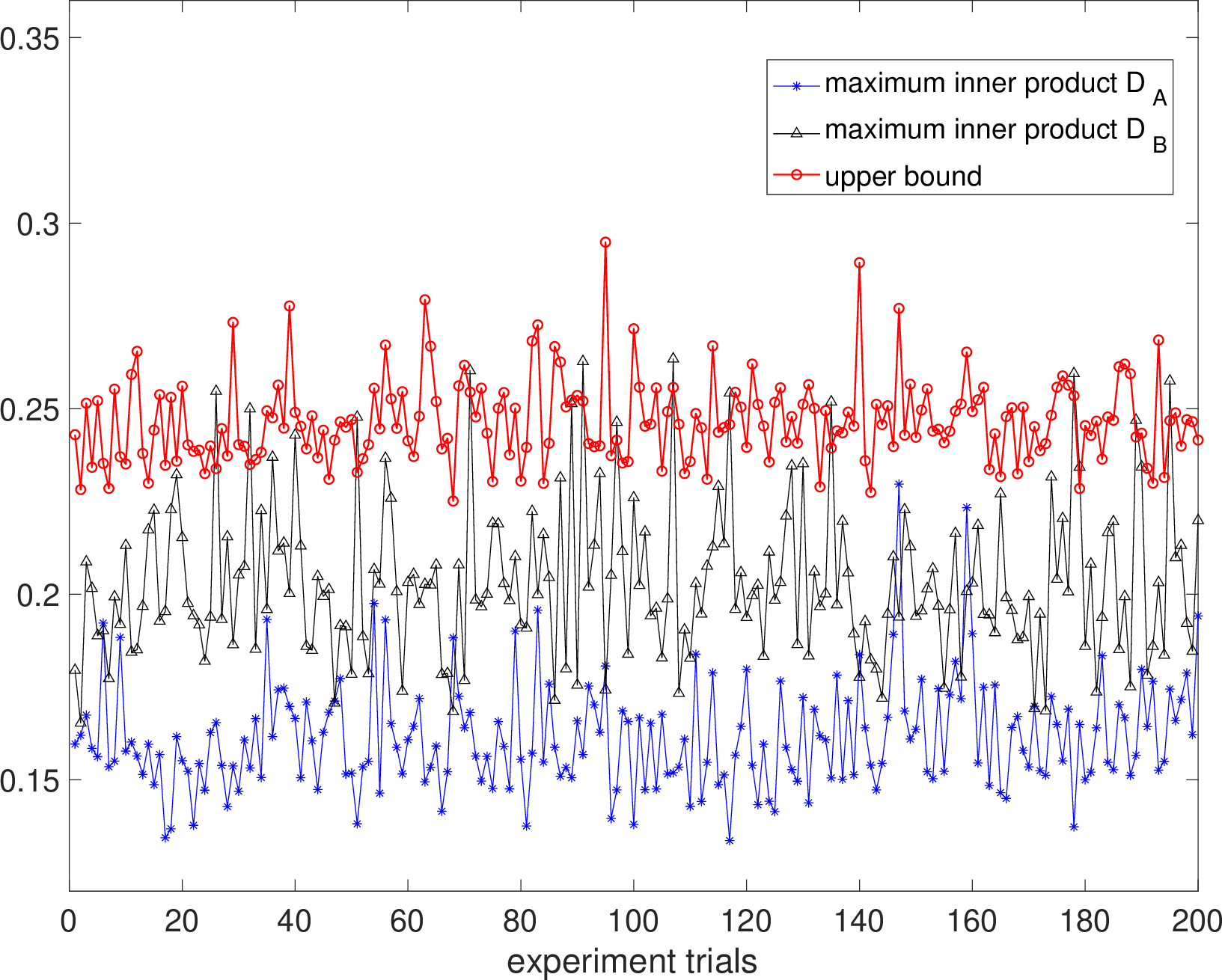}
        \label{fig:lemma15}    
    }
    \hfill
    \subfigure[Lemma \ref{BS_DaDb}, $\| \textbf{D}_A\|_2$ and $\| \textbf{D}_B\|_2$]{
        \includegraphics[width=0.45\linewidth]{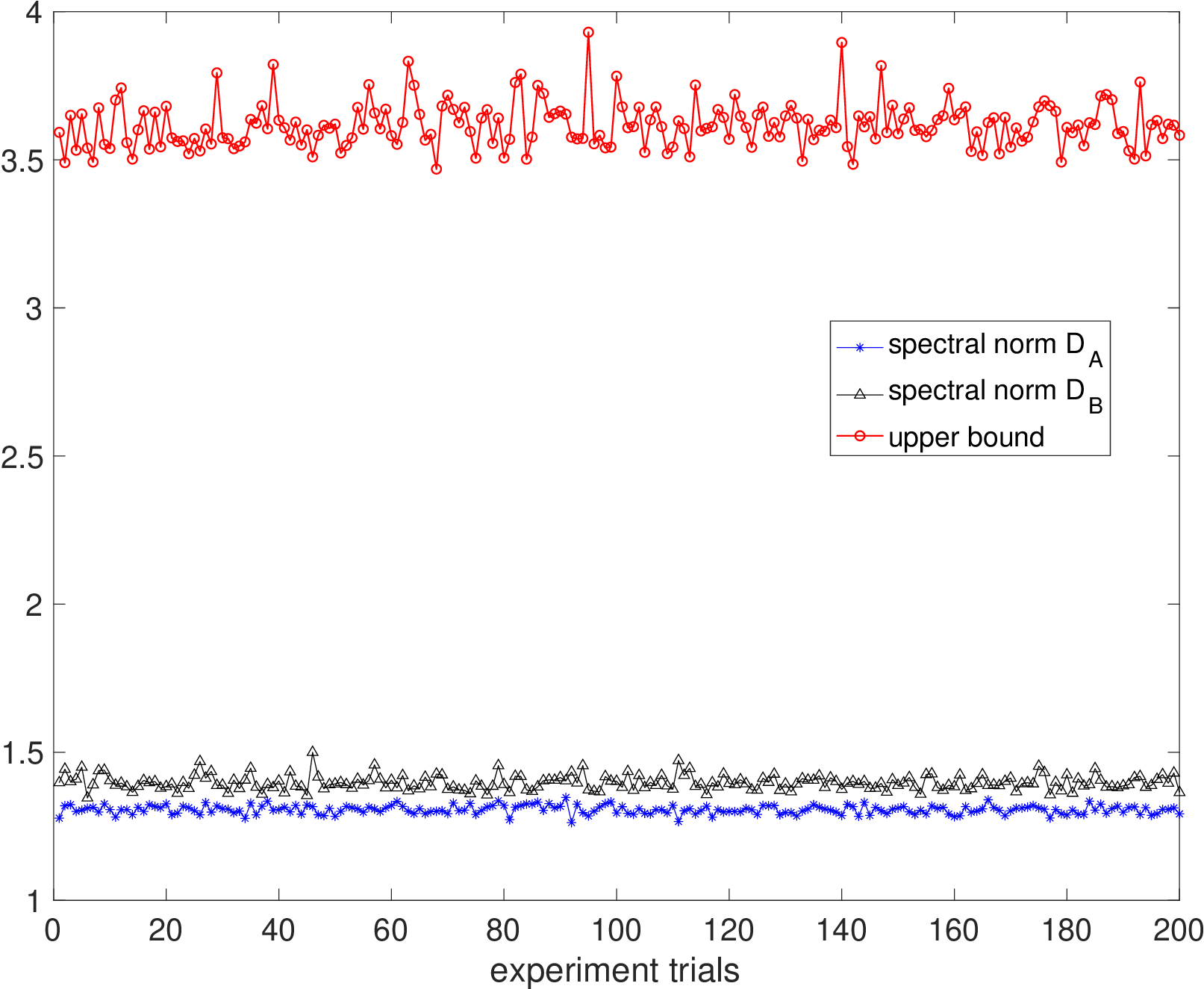}
        \label{fig:lemma16}    
    }
    \caption{Comparison of the maximum inner product between the auxiliary atoms and the spectral norm of the auxiliary dictionary matrix with their theoreticl bounds over 200 trials for $\textbf{D}_A$ and $\textbf{D}_B$.}   \label{fig:dic_bound} 
\end{figure}

\subsection{Auxiliary Observation Matrices: $\textbf{Y}_U$ and $\textbf{Y}_V$}

\begin{figure}[t]
    \centering
    \subfigure[Lemma \ref{RowLength_YuYv}, $\left\|\textbf{Y}_U^{(i)} \right\|_2$]{
        \includegraphics[width=0.45\linewidth]{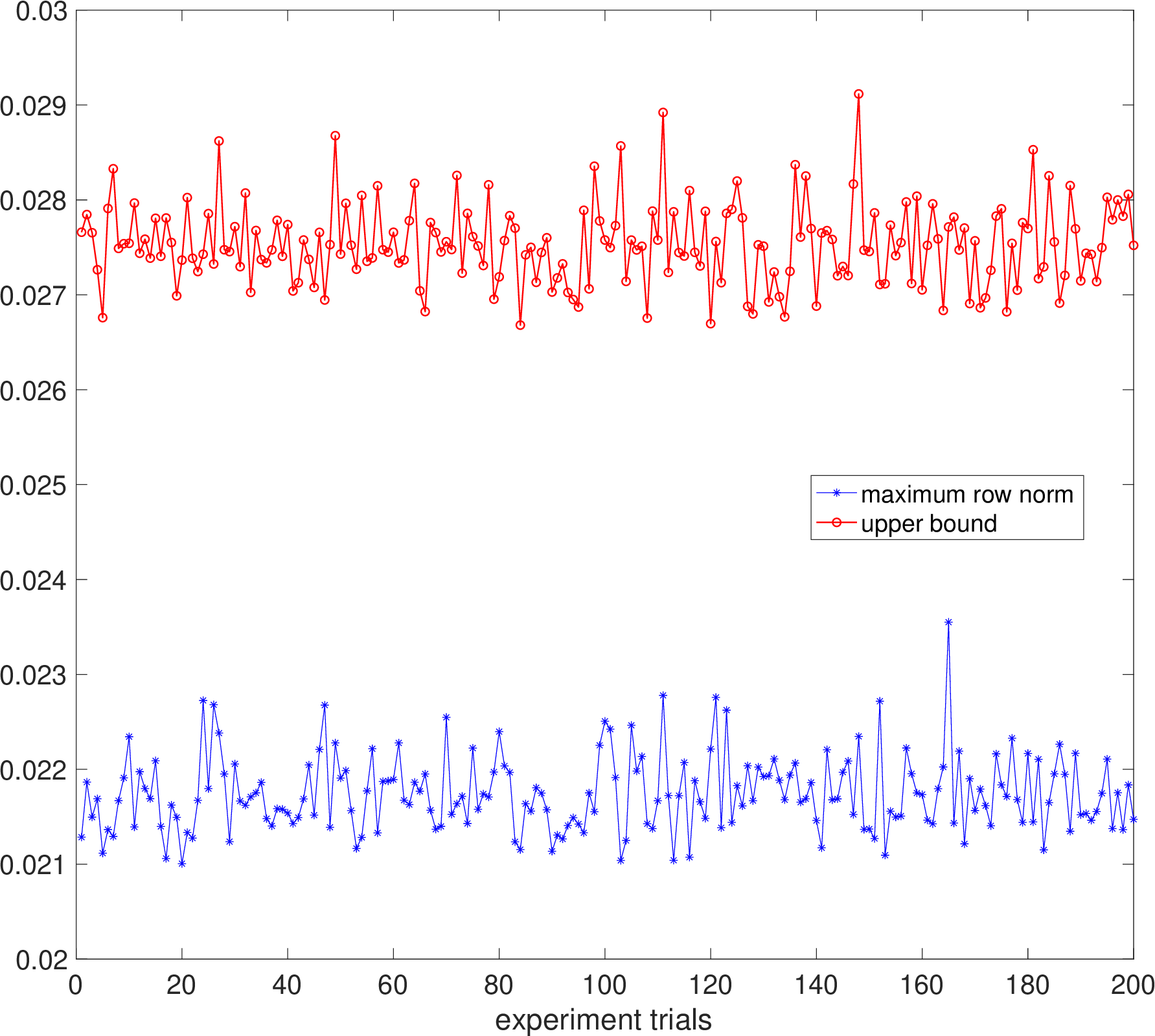}
        \label{fig:lemma17}    
    }
    \hfill
    \subfigure[Lemma \ref{SSN_YuYv}, $\left\| \textbf{Y}_U^{(I_U)} \right\|_2$]{
        \includegraphics[width=0.45\linewidth]{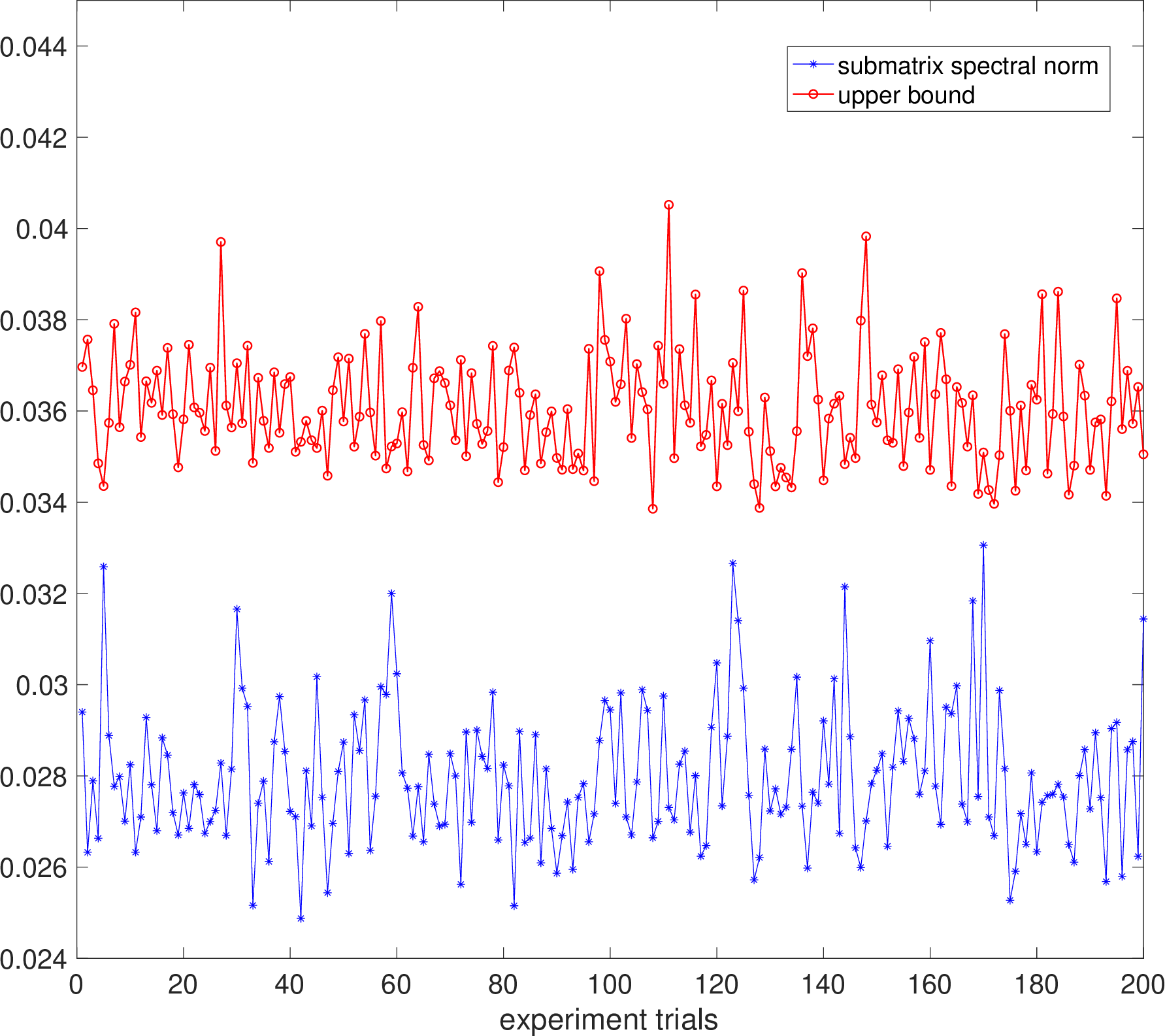}
        \label{fig:lemma18}    
    }
    \caption{Comparison of the maximum  row $l_2$-norm  of $\textbf{Y}_U$ and the spectral norm of a random selected submatrix of $\textbf{Y}_U$ with their theoretical bounds over 200 trials.}   \label{fig:observationY_bound} 
\end{figure}

Leveraging the results established for $\textbf{X}_A$ and $\textbf{X}_B$, together with those for $\textbf{D}_A$ and $\textbf{D}_B$, we first derive bounds on the row norms of the two  auxiliary observation matrices $\textbf{Y}_U$ and $\textbf{Y}_V$ in Lemma \ref{RowLength_YuYv}.
These are then  used to establish bounds on the spectral norms of submatrices of the auxiliary observation matrices.
The resulting bounds are presented in Lemma \ref{SSN_YuYv}.

\begin{lemma}[Bounded Row Norm, $\textbf{Y}_U$ and $\textbf{Y}_V$] \label{RowLength_YuYv}
 Suppose Assumptions  \ref{SC}-\ref{LI} hold.
For any $0<\Delta<1$,  the row $l_2$-norms of   $\textbf{Y}_U = \textbf{X}_A\textbf{D}_B $  and $\textbf{Y}_V = \textbf{X}_B\textbf{D}_A$  are upper bounded  by
\begin{align} 
\label{eq:Row_Yu}
\left\|\textbf{Y}_U^{(i)} \right\|_2 < \;  &  u_sM^{(A)}_{\max}   \sigma_B \sqrt{\frac{ (1+ \Delta )( 1+ s_A\mu_D )  s_As_B }{nk }  } ,  \\
\label{eq:Row_Yv}
 \left\|  \textbf{Y}_V^{(i)} \right\|_2 < \;  & u_s M^{(B)}_{\max}   \sigma_A \sqrt{\frac{ (1+ \Delta ) ( 1+ s_B\mu_D ) s_As_B }{mk } } , 
 \end{align}
 with probabilities at least $1-ke^{-\frac{C_B\Delta^2  \sigma_B^2m}{k {M_{\max}^{(B)}}^2}}$   and   $1-ke^{-\frac{C_A\Delta^2  \sigma_A^2n}{k {M_{\max}^{(A)}}^2}}$,  respectively.  
 \end{lemma}

\begin{lemma}[Submatrix Spectral Norm, $\textbf{Y}_U$ and $\textbf{Y}_V$] \label{SSN_YuYv}
 Suppose Assumptions  \ref{SC}-\ref{LI} hold.
  Let $I_U$ and $I_V$  denote the index sets of randomly selected rows of $\textbf{Y}_U = \textbf{X}_A\textbf{D}_B $  and   $\textbf{Y}_V = \textbf{X}_B\textbf{D}_A$, respectively, and let $|I_U|$ and $|I_V|$ denote their cardinalities.
  Let $\textbf{Y}_U^{(I_U)}$ and $\textbf{Y}_V^{(I_V)}$ denote the corresponding row submatrices.
 For any $0<\Delta_Y, \Delta<1$, the spectral norms of $\textbf{Y}_U^{(I_U)}$ and $\textbf{Y}_V^{(I_V)}$ are upper bounded by 
\begin{equation} 
\label{eq:row_subYuYv}
\left\| \textbf{Y}_U^{(I_U)} \right\|_2 \leq        \gamma    \sqrt{\frac{\left| I_U \right|s_As_B}{nk^2}}, \;  \left\| \textbf{Y}_V^{(I_V)} \right\|_2  \leq  \gamma \sqrt{\frac{\left| I_V \right|s_As_B}{mk^2}},
 \end{equation}
 with probabilities at least $p_u$ and $p_v$, respectively, where
 \begin{align}
     p_u = \; & 1- ke^{-\frac{C_B\Delta^2  \sigma_B^2m}{k {M_{\max}^{(B)}}^2}}-de^{- \frac{ C_U \Delta_Y^2\sigma_A^2\left| I_U \right| \left(1+ \mu_D(d-1)\right)}{ {M_{\max}^{(A)}}^2( 1+ s_A\mu_D)}},
     \\
     p_v =\; & 1- ke^{-\frac{C_A\Delta^2  \sigma_A^2n}{k {M_{\max}^{(A)}}^2}}-de^{- \frac{ C_V \Delta_Y^2\sigma_B^2\left| I_V \right| \left(1+ \mu_D(d-1)\right)}{ {M_{\max}^{(B)}}^2( 1+ s_B\mu_D)}}.
 \end{align}
Here, $C_A, C_B, C_U, C_V>0$ are positive universal constants.
\end{lemma}

\paragraph{Empirical Illustration:}

We demonstrate the  upper bounds derived in Lemmas \ref{RowLength_YuYv} and \ref{SSN_YuYv} for $\textbf{Y}_U \in \mathbb{R}^{1500\times 50}$ using exactly the same experiment setting as in the empirical illustration of  Section \ref{sec:prop_coef_aux}.
In each of the 200 trials,  we compute $\textbf{Y}_U$, estimate $l_s$, $u_s$ and $\mu_s$, and randomly select $|I_U|=20$ rows to 
 illustrate the result in Lemma  \ref{SSN_YuYv}.
The theoretical bounds are computed with $\Delta=0.01$.
Figure \ref{fig:observationY_bound} presents the comparison results, where the derived upper bounds hold in all 200 trials while exhibiting reasonable tightness.

\section{Main  Results}
\label{rel_res}

Having introduced the problem formulation, data-generating process, factorisation models, algorithmic framework, and the associated analyses of the problem matrices and algorithmic properties,  we now present our main theoretical results on the role of coefficient sparsity in matrix tri-factorisation.
We first analyse the auxiliary tri-factorisation problem $\textbf{R}  \approx \textbf{X}_A \tilde{\textbf{S}}\textbf{X}_B^T$ using Algorithm \ref{alg:rl}, investigating  how the coefficient sparsity levels $s_A$ and $s_B$ affect the recovery of the  auxiliary coefficient  matrices $\textbf{X}_A$ and $\textbf{X}_B$, and the auxiliary latent relation  matrix $\tilde{\textbf{S}}$.
We establish sufficient conditions under which Algorithm \ref{alg:rl} asymptotically recovers $\textbf{X}_A$, $\textbf{X}_B$, and $\tilde{\textbf{S}}$ from the observed $\mathbf{R}$ with high probability, while recovering the auxiliary dictionary matrices $\textbf{D}_A$ and $\textbf{D}_B$ as an intermediate result in the proof.
We next examine coefficient sparsity from a different perspective through the spectral dictionary approximation computed using Algorithm~\ref{alg:ini}, deriving an upper bound on its error.
These analyses reveal how coefficient sparsity affects both  the recovery conditions and the corresponding estimation errors. 

We show  that greater sparsity   (smaller  $s_A$ and $s_B$)  accelerates the convergence of the estimation errors in  Algorithm \ref{alg:rl}, reduces the approximation error of Algorithm \ref{alg:ini}, and relaxes the conditions required by the corresponding recovery guarantees.
Finally, to connect the auxiliary analysis with the original problem, we investigate the structural relationship  between  the auxiliary tri-factorisation  $\textbf{R}  \approx \textbf{X}_A \tilde{\textbf{S}}\textbf{X}_B^T$ and the original tri-factorisation  $\textbf{R}  \approx \textbf{A}\textbf{S}\textbf{B}^T$.
In particular, we characterise their shared structure,  quantify the structural shifts in the corresponding coefficient and latent relation matrices, and identify the factors  governing these shifts.
Our analysis aggregates the parameters defining the generative models into a collection of problem constants listed in Table~\ref{tab:constant-definitions}, and expresses the theoretical results in terms of these constants to highlight the role of sparsity.

\subsection{Auxiliary Model Identification}

\subsubsection{Theoretical Results}
We first present the model identification results for the auxiliary tri-factorisation model $  \textbf{R}  \approx \textbf{X}_A \tilde{\textbf{S}}\textbf{X}_B^T$ in Theorem \ref{main_res}, by analysing the recovery errors of the  auxiliary coefficient, dictionary and latent relation matrices estimated using Algorithm \ref{alg:rl}.
When applying Algorithm \ref{alg:dl}, the lasso accuracy and sparsity control parameters  are set as  $\epsilon_t =  M_A\epsilon_{B_{t-1}}^D $ and $ \rho^{(s)}_t = 8.6 M_A\epsilon_{B_{t-1}}^D$, respectively, for the auxiliary observation $\textbf{Y}_U$.
For $\textbf{Y}_V$,  they are set as $\epsilon_t =  M_B\epsilon_{A_{t-1}}^D $ and $\rho^{(s)}_t = 8.6 M_B\epsilon_{A_{t-1}}^D$. 

\begin{theorem}[Auxiliary Recovery Guarantees] \label{main_res}
Suppose Assumptions \ref{SC}-\ref{LI} hold. 
Given $0< \Delta<1$, $0<\Delta_X <  \frac{l_s(1-\Delta) \rho_{AB} }{u_s(1+\Delta) } $, and   $0<\eta < \min\left(\frac{1}{\sqrt{s_A}}, \frac{1}{\sqrt{s_B}}\right)$, 
suppose the coefficient sparsity satisfies  
\begin{equation}
\label{main:sparsity_condition}
     \max(s_A, s_B)   \leq \min\left( \frac{0.05}{ \mu_D}, \left( \frac{ \min\left(\theta_{13}^{(A)}, \theta_{13}^{(B)}\right)\eta^2 k}{   1+\mu_D(d-1)  }\right)^{\frac{1}{3}}\right).
\end{equation}
Suppose Algorithm \ref{alg:rl} is initialised with  auxiliary dictionary estimates that satisfy  
\begin{align}
\label{eq:final_dic_condA}
     \epsilon_{A_{0}}^D  \leq\; &    \min\left( \frac{1}{40\sqrt{s_B}}, \frac{   \theta_{10}^{(B)}  }{ s_B}, \frac{ \theta_{12}^{(B)}}{\sqrt{s_B^3}},   \frac{1}{ \theta_{11}^{(B)}\sqrt{s_As_B^2}},  \frac{ \sqrt{k}}{  5\theta_9^{(B)}    \max(s_A,s_B)^2  }\right),  \\
\label{eq:final_dic_condB}
         \epsilon_{B_{0}}^D  \leq   \; & \min\left( \frac{1}{40\sqrt{s_A}}, \frac{   \theta_{10}^{(A)}  }{ s_A }, \frac{ \theta_{12}^{(A)}}{\sqrt{s_A^3}},   \frac{1}{ \theta_{11}^{(A)}\sqrt{s_A^2s_B}},  \frac{ \sqrt{k}}{  5\theta_9^{(A)}    \max(s_A,s_B)^2 }\right).
\end{align}
Suppose the initial estimates are refined for $T$ iterations using Algorithm \ref{alg:dl}.
Then, the auxiliary coefficient  estimates produced using Algorithm \ref{alg:rl} satisfy  $supp\left(\hat{\textbf{X}}_{A_T}^{(i)}  \right) = supp\left(   \textbf{X}_{A}^{(i)}\right) \;\forall i\in [n]$ and  $supp\left(\hat{\textbf{X}}_{B_T}^{(i)}  \right) = supp\left(   \textbf{X}_{B}^{(i)}\right) \;\forall i\in [m]$, together with 
\begin{align}
\label{theo:coef_a_full}
    \epsilon_{A_T}^X      \leq \; & 8.5  (\eta\sqrt{s_A})^{T-1} \epsilon_{B_0}^D   M^{(A)}_{\max}   u_s\sigma_B   \sqrt{\frac{(1+ \Delta )   s_B }{nk }}, \\
\label{theo:coef_b_full}
      \epsilon_{B_T}^X    \leq \; &  8.5  (\eta\sqrt{s_B})^{T-1} \epsilon_{A_0}^D  M^{(B)}_{\max}   u_s\sigma_A  \sqrt{\frac{(1+ \Delta )  s_A }{mk }}.
\end{align}
Letting $\delta = \frac{ 0.1u_s^2 \sigma_A^2\sigma_B^2 \Delta s_As_B }{k^2 }$, the bound in Eq. (\ref{theo:coef_a_full})  holds with probability at least 
\begin{align}
\label{eq:main_pa}
\nonumber
    p_A = \;& 1 -  11 ke^{-\frac{C_B\Delta^2  \sigma_B^2m}{k  {M_{\max}^{(B)}}^2}} - 2ke^{-\frac{ \hat{C}_Al_s^2 \sigma_A^2 (1- \Delta ) \Delta_X^2 ns_B   }{u_s^2 {M^{(A)}_{\max}}^2   (1+ \Delta ) ks_A }}  -  ke^{-\frac{ \hat{C}_Al_s^2 \sigma_A^2 (1- \Delta ) \Delta_X^2 ns_B   }{u_s^2 {M^{(A)}_{\max}}^2   (1+ \Delta ) k^2 }} - 5ke^{-\frac{\tilde{C}_An}{ks_A}}  \\
    \; & -  2ke^{-\frac{\tilde{C}_A \left\lfloor\frac{s_An}{2k} \right\rfloor}{ks_A}}  -4ke^{-\frac{  ns_A}{16k}} - 2e^{- \frac{\frac{1}{2}nk^3\delta^2}{s_A\sigma_A^4 u_s^4\sigma_B^4(1+ \Delta )^2s_B^2 +\frac{1}{3}\left( {M^{(A)}_{\max}}^2k +s_A\sigma_A^2\right)  u_s^2\sigma_B^2(1+ \Delta )  s_B k \delta }},
\end{align}
and the bound  in Eq. (\ref{theo:coef_b_full}) holds with probability at least
\begin{align}
\label{eq:main_pb}
\nonumber
    p_B = \;& 1 -  11 ke^{-\frac{C_A\Delta^2  \sigma_A^2n}{k  {M_{\max}^{(A)}}^2}} - 2ke^{-\frac{ \hat{C}_Bl_s^2 \sigma_B^2 (1- \Delta ) \Delta_X^2 ms_A   }{u_s^2 {M^{(B)}_{\max}}^2   (1+ \Delta ) ks_B }}  -  ke^{-\frac{ \hat{C}_Bl_s^2 \sigma_B^2 (1- \Delta ) \Delta_X^2 ms_A   }{u_s^2 {M^{(B)}_{\max}}^2   (1+ \Delta ) k^2 }} - 5ke^{-\frac{\tilde{C}_Bm}{ks_B}}  \\
    \; & -  2ke^{-\frac{\tilde{C}_B \left\lfloor\frac{s_Bm}{2k} \right\rfloor}{ks_B}}  -4ke^{-\frac{  ms_B}{16k}} - 2e^{- \frac{\frac{1}{2}mk^3\delta^2}{s_B\sigma_B^4 u_s^4\sigma_A^4(1+ \Delta )^2s_A^2 +\frac{1}{3}\left( {M^{(B)}_{\max}}^2k +s_B\sigma_B^2\right)  u_s^2\sigma_A^2(1+ \Delta )  s_A k \delta  }}.
\end{align} 
Finally, the auxiliary  latent relation  estimate produced using Algorithm \ref{alg:rl}  satisfies
\begin{equation}
\label{theo:latent_s}
     \epsilon_S  < 3\max\left((\eta\sqrt{s_B})^T  \epsilon_{A_{0}}^D, (\eta\sqrt{s_A})^T  \epsilon_{B_{0}}^D \right)\sigma_{{\min}_R}^{-1},
 \end{equation}
 with probability at least $p_S=p_Ap_B$.
 \end{theorem}

\paragraph{Proof Sketch:}
We establish recovery guarantees by showing that, under suitable conditions, the estimation errors  of Algorithm  \ref{alg:rl}  decrease linearly to zero for the auxiliary coefficient and latent relation matrices.
The complete proof is provided in Appendix \ref{app:main_proof1}, and  proceeds by deriving recursive bounds  for the auxiliary coefficient and dictionary estimation errors.
Let the auxiliary coefficient error matrices at $t$-th iteration be $\bm\Delta_{A_t} =\textbf{X}_A -  \hat{\textbf{X}}_{A_{t}}$ and $ \bm\Delta_{B_t} = \textbf{X}_B- \hat{\textbf{X}}_{B_{t}}$.
Existing results, restated in Lemma \ref{dic_error},   express the dictionary error $\textmd{dist}\left(\hat{\textbf{D}}_{A_t/B_t}^{(i)}, \textbf{D}_{A/B}^{(i)} \right)$ in terms of several quantities involving the  coefficient estimates, including $\left\|\left(\hat{\textbf{X}} _{A/B}^T \bm\Delta _{A/B}\right)^{(i)}_{\setminus i} \right\|_2 $,    $\left\|\left(\hat{\textbf{X}} _{A/B}\right)_i^T\left(\hat{\textbf{X}}_{A/B}\right)_{\setminus i}\right\|_2$, and     $\|\bm\Delta_{X_{A/B}}\|_{\infty}$, together with $\sigma_{\min}(\textbf{X}_{A/B})$, $\|\textbf{X}_{A/B} \|_2$, and   $\|\textbf{D}_{A/B} \|_2$ (where, for simplicity, we omit the subscripts).
The results  in Section \ref{sec:matrix_res} provide high-probability bounds that facilitate the analysis of these  matrix  quantities.
Algorithm \ref{alg:rl}  alternates between  coefficient estimation and   dictionary update using Algorithm \ref{alg:dl}, for which Lemma \ref{coeff_error}  bounds its coefficient estimation error at each iteration,  yielding an upper bound on $\|\bm\Delta_{X_{A/B}}\|_{\infty}$.
All these bound results are then substituted into the dictionary error bound  of  Lemma \ref{dic_error} to establish a recursive contraction of the dcitionary estimation error.
Combining the coefficient and dictionary error bounds on    $\|\bm\Delta_{X_{A/B}}\|_{\infty}$ and $\textmd{dist}\left(\hat{\textbf{D}}_{A_t/B_t}^{(i)}, \textbf{D}_{A/B}^{(i)} \right)$ yields the stated recovery guarantees for the auxiliary dictionary,  coefficient, and latent relation  matrices, corresponding to $\epsilon_{A_t/B_t}^D$, $\epsilon_{A_T/B_T}^X$, and $ \epsilon_S$, respectively.   
The same analysis also identifies the sparsity and initialisation conditions required for linear convergence and quantifies the associated high-probability guarantees.

\paragraph{Remarks:}
Theorem \ref{main_res}  establishes sufficient conditions  for the asymptotic recovery of the auxiliary coefficient matrices $\textbf{X}_A$ and $\textbf{X}_B$, and the auxiliary latent relation matrix $\tilde{\textbf{S}}$, from the observation matrix $\textbf{R}$ (up to sign and permutation), together with explicit convergence rates and high-probability guarantees. 
The  asymptotic recovery of the auxiliary dictionary matrices is established as an intermediate result in the proof.
Under these conditions, the recovery errors decay geometrically with the refinement iterations.
The recovery conditions consist of a sparsity requirement and an initialisation requirement.  
Specifically, Eq. (\ref{main:sparsity_condition})  requires the underlying coefficient vectors to be sufficiently sparse, while
Eqs. (\ref{eq:final_dic_condA}) and (\ref{eq:final_dic_condB}) require sufficiently accurate initial auxiliary dictionary estimates. 
These initilisation conditions becomes less restrictive as the coefficient sparsity increases.
The problem constants  further reveal how the recovery conditions depend on the underlying generative model. 
In particular, all three conditions become less restrictive as  $ \frac{\mu_s}{\sqrt{d}}$ (through its contribution to $\mu_D$) and $\Delta$ decrease.
By Assumption \ref{LI},   smaller values of $\frac{\mu_s}{\sqrt{d}}$ correspond  to   weaker correlations among the latent components.  
Meanwhile, decreasing $\Delta$ relaxes the recovery conditions and tightens the recovery error bounds, but lowers the associated success probabilities, thereby revealing a trade-off between stronger recoverability and  stronger probabilistic guarantees.
%
%
More importantly, the error bounds in Eqs. (\ref{theo:coef_a_full}),   (\ref{theo:coef_b_full}), and   (\ref{theo:latent_s}) decrease geometrically with contraction factors $\eta \sqrt{s_A}$ or $\eta \sqrt{s_B}$, establishing that the convergence rate depends explicitly on the coefficient sparsity.
Increasing the coefficient sparsity  accelerates convergence and enlarges the admissible range of $\eta$.
The later in turn relaxes the recovery conditions. 
Consequently, this theorem provides a rigorous characterisation of the fundamental role of coefficient sparsity in matrix tri-factorisation.
In addition, the recovery error bounds in Eqs. (\ref{theo:coef_a_full}) and (\ref{theo:coef_b_full}), together with the success probabilities in Eqs. (\ref{eq:main_pa}) and (\ref{eq:main_pb}), establish sample complexity results that characterise how the recovery performance improves as more observations are collected, i.e., as $n$ and $m$ increase.

\subsubsection{Empirical Illustration}
We empirically validate the recovery guarantees established in Theorem \ref{main_res} through a series of experiments, in which Algorithm \ref{alg:rl} is used to estimate  $\textbf{X}_A$, $\textbf{X}_B$ and $\tilde{\textbf{S}}$ from  the observation  $\textbf{R}$. 
Specifically, we perform three sets of experiments to investigate how the following factors affect  the recovery errors: (1) the coefficient sparsity levels ($s_A$ and $s_B$), (2) the  initial errors of the auxiliary dictionary estimates ($\epsilon_{A_{0}}^D$ and $\epsilon_{B_{0}}^D$), and (3) the incoherence property  of the latent relation matrix  ($\frac{\mu_s}{\sqrt{d}}$).

To implement Algorithm \ref{alg:rl}, we employ the   FISTA algorithm \citep{Beck09} to solve the  Lasso  problem at Step \ref{alg2:step:Lasso} of Algorithm \ref{alg:dl}.
To focus on observing the effect of sparsity,  FISTA  is provided with the true coefficient sparsity level,  evaluating whether the algorithm   correctly recovers both the support  and numerical values of the nonzero auxiliary coefficients.
Before evaluating the recovery performance, we resolve the permutation and sign ambiguities by aligning the estimated dictionary atoms and coefficient vectors with the ground truth.
The optimal permutation is obtained by solving a linear assignment problem.
For example, the rows of $\tilde{\textbf{X}}_{A}$ are aligned with those of $\textbf{X}_{A}$ by solving  $\max_{\bm\Pi\in \{0,1\}^{k\times k}} \text{tr}\left( \tilde{\textbf{X}}_{A}\bm\Pi \textbf{X}_{A}^T \right)  = \text{tr}\left( \tilde{\textbf{X}}_{A}\textbf{X}_{A}^T \bm\Pi \right)$  using a mixed-integer linear programming solver.
The same procedure is applied to  $\tilde{\textbf{X}}_{B}$.
Recovery performance is evaluated after alignment  using the average relative coefficient recovery errors 
\begin{equation}
    E_{\textbf{X}_{A}} =\frac{1}{n}\sum_{i=1}^n \frac{\left\| \tilde{\textbf{X}}_{A}^{(i)}-  \textbf{X}_{A}^{(i)}\right\|_2}{\left\| \textbf{X}_{A}^{(i)} \right\|_2},\; E_{\textbf{X}_{B}} =\frac{1}{m}\sum_{i=1}^m \frac{\left\| \tilde{\textbf{X}}_{B}^{(i)}-  \textbf{X}_{B}^{(i)}\right\|_2}{\left\| \textbf{X}_{B}^{(i)} \right\|_2},
\end{equation}
and the average dictionary recovery error
\begin{equation}
    E_{\textbf{D}_{A/B}} =\frac{1}{k}\sum_{i=1}^k  \left\| \tilde{\textbf{D}}_{A/B}^{(i)}-  \textbf{D}_{A/B}^{(i)}\right\|_2. 
\end{equation}
We  evaluate the   accuracy of the recovered support  for each   coefficient vector using 
\begin{equation}
    A_{supp_{\textbf{X}_{A}}} =\frac{1}{n}\sum_{i=1}^n \frac{\left| supp\left(\textbf{X}_{A}^{(i)}\right) \cap supp\left(\tilde{\textbf{X}}_{A}^{(i)}\right) \right|}{s_A},\; A_{supp_{\textbf{X}_{B}}} =\frac{1}{m}\sum_{i=1}^m \frac{\left| supp\left(\textbf{X}_{B}^{(i)}\right) \cap supp\left(\tilde{\textbf{X}}_{B}^{(i)}\right) \right|}{s_B}.
\end{equation}
Finally,   the recovery error of the auxiliary latent relation matrix is assessed by   
\begin{equation}
    E_{\textbf{S}} = \frac{\left\|\tilde{\textbf{S}} -\textbf{S}\right\|_2}{\left\| \textbf{S}\right\|_2},
\end{equation}
where the estimated latent relation matrix is first permuted using the   permutation matrices obtained from the alignment of $\textbf{X}_{A}$ and $\textbf{X}_{B}$.

 \begin{figure}[t]
    \centering
    \subfigure[ error convergence, $s_A=s_B=3$]{
    \includegraphics[width=0.45\linewidth]{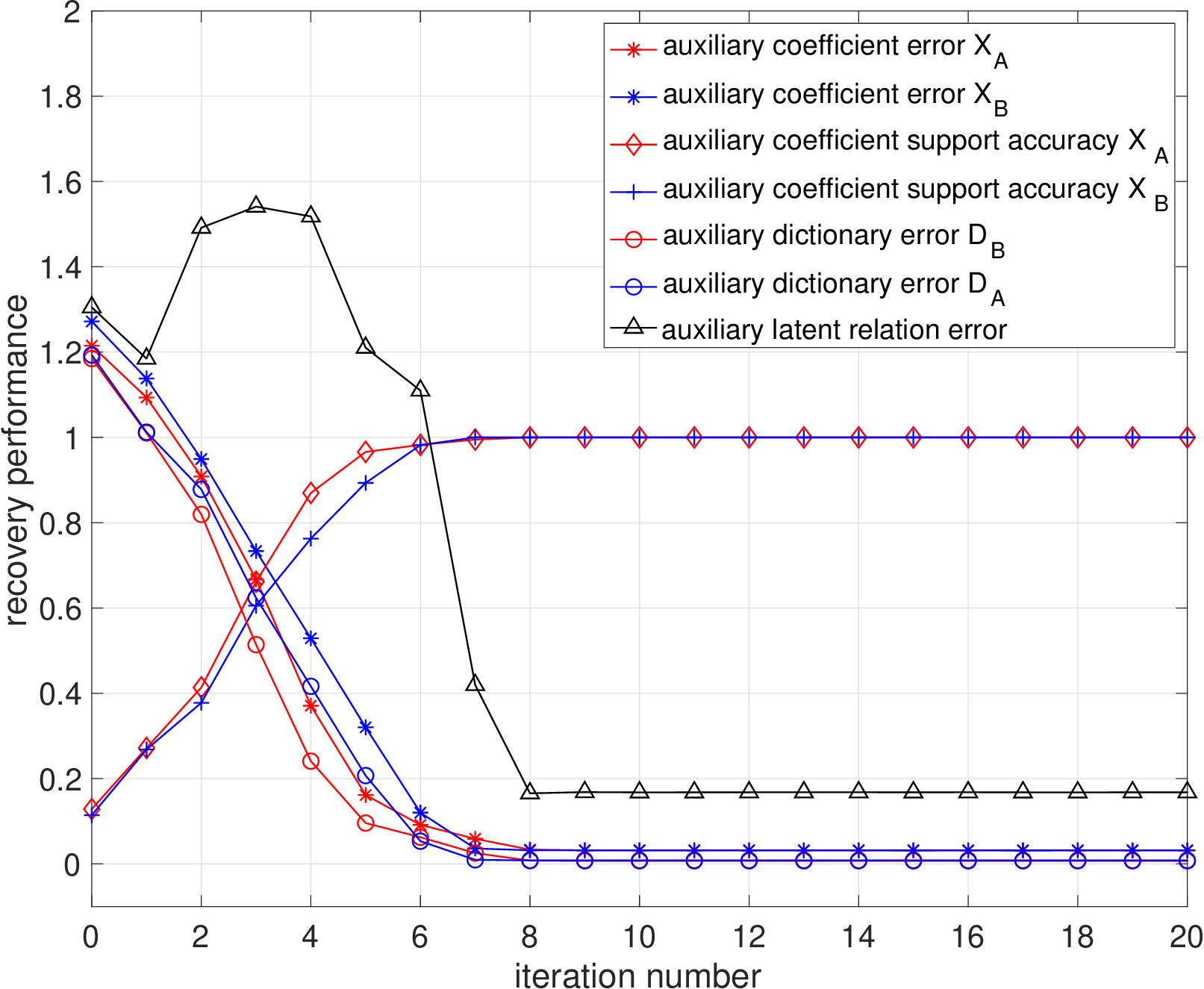}
    \label{fig:test_alg1}    
  }
   \hfill
\subfigure[error change over varying sparsity level ]{
    \includegraphics[width=0.45\linewidth]{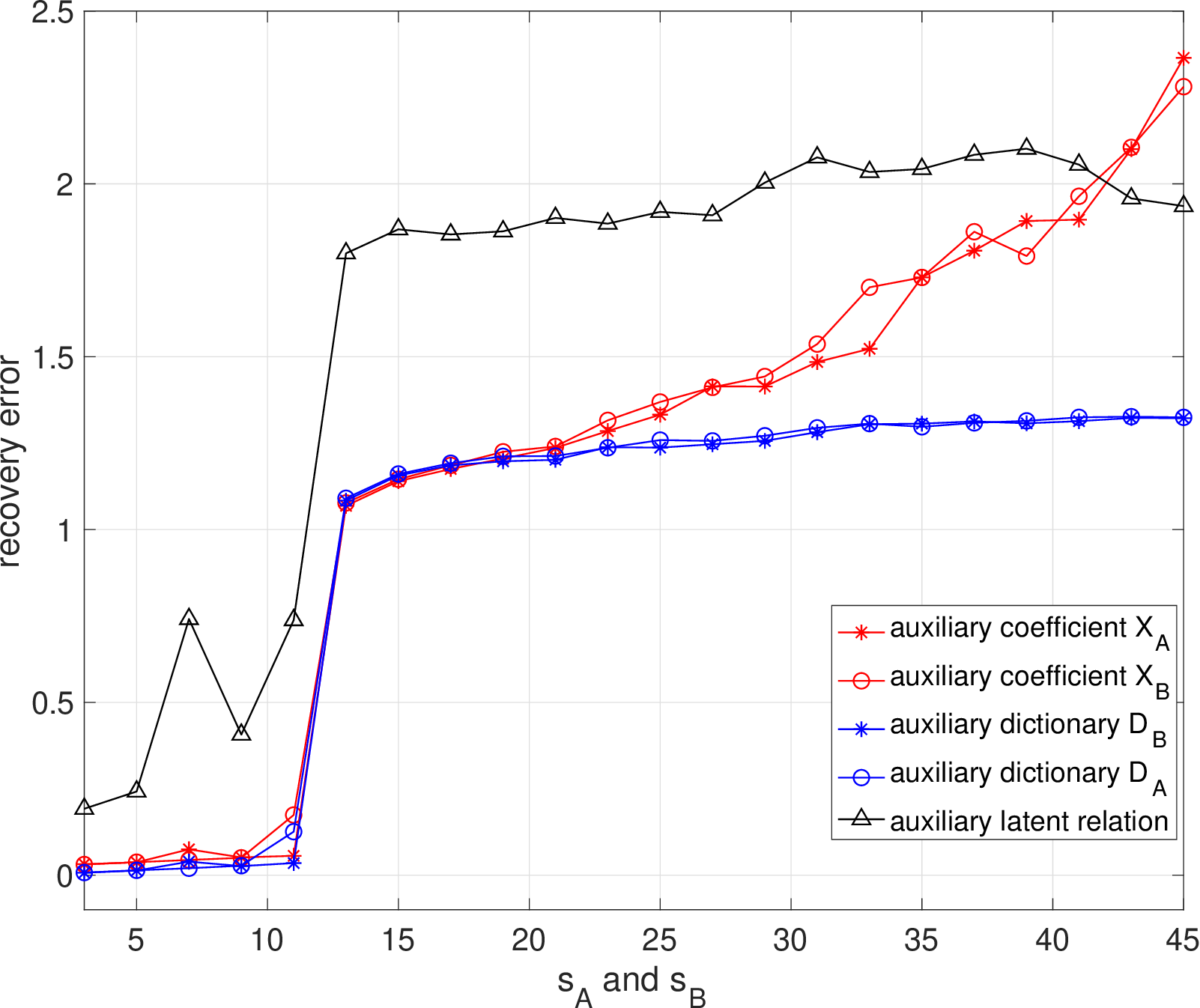}
    \label{fig:test_alg1_sparsity}    
  }
 \caption{  Illustration of (a) how the recovery errors and coefficient support accuracies converge for $s_A=s_B=3$,   (b)  how the final recovery errors and coefficient support accuracies change with varying sparsity levels. }     \label{fig:alg1_A}  
\end{figure}

 \begin{figure}[t]
    \centering
    \subfigure[ error convergence, $s_A=s_B=9$]{
    \includegraphics[width=0.45\linewidth]{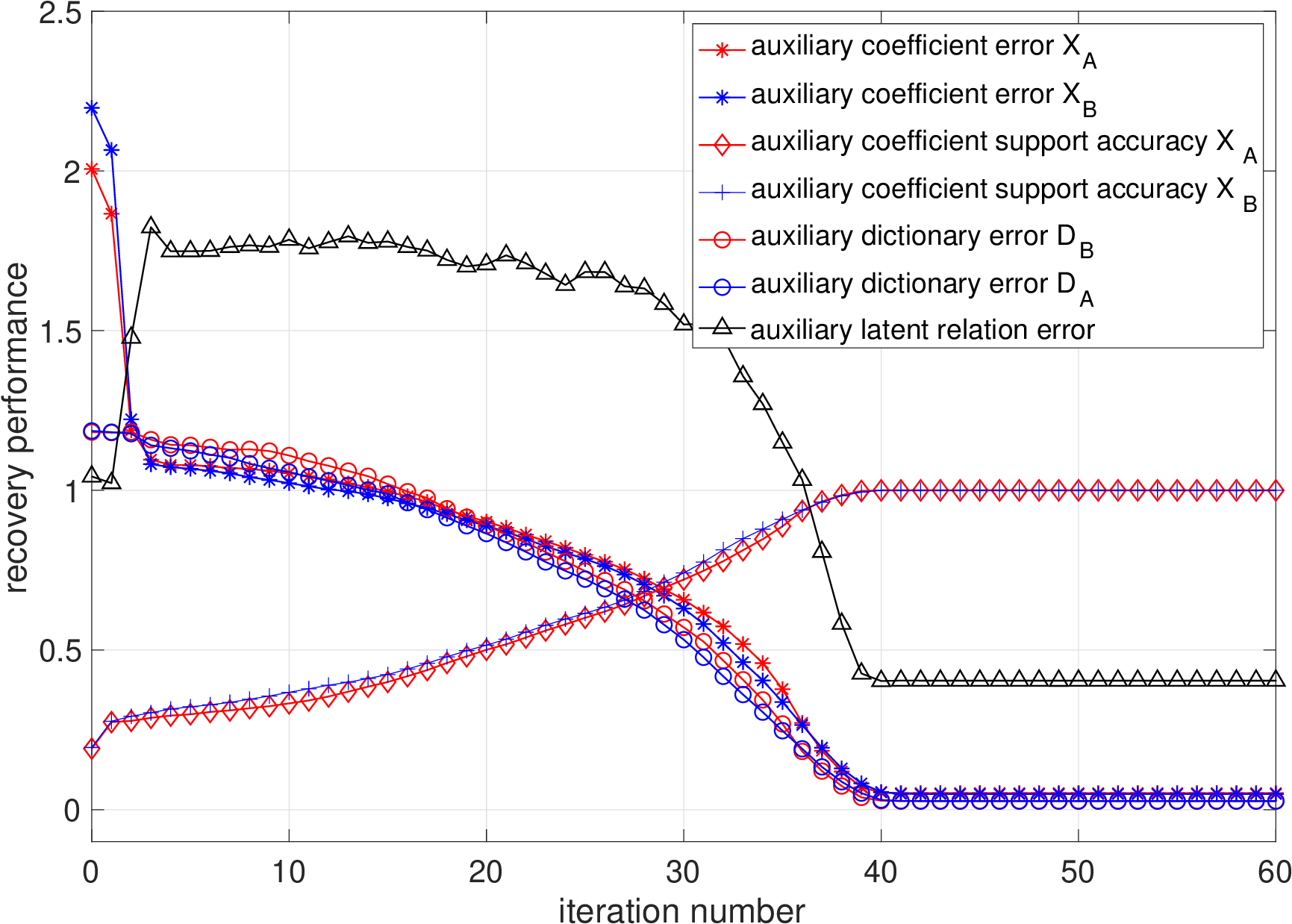}
    \label{fig:test_alg1_s9}    
  }
   \hfill
\subfigure[error convergence, $s_A=s_B=45$ ]{
    \includegraphics[width=0.45\linewidth]{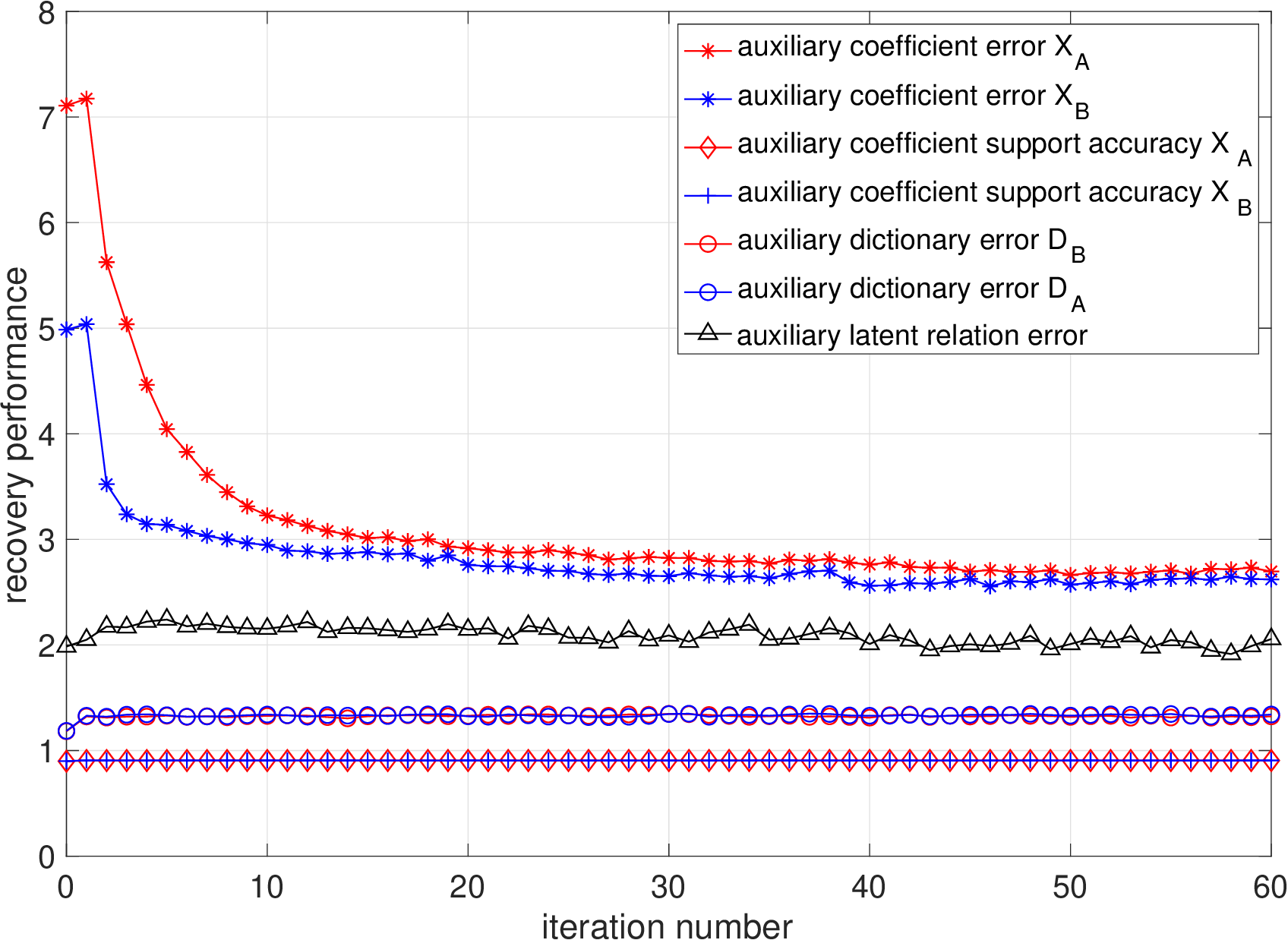}
    \label{fig:test_alg1_s45}    
  }
 \caption{  Illustration of  how the recovery errors and coefficient support accuracies converge for $s_A=s_B=9$ (a) and for   $s_A=s_B=45$  (b). }     \label{fig:alg1_B}  
\end{figure}

\paragraph{\textit{Experiment 1:}}
We investigate how coefficient sparsity affects the recovery performance of Algorithm \ref{alg:rl}.
Following the same  generative model used in the empirical illustration of  Section \ref{sec:prop_coef_aux},  we generate auxiliary coefficient matrices $\textbf{X}_{A}\in\mathbb{R}^{1500\times 50}$  and $\textbf{X}_{B}\in\mathbb{R}^{1200\times 50}$ with $s_A=s_B=3$, auxiliary dictionary matrices $\textbf{D}_{A}\in\mathbb{R}^{50\times 50}$  and $\textbf{D}_{B}\in\mathbb{R}^{50\times 50}$, and the auxiliary latent relation matrix  $\tilde{\textbf{S}} \in\mathbb{R}^{50\times 50}$.
 Algorithm \ref{alg:rl}  is initialised using a random   matrix $\textbf{D}_0 \in\mathbb{R}^{50 \times 50}$, whose entries are sampled independently from the standard normal distribution and whose rows are subsequently normalised to have unit   $l_2$-norm.
 The algorithm is then optimised  for $T=20$ iterations.
The experiment is repeated    5 times.
The averaged recovery errors and coefficient support accuracies are reported in Figure \ref{fig:test_alg1}.
Figure \ref{fig:test_alg1}  shows that the recovery errors of the auxiliary coefficient and dictionary matrices decreases rapidly, approaching zero after approximately 8 iterations, while the coefficient support accuracies approach $100\%$.
In contrast, the recovery error of the auxiliary latent relation matrix converges to approximately $0.17$.

To investigate the effect of coefficient sparsity,  we increase the iteration number to $T=200$ and vary the sparsity level by setting $s_A=s_B$   from 3 to 45  in increments of 2.
For each sparsity level,   the experiment is repeated  5 times, where the observation matrix $\textbf{R}$ is regenerated independently.
Figure \ref{fig:test_alg1_sparsity} reports the average recovery performance over the 5 trails.
The results show that the recovery performance deteriorates substantially once $s_A=s_B$ exceeds approximately 12.
This is consistent with the theoretical prediction that the recovery becomes more challenging as the coefficient vectors become less sparse.
To further illustrate this behavior,  Figure \ref{fig:alg1_B} compares the convergence trajectories for $s_A=s_B=9$ and $s_A=s_B=45$.
Compared with the case of $s_A=s_B=3$ shown in Figure \ref{fig:test_alg1}, Figure \ref{fig:test_alg1_s9} converges much more slowly when $s_A=s_B=9$.
Although the coefficient and dictionary errors eventually stabilise for $s_A=s_B=45$, they remain substantially larger than those observed under higher sparsity, indicating that the algorithm fails to accurately recover the underlying factor matrices.
When $s_A=s_B=45$,  the supports contain nearly all  indices; therefore,  the support recovery accuracy is no longer an informative performance metric.

 \begin{figure}[t]
    \centering
    \subfigure[initial error]{
    \includegraphics[width=0.45\linewidth]{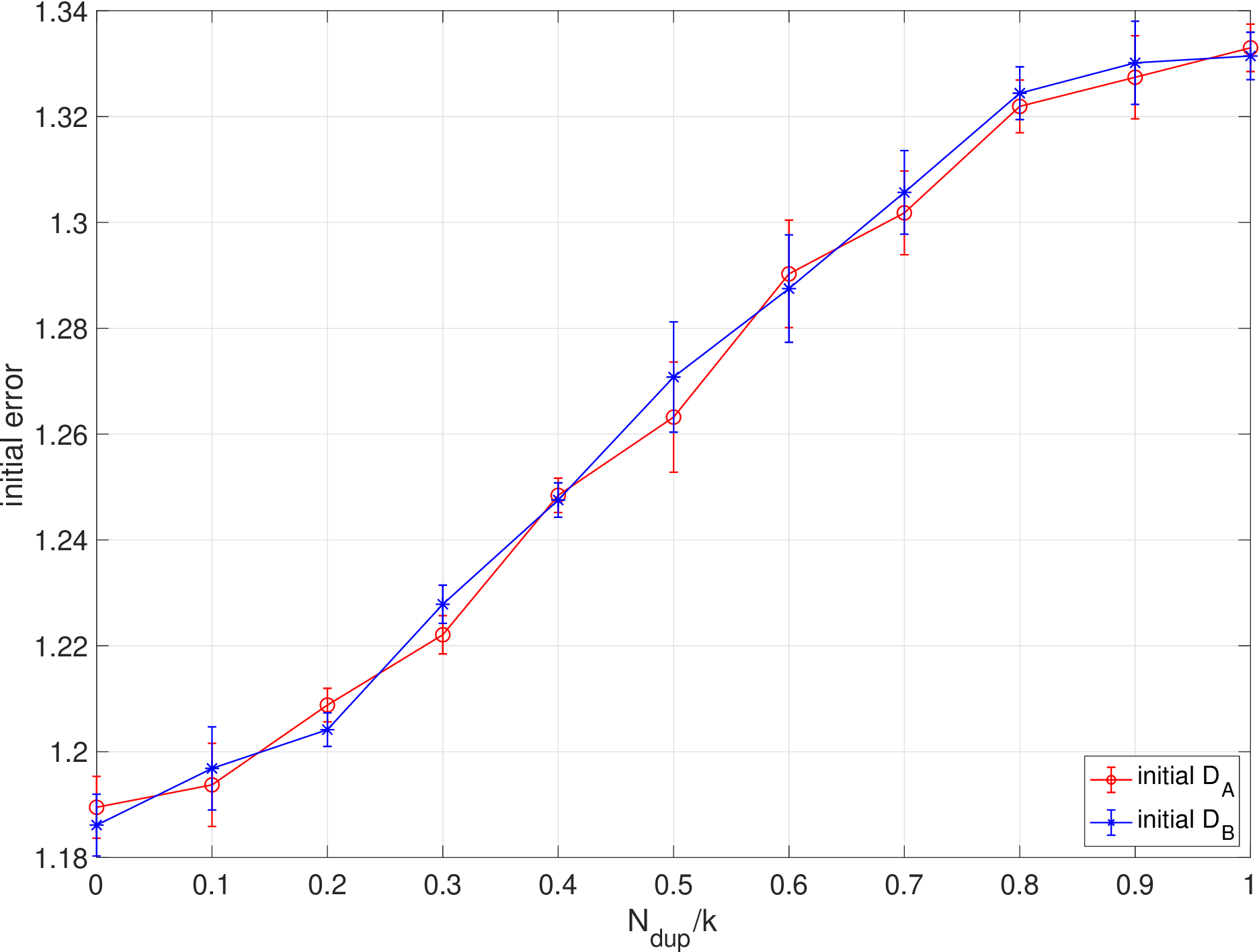}
    \label{fig:ini_quality}    
  }
   \hfill
\subfigure[recovery error]{
    \includegraphics[width=0.45\linewidth]{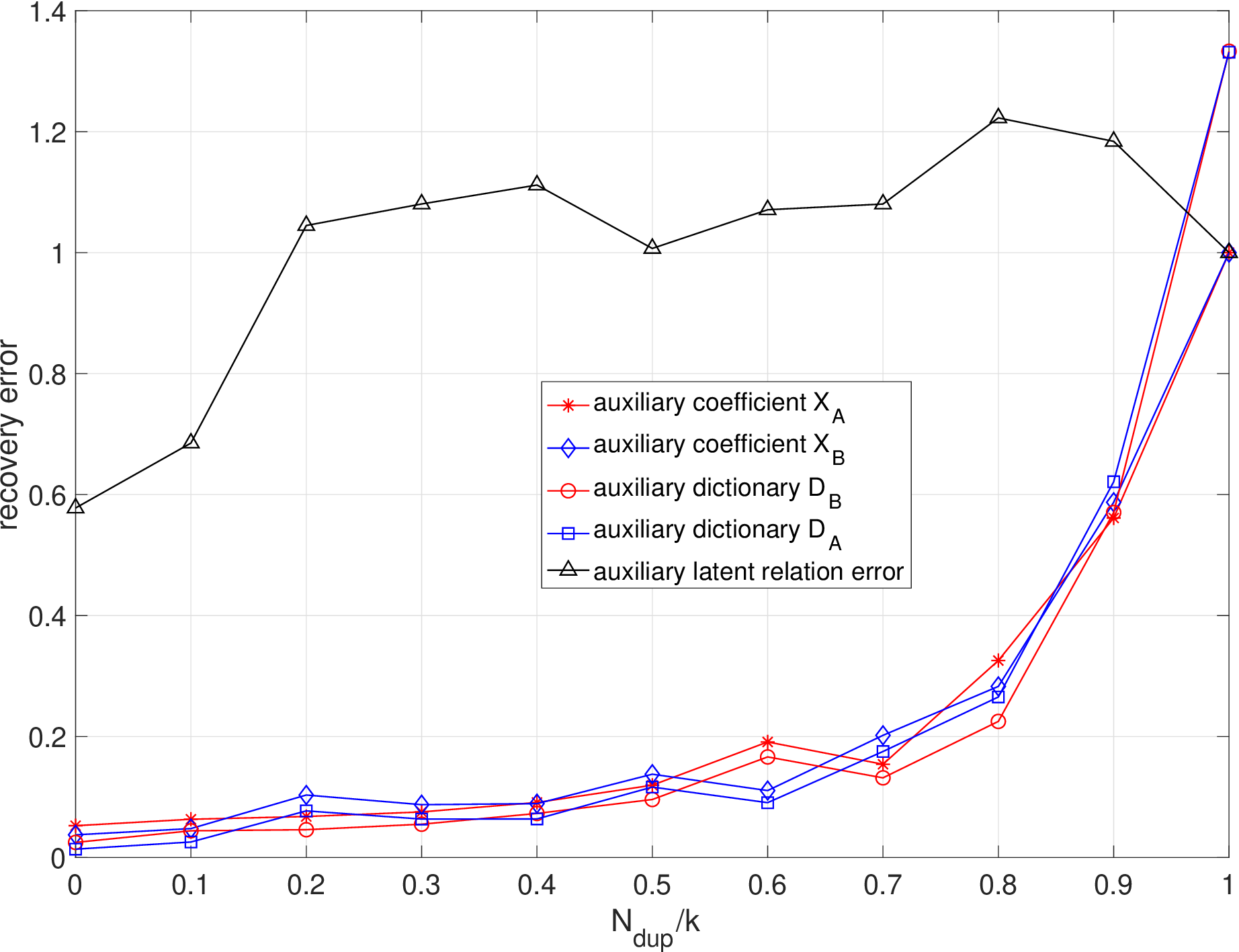}
    \label{fig:alg1_ini}    
  }
 \caption{ Illustration of the effect of the number of duplicated initial atoms, measured by $\frac{N_{\text{dup}}}{k}$, on (a) the initial auxiliary dictionary error and (b) the recovery errors of the estimated auxiliary matrices.}     \label{fig:alg1_C} 
\end{figure}

\begin{figure}[t]
    \centering
       \subfigure[error convergence, $d=40$ ]{
    \includegraphics[width=0.45\linewidth]{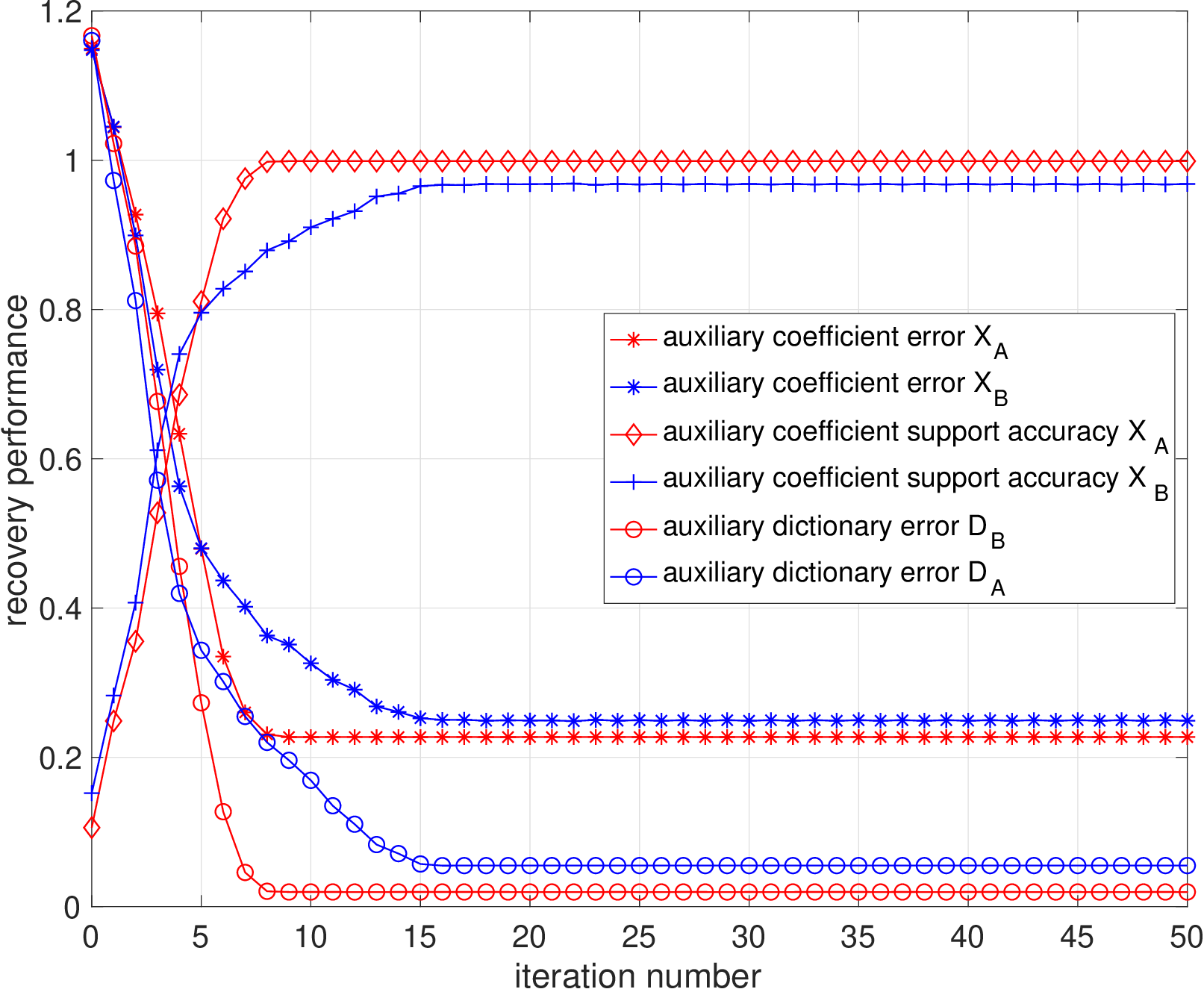}
    \label{fig:test_alg1_d40}    
  } 
 \hfill 
   \subfigure[ error change ]{
    \includegraphics[width=0.45\linewidth]{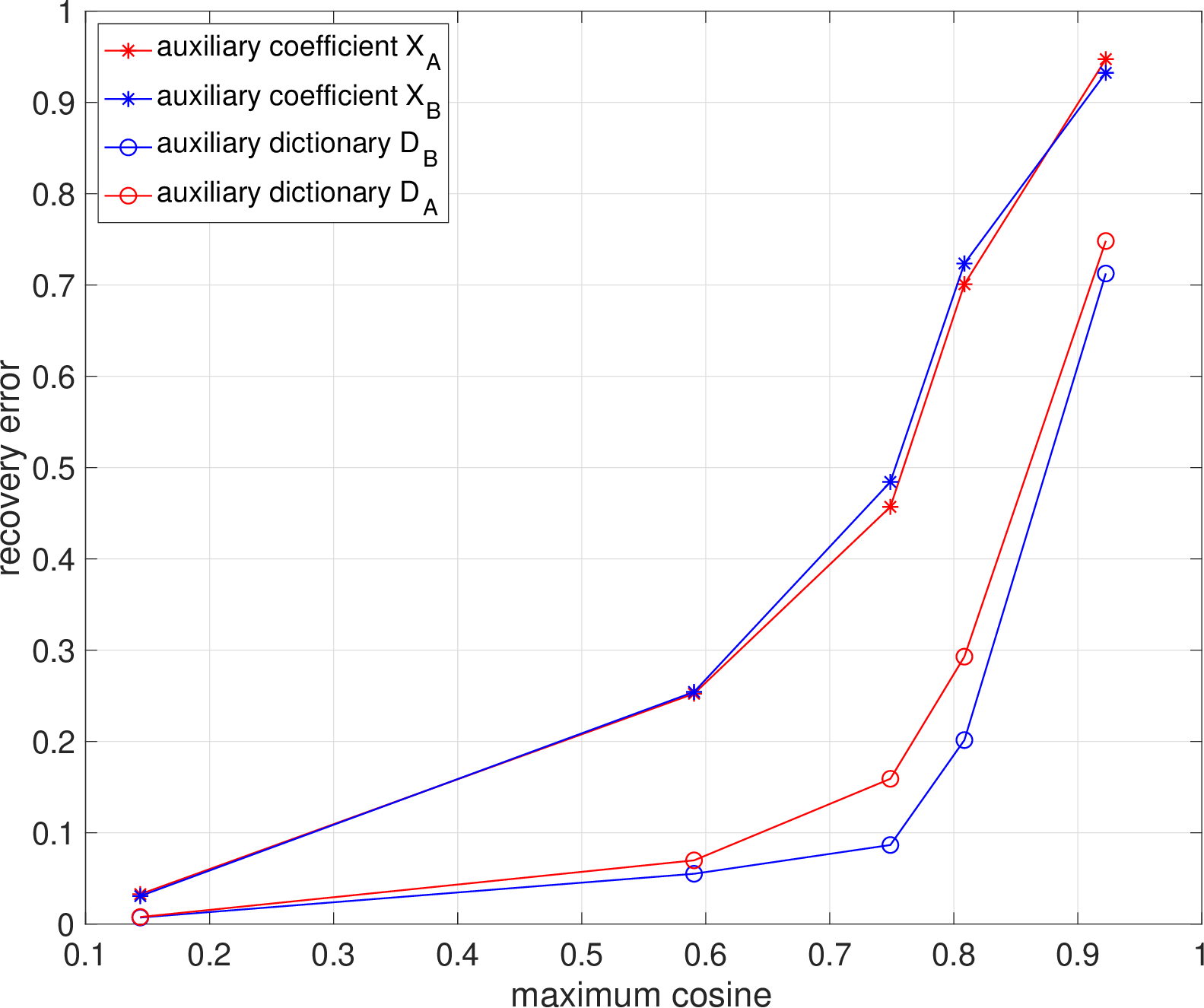}
    \label{fig:test_alg1_incoherence}    
  }
 \caption{  (a)  Illustration of  how the coefficient and dictionary recovery errors  and coefficient support accuracies converge for $d=40$.  (b) Illustration of  how the final coefficient and dictionary recovery errors change with increasing incoherence of $\textbf{S}$. }     \label{fig:alg1_D}  
\end{figure}

\paragraph{\textit{Experiment 2:}}

We investigate how the quality of the initial auxiliary dictionary estimates affects the recovery performance.
Fixing the sparsity level at $s_A=s_B=5$, we vary  the initialisation quality  by introducing  $N_{\text{dup}}$ duplicated atoms into the initial dictionary $\textbf{D}_0$.
Starting from a randomly generated $\textbf{D}_0$ constructed as in the previous experiment, we  randomly select its $N_{\text{dup}}$ rows and replace them with a common vector whose entries are identical and whose $l_2$-norm is normalised to one.
Consequently, increasing   $N_{\text{dup}}$ reduces the quality  of the dictionary initialisation.
We vary the duplication ratio $\frac{N_{\text{dup}}}{k}$ from $0$ to $1$  in increments of $0.1$.
For each duplication ratio, the experiment is repeated  5 times, where both the observation matrix  and the initial dictionaries are regenerated independently.
The average  performance over the 5 trials is reported.

To quantify the initilisation quality, we first align the  initial auxiliary dictionary atoms  with the  ground truth by resolving the permutation and sign ambiguities, and then compute the dictionary recovery error.
 Figure \ref{fig:ini_quality} reports the average initial dictionary error, together with its standard deviation  across the 5 trials.
As expected, increasing the number of duplicated atoms  progressively degrades the initialisation quality.
Figure \ref{fig:alg1_ini} reports the recovery performance obtained from these initialisations using Algorithm \ref{alg:rl}.
The recovery error of the auxiliary latent relation matrix $\tilde{\textbf{S}}$ increases noticeably at $\frac{N_{\text{dup}}}{k} =0.2$, whereas the recovery errors of the auxiliary coefficient and dictionary matrices remain relatively stable until $\frac{N_{\text{dup}}}{k} \approx 0.6$ and they increase sharply after $\frac{N_{\text{dup}}}{k} = 0.8$. 
These results suggest that Algorithm~\ref{alg:rl} is relatively robust to moderate initialisation errors, but its recovery performance deteriorates substantially once the initialisation quality falls below a certain threshold.

\paragraph{\textit{Experiment 3:}}

We investigate how the incoherence property of the latent relation matrix affects the recovery performance.
By definition, the incoherence constant $\frac{\mu_s}{\sqrt{d}}$ of $\textbf{S}$ is approximated by  the maximum off-diagonal absolute cosine similarity among the rows and columns of $\textbf{S}$.
In the previous experiments,  the latent relation matrix was generated as $\textbf{S} =  \textbf{O}  + 0.2\textbf{L}$, following the construction  in  Section \ref{sec:prop_coef_aux}.
This produces a full-rank  matrix  whose rows and columns exhibit relatively weak pairwise correlations,  resulting in a small incoherence constant.
To generate latent relation matrices with larger incoherence constants,  we instead construct rank-deficient matrices as     $\textbf{S} =  \textbf{U}\textbf{V}^T  $, where $\textbf{U}\in \mathbb{R}^{50\times d}$ and $\textbf{V}\in \mathbb{R}^{50\times d}$ contain  the first $d$ left and right singular vectors, respectively, of a random orthonormal matrix.
As the rank $d$ decreases, the rows and columns of $\textbf{S}$ become  increasingly correlated,  leading to larger off-diagonal absolute cosine similarities,  and hence a larger incoherence constant.
Fixing $s_A=s_B=3$, we vary the rank $d$ from 40 to 10 with decrements of  10.
Together with the full-rank construction $\textbf{O}  + 0.2\textbf{L}$, this yields a collection of   generative models    with increasing incoherence constant of $\textbf{S}$.
For each configuration, we repeat the experiment  independently (including  data generation and running Algorithm \ref{alg:rl}) 5 times, and report the average recovery performance.

Figure \ref{fig:test_alg1_d40} illustrates the convergence of the recovery errors for the auxiliary coefficient and latent relation matrices when $d=40$, corresponding to an approximate incoherence constant of 0.59.
Compared with the full-rank case shown in Fig \ref{fig:test_alg1}, whose approximate incoherence constant is 0.14, the recovery errors converge more slowly and stabilise at higher errors.  
 Figure \ref{fig:test_alg1_incoherence} summarises  the final recovery errors of the  auxiliary coefficient and dictionary matrices, as the approximate incoherence constant increases from 0.14 to 0.92.
The recovery errors increase substantially with the incoherence constant, demonstrating that stronger correlations among the latent components make the auxiliary tri-factorisation problem increasingly difficult.
This observation is consistent with Theorem \ref{main_res}, which predicts that smaller values of $\frac{\mu_s}{\sqrt{d}}$ relax the recovery conditions and thereby improve the recoverability of the auxiliary factor matrices. 

\subsection{Spectral Approximation of Auxiliary Dictionaries}

Unlike Algorithm  \ref{alg:dl}, which employs an optimisation procedure to iteratively refine the dictionary estimates, Algorithm \ref{alg:ini} performs a one-step estimation that directly approximates the auxiliary dictionary matrices   from the observation $\textbf{R}$ through a sequence of SVD operations.
In the following theorem, we derive an upper bound on its estimation error and establish the conditions under which this bound holds.
To state the theorem, we define $N_u = \min_{ij} \left|N_{\rho_u}
\left(\textbf{Y}_U^{(i)}, \textbf{Y}_U^{(j)}, Y_U \right)\right|$ and $N_v = \min_{ij} \left|N_{\rho_v}
\left(\textbf{Y}_V^{(i)}, \textbf{Y}_V^{(j)}, Y_V \right)\right|$, which denote the minimum sizes of the   shared $\rho$-neighbour sets  over  all pairs of rows in $\textbf{Y}_U$ and $\textbf{Y}_V$.
Here $Y_U$ and  $Y_V$  denote  the sets of rows in $\textbf{Y}_U$ and $\textbf{Y}_V$, respectively.

\begin{theorem}[Spectral Dictionary Estimation] \label{main_res2}
Suppose Assumptions \ref{SC}-\ref{LI} hold. 
Given $0< \Delta<1$, suppose
 \begin{equation}
\label{s_AB_condition1}
  s_A \leq   \frac{ {M^{(A)}_{\min}}  l_s     }{ {M^{(A)}_{\max}}u_s} \sqrt{\frac{ 1- \Delta     }{2 \mu_D(1+ \Delta )}},\;  s_B \leq   \frac{ {M^{(B)}_{\min}}  l_s     }{ {M^{(B)}_{\max}}u_s} \sqrt{\frac{ 1- \Delta     }{2 \mu_D(1+ \Delta )}}.
\end{equation}
Suppose Algorithm \ref{alg:ini} is applied with $\rho_p =\frac{62}{64}$, where the  correlation thresholds $\rho_u$ and $\rho_v$  for the inputs $\textbf{Y}_u$  and $\textbf{Y}_v$, respectively, satisfy 

\begin{align}
 \label{eq:threshold_rho_u}
& \frac{ {M^{(A)}_{\max}}^2  u_s^2\sigma_B^2 (1+ \Delta ) \mu_D s_Bs_A^2 }{nk } \leq \rho_u  \leq  \frac{ {M^{(A)}_{\min}}^2  l_s^2\sigma_B^2 (1- \Delta ) s_B  }{nk } - \frac{ {M^{(A)}_{\max}}^2  u_s^2\sigma_B^2 (1+ \Delta ) \mu_D s_Bs_A^2 }{nk }, \\
\label{eq:threshold_rho_v}
& \frac{{M^{(B)}_{\max}}^2  u_s^2\sigma_A^2 (1+ \Delta ) \mu_D s_As_B^2 }{mk } \leq \rho_v \leq  \frac{ {M^{(B)}_{\min}}^2  l_s^2\sigma_A^2 (1- \Delta ) s_A  }{mk }-\frac{{M^{(B)}_{\max}}^2  u_s^2\sigma_A^2 (1+ \Delta ) \mu_D s_As_B^2 }{mk }.
\end{align}
Given  $0<\Delta_N <1-\frac{\max\left(s_A^3, s_B^3\right)}{k}$ and $\gamma\leq \frac{1}{64}$,   the auxiliary dictionary estimates produced using Algorithm \ref{alg:ini}   satisfy
\begin{equation}
\label{theo:alg3_res}
     \epsilon_{A}^D    \leq \frac{\theta_{14}^{(B)}}{   1-\frac{s_B^3}{k} -\Delta_N  } \sqrt{\frac{s_B}{k}}, \; \epsilon_{B}^D    \leq \frac{\theta_{14}^{(A)}}{   1-\frac{s_A^3}{k} -\Delta_N  } \sqrt{\frac{s_A}{k}},
\end{equation}
with probabilities at least $p_v$ and $p_u$, respectively, where
\begin{align}
\label{alg3:prob_pv}
    p_v = \; & 1- 2ke^{-\frac{C_A\Delta^2  \sigma_A^2n}{k {M_{\max}^{(A)}}^2}}-de^{- \frac{ C_V \Delta_Y^2\sigma_B^2\left| I_V \right| \left(1+ \mu_D(d-1)\right)}{ {M_{\max}^{(B)}}^2( 1+ s_B\mu_D)}}-e^{-2\Delta_N^2\left|N_{\rho_v}\right|}-2e^{- 2\left|N_{\rho_v}\right|(\left|N_{\rho_v}\right|-1)\gamma^2},\\
\label{alg3:prob_pu}
    p_u = \; & 1- 2ke^{-\frac{C_B\Delta^2  \sigma_B^2m}{k {M_{\max}^{(B)}}^2}}-de^{- \frac{ C_U \Delta_Y^2\sigma_A^2\left| I_U \right| \left(1+ \mu_D(d-1)\right)}{ {M_{\max}^{(A)}}^2( 1+ s_A\mu_D)}}-e^{-2\Delta_N^2\left|N_{\rho_u}\right|}-2e^{- 2\left|N_{\rho_u}\right|(\left|N_{\rho_u}\right|-1)\gamma^2}.
\end{align}
\end{theorem}

\paragraph{Proof Sketch:} 
The complete proof  is provided  in Appendix \ref{app:main_proof_alg3} and proceeds in three main steps. 
First, we establish  Lemma  \ref{threshold_YuYv}, which derives the correlation threshold requirements in Eqs. (\ref{eq:threshold_rho_u}) and   (\ref{eq:threshold_rho_v}), by building on   Lemma \ref{UniqueAtom} that characterises the conditions under which  the unique-intersection threshold $\rho_p=\frac{62}{64}$ is applicable.
Next, we establish Lemma   \ref{singular_YuYv}, which bounds  the recovery error  of the shared atom associated with   each unique-intersection pair.
Finally, the theorem follows  by combining these two   lemmas.

\paragraph{Remarks:} 
Theorem \ref{main_res2} complements Theorem \ref{main_res} by  providing a different algorithmic perspective on the role of sparsity.
It characterises how sparsity affects the spectral approximation of the auxiliary dictionary matrices produced by Algorithm \ref{alg:ini}.
The sparsity condition in Eq. (\ref{s_AB_condition1}) guarantees that the admissible ranges for the correlation thresholds in Eqs. (\ref{eq:threshold_rho_u}) and (\ref{eq:threshold_rho_v}) are non-empty.
This condition  becomes less restrictive when the coefficient amplitudes are more uniform, corresponding to  larger ratios $ \frac{ {M^{(A)}_{\min}} }{ {M^{(A)}_{\max}} }$ and $ \frac{ {M^{(B)}_{\min}} }{ {M^{(B)}_{\max}} }$;  when the row and column norms of the latent relation matrix are more balanced, corresponding to  a larger ratio $\frac{l_s}{u_s}$;   and when  the rows and columns of the  latent relation matrix are less correlated, corresponding to  a smaller value of $\frac{\mu_d}{\sqrt{d}}$ (congtributing to  $\mu_D$).
Furthermore, Eq. (\ref{theo:alg3_res}) explicitly quantifies the effect of sparsity on the approximation error,  showing that the upper bounds on $\epsilon_A^D$ and $\epsilon_B^D$ decrease as $s_A$ and $s_B$ decrease.
In contrast to the sample complexity results in Theorem \ref{main_res}, the error bounds in Eq. (\ref{theo:alg3_res}) do not depend on the observation numbers $n$ and $m$.
Rather, the numbers of observations influence only the success probabilities in  Eqs. (\ref{alg3:prob_pv}) and (\ref{alg3:prob_pu}), through $n$, $m$,  also $\left|N_{\rho_u}\right|$ and $\left|N_{\rho_v}\right|$.

\begin{figure}[t]
    \centering
\subfigure[$\rho_v$-correlation graph,  $s_A=s_B=2$]{
    \includegraphics[width=0.42 \linewidth]{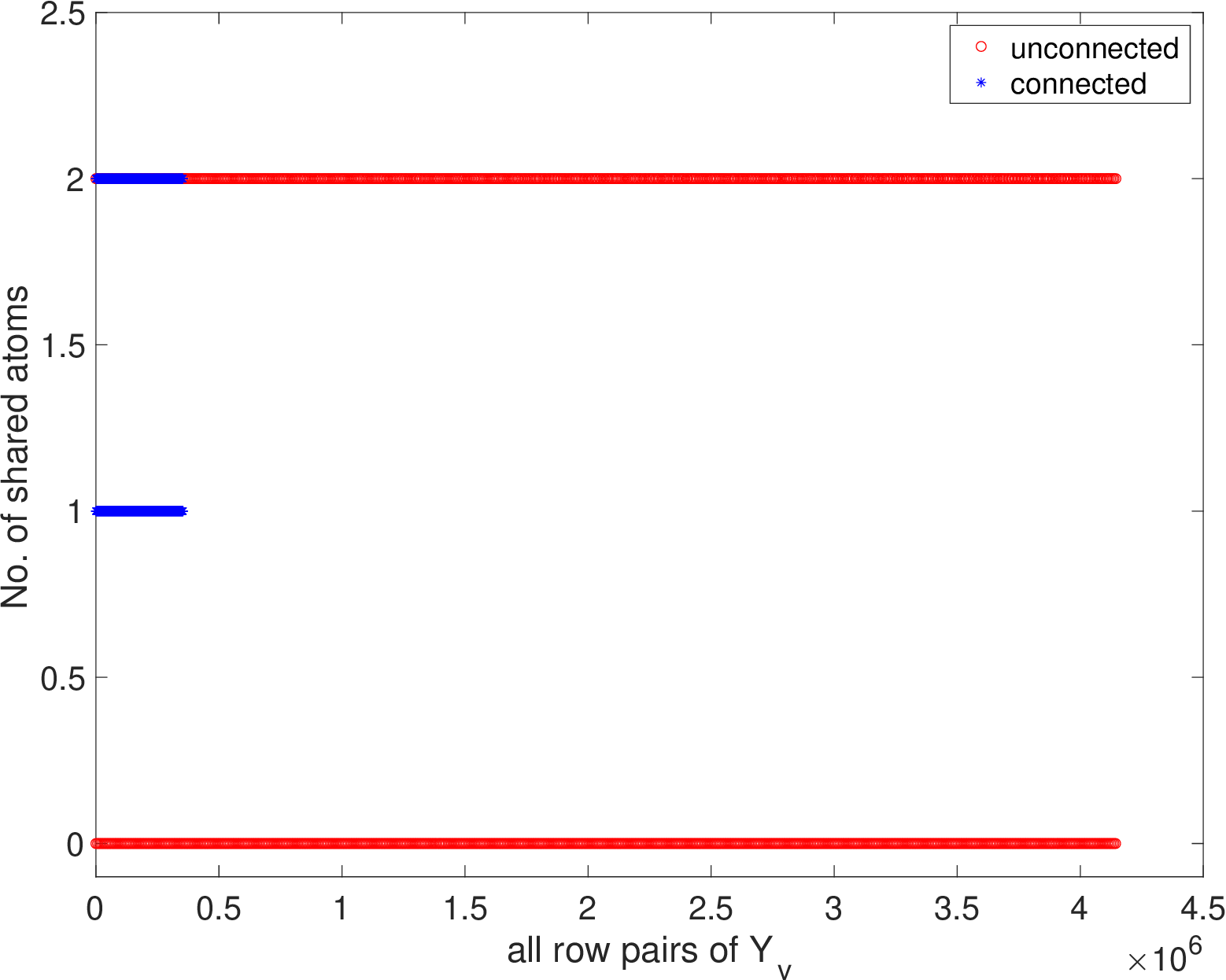}
    \label{fig:alg3_a}    
  }
   \hfill   
\subfigure[ unique-intersection score,  $s_A=s_B=2$  ]{
    \includegraphics[width=0.49\linewidth]{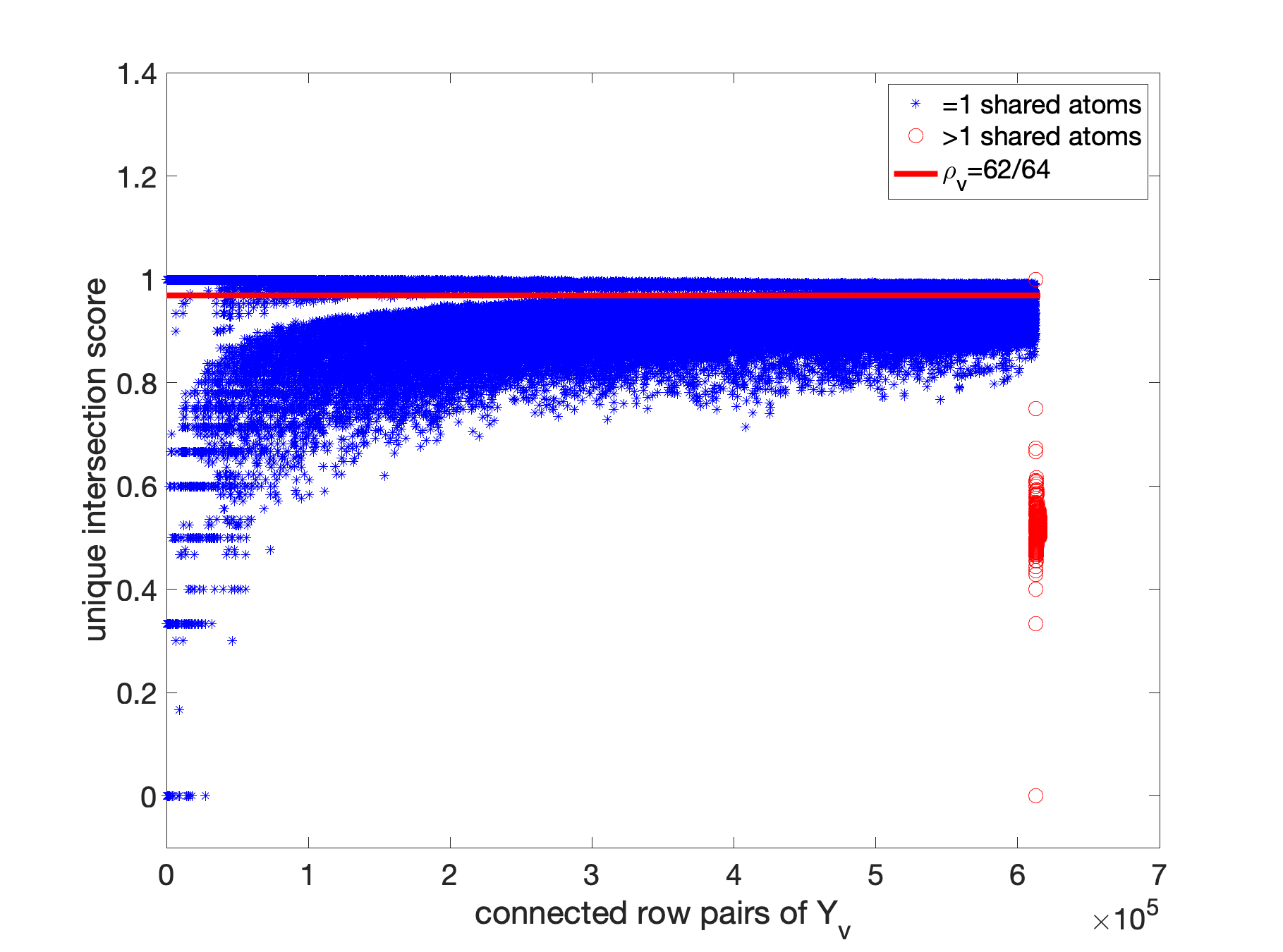}
    \label{fig:alg3_b}    
  }
    \subfigure[$\rho_v$-correlation graph,  $s_A=s_B=3$]{
    \includegraphics[width=0.42 \linewidth]{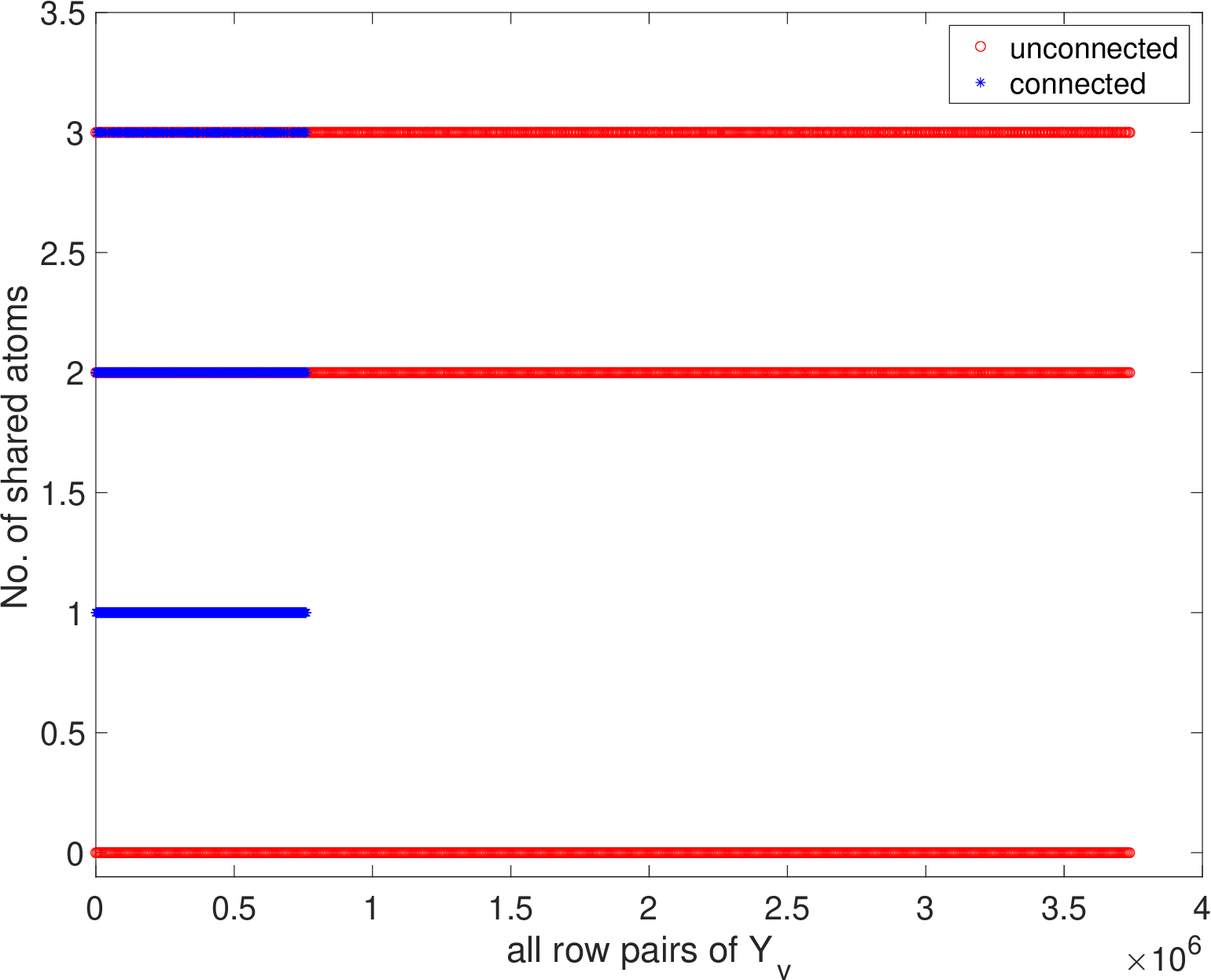}
    \label{fig:alg3_a_s3}    
  }
   \hfill
\subfigure[ unique-intersection score,  $s_A=s_B=3$  ]{
    \includegraphics[width=0.49\linewidth]{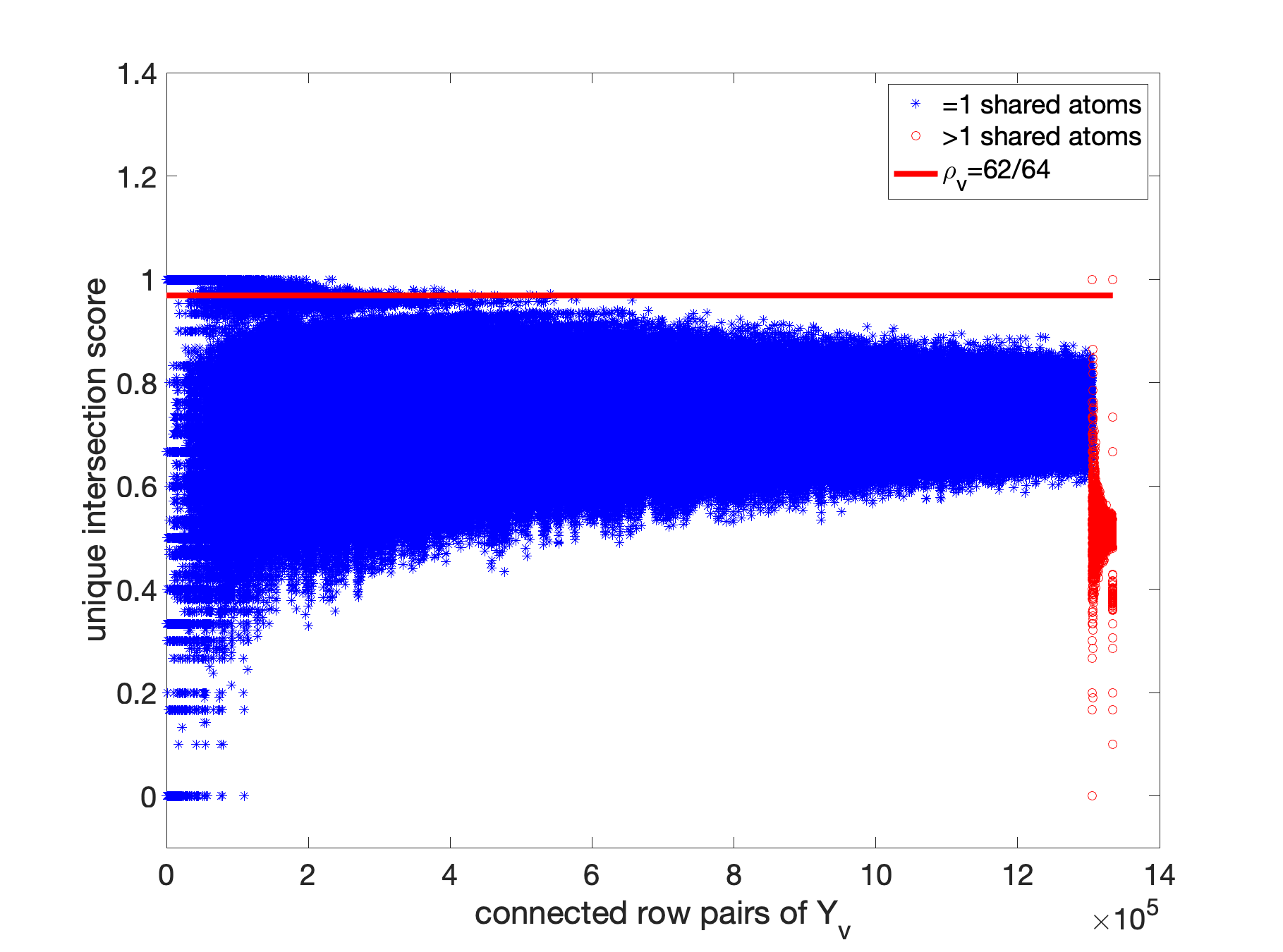}
    \label{fig:alg3_b_s3}    
  }
 \caption{(a) and (c) show the number  of shared atoms for   connected and unconnected row pairs in the $\rho_v$-correlation graph, for $s_A=s_B=2$ and $s_A=s_B=3$, respectively. (b) and (d) show the unique-intersection scores  $\frac{N_e}{\frac{1}{2}N(N-1)}$  computed in Step \ref{alg3:unique_score} of Algorithm \ref{alg:unique} for all connected row pairs, under $s_A=s_B=2$ and $s_A=s_B=3$, respectively, together with the unique-intersection threshold $\rho_p=\frac{62}{64}$.}     \label{fig:alg3_plot} 
\end{figure}

\begin{figure}[t]
    \centering
  \includegraphics[width=0.6 \linewidth]{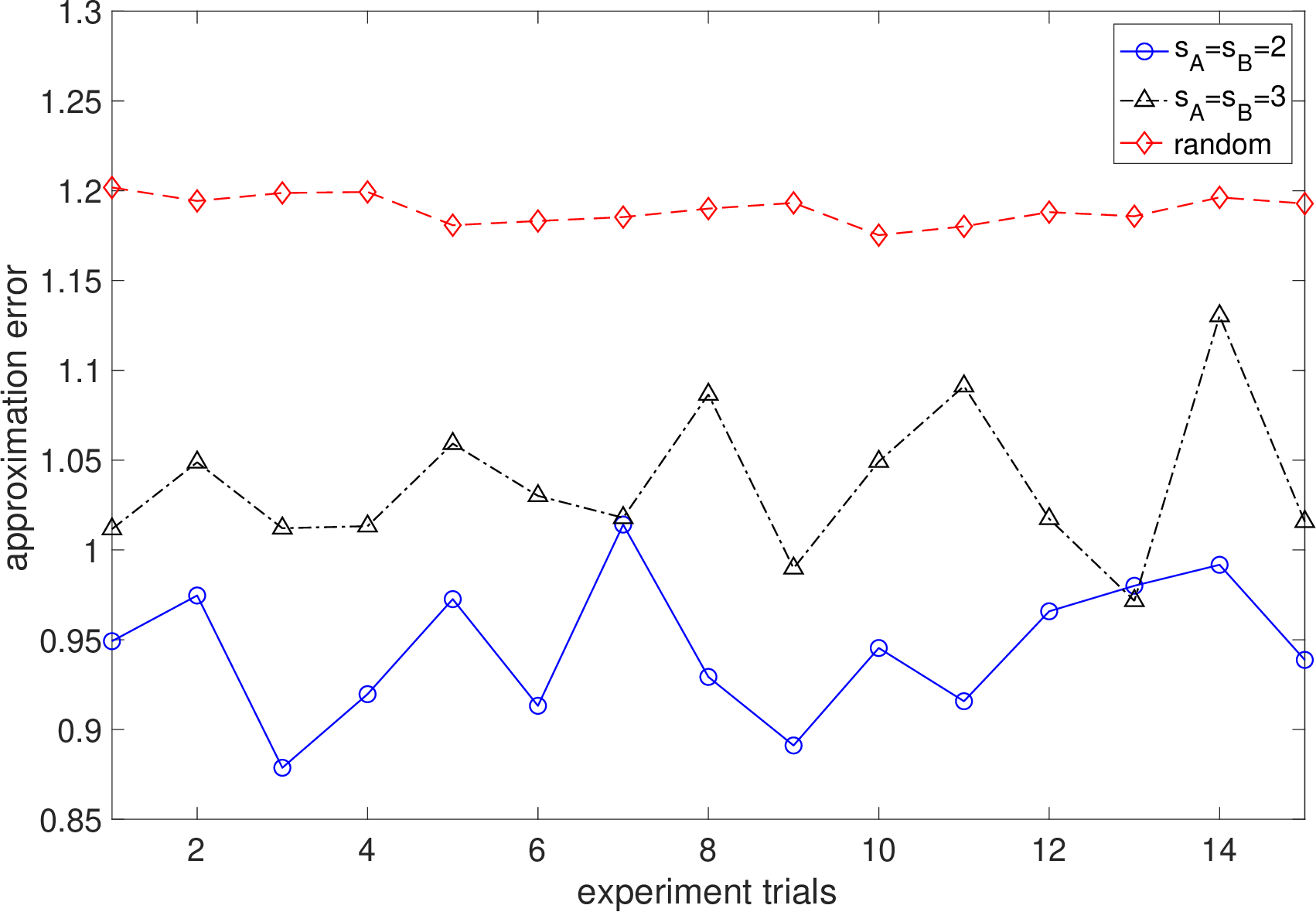}    
 \caption{ Comparison of recovery errors of dictionary estimates produced by Algorithm \ref{alg:ini} with  $s_A=s_B=2$ and  $s_A=s_B=3$, against the error of  random estimate. }         \label{fig:alg3_perf} 
\end{figure}

\paragraph{Empirical Illustration:}
We use the same data generation process as in the previous experiments with $n=3000$, $m=4000$, $k=50$, ${m_{\max}^{(A)}}={m_{\max}^{(B)}}=1$, ${M_{\max}^{(A)}}= {M_{\max}^{(B)}}=1.1$, and $\textbf{S} =  \textbf{O}  + 10^{-3}\textbf{L}$.
This configuration enables the sparsity condition in Eq. (\ref{s_AB_condition1}) for both $s_A=s_B=2$ and $s_A=s_B=3$.
Algorithm $\ref{alg:ini}$ is applied to approximate $\textbf{D}_A$ from $\textbf{Y}_v$  with hyperparameters $\rho_p = \frac{62}{64}$, $\xi=0.6$, and  a value of $\rho_v$  chosen to satisfy Eq. (\ref{eq:threshold_rho_v}) and lie close to its upper bound.

We first demonstrate the observations for the sparsity level $s_A=s_B=2$.
Figures \ref{fig:alg3_a} and \ref{fig:alg3_b} illustrate the effectiveness of the adopted  thresholds   $\rho_v$ and $\rho_p$, respectively.
As shown in Figure \ref{fig:alg3_a}, all row pairs that share exactly one atom correspond to connected pairs in the $\rho_v$-correlation graph of $\mathbf{Y}_v$.
Figure \ref{fig:alg3_b} further shows that,  among all the  connected row pairs, those sharing more than one atom  (highlighted by red circles) generally exhibit low unique-intersection scores. 
Consequently, the threshold $\rho_p=\frac{62}{64}$ successfully filters out nearly all such pairs.
We then repeat the  experiment by increasing the number of nonzero coefficients to  $s_A=s_B=3$.
Figure \ref{fig:alg3_a_s3}  confirms that  row pairs sharing exactly one atom again all appear as connected pairs.
Figure \ref{fig:alg3_b_s3} shows that the threshold $\rho_p=\frac{62}{64}$ filters out  most  pairs sharing more than one atom, although it retains more such pairs than in the case  $s_A=s_B=2$.
We compare the estimates of $\mathbf{D}_A$ produced by Algorithm \ref{alg:ini} for $s_A=s_B=2$ and  $s_A=s_B=3$ with a random estimate  obtained by sampling each entry independently from a  standard normal distribution and normalising each row to uni norm.
Figure \ref{fig:alg3_perf} reports the recovery errors over 15 independent  trials. 
The algorithm consistently produces dictionary estimates that are closer to the true dictionary than a random guess. 
Overall, increasing the number of nonzero coefficients from 2 to 3 results in higher approximation errors.

\subsection{Auxiliary Model Discrepancy}

As discussed in Section \ref{sec:AuxTri}, the transformation from  $\textbf{A}$, $\textbf{B}$, and $ \textbf{S} $ to  $\textbf{X}_A$, $\textbf{X}_B$, and $\tilde{\textbf{S}} $  introduces only bounded changes  to the norms   and pairwise angles of the coefficient vectors  and to the conditioning of the latent relation matrix.
These  structural changes are characterised in Theorem \ref{main_res3}.
Specifically, to quantify the corresponding norm changes, we define the following coefficient-norm shift factors:
\begin{equation}
    \delta_{l_2}^{(A)} =\frac{\max_{i=1}^n \frac{\left\|\textbf{X}_A^{(i)}\right\|_2 }{\left\|\textbf{A}^{(i)}\right\|_2 } }{\min_{i=1}^n \frac{\left\|\textbf{X}_A^{(i)}\right\|_2 }{\left\|\textbf{A}^{(i)}\right\|_2 } }, \; \delta_{l_2}^{(B)} =\frac{\max_{i=1}^n \frac{\left\|\textbf{X}_B^{(i)}\right\|_2 }{\left\|\textbf{B}^{(i)}\right\|_2 } }{\min_{i=1}^n \frac{\left\|\textbf{X}_B^{(i)}\right\|_2 }{\left\|\textbf{B}^{(i)}\right\|_2 } }.
\end{equation}
The corresponding angle changes are quantified by comparing the cosine similarities between coefficient vectors before and after the transformation, through the following coefficient-angle shift factors:
\begin{align}
     \delta_{\text{cos}}^{(A)}=\; & \left|  \text{cos}\left(\textbf{X}_A^{(i)}, \textbf{X}_A^{(j)}\right) -\text{cos}\left(\textbf{A}^{(i)}, \textbf{A}^{(j)}\right) \right|,\\
     \delta_{\text{cos}}^{(B)}=\; & \left|  \text{cos}\left(\textbf{X}_B^{(i)}, \textbf{X}_B^{(j)}\right) -\text{cos}\left(\textbf{B}^{(i)}, \textbf{B}^{(j)}\right) \right|.
\end{align}
The conditioning change is quantified by the following condition-number shift factor of the latent relation matrix:
\begin{equation}
    \delta_S= \sqrt{\frac{\kappa\left(\tilde{\textbf{S}}\right)}{\kappa\left(\textbf{S}\right)} }.
\end{equation}
Shift factors $\delta_{l_2}^{(A)}=\delta_{l_2}^{(B)}=\delta_S=1$ and $\delta_{\mathrm{cos}}^{(A)}=\delta_{\mathrm{cos}}^{(B)}=0$ indicate that the original and auxiliary tri-factorisations have identical coefficient norms, pairwise coefficient angles, and latent-relation  condition numbers.
Using Lemma \ref{RowLengthF}, we derive bounds on these  shift factors.
The    proof  of Theorem \ref{main_res3} is  provided in Appendix \ref{app:model_gap}.

\begin{theorem}[Structural Shift Characterisation] \label{main_res3}
Let $\mathbf{R}=\mathbf{A}\mathbf{S}\mathbf{B}^T$ be a tri-factorisation and consider its corresponding auxiliary tri-factorisation $\mathbf{R}=\mathbf{X}_A\tilde{\mathbf{S}}\mathbf{X}_B^T$ constructed according to Definition~\ref{gen_aux}.
Then the original and auxiliary coefficient matrices satisfy
\begin{equation}
    supp\left(\textbf{A}^{(i)}\right) =supp\left(\textbf{X}_A^{(i)}\right), \; supp\left(\textbf{B}^{(i)}\right) =supp\left(\textbf{X}_B^{(i)}\right).
\end{equation}
Suppose Assumptions \ref{SC}-\ref{LI} hold.   
Given $0<\Delta<1$,  the coefficient-norm and condition-number shift factors satisfy
\begin{equation}
\label{eq:shift_bound}
  \delta_{l_2}^{(A)}, \delta_{l_2}^{(B)}, \delta_{S} \in \left[\frac{l_s}{u_s}\sqrt{\frac{1-\Delta}{1+\Delta}},   \frac{u_s}{l_s}\sqrt{\frac{1+\Delta}{1-\Delta}}\right],
\end{equation}
with probabilities  at least $1-ke^{-\frac{C_B\Delta^2  \sigma_B^2m}{k  {M_{\max}^{(B)}}^2}}$, $1-ke^{-\frac{C_A\Delta^2  \sigma_A^2n}{k  {M_{\max}^{(A)}}^2}}$, and $1-ke^{-\frac{C_A\Delta^2  \sigma_A^2n}{k  {M_{\max}^{(A)}}^2}}-ke^{-\frac{C_B\Delta^2  \sigma_B^2m}{k  {M_{\max}^{(B)}}^2}}$,  respectively.
Expressed in terms of the the quantities $D_A$ and $D_B$ defined in Table \ref{tab:constant-definitions}, the coefficient-angle shift factors satisfy
\begin{equation}
    \label{eq:cosine_shift_bound}
     \delta_{\text{cos}}^{(A)}   \leq \left\{   \begin{array}{ll}
         \frac{s_AD_A}{k},  &  \text{if } \text{cos}\left(\textbf{A}^{(i)}, \textbf{A}^{(j)}\right) \geq 0, \\
        \frac{2s_AD_A}{k}, & \text{otherwise},
     \end{array}  \right. \;  \delta_{\text{cos}}^{(B)}   \leq \left\{   \begin{array}{ll}
         \frac{s_BD_B}{k},  &  \text{if } \text{cos}\left(\textbf{B}^{(i)}, \textbf{B}^{(j)}\right) \geq 0, \\
        \frac{2s_BD_B}{k}, & \text{otherwise},
     \end{array}  \right.
\end{equation}
 with probabilities at least $1-ke^{-\frac{C_B\Delta^2  \sigma_B^2m}{k  {M_{\max}^{(B)}}^2}}$ and $1-ke^{-\frac{C_A\Delta^2  \sigma_A^2n}{k  {M_{\max}^{(A)}}^2}}$, respectively.
\end{theorem}

\paragraph{Remarks:}
The structural shift arises from the transformation between the original and auxiliary factor matrices via the scaling factors  $\left\|\textbf{F}_A^{(i)} \right\|_2$ and $\left\|\textbf{F}_B^{(i)} \right\|_2$.
Consequently, the structural shift bounds  rely on the upper and lower bounds on $\left\|\textbf{F}_A^{(i)} \right\|_2$ and $\left\|\textbf{F}_B^{(i)} \right\|_2$  established  in Lemma \ref{RowLengthF}.
The derived bounds  depend  on the ratios  $\frac{u_s}{l_s}$,   $\frac{M_{\max}^{(A)}}{M_{\min}^{(A)}}$, and $\frac{M_{\max}^{(B)}}{M_{\min}^{(B)}}$, as well as  the parameter $\Delta$ for controlling the probabilities with which the bounds  hold.
These dependencies enter through the quantities $\frac{l_s}{u_s}\sqrt{\frac{1-\Delta}{1+\Delta}} $, $ \frac{u_s}{l_s}\sqrt{\frac{1+\Delta}{1-\Delta}}$, $D_A$ and $D_B$.
Smaller variations in the row and column norms of the latent relation matrix and in the nonzero coefficient amplitudes  lead to milder structural shifts after the transformation.
There is a trade-off between the tightness of the bounds and their associated success probabilities.
Finally, the upper bounds of the angle shift factors explicitly depend  on the coefficient sparsity ratios $\frac{s_A}{k}$ and $\frac{s_B}{k}$.  
Sparser coefficient vectors therefore yield  upper bounds on the angle shifts that are closer to zero, providing stronger guarantees of pairwise angle preservation.

\paragraph{Empirical Illustration:}

\begin{figure}[t]
    \centering
\subfigure[Theorem \ref{main_res3}, $\delta_{l_2}^{(A)}$ and $\delta_{S} $ ]{
    \includegraphics[width=0.45\linewidth]{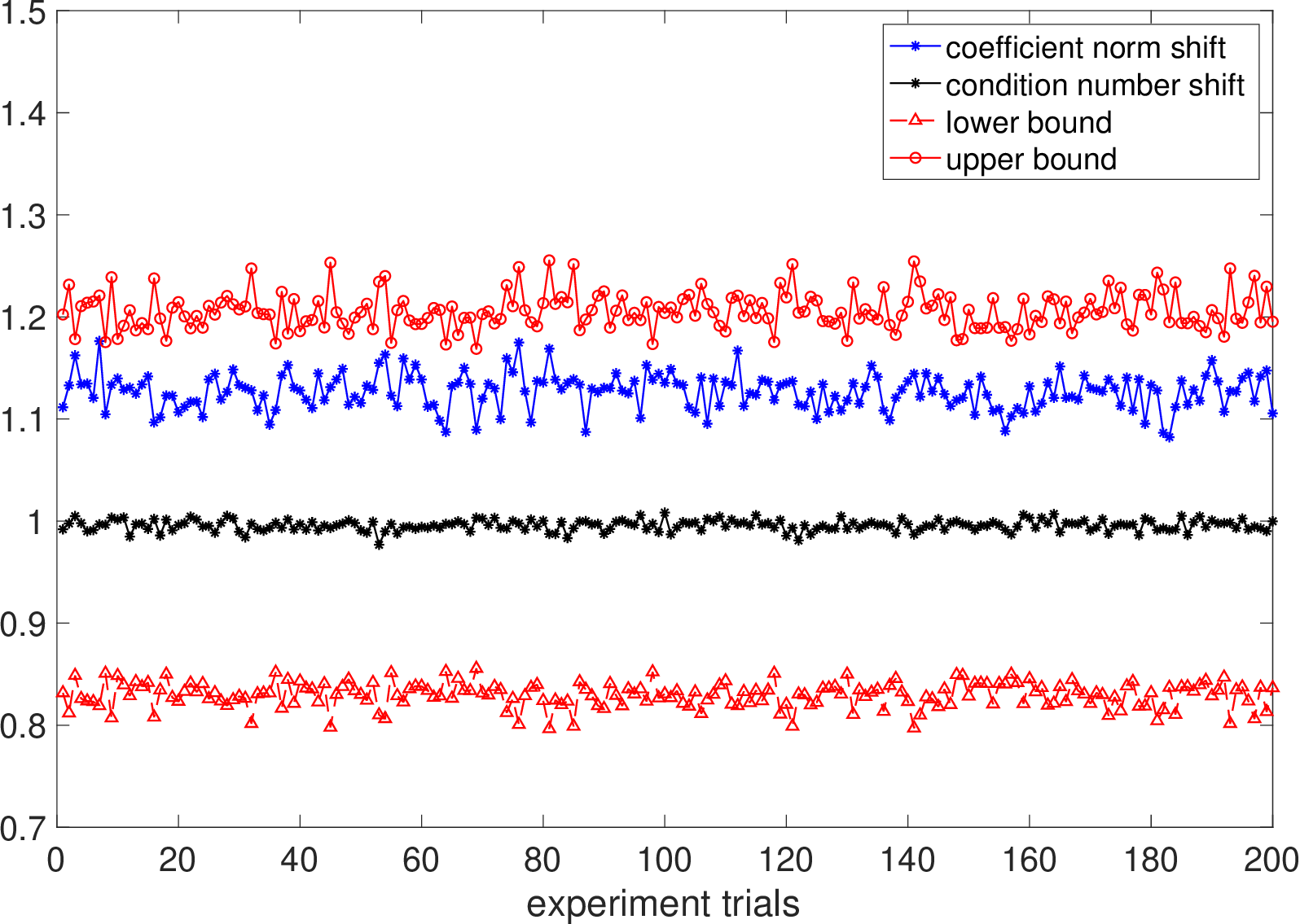}
    \label{fig:theorem22a}    
  }
   \hfill
\subfigure[Theorem \ref{main_res3}, $\delta_{\text{cos}}^{(A)} $ ]{
    \includegraphics[width=0.45\linewidth]{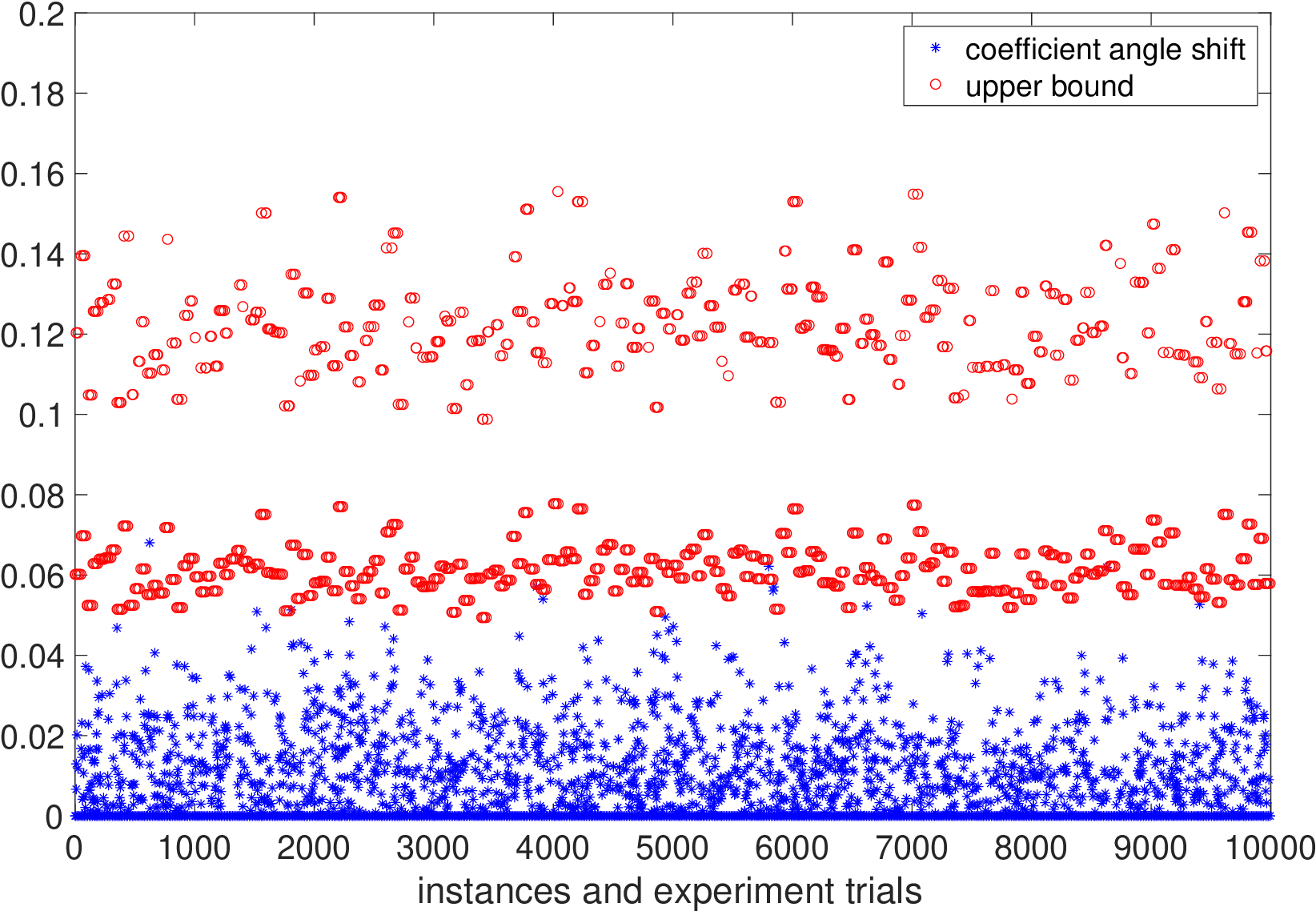}
    \label{fig:theorem22b}    
  }
 \caption{Comparison of observed shift factors and their theoretical bounds over 200 random trials. (a) Coefficient-norm and condition-number shift factors, $\delta_{l_2}^{(A)}$ and $\delta_S$. (b) Coefficient-angle shift factor, $\delta_{\mathrm{cos}}^{(A)}$, based on 50 randomly selected coefficient-vector pairs per trial.}     \label{fig:auxiliary_gap} 
\end{figure}

We empirically validate the structural shift results established in Theorem \ref{main_res3} for the left coefficient matrix ($\textbf{A}\rightarrow\textbf{X}_A $) and the latent relation matrix ($\textbf{S}\rightarrow \tilde{\textbf{S}}$)  using the same experiment setting as  that in Section \ref{sec:prop_coef_aux}.
All theoretical bounds   are computed  using $\Delta = 0.05$. 
We  generate random triples of $\textbf{A}$, $\textbf{B}$, and $\textbf{S}$, construct the corresponding auxiliary matrices $\textbf{X}_A$, $\textbf{X}_B$, and $\tilde{\textbf{S}}$, and observe the resulting structural changes over 200 independent trials.

Figure \ref{fig:theorem22a} compares the observed shift factors $ \delta_{l_2}^{(A)}$ and $ \delta_{S}$  with their theoretical lower and upper bounds  over 200  trials.    
The condition-number shift remains close to 1 with  minor variation ($0.9958\pm 0.0053$), indicating that the transformation from $\textbf{S}$ to $\tilde{\textbf{S}}$ introduces negligible changes to its conditioning.
The coefficient-norm shift exhibits slightly larger variation, staying close to 1.1   ($1.1269\pm 0.0175$).
The theoretical  bounds are satisfied in all  trials.
Figure~\ref{fig:theorem22b} compares  the angle shift factor $\delta_{\text{cos}}^{(A)}$ with their theoretical upper bounds for 50 randomly selected pairs of coefficient vectors per trial. 
The original and auxiliary cosine similarities exhibit strong agreement, as evidenced by the near-zero empirical values of $\delta_{\text{cos}}^{(A)}$ ($0.0024\pm 0.0068$).
The theoretical bounds are satisfied for nearly all $200\times50=10^4$ examined pairs except for 2 pairs.
Overall, Figure~\ref{fig:auxiliary_gap} demonstrates strong geometric and spectral alignment between the original and auxiliary tri-factorisations (near-one $ \delta_{l_2}^{(A)}$ and $ \delta_{S}$, near-zero $\delta_{\text{cos}}^{(A)}$),  indicating that recovering the auxiliary matrices can faithfully preserve  the original   latent structure.

\section{Conclusion and Future Works}

To the best of our knowledge, this paper presents the first rigorous theoretical study of sparsity-induced identifiability in general real-valued matrix tri-factorisation. 
We establish   recovery guarantees   showing that sparsity fundamentally governs  identifiability and recoverability. 
The key enabler of this analysis is  a novel decomposition strategy that transforms an otherwise intractable theoretical problem into two coupled auxiliary factorisation problems while preserving the structural information required to recover the original factor matrices.
We further show that the auxiliary formulation faithfully preserves the essential structural properties of the original factorisation, providing a rigorous justification for analysing it as a structurally consistent surrogate.
Extensive Monte Carlo experiments demonstrate close agreement between the theoretical results and empirical observations.

Our theory reveals a simple mechanism: \textit{coefficient sparsity shapes the geometry of the auxiliary factorisation problems, and this geometry, in turn, governs identifiability and recovery}.
For instance, the row-wise sparsity of   $\textbf{A}$ and  $\textbf{B}$ governs the column norm and singular values of the auxiliary coefficient matrices $\textbf{X}_A$ and $\textbf{X}_B$, together with the row and submatrix norms of the auxiliary observation matrices $\textbf{Y}_U$ and $\textbf{Y}_V$.
In contrast, the derived bounds on the spectral norm and incoherence of the auxiliary dictionary matrices $\textbf{D}_A$ and $\textbf{D}_B$ are independent of sparsity because their atoms are  $l_2$-normalised, depending instead on the incoherence of $\textbf{S}$.
Consequently,   sparsity governs the recovery guarantees by relaxing the required recovery conditions,  determining the convergence factors of Algorithms \ref{alg:rl}  and \ref{alg:dl} through $\eta\sqrt{s_A}$ and $\eta\sqrt{s_B}$,  the error bounds  through $ \sqrt{\frac{s_A}{k}}$ and  $\sqrt{\frac{s_B}{k}}$ in Algorithms \ref{alg:rl}-\ref{alg:ini} and additionally through $\sqrt{\frac{s_A^3}{k}}$ and  $\sqrt{\frac{s_B^3}{k}}$  in Algorithm \ref{alg:ini}.
Furthermore, it improves the  angular consistency between the original and auxiliary coefficient vectors via   $\frac{s_A}{k}$ and $\frac{s_B}{k}$.
These uncover a role for sparsity analogous to that in classical dictionary learning and sparse coding, while extending rigorous identifiability and recovery guarantees to general real-valued matrix tri-factorisation.

Our theoretical framework opens several promising research directions. 
First, the theoretical conditions identified by our analysis suggest practical principles for regularisation and model design. 
In particular, promoting coefficient sparsity, reducing variations in coefficient amplitudes and latent-relation norms, and encouraging incoherent latent representations may yield models that are not only more interpretable but also more identifiable and easier to recover. 
These insights may also guide the modelling of latent components in applications such as layer-wise matrix factorisation and disentangled representation learning.

More broadly, the proposed decomposition strategy provides a general blueprint for extending theoretical results beyond sparsity. 
Rather than analysing complex factorisation models directly, future work can leverage the auxiliary formulation to systematically lift existing theory, including structural constraints, recovery guarantees, and identifiability, for two-factor matrix factorisation to tri-factorisation or more general multi-factor models. 
We hope this perspective will stimulate further theoretical advances in structured matrix and tensor factorisation.

\section*{Acknowledgement}
I sincerely thank Prof. Yannis Goulermas for his inspirational and encouraging discussions, and Dr. Yian Deng for assisting in checking an early version of the proofs in Section \ref{sec:matrix_res} and of several supporting lemmas.

\appendix

\section{Used Existing Results}
\label{app:existing_res}

Our theoretical analysis builds upon a collection of classical concepts and fundamental results in random matrix theory.
These include the matrix incoherence, restricted isometry property,   inner-product-induced distance measure between vectors, Bernstein inequality,  spectrum, inverse and norm properties of random matrices,  as well as inequalities  relating the norms and sparsity of random sparse matrices. 
Our analysis also  relies  on a key theorem developed by \citet{Can08} and later restated   by \citet{Aga16} in the context of compressed sensing, which establishes recovery conditions and theoretical guarantee  for the Lasso estimator.
In addition, we employ several results from \citet{Aga16,Aga17} concerning dictionary estimate and atom properties to facilitate the analysis of Algorithms \ref{alg:dl} and \ref{alg:ini}.
Collectively, these  results form the theoretical basis of our analysis.

\subsection{Preliminary on Random Matrices}
\label{app:existing_theo}

We first introduce three concepts related to matrix incoherence, followed by the concept of restricted isometry constant $\delta_s$.
Given a matrix,  the restricted isometry constant  quantifies  the extent to which the rows of every sub-matrix formed by selecting at most $s$ rows are approximately orthonormal.

\begin{definition}[Matrix $\mu_0$-Incoherence]
Let  $\textbf{R}\in \mathbb{R}^{n  \times m}$ be a rank-$r$  matrix  with compact SVD $\textbf{R}= \textbf{U}\bm\Sigma\textbf{V}^T$. 
The matrix $\textbf{R}$ is said to be $\mu_0$-incoherent, for some $\mu_0>0$, if the rows of its singular vector matrices satisfy
\begin{align}
&\left\|\bm U^{(i)}\right\|_2  \leq \mu_0\sqrt{\frac{ r}{n}}, \forall  i \in [n], \\
&\left\|\bm V^{(i)}\right\|_2  \leq \mu_0\sqrt{\frac{ r}{m}}, \forall  i \in [m].
\end{align}
\end{definition}

\begin{definition}[Matrix $(\mu_0, \mu_1)$-Incoherence]
Let  $\textbf{R}\in \mathbb{R}^{n  \times m}$ be a rank-$r$  matrix  with compact SVD $\textbf{R}= \textbf{U}\bm\Sigma\textbf{V}^T$. 
The matrix $\textbf{R}$ is said to be $(\mu_0, \mu_1)$-incoherent, for some $\mu_0, \mu_1>0$, if, in addition to being $\mu_0$-incoherent,   the maximum entry of the matrix $\sum_{i=1}^r {\bm U^{(i)}}^T\bm V^{(i)}$ is   bounded by $\mu_1\sqrt{\frac{r}{nm}}$.
\end{definition}

\noindent 
The above concept of  $(\mu_0, \mu_1)$-incoherence is similar to the  one defined in \citet{Kes10} but uses  the definition of $\mu_1$ in \citet{Can09}.   
It is worth to mention that if the matrix $\textbf{R}$  is $\mu_0$-incoherent, it is also $(\mu_0, \mu_1)$-incoherent with $\mu_1 = \mu_0\sqrt{r}$, resulted from   Cauchy-Schwarz inequality.

\begin{definition}[Pairwise Incoherence]
\label{def:incoherence}
Let $\textbf{M} \in \mathbb{R}^{k  \times d}$ be matrix whose rows have unit $l_2$-norm, i.e., $\left\|\bm M^{(i)}\right\|_2 =1$.
The matrix $\textbf{M}$ is said to be pairwise incoherent, for some $\mu>0$, if every pair of its rows  satisfies
\begin{equation}
\left|\left\langle\bm M^{(i)}, \bm M^{(j)} \right\rangle\right| \leq \frac{\mu}{\sqrt{d}},\textmd{ for } i\neq j \text{ and } i,j \in [k].
\end{equation}
\end{definition}

\begin{definition}[Restricted Isometry Constant $\delta_s$]
\label{def:RIC}
For  each integer $s\in [k]$, the restricted isometry constant $\delta_s\in \mathbb{R}$ of a  matrix $\textbf{M} \in \mathbb{R}^{k  \times d}$  is defined as the smallest nonnegative constant such that, for every vector $\textbf{x}\in\mathbb{R}^{k}$ with at most $s$ nonzero entries, it has
\begin{equation}\label{eq:RIC1}
(1-\delta_s)\|\textbf{x}\|_2^2\leq \|\textbf{x}\textbf{M}\|_2^2\leq(1+\delta_s)\|\textbf{x}\|_2^2, 
\end{equation}
Equivalently,  
\begin{equation} \label{eq:RIC2}
1-\delta_s \leq \sigma^2_{\min}\left(\textbf{M}^{supp(\textbf{x})}\right)  \leq \sigma^2_{\max}\left(\textbf{M}^{supp(\textbf{x})}\right) \leq 1+\delta_s.
\end{equation}
\end{definition}
 


The inner-product-induced distance (Definition \ref{def:dist}) provides an upper bound on the Euclidean distance between two vectors, and has been shown to be effective for analysing dictionary estimation errors \citep{Aga16}. 
Lemma \ref{dist_lemma} states the corresponding bound.

\begin{definition} [Inner-product-induced Distance]
\label{def:dist}
The inner-product-induced distance between  two vectors $\textbf{x}, \textbf{y}\in \mathbb{R}^d$ is defined as
\begin{equation}
\textmd{dist}(\textbf{x}, \textbf{y}) = \sup_{\textbf{z} \perp \textbf{y}} \frac{\langle\textbf{z}, \textbf{x}\rangle}{\|\textbf{z}\|_2\|\textbf{x}\|_2} = \sup_{\textbf{z} \perp \textbf{x}} \frac{\langle\textbf{z}, \textbf{y}\rangle}{\|\textbf{z}\|_2\|\textbf{y}\|_2},
\end{equation}
where $\textbf{z} \perp \textbf{y}$  denotes that $\textbf{z} $ is orthogonal to $\textbf{y}$, i.e., $\langle\textbf{z}, \textbf{y}\rangle =0$. 
The inner-product-induced distance between  two matrices $\textbf{X}, \textbf{Y}\in \mathbb{R}^{k\times d}$ is defined by extending the above vector distance:
\begin{equation}
\textmd{dist}(\textbf{X}, \textbf{Y}) = \sup_{i\in [k]}\textmd{dist}\left(\textbf{X}^{(i)}, \textbf{Y}^{(i)}\right).
\end{equation}
\end{definition}

\begin{lemma}   
\label{dist_lemma}
The inner-product-induced distance between  two vectors $\textbf{x}, \textbf{y}\in \mathbb{R}^d$ satisfies
\begin{equation}
\label{lemma_distD}
\textmd{dist}(\textbf{x}, \textbf{y}) \leq \min_{a\in\{-1, +1\}} \|a\textbf{x} -\textbf{y}\|_2 \leq \sqrt{2} \textmd{dist}(\textbf{x}, \textbf{y}).
\end{equation}
\end{lemma}

We restate below several existing results. 
Theorem \ref{SpecRand} is a restatement of Theorem 5.44 of \citet{Ver12}.
Lemma \ref{inv1} presents a fundamental result on matrix inverse \citep{Mil81}.
Lemma \ref{bern_ineq} presents one of the well-known Bernstein inequalities.
Lemmas \ref{random_spars} and \ref{random_nonzero} are restatements of Lemma 12 and Lemma 19.1 from \citet{Aga16}, respectively.

\begin{theorem} [Spectrum of Random Matrix] \label{SpecRand}
Let  $\textbf{X}\in \mathbb{R}^{n\times k}$ be a random matrix  whose rows are independent random vectors with common second moment matrix $\bm\Sigma$  and satisfy  $\left\| \textbf{X}^{(i)}\right\|_2 \leq \sqrt{M}$. 
Then, for any $t>0$, with probability at least $p = 1-ke^{-Ct^2}$, the following inequality holds
\begin{equation}
\left\| \frac{1}{n}\textbf{X}^T\textbf{X} - \bm\Sigma\right\|_2 \leq \max \left(\| \bm\Sigma\|_2^{\frac{1}{2}}\gamma, \gamma^2\right),
\end{equation}
where $\gamma = t\sqrt{\frac{M}{n}}$ and $C>0$  is a positive constant.
In particular, this inequality yields
\begin{equation}
\label{SpecRand:eq2}
\left\|\textbf{X}\right\|_2 \leq \| \bm\Sigma\|_2^{\frac{1}{2}}\sqrt{n} +t\sqrt{M}.
\end{equation}
\end{theorem}

\begin{lemma} [Inverse of Matrix Sum]  \label{inv1}
Let $\textbf{X}, \textbf{Y} \in \mathbb{R}^{n\times n}$  be two square matrices. 
Assume that $\textbf{X}$ and $\textbf{X}+\textbf{Y}$  are non-singular, and $\textbf{Y}$  is  rank-1.
Then, it has
\begin{equation}
(\textbf{X}+\textbf{Y})^{-1} = \textbf{X}^{-1} - \frac{1}{1+z}\textbf{X}^{-1}\textbf{Y}\textbf{X}^{-1},
\end{equation}
where $z = \textmd{tr}\left(\textbf{Y}\textbf{X}^{-1}\right)$.
\end{lemma}

\begin{lemma}[Bernstein Inequality] \label{bern_ineq}
Let $X_1,\;X_2,\;\ldots,  X_n$ be independent random scalar variables with zero mean, and  suppose $|X_i| \leq 
R$.
For any   constant $\delta>0$, it then has
\begin{equation}
P\left(\sum_{i=1}^n X_i \geq t\right) \leq  e^{-\frac{\frac{1}{2}\delta^2}{\sum_{i=1}^n  E\left[X_i^2\right] + \frac{R\delta}{3}}}.
\end{equation}
\end{lemma}

\begin{lemma} [Norm Inequality of Random Sparse Matrix]  \label{random_spars}
Let $\textbf{X}\in \mathbb{R}^{n\times k}$ be a random matrix  whose support is generated according to  Assumption \ref{SC}.
Then, for every  matrix $\textbf{Y} \in \mathbb{R}^{n\times k}$ with $supp(\textbf{Y}) \subseteq  supp(\textbf{X})$,  with probability at least $1-ke^{-\frac{Cn}{ks}}$, the following inequality holds 
\begin{equation}
\|\textbf{Y}\|_{2} \leq 2 \|\textbf{Y}\|_{\infty}\sqrt{\frac{ns^2}{k}},
\end{equation}
where $C>0$ is a universal constant.
\end{lemma}

\begin{lemma} [Nonzero Set of Random Sparse Matrix]  \label{random_nonzero}
Let $\textbf{X}\in \mathbb{R}^{n\times k}$ be a random matrix  whose support is generated according to  Assumption \ref{SC}.
Define the indicator variable  $\chi_{ij} = 1$ if $j\in \textmd{supp}\left(\textbf{X}^{(i)}\right)$ while $\chi_{ij} = 0$ otherwise. 
Then, for any  $\delta >0$ and any $j\in  [k]$,    the following inequality holds   with probability at least $1-2ke^{-\frac{\delta^2ns}{4k}}$:
\begin{equation}
(1-\delta)\frac{sn}{k} \leq \sum_{i=1}^n \chi_{ij} \leq (1+\delta)\frac{sn}{k}.
\end{equation}
\end{lemma}
This result implies that, among the $n$ randomly generated $k$-dimensional vectors each containing at most $s$  nonzero entries, the number of vectors whose $j$-th entry is  non-zero concentrates around $\frac{sn}{k}$, with the stated probability.

\subsection{A Classical Result in Compressed Sensing}
The following theorem, originally established by \citet{Can08} and later restated by \citet{Aga16}, provides sufficient conditions for approximate  sparse recovery and establishes recovery guarantees for the constrained Lasso estimator.

\begin{theorem} [Noisy Sparse Recovery] \label{Candes}
Let $\textbf{y} = \textbf{x}\textbf{D} +\textbf{e}$ be a $k$-dimensional   vector  where $\textbf{D}\in \mathbb{R}^{k\times d}$, $\textbf{x}\in \mathbb{R}^k$, and $\textbf{e} \in \mathbb{R}^{d}$. 
Suppose that $\textbf{x}$ contains at most $s$ nonzero entries, $\|\textbf{e}\|_2\leq \epsilon$, and  the restricted isometry constant of $\textbf{D}$  satisfies $\delta_{2s}\leq \sqrt{2}-1$. 
Consider the following constrained Lasso problem:
\begin{align}
\label{eq:lasso}
\hat{\textbf{x}} =\; &\arg\min_{\textbf{z}\in \mathbb{R}^{k}}  \|\textbf{z}\|_1, \\
\nonumber
& \textmd{subject to }  \|\textbf{y} - \textbf{z}\textbf{D}\|_2 \leq \epsilon.
\end{align}
There exists an explicit constant $C>0$ so that the solution satisfies
$\left\|\hat{\textbf{x}}- \textbf{x}\right\|_2 \leq C\epsilon$. 
In particular, when $\delta_{2s}\leq 0.2$, $C=8.5$ suffices.
\end{theorem}

\subsection{ An Existing Bound  on Dictionary Estimation Error}
\label{sec:existing_dl}

Denote the standard dictionary learning   model  by $\textbf{Y} \approx \textbf{X}\textbf{D}$.
After obtaining a coefficient estimate $\hat{\textbf{X}}$, the  dictionary estimate is updated according to  $\hat{\textbf{D}} = \hat{\textbf{X}}^{\dagger} \textbf{Y} =\hat{\textbf{X}}^{\dagger} \textbf{X}\textbf{D}$, as described in Step \ref{step:dic_comp} of Algorithm \ref{alg:dl}. 
 \citet{Aga16} established an upper bound on the inner-product-induced distance between the estimated and true dictionaries.
We summarise the results from their proof of Theorem 1, Lemma 13 and Lemma 20, which are relevant to our theory development, in the following lemma.

\begin{lemma}[Dictionary Estimation Error] \label{dic_error}
Let  $\textbf{Y} = \textbf{X}\textbf{D}$, where $\textbf{Y}\in \mathbb{R}^{n\times d}$, $\textbf{X}\in \mathbb{R}^{n\times k}$, and $\textbf{D}\in \mathbb{R}^{k\times d}$.
Suppose that $\hat{\textbf{X}}\in \mathbb{R}^{n\times k}$  is an estimate of the coefficient matrix. 
Assuming  that   $ \textbf{X}$ is non-singular, and compute the dictionary estimate  as $\hat{\textbf{D}} = \hat{\textbf{X}}^{\dagger} \textbf{Y} =\hat{\textbf{X}}^{\dagger} \textbf{X}\textbf{D}$.
Define the coefficient estimation error by $\bm\Delta_X = \textbf{X} - \hat{\textbf{X}} $. 
Then the distance between the $i$-th estimated dictionary atom and the corresponding true atom satisfies
\begin{equation}
\label{eq:dic_error1}
\textmd{dist}\left(\hat{\textbf{D}}^{(i)}, \textbf{D}^{(i)} \right) \leq \frac{ \left\|\left( \hat{\textbf{X}}^{\dagger}\textbf{X}\right)^{(i)}_{\setminus i} \right\|_2 \|\textbf{D}\|_2}{1-  \left\| \left(\hat{\textbf{X}}^T\hat{\textbf{X}}\right)^{-1}\right\|_2\left\| \hat{\textbf{X}} \right\|_2  \left\| \bm\Delta_X \right\|_2 - \left\|\left(\hat{\textbf{X}}^{\dagger}\bm\Delta_X\right)^{(i)}_{\setminus i} \right\|_2  \|\textbf{D}\|_2 }. 
\end{equation}
Furthermore,  
\begin{align}
\label{eq:quatity3}
\nonumber
\left\|\left(\hat{\textbf{X}}^{\dagger}\bm\Delta_X\right)^{(i)}_{\setminus i} \right\|_2 = \; & \left\|\left( \hat{\textbf{X}}^{\dagger}\textbf{X}\right)^{(i)}_{\setminus i} \right\|_2 \\
\leq\; &    \left\| \left(\hat{\textbf{X}}^T\hat{\textbf{X}}\right)^{-1}\right\|_2  \left(\left\|\left(\hat{\textbf{X} }^T \bm\Delta_X\right)^{(i)}_{\setminus i} \right\|_2 +  \left\| \hat{\textbf{X}} \right\|_2  \left\| \bm\Delta_X \right\|_2\left\| \left(\hat{\textbf{X}}^T\hat{\textbf{X}}\right)^{-1}\right\|_2\left\|\hat{\textbf{X}}_i^T\hat{\textbf{X}}_{\setminus i}\right\|_2    \right).
\end{align}
and
\begin{align}
\label{eq:quantity1}
 \left\| \hat{\textbf{X}} \right\|_2  \leq \; & \|\textbf{X}\|_2+ \|\bm\Delta_X\|_2, \\
\label{eq:quantity2}
 \left\|\left(\hat{\textbf{X}}^T\hat{\textbf{X}}\right)^{-1}\right\|_2 \leq\; &\left({\sigma_{\min}\left(\textbf{X} \right)}^2 -   \|\bm\Delta_X\|_2^2 -2 \|\bm\Delta_X\|_2\|\textbf{X}\|_2 \right)^{-1}.
\end{align}
\end{lemma}

\subsection{Known Results on Shared Dictionary Atoms}
\label{sec:existing_ini}

Given a dictionary $\textbf{D}\in \mathbb{R}^{k\times d}$, let  $\textbf{y}  \in \mathbb{R}^d$ be a row vector obtained as a linear combination of the dictionary atoms:  $\textbf{y}  = \sum_{i=1}^k x_i \textbf{D}^{(i)}$, where the coefficient vector $\textbf{x} =[x_1, x_2, \ldots, x_k]$ is generated according to the   element-wise process as in Definition \ref{def:sparse:gen},  integrating Assumptions \ref{SC} and \ref{NZC}.
\begin{definition}[Sparse Vector Generation]  
\label{def:sparse:gen}
Generate a random row vector $\textbf{x}\in \mathbb{R}^k$  with at most $s$ nonzero entries according to the following  element-wise construction:   $\forall i \in [k]$,
\begin{equation} \label{genCoefficent1}
x_i= M_{i}\chi_{i},  
\end{equation}
where  $M_{i}\in \mathbb{R} $ is  a  real-valued random variable and $\chi_{i} \in \{0,1\}$ is a binary random variable.
Specifically,   the random variables $\{M_i\}_{i=1}^{k}$ are drawn independently and identically distributed with mean $\mu$ and variance $\sigma^2$.
The binary variables $\{\chi_i\}_{i=1}^{k}$ are generated by first selecting the support $supp(\textbf{x})$ uniformly at random from all subsets of $[k]$ of cardinality $s$, and then setting $\chi_{i} = 1$ if $i\in \textmd{supp}\left(\textbf{x}\right)$, while $\chi_{i } = 0$ otherwise.
\end{definition}

Let $I(\textbf{y}, \textbf{D}) =supp(\textbf{x})$ denote the index set of    atoms  contributing to the representation of $\textbf{y}$.
For two vectors $\textbf{y}_i$ and $\textbf{y}_j$, the quantity $|I(\textbf{y}_i, \textbf{D})  \cap I(\textbf{y}_j, \textbf{D})| $   denotes the number of  shared   atoms.
We  also write $D(\textbf{y}_i, \textbf{y}_j) = I(\textbf{y}_i, \textbf{D})  \cap I(\textbf{y}_j, \textbf{D})$ for their shared support.
If $|I(\textbf{y}_i, \textbf{D})  \cap I(\textbf{y}_j, \textbf{D})| =1$, then $\textbf{y}_i$ and $\textbf{y}_j$ form a unique intersection  pair.
The shared atom can then be estimated from $\textbf{y}_i$ and $\textbf{y}_j$ using Steps \ref{step:atom1} and \ref{step:atom2} of Algorithm \ref{alg:ini}.
Consequently,  identifying unique intersection  pair becomes a key step in estimating dictionary atoms.
\citet{Aga17} proposed to identify    unique intersection pairs  by constructing  the $\rho$-correlation graph   $G_{\rho}(Y)$ (Definition \ref{def:rho_corr_graph}) and examining the shared $\rho$-neighbour set  $N_{\rho}(\textbf{y}_i, \textbf{y}_j, Y)$ (Definition \ref{def:share_neighbour}) for each connected pair of vertices in the graph.
This procedure forms the basis of Algorithms \ref{alg:ini} and \ref{alg:unique}.
Specifically,  for each connected pair  in $G_{\rho}(Y)$, the connectivity density of the  $\rho$-correlation graph constructed from $N_{\rho}(\textbf{y}_i, \textbf{y}_j, Y)$ is assessed.
The success of this procedure depends on an appropriate choice of the correlation threshold $\rho$, for which a sufficient condition   is given in the following assumption.
\begin{assumption}[Correlation Threshold $\rho$] \label{threshold_rho}
Given a dictionary $\textbf{D}\in \mathbb{R}^{k\times d}$, let $Y=\{\textbf{y}_i\}_{i=1}^n$ be a collection of vectors, each represented as a linear combination of rows in $\textbf{D}$, where the  coefficient vectors  are generated according to  Definition \ref{def:sparse:gen}.   
Assume that the correlation threshold $\rho$ satisfies the following conditions for any pair of vectors $\textbf{y}_i, \textbf{y}_j \in Y$:
\begin{itemize}
\item If  $|I(\textbf{y}_i, \textbf{D})  \cap I(\textbf{y}_j, \textbf{D})| =1$, then $ \left|\textbf{y}_i\textbf{y}_j^T\right| > \rho$.
\item If $ \left|\textbf{y}_i\textbf{y}_j^T\right| > \rho$, then  $|I(\textbf{y}_i, \textbf{D})  \cap I(\textbf{y}_j, \textbf{D})| \geq 1$.
\end{itemize}
\end{assumption}

Our analysis relies on key results from the proofs of Lemma 3.1, Proposition 3.1 and Proposition 3.2 in \citet{Aga17},   restated below as   Lemmas \ref{UniqueAtom} and \ref{AtomSharing}.
The choice of the unique intersection threshold $\rho_p =\frac{62}{64}$ in Algorithm \ref{alg:unique} is motivated by Lemma \ref{UniqueAtom}.
\begin{lemma}[Unique Intersection Condition] \label{UniqueAtom}
Given a dictionary $\textbf{D}\in \mathbb{R}^{k\times d}$, let $Y=\{\textbf{y}_i\}_{i=1}^n$ be a collection of vectors, each represented as a linear combination of the rows in $\textbf{D}$, where the  coefficient vectors  are generated according to  Definition \ref{def:sparse:gen}.     
For any  $\textbf{y}_i, \textbf{y}_j \in Y$, define 
\begin{equation}
P(\textbf{y}_i,\textbf{y}_j, Y) =  \left\{ (t, l)\left|  t,l\in  [k],  \textbf{y}_t, \textbf{y}_l  \in N_{\rho}(\textbf{y}_i, \textbf{y}_j, Y),  \left|\textbf{y}_t\textbf{y}_l^T\right| > \rho\right.\right\}, 
\end{equation}
and let $\bar{N} = \frac{1}{2}\left|N_{\rho}(\textbf{y}_i, \textbf{y}_j, Y)\right| \left(\left|N_{\rho}(\textbf{y}_i, \textbf{y}_j, Y)\right|-1\right)$.
Suppose that $s^3 \leq \frac{k}{1536}$ and that the correlation threshold  $\rho$  satisfies Assumption \ref{threshold_rho}.
Then the following holds 
\begin{equation}
     \left\{\begin{array}{ll}
         \frac{\left|P(\textbf{y}_i,\textbf{y}_j, Y)\right|}{ \bar{N}}  \geq  \frac{62}{64},  &    \textmd{if }   |I(\textbf{y}_i, \textbf{D})  \cap I(\textbf{y}_j, \textbf{D})| =1 , \\
         \frac{\left|P(\textbf{y}_i,\textbf{y}_j, Y)\right|}{ \bar{N}}  \leq  \frac{61}{64}, &   \textmd{otherwise},
        \end{array}
        \right.
\end{equation}
with probability at least $1-2e^{- 2\bar{N}\gamma^2}$ where $\gamma\leq \frac{1}{64}$ is  a  constant.  
\end{lemma}

\begin{lemma}[Unique Intersection Set Size] \label{AtomSharing}
Given a dictionary $\textbf{D}\in \mathbb{R}^{k\times d}$, let $Y=\{\textbf{y}_i\}_{i=1}^n$ be a collection of vectors, each represented as a linear combination of the rows in $\textbf{D}$, where the  coefficient vectors  are generated according to  Definition \ref{def:sparse:gen}.       
Let $\textbf{y}_i, \textbf{y}_j \in Y$ be a unique intersection pair,  and suppose they share the single   atom $\textbf{D}^{(p)}$, i.e.,  $D(\textbf{y}_i, \textbf{y}_j) = \left\{\textbf{D}^{(p)}\right\}$. 
Partition $N_{\rho}(\textbf{y}_i, \textbf{y}_j, Y) $ into the following two disjoint subsets:
\begin{align}
\label{eq:set1}
\tilde{N}_{\rho}(\textbf{y}_i, \textbf{y}_j, Y)  = & \left\{ t \left| t\in N_{\rho}(\textbf{y}_i, \textbf{y}_j, Y), D(\textbf{y}_i, \textbf{y}_t) =D(\textbf{y}_j, \textbf{y}_t) =\left \{\textbf{D}^{(p)}\right\} \right.\right\} ,\\
\label{eq:set2}
\tilde{N}_{\rho}^{\neg}(\textbf{y}_i, \textbf{y}_j, Y)  = &N_{\rho}(\textbf{y}_i, \textbf{y}_j, Y) \setminus \tilde{N}_{\rho}(\textbf{y}_i, \textbf{y}_j, Y).
\end{align}
Then  the following bounds hold:
\begin{align}
\left|\tilde{N}_{\rho}(\textbf{y}_i, \textbf{y}_j, Y) \right| \geq\;  & \left(1-\frac{s^3}{k} -\Delta\right) \left|N_{\rho}(\textbf{y}_i, \textbf{y}_j, Y)\right| ,\\
\left|\tilde{N}_{\rho}^{\neg}(\textbf{y}_i, \textbf{y}_j, Y) \right|  \leq \; & \left(\frac{s^3}{k} +\Delta\right) \left|N_{\rho}(\textbf{y}_i, \textbf{y}_j, Y)\right|, 
\end{align}
with probability at least $1-e^{-2\Delta^2\left|N_{\rho}(\textbf{y}_i, \textbf{y}_j, Y)\right|}$ where $\Delta>0$.
\end{lemma}
In the above, the set  $P(\textbf{y}_i,\textbf{y}_j, Y)$ contains all the connected pairs in the graph  $G_{\rho}\left(N_{\rho}(\textbf{y}_i, \textbf{y}_j, Y)\right)$,  while  $\bar{N}$ denotes the total   number of possible edges in this graph.

\section{Proofs of Key Matrix Properties}
\label{app_proof_matrix}

\subsection{Proof of Lemma  \ref{singularAB}}
\label{proof_singularAB}

\subsubsection{A Supporting Lemma}

To facilitate the derivation of the bounds in Lemma \ref{singularAB}, we first derive a general expression for the singular values of the second-moment matrix of random sparse vectors generated according to Definition~\ref{def:sparse:gen}. The resulting expressions are presented in Lemma \ref{cov1}.
\begin{lemma}[Singular Values of the Second Moment Matrix, $x_i= M_{i}\chi_{i}$] \label{cov1}
Let  $\textbf{x}\in \mathbb{R}^k$ be a random  vector generated according to Definition \ref{def:sparse:gen}, containing at most $s$ nonzero entries.  
Then, its second moment  matrix  $\bm\Sigma =  E\left[\textbf{x}^T\textbf{x} \right]$ has one singular value equal to  $\sigma_{\max}(\bm\Sigma)$, and the remaining $k-1$ singular values equal to $\sigma_{\min}(\bm\Sigma)$, where 
 \begin{align}
 \sigma_{\max}(\bm\Sigma)  =  \; &  \frac{s^2\mu^2 +s\sigma^2 }{k}, \\
\sigma_{\min}(\bm\Sigma)  = \;  & \frac{s\left(\mu^2 +\sigma^2\right)}{k}-\frac{s(s-1)\mu^2}{k(k-1)}.
\end{align}
\end{lemma}

\begin{proof}
We utilise the following basic results on $\chi_{i}$  and $M_i$ \citep{Aga16}: 
\begin{align}
\nonumber
&E[\chi_i] =  \frac{s}{k}, \; E\left[\chi_i^2\right] =\frac{s}{k}, \; E[\chi_i\chi_j] = \frac{s(s-1)}{k(k-1)}, \\
\nonumber
& E[M_i]  = \mu, \;E\left[M_i^2\right]  =\mu^2 +\sigma^2, \; E[M_iM_j] =  \mu^2.
\end{align}
The diagonal element of $\bm\Sigma$ is 
\begin{equation}
\bm\Sigma_{ii} = E\left[M_{i}^2\chi_{i}^2\right] =E\left[M_{i}^2\right] E\left[\chi_{i}^2\right]  =\frac{s\left(\mu^2 +\sigma^2\right)}{k}.
\end{equation}
The off-diagonal element of $\bm\Sigma$ is
\begin{equation}
\bm\Sigma_{ij} = E[\chi_{i}M_{i}   \chi_{j}M_{j} ]  =E[M_{i}M_{j}]E[\chi_{i}\chi_j] = \frac{s(s-1)}{k(k-1)}\mu^2 .
\end{equation}
The second moment matrix of $\textbf{x}$ is 
\begin{equation}
\bm\Sigma= E[\textbf{x}\textbf{x}^T] = \left(\frac{s\left(\mu^2 +\sigma^2\right)}{k}-\frac{s(s-1)\mu^2}{k(k-1)}\right) \textbf{I}_k + \frac{s(s-1)\mu^2}{k(k-1)}\textbf{1}_k\textbf{1}_k^T.
\end{equation}
Define
\begin{align}
\label{eq:a}
a= \;& \frac{s\left(\mu^2 +\sigma^2\right)}{k}-\frac{s(s-1)\mu^2}{k(k-1)}, \\
\label{eq:b}
b= \;& \frac{s(s-1)\mu^2}{k(k-1)}.
\end{align}
Eigen-decompose the size-$k$ matrix of ones as  
$\textbf{1}_k\textbf{1}_k^T = \textbf{U}_{1} \left[
\begin{array}{cc c c}
  k & 0 &\ldots &0     \\
  0&   0  &\ldots &0   \\
  0&   0  &\ldots &0    \\  
\end{array}\right ]
\textbf{U}_{1}^T$. 
Subsequently, we have 
\begin{equation}
\bm\Sigma = a \textbf{I}_k  +b \textbf{1}_k\textbf{1}_k^T = \textbf{U}_{1} \left[
\begin{array}{cc c c}
  a+kb & 0 &\ldots &0     \\
  0&   a  &\ldots &0   \\
  0&   0  &\ldots &a    \\  
\end{array}\right ]
\textbf{U}_{1}^T.
\end{equation}
Therefore
 \begin{align}
 \label{eq_akb1}
\sigma_{\max}(\bm\Sigma) =&\; a+kb =   \frac{s^2\mu^2 +s\sigma^2 }{k}, \\ 
 \label{eq_akb2}
\sigma_{\min}(\bm\Sigma)  =&\;  a =\frac{s\left(\mu^2 +\sigma^2\right)}{k}-\frac{s(s-1)\mu^2}{k(k-1)}.
\end{align}
This completes the proof. 

\end{proof}

\subsubsection{Main Proof}

\begin{proof}
We provide proof of Lemma  \ref{singularAB} for $\textbf{A}$, and the same applies to $\textbf{B}$.  
The rows of $\textbf{A}$ are independent random vectors sampled by following the generating process described in Definition \ref{def:sparse:gen}, with a zero mean $\mu_A =0$ and a fixed variance $\sigma_A^2$. 
These rows share the same second moment matrix, denoted by $\bar{\bm\Sigma}_A$.
Applying Lemma \ref{cov1}, $\bar{\bm\Sigma}_A$ has $k$ equal singular values  given as
\begin{equation}
\label{eq:sigAA}
    \sigma_{\max}\left(\bar{\bm\Sigma}_A\right)  = \left\| \bar{\bm\Sigma}_A\right\|_2  =   \sigma_{\min}\left(\bar{\bm\Sigma}_A\right) = \frac{s_A\sigma_A^2}{k}.
\end{equation}
Each row of $\textbf{A}$  satisfies $\left\|\textbf{A}^{(i)}\right\|_2 \leq \sqrt{s_A} M^{(A)}_{\max}$.  
Applying Theorem \ref{SpecRand} with $\gamma = t\sqrt{\frac{s_A{M^{(A)}_{\max}}^2}{n}}$, the following holds with probability at least   $1 - ke^{-C_At^2}$:
\begin{align}
\left\|\frac{1}{n} \textbf{A}^T\textbf{A}- \bar{\bm\Sigma}_A\right\|_2 \leq\; & \max\left(\left\| \bar{\bm\Sigma}_A\right\|_2^{\frac{1}{2}} t\sqrt{\frac{s_A{M^{(A)}_{\max}}^2}{n}}, \frac{t^2s_A{M^{(A)}_{\max}}^2}{n}\right)\\
\nonumber
= \; & \max\left(  t\sqrt{\frac{\sigma_A^2s_A^2{M^{(A)}_{\max}}^2}{nk}}, \frac{t^2s_A{M^{(A)}_{\max}}^2}{n}\right)\\
\nonumber
=\; & \max\left( t\sqrt{\frac{\sigma_A^2{M^{(A)}_{\max}}^2k}{n  }} , \frac{t^2{M^{(A)}_{\max}}^2k}{n }\right)\frac{s_A}{k}  \\
\nonumber
=\; &\max\left(  \delta\sigma_A , \delta^2\right)\frac{s_A }{k},
\end{align}
where 
\begin{equation}
\label{t_delta}
t = \delta\sqrt{\frac{n }{k{M^{(A)}_{\max}}^2}}.
\end{equation}
For a given $0< \Delta <1$, we want to choose a sufficiently small value, i.e.,  $\delta =  \Delta \sigma_A $, such that  
\begin{equation}
\label{choice_delta}
\max\left(  \delta  \sigma_A , \delta^2\right)  =   \max\left(   \Delta \sigma_A^2  ,   \Delta^2\sigma_A^2  \right)  =  \Delta \sigma_A^2 ,
\end{equation}
and such a choice of $\delta$   results in 
\begin{equation}
\label{eq:sigma_bound2}
  \sigma_{\max} \left(\frac{1}{n} \textbf{A}^T\textbf{A}- \bar{\bm\Sigma}_A\right) = \left\|\frac{1}{n} \textbf{A}^T\textbf{A}- \bar{\bm\Sigma}_A\right\|_2 \leq   \frac{\Delta s_A\sigma_A^2}{ k },  
\end{equation}
Applying the triangle inequality, as well as  Eqs. (\ref{eq:sigAA}) and (\ref{eq:sigma_bound2}), we obtain  the following:
\begin{equation}
    \left\|\frac{1}{n} \textbf{A}^T \textbf{A}\right\|_2 \leq  \left\|\frac{1}{n} \textbf{A}^T\textbf{A}- \bar{\bm\Sigma}_A\right\|_2  +\|\bar{\bm\Sigma}_A\|_2 <  \frac{\Delta  s_A\sigma_A^2 }{ k }  + \frac{  s_A\sigma_A^2}{k} =\frac{(1+ \Delta )s_A\sigma_A^2}{k },
\end{equation}
which results in  
\begin{equation}
\label{sigAbound_max}
\sigma_{\max}( \textbf{A})  = \left\|\textbf{A}^T \textbf{A}\right\|_2^{\frac{1}{2}} < \sqrt{\frac{(1+ \Delta ) n s_A\sigma_A^2 }{k } }.
\end{equation}
Similarly, applying singular value inequality of matrix sum, it has  
\begin{equation}
 \sigma_{\min}\left(\frac{1}{n}  \textbf{A}^T \textbf{A}\right) \geq   \sigma_{\min}(\bar{\bm\Sigma}_A)  - \sigma_{\max} \left(\frac{1}{n} \textbf{A}^T\textbf{A}- \bar{\bm\Sigma}_A\right)   \geq   \frac{s_A\sigma_A^2}{ k }   -   \frac{\Delta s_A\sigma_A^2}{  k }   = \frac{(1- \Delta )s_A\sigma_A^2}{ k }   ,
\end{equation}
and therefore 
\begin{equation}
\label{sigAbound_min}
\sigma_{\min}( \textbf{A})  \geq \sqrt{\frac{(1- \Delta)ns_A\sigma_A^2 }{ k }}.
\end{equation}
Calculating $t$ from $\delta =  \Delta \sigma_A $ by Eq . (\ref{t_delta}),  we obtain $t =  \delta\sqrt{\frac{n }{k{M^{(A)}_{\max}}^2}} = \sqrt{\frac{\Delta^2 n\sigma_A^2}{k{M^{(A)}_{\max}}^2}}$.  
Eqs. (\ref{sigAbound_max}) and (\ref{sigAbound_min}) hold with probability at least $1-ke^{-\frac{C_A\Delta^2  \sigma_A^2n}{k{M_{\max}^{(A)}}^2}}$. 
This completes the proof.

\end{proof}

\subsection{Proof of Lemma \ref{RowLengthF} }

\begin{proof}
We prove for the $\textbf{F}_B$  case and the same applies to $\textbf{F}_A$. Overall the proof builds on the fact $\textmd{rank}(\textbf{R})=\textmd{rank}\left(\textbf{A}\textbf{S}\textbf{B}^T\right)=\textmd{rank}(\textbf{S})=d$  that holds with a high probability.

 Among the $k$ left singular vectors of  $\textbf{B} \in \mathbb{R}^{m\times k}$,  there are $d$ vectors (stored as columns of   $\textbf{U}_{B}^{(R)} \in \mathbb{R}^{m\times d}$)  lying within the subspace spanned by the right singular vectors of $\textbf{R}$ (stored as columns of   $\textbf{V}_R\in  R^{m\times d} $), otherwise it would contradict the rank-$d$ fact stated above.  Subsequently, there always exists an orthogonal matrix $\textbf{O}_1\in \mathbb{R}^{d\times d}$ so that $\textbf{U}_{B}^{(R)}= \textbf{V}_R \textbf{O}_1$.  
 The remaining $k-d$ left singular vectors  of  $\textbf{B} $ (stored as columns of   $\textbf{U}_{B}^{(R_\perp)}\in \mathbb{R}^{m\times (k-d)}$) are orthogonal to the subspace spanned by $\textbf{V}_R$,  thus  $\textbf{V}_R^T\textbf{U}_{B}^{(R_\perp)} =\textbf{0}$.  
 Present SVD of $\textbf{B}$  using $\textbf{U}_{B}^{(R)}$ and $\textbf{U}_{B}^{(R_\perp)}$  along with their corresponding singular values stored as diagonals of $\bm\Sigma_{B}^{(R)}\in \mathbb{R}^{d\times d} $ and $\bm\Sigma_{B}^{(R_\perp)} \in \mathbb{R}^{(k-d)\times (k-d)} $,  also the corresponding  right singular vectors stored as columns of $\textbf{V}_{B}^{(R)} \in \mathbb{R}^{k\times d} $ and $\textbf{V}_{B}^{(R_\perp)}\in \mathbb{R}^{k\times (k-d)}$, we have
\begin{equation}
\label{SVD_B}
\textbf{B} = \left[\textbf{U}_{B}^{(R)}, \textbf{U}_{B}^{(R_\perp)} \right] \left [
\begin{array}{cc}
  \bm\Sigma_{B}^{(R)} &     \textbf {0}  \\
 \textbf {0} &      \bm\Sigma_{B}^{(R_\perp)}  
\end{array}
\right ] \left[\textbf{V}_{B}^{(R)} , \textbf{V}_{B}^{(R_\perp)} \right]^T.
\end{equation}
Utilizing $\textbf{V}_R^T\textbf{U}_{B}^{(R_\perp)}= \textbf{0}$ and $\textbf{U}_{B}^{(R)}= \textbf{V}_R \textbf{O}_1$, we have
 \begin{equation}
 \textbf{B}^T\textbf{V}_R =  \textbf{V}_{B}^{(R)} \bm\Sigma_{B}^{(R)}{\textbf{U}_{B}^{(R)}}^T\textbf{V}_R =  \textbf{V}_{B}^{(R)} \bm\Sigma_{B}^{(R)}\textbf{O}_1^T.
\end{equation}
Starting  from  $\textmd{rank}\left(\textbf{S}\textbf{B}^T\textbf{V}_R\right) = d $ that is also a result of the rank-$d$ fact as stated in the beginning   and applying  compact SVD of $\textbf{S}$, we have  
\begin{equation}
\textmd{rank}\left(\textbf{U}_S\bm\Sigma_S\textbf{V}_S^T\textbf{B}^T\textbf{V}_R\right) =\textmd{rank}\left(\textbf{U}_S\bm\Sigma_S\textbf{V}_S^T\textbf{V}_{B}^{(R)} \bm\Sigma_{B}^{(R)}\textbf{O}_1^T\right)  = d.
\end{equation}
This indicates $\textmd{rank}\left(\textbf{V}_S^T\textbf{V}_{B}^{(R)}\right)=d$. Given both $\textbf{V}_S, \textbf{V}_{B}^{(R)} \in \mathbb{R}^{k\times d}$ contain $d$ orthogonal basis vectors of a $k$-dimensional space,  they must span the same subspace in order to satisfy $\textmd{rank}\left(\textbf{V}_S^T\textbf{V}_{B}^{(R)}\right)=d$. Therefore, there exists an orthogonal matrix $\textbf{O}_2\in \mathbb{R}^{d\times d}$ such that $\textbf{V}_{B}^{(R)} = \textbf{V}_S\textbf{O}_2$. Letting $\textbf{z}_i = \textbf{S}^{(i)}  \textbf{V}_{B}^{(R)} = \textbf{U}_S^{(i)}\bm\Sigma_S\textbf{V}_S^T \textbf{V}_{B}^{(R)} = \textbf{U}_S^{(i)}\bm\Sigma_S\textbf{O}_2$,  it has
\begin{equation}
\label{si_zi}
\|\textbf{z}_i\|_2^2 =  \textbf{U}_S^{(i)}\bm\Sigma_S\textbf{O}_2\textbf{O}_2^T \bm\Sigma_S {\textbf{U}_S^{(i)}}= \textbf{U}_S^{(i)}\bm\Sigma_S \textbf{V}_S^T\textbf{V}_S\bm\Sigma_S{ \textbf{U}_S^{(i)}} ^T  =  \left\|\textbf{S}^{(i)} \right\|_2^2.
\end{equation}
Now, we use  $\textbf{z}_i$ to re-express ${\textbf{F}}_B^{(i)}$ and obtain
\begin{equation}
\label{fi_zi}
\left\| {\textbf{F}}_B^{(i)}  \right \|_2  =    \left\| \textbf{S}^{(i)}  \textbf{B}^T\textbf{V}_R\right \|_2   = \left\| \textbf{S}^{(i)} \textbf{V}_{B}^{(R)} \bm\Sigma_{B}^{(R)}\textbf{O}_1\right\|_2=  \left\| \textbf{z}_i \bm\Sigma_{B}^{(R)}\right \|_2.
\end{equation}
With Eq. (\ref{si_zi}), it is easy to derive
\begin{equation} 
\label{length_inequality}
 \left \|\textbf{S}^{(i)} \right\|_2^2 \sigma_{\min}^2\left( \textbf{B}\right) = \left\| \textbf{z}_i\right \|_2^2 \sigma_{\min}^2\left( \textbf{B}\right) \leq    \left\| \textbf{z}_i \bm\Sigma_{B}^{(R)}\right \|_2^2     \leq \left\| \textbf{z}_i\right \|_2^2\sigma_{\max}^2\left( \textbf{B}\right) = \left \|\textbf{S}^{(i)} \right\|_2^2 \sigma_{\max}^2\left( \textbf{B}\right).
\end{equation}
Combining the above with  $l_s  \leq\left \|\textbf{S}^{(i)} \right\|_2  \leq u_s$ (Assumption \ref{LS}) and Lemma \ref{singularAB},  it has
\begin{equation}
l_s\sigma_B\sqrt{\frac{(1- \Delta ) m s_B }{k }}   \leq   \left\| {\textbf{F}}_B^{(i)}  \right \|_2   \leq    u_s\sigma_B\sqrt{\frac{(1+ \Delta ) m s_B }{k }},
\end{equation}
which holds with probability at least $1-ke^{-\frac{C_B\Delta^2  \sigma_B^2m}{k s_B{M_{\max}^{(B)}}^2}}$.
\end{proof}

\subsection{Proof of Lemma \ref{singularXAXB}}

\subsubsection{A Supporting Lemma}
\label{app:spectrum_proof}

To characterise how $\textbf{X}_A$ and $\textbf{X}_B$ are scaled from $\textbf{A}$ and $\textbf{B}$, we introduce the scaling process in Definition \ref{def:sparse:gen_scaled} and bound the extreme singular values of the second-moment matrix of a scaled random vector in Lemma \ref{cov2}.
\begin{definition}[Scaled Sparse Vector Generation]  
\label{def:sparse:gen_scaled}
Generate a random row vector $\textbf{x}\in \mathbb{R}^k$  with at most $s$ nonzero entries  according to the following  element-wise process:  
\begin{equation} \label{genCoefficent2}
x_i= M_{i}\chi_{i}c_i, \; \forall i \in [k],
\end{equation}
where $c_i\in \mathbb{R}$ is a nonzero scaling constant satisfying  $0 < c_l\leq |c_i|  \leq c_u$,  and  $\chi_{i}$   and $M_i$ are generated as in Definition \ref{def:sparse:gen}.
\end{definition}
\begin{lemma}[Second Moment Singular Values, $x_i= M_{i}\chi_{i}c_i$] \label{cov2}
Consider a random row vector $\textbf{x}\in \mathbb{R}^k$  with at most $s$ nonzero entries generated by  following the process described in Definition \ref{def:sparse:gen_scaled}. 
The singular values of the second moment matrix of $\textbf{x}$ are bounded by
\begin{align}
\label{eq:bound1}
& \sigma_{\max}(\bm\Sigma)  \leq     \frac{c_u^2\left(s^2\mu^2 +s\sigma^2\right) }{k}, \\
\label{eq:bound2}
&\sigma_{\min}(\bm\Sigma) \geq  c_l^2 \left(\frac{s\left(\mu^2 +\sigma^2\right)}{k}-\frac{s(s-1)\mu^2}{k(k-1)}\right) \left(1+\frac{(s-1)k\mu^2}{\left(s\mu^2+\sigma^2\right)(k-1)} \right)^{-1}.  
\end{align}
\end{lemma}

\begin{proof}
Building on the same results on  $\chi_{i}$  and $M_i$ as used in Lemma \ref{cov1},  the diagonal element of $\bm\Sigma$ is 
\begin{equation}
\label{eq:statistics1}
\bm\Sigma_{ii} = E\left[M_{i}^2\chi_{i}^2c_i^2\right] =E\left[M_{i}^2\right] E\left[\chi_{i}^2\right] c_i^2 =\frac{s\left(\mu^2 +\sigma^2\right)c_i^2}{k}.
\end{equation}
The off-diagonal element of $\bm\Sigma$ is
\begin{equation}
\label{eq:statistics2}
\bm\Sigma_{ij} = E[\chi_{i}M_{i}c_i  \chi_{j}M_{j}c_j ]  =E[M_{i}M_{j}]E[\chi_{i}\chi_j]c_ic_j = \frac{s(s-1)}{k(k-1)}\mu^2c_ic_j.
\end{equation}
Defining  $a$ and $b$ as in Eqs. (\ref{eq:a}) and (\ref{eq:b}) and
letting $\textbf{c}=[c_1,c_2,\ldots, c_k]^T$,  the second moment matrix of $\textbf{x}$ can be expressed as
\begin{equation}
\bm\Sigma = a\text{diag}(\textbf{c}\circ \textbf{c}) + b\textbf{c}\textbf{c}^T.
\end{equation}
Applying  the triangle inequality for norm,   the expression of $a+kb$ in Eq. (\ref{eq_akb1}), and the fact that $  |c_i|  \leq c_u$,  it has 
\begin{align}
\sigma_{\max}(\bm\Sigma)  =  \|\bm\Sigma\|_2  \leq \;  &  a \left\|\text{diag}(\textbf{c}\circ \textbf{c}) \right\|_2 + b\left\|\textbf{c}\textbf{c}^T\right\|_2  \leq (a+kb)c_u^2 = \frac{c_u^2\left(s^2\mu^2 +s\sigma^2\right) }{k}.
\end{align}

Defining $\textbf{c}^-=\left[c_1^{-1},c_2^{-1},\ldots, c_k^{-1}\right]^T$  and  applying  the matrix inverse result in Lemma  \ref{inv1}, triangle inequality of norm, $ |c_i|  \geq c_l$, and  the fact $ \textmd{tr}\left(\textbf{c}\textbf{c}^T \text{diag}^{-1}(\textbf{c}\circ \textbf{c})\right) = k$, we have
\begin{align}
\nonumber
\left \|\bm\Sigma^{-1}\right\|_2    =\; & \left\|a^{-1}  \text{diag}^{-1}(\textbf{c}\circ \textbf{c}) +\frac{a^{-2}b}{1+ ka^{-1}b }\text{diag}^{-1}(\textbf{c}\circ \textbf{c})\textbf{c}\textbf{c}^T \text{diag}^{-1}(\textbf{c}\circ \textbf{c}) \right\|_2\\
\nonumber
\leq \; &    a^{-1} \left\|\text{diag}\left(\textbf{c}^-\circ \textbf{c}^-\right)\right\|_2+ \frac{a^{-1}b}{a+kb} \|\textbf{c}^-\|_2^2  \leq  c_l^{-2} a^{-1}\left(1 + \frac{kb}{a+kb}\right). 
\end{align}
By applying Eqs. (\ref{eq:b}), (\ref{eq_akb1}) and (\ref{eq_akb2}), it has
\begin{equation}
 \left \|\bm\Sigma^{-1}\right\|_2  \leq c_l^{-2}\left(\frac{s\left(\mu^2 +\sigma^2\right)}{k}-\frac{s(s-1)\mu^2}{k(k-1)}\right)^{-1} \left(1+\frac{(s-1)k\mu^2}{\left(s\mu^2+\sigma^2\right)(k-1)} \right),
\end{equation}
which gives  
\begin{align} 
\sigma_{\min}(\bm\Sigma) = \frac{1}{\left \|\bm\Sigma^{-1}\right\|_2 } \geq c_l^2 \left(\frac{s\left(\mu^2 +\sigma^2\right)}{k}-\frac{s(s-1)\mu^2}{k(k-1)}\right) \left(1+\frac{(s-1)k\mu^2}{\left(s\mu^2+\sigma^2\right)(k-1)} \right)^{-1}.
\end{align}


\end{proof}

\subsubsection{Main Proof}

\begin{proof}
We provide proof for Lemma \ref{singularXAXB} for $\textbf{X}_A$, and the same applies to $\textbf{X}_B$. 
The rows of $\textbf{X}_A$ are generated by   Definition \ref{def:sparse:gen_scaled} with  zero mean $\mu_A=0$, fixed variance $\sigma_A^2$, and  a bounded scaling constant  $c_l \leq c_i=\left\|{\hat{\textbf{F}}}_B^{(i)}\right\|_2 \leq c_u$ for which, according to Corollary \ref{RowLengthFhat},  it has 
\begin{equation} \label{clu}
c_l = l_s\sigma_B\sqrt{\frac{(1- \Delta )  s_B }{nk }}     , \; c_u = u_s\sigma_B\sqrt{\frac{(1+ \Delta )  s_B }{nk }}. 
\end{equation}
Applying  Lemma  \ref{cov2} with $\mu=0$,   the extreme singular values of the second moment matrix of $\textbf{X}_A$ are bounded by
\begin{align} 
\label{ieq_max}
& \sigma_{\max}(\bm\Sigma_A) =  \left\|  \bm\Sigma_A\right\|_2 \leq  \frac{ c_u^2\sigma_A^2 s_A }{k} = \frac{ \eta_u^2 s_A }{k}, \\
\label{ieq_min}
  &\sigma_{\min}(\bm\Sigma_A) \geq \frac{ c_l^2\sigma_A^2 s_A }{k} = \frac{ \eta_l^2 s_A }{k},  
\end{align}
where $\eta_l = c_l \sigma_A $ and $\eta_u = c_u \sigma_A $  are introduced to simplify notations, i.e.,
\begin{equation} \label{etalu}
\eta_l = l_s\sigma_A\sigma_B\sqrt{\frac{(1- \Delta )  s_B }{nk }}     , \; \eta_u = u_s\sigma_A\sigma_B\sqrt{\frac{(1+ \Delta )  s_B }{nk }}.
\end{equation}
Applying Eq. (\ref{bound_maxA}), we have
\begin{equation} 
\left\|\textbf{X}_A^{(i)}\right\|_2   \leq \sqrt{s_A}M_A  = M^{(A)}_{\max}   u_s\sigma_B\sqrt{\frac{(1+ \Delta )   s_As_B }{nk }} = \eta\sqrt{ (1+ \Delta )s_B},
\end{equation}
where $\eta = M^{(A)}_{\max}   u_s\sigma_B\sqrt{\frac{s_A}{nk}} $ is introduced to simplify notations.

Applying Theorem \ref{SpecRand} with $\gamma = t_X\eta\sqrt{\frac{(1+ \Delta )s_B }{n }}$ and then incorporating Eq. (\ref{ieq_max}), there exists a constant $\hat{C}_A$ such that the following holds with probability at least   $1-ke^{-\hat{C}_At_X^2}$:
\begin{align}
\left\|\frac{1}{n} \textbf{X}_A^T\textbf{X}_A- \bm\Sigma_A\right\|_2  \leq\; & \max\left(\left\|  \bm\Sigma_A\right\|_2^{\frac{1}{2}} t_X\eta\sqrt{\frac{(1+ \Delta )s_B }{n}},   \frac{ t_X^2\eta^2(1+ \Delta )s_B}{n}\right)\\
\nonumber
\leq \; & \max\left(   \eta_u\sqrt{\frac{ t_X^2 \eta^2   (1+ \Delta )s_As_B }{nk}},   \frac{ t_X^2\eta^2(1+ \Delta )s_B}{n}\right)\\
\nonumber
\leq \; & \max\left(   \eta_u\sqrt{\frac{t_X^2\eta^2(1+ \Delta )k  }{n  }}, \frac{ t_X^2\eta^2 (1+ \Delta )k }{n}\right)  \frac{\max(s_A, s_B)}{k}  \\
\nonumber
=\; &\max\left( \eta_u\delta, \delta^2\right)\frac{\max(s_A, s_B)}{k}  ,
\end{align}
where 
\begin{equation}
\label{t_delta2}
    t_X=  \delta\sqrt{\frac{ n  }{\eta^2(1+ \Delta )k }}.
\end{equation}
We want to choose a sufficiently small value for $\delta$, i.e., $\delta = \eta_l\Delta_X$ with $0<\Delta_X<1$, such that  
\begin{equation}
\label{choice_delta2}
\max\left( \eta_u  \delta, \delta^2\right)  = \max\left( \eta_l\eta_u\Delta_X   ,  \eta_l^2\Delta_X^2\right)  = \eta_l  \eta_u\Delta_X.
\end{equation}
This results in the following 
\begin{equation}
\label{Delta_bound}
\sigma_{\max} \left(\frac{1}{n} \textbf{X}_A^T\textbf{X}_A-  \bm\Sigma_A\right) = \left\|\frac{1}{n} \textbf{X}_A^T\textbf{X}_A- \bm\Sigma_A\right\|_2   \leq \frac{\eta_l \eta_u \Delta_X\max(s_A, s_B)}{k} .
\end{equation}
Combining  Eqs. (\ref{Delta_bound}) and (\ref{ieq_max}), it enables the following bound after applying the triangle inequality:
\begin{equation}
\nonumber
\left\|\frac{1}{n}  \textbf{X}_A^T\textbf{X}_A \right\|_2 \leq \left\|\frac{1}{n} \textbf{X}_A^T\textbf{X}_A - \bm\Sigma_A\right\|_2  +\|\bm\Sigma_A\|_2 
\leq  \frac{\eta_l\eta_u \Delta_X\max(s_A, s_B)}{k}+  \frac{ \eta_u^2 s_A }{k}  \leq  \frac{\eta_u^2(1+\Delta_X)\max(s_A, s_B)}{k}.
\end{equation}
Incorporating the formulation of $\eta_u$ in Eq. (\ref{etalu}) and multiplying both sides by $n$, it has
\begin{align}
\label{sigXAbound_max}
\nonumber
\sigma_{\max}( \textbf{X}_A)  \leq\;& \eta_u\sqrt{\frac{(1+\Delta_X)n\max(s_A, s_B)}{k}} \\
 \leq \; &   \frac{u_s\sigma_A\sigma_B\sqrt{(1+\Delta)(1+\Delta_X)} \max(s_A, s_B)}{k}  .
\end{align}
Similarly, applying singular value inequality of matrix sum, also combining Eq. (\ref{Delta_bound}) and Eq. (\ref{ieq_min}), it has  
\begin{equation}
 \sigma_{\min}\left(\frac{1}{n}  \textbf{X}_A^T\textbf{X}_A \right) \geq   \sigma_{\min}(\bm\Sigma_A)  - \sigma_{\max} \left(\frac{1}{n} \textbf{X}_A^T\textbf{X}_A-  \bm\Sigma_A\right)   \geq \frac{ \eta_l^2 s_A }{k} - \frac{\eta_l\eta_u \Delta_X\max(s_A, s_B)}{k} ,
\end{equation}
and therefore 
\begin{align}
\label{sigXAbound_min}
\nonumber
\sigma_{\min}( \textbf{X}_A)  \geq \; & \eta_l\sqrt{\frac{\left(s_A-\frac{\eta_u}{\eta_l}\Delta_X\max(s_A, s_B)\right) n}{k}} \\
\nonumber
\geq \; & \eta_l  \sqrt{\frac{\left(\frac{\min(s_A, s_B)}{\max(s_A,s_B)}-\frac{\eta_u}{\eta_l}\Delta_X\right) n\max(s_A, s_B)}{k}} \\
\geq\; & l_s\sigma_A\sigma_B\sqrt{1-\Delta}\left(\frac{\min(s_A, s_B)}{\max(s_A,s_B)}-\frac{(1+\Delta)u_s\Delta_X}{(1-\Delta)l_s}\right)^{\frac{1}{2}} \frac{ \min(s_A, s_B)}{k } . 
\end{align}
To guarantee a positive lower bound, we require that 
\begin{equation}
     \Delta_X <   \frac{l_s(1-\Delta)\min(s_A, s_B)}{u_s(1+\Delta)\max(s_A, s_B)}.
\end{equation}
Calculating $t_X$ from $\delta =   \eta_l \Delta_X$ by Eq . (\ref{t_delta}),  we obtain 
\begin{equation}
    t_X =  \sqrt{\frac{ \eta_l^2\Delta_X^2n  }{\eta^2(1+ \Delta )k }} =  \sqrt{ \frac{   \eta_l^2  \Delta_X^2n^2  }{ {M^{(A)}_{\max}}^2   u_s^2\sigma_B^2(1+ \Delta )s_A  }} = \sqrt{\frac{ l_s^2 \sigma_A^2 (1- \Delta ) \Delta_X^2  s_B n  }{ {M^{(A)}_{\max}}^2   u_s^2(1+ \Delta ) s_Ak }}.
\end{equation}
Overall, for Eqs. (\ref{sigAbound_max}) and (\ref{sigAbound_min}) to hold, it requiresprobability at least $1-ke^{-C_Bt^2}-ke^{-\hat{C}_At_X^2}$. This is because the results of  $c_l$, $c_u$ and $M_A$  build  on  the bounds of   $\left\|{\hat{\textbf{F}}}_B^{(i)}\right\|_2$   in Corollary \ref{RowLengthFhat}, which  need to hold jointly. 
This completes the proof.

\end{proof}

\subsection{Proof of Lemma \ref{columnXAXB}}
\begin{proof}
Each element in the $j$-th column of $\textbf{X}_A$ is generated by following Definition (\ref{def:sparse:gen_scaled}) with zero mean $\mu_A=0$, fixed  variance $\sigma_A^2$, and scaling constant $c_j = \left\|\hat{\textbf{F}}_B^{(j)} \right\|_2$. 
Applying Eq. (\ref{eq:statistics1}) and Corollary  \ref{RowLengthFhat},  the second moment  $ E\left[x_j^2\right]$ satisfies 
\begin{equation}
\label{eq:mean_bound}
  \frac{l_s^2 \sigma_A^2\sigma_B^2(1- \Delta )  s_As_B }{nk^2 }  \leq E\left[x_j^2\right] = \frac{s_A\sigma_A^2\left\|\hat{\textbf{F}}_B^{(j)} \right\|_2^2}{k} \leq  \frac{u_s^2 \sigma_A^2\sigma_B^2(1+ \Delta )  s_As_B }{nk^2 }   .
\end{equation}
We define another set of random variables  $\left\{W_i =(X_A)^2_{ij} - E\left[x_j^2\right] \right\}_{i=1}^n $, for a given $j$.
By definition, they are random variables with zero mean, i.e., $E[W_i] =0 $, also it has
\begin{equation}
 \|({\textbf{X}_A})_j\|_2^2 =   \sum_{i=1}^n(X_A)^2_{ij} = \sum_{i=1}^n   W_i  + nE\left[x_j^2\right].
\end{equation}
For any positive $\delta>0$, it has
\begin{equation}
    p\left(\left| \sum_{i=1}^n W_i \right| \leq \delta\right)  = p\left(  n E\left[x_j^2\right]  - \delta  \leq \left\|({\textbf{X}_A})_j\right\|_2^2    \leq n E\left[x_j^2\right]  + \delta \right).
\end{equation}
This means that, with probability  $ p\left(\left| \sum_{i=1}^n W_i \right| \leq \delta\right)$, the quantity $\left\|({\textbf{X}_A})_j\right\|_2^2  $ falls within the interval $\left[E\left[x_j^2\right]  - \delta, E\left[x_j^2\right]  + \delta \right]$, which, together with Eq. (\ref{eq:mean_bound}),  results in the following bounds for $\left\|({\textbf{X}_A})_j\right\|_2^2$, as
\begin{equation}
\label{eq:mean_bound2}
  \frac{l_s^2 \sigma_A^2\sigma_B^2(1- \Delta )  s_As_B }{k^2 } -\delta \leq \left\|({\textbf{X}_A})_j\right\|_2^2   \leq  \frac{u_s^2 \sigma_A^2\sigma_B^2(1+ \Delta )  s_As_B }{k^2 } +\delta,
\end{equation}
Below we derive the probability for the above bounds to hold.

Applying  Bernstein's inequality in Lemma \ref{bern_ineq},  for  $|W_i| \leq R$, it has
\begin{equation}
p\left(\left| \sum_{i=1}^n W_i \right| \leq \delta\right) \geq 1-2e^{-\frac{\frac{1}{2}\delta^2}{nE\left[W_i^2\right] + \frac{R\delta}{3}}}.
\end{equation}
Applying Eqs. (\ref{bound_maxA}) and (\ref{eq:mean_bound}),  we compute $R$ by deriving an upper bound for $|W_i|$ as below:
\begin{equation}
|W_i| =  \left |(X_A)^2_{ij} - E\left[x_j^2\right] \right | \leq M_A^2 +E\left[x_j^2\right]    
\leq     \frac{ {M^{(A)}_{\max}}^2   u_s^2\sigma_B^2(1+ \Delta )  s_B }{nk }  + \frac{u_s^2 \sigma_A^2\sigma_B^2(1+ \Delta )  s_As_B }{nk^2 }.
\end{equation}
Combining the last two terms above and setting it as $R$, it has 
\begin{equation}
\label{eq:R}
    R=  \frac{\left( {M^{(A)}_{\max}}^2k +s_A\sigma_A^2\right)  u_s^2\sigma_B^2(1+ \Delta )  s_B }{nk^2 }.
\end{equation}
Next, we bound $E\left[W_i^2\right]$ by
\begin{align}
\nonumber
E\left[W_i^2\right] =\;&  E\left[(X_A)^4_{ij}\right] - \left(E\left[x_j^2\right] \right)^2 \leq E[\chi_i]E\left[M_i^4\right]c^4_j \leq E[\chi_i]\left(E\left[M_i^2\right]\right)^2c^4_j = \frac{s_A\sigma_A^4\left\|\hat{\textbf{F}}_B^{(j)} \right\|_2^4}{k}, 
\end{align}
where by H\"{o}lder's inequality $E\left[M_i^4\right] \leq E\left[M_i^2\right]E\left[M_i^2\right]$.
Applying   Corollary  \ref{RowLengthFhat}, it  has
\begin{equation}
\label{eq:EW2}
   E\left[W_i^2\right] \leq \frac{ s_A\sigma_A^4 u_s^4\sigma_B^4(1+ \Delta )^2s_B^2 }{n^2k^3}  .
\end{equation}
Incorporate Eqs. (\ref{eq:R}) and (\ref{eq:EW2}) into the probability formula, it has
\begin{equation}
\label{eq:colum_prob}
 1 -2e^{-\frac{\frac{1}{2}\delta^2}{nE\left[W_i^2\right] +\frac{R\delta}{3}}}  
\geq   1- 2e^{- \frac{\frac{1}{2}nk^3\delta^2}{s_A\sigma_A^4 u_s^4\sigma_B^4(1+ \Delta )^2s_B^2 +\frac{1}{3}\left( {M^{(A)}_{\max}}^2k +s_A\sigma_A^2\right)  u_s^2\sigma_B^2(1+ \Delta )  s_B k \delta }}.
\end{equation}
Also, when developing the results, we use the bound result for $\left\|\hat{\textbf{F}}_B^{(j)} \right\|_2$ from Corollary  \ref{RowLengthFhat}, which holds with probability at least $1-ke^{-\frac{C_B\Delta^2  \sigma_B^2m}{k s_B{M_{\max}^{(B)}}^2}}$.
It is combined with Eq. (\ref{eq:colum_prob}) to obtain the joint probability bound as stated in the lemma.
 This completes the proof.

\end{proof}

\subsection{Proof of Lemma \ref{MC_DaDb}}
\begin{proof}
We provide the proof  for $\textbf{D}_B$   and the same applies to $\textbf{D}_A$. 
Following a similar strategy as  in proof  of  Lemma \ref{RowLengthF} with the same definition of $\textbf{z}_i = \textbf{S}^{(i)} \textbf{V}_{B}^{(R)}$, we have shown in Eqs. (\ref{si_zi}) and (\ref{fi_zi}) that $\left\| \textbf{z}_i \right \|_2 = \left\| \bm S^{(i)} \right \|_2$ and $\left\| {\textbf{F}}_B^{(i)}  \right \|_2 =  \left\| \textbf{z}_i \bm\Sigma_{B}^{(R)}\right \|_2$.
Similarly, it has
\begin{align}
\label{fi_zi_inner}
\nonumber
\left\langle {\textbf{F}}_B^{(i)} , {\textbf{F}}_B^{(j)}  \right \rangle  = \;& \left\langle \textbf{S}^{(i)}\textbf{B}^T\textbf{V}_R,  \textbf{S}^{(j)}\textbf{B}^T\textbf{V}_R \right \rangle\\  
=\; &  \textbf{S}^{(i)} \textbf{V}_{B}^{(R)} \bm\Sigma_{B}^{(R)}\textbf{O}_1 \textbf{O}_1 ^T  \bm\Sigma_{B}^{(R)}  {\textbf{V}_{B}^{(R)} }^T {\textbf{S}^{(i)}}^T =  \left\langle  \textbf{z}_i \bm\Sigma_{B}^{(R)}, \textbf{z}_j \bm\Sigma_{B}^{(R)}\right \rangle .
\end{align}
and 
\begin{equation}
\label{si_zi_inner}
\left\langle\textbf{z}_i, \textbf{z}_j \right \rangle =  \textbf{U}_S^{(i)}\bm\Sigma_S\textbf{O}_2\textbf{O}_2^T \bm\Sigma_S {\textbf{U}_S^{(j)}}= \textbf{U}_S^{(i)}\bm\Sigma_S \textbf{V}_S^T\textbf{V}_S\bm\Sigma_S{ \textbf{U}_S^{(j)}} ^T  =  \left\langle \textbf{S}^{(i)}, \textbf{S}^{(j)} \right \rangle.
\end{equation}
Utilising the above and denoting the singular values in $\bm\Sigma_{B}^{(R)} \in \mathbb{R}^{d\times d}$ by $\{\sigma_i\}_{i=1}^d$ where $\max\left(\{\sigma_i\}_{i=1}^d\right)\leq \sigma_{\max}^2( \textbf{B} )$ and $\min\left(\{\sigma_i\}_{i=1}^d\right)\geq \sigma_{\min}^2( \textbf{B} )$, we  have 
\begin{equation} 
\label{express_D}
 \left |   \left \langle \textbf{D}_B^{(i)}, \textbf{D}_B^{(j)}\right \rangle  \right | = \frac{\left | \left\langle \textbf{F}_B^{(i)} , \textbf{F}_B^{(j)}     \right\rangle  \right |}{\left\| \textbf{F}_B^{(i)} \right \|_2\left\| \textbf{F}_B^{(j)} \right \|_2} =  \frac{\left | \left\langle \textbf{z}_i \bm\Sigma_{B}^{(R)} , \textbf{z}_j \bm\Sigma_{B}^{(R)}    \right\rangle  \right |}{\left\| \textbf{z}_i \bm\Sigma_{B}^{(R)} \right \|_2\left\| \textbf{z}_j \bm\Sigma_{B}^{(R)} \right \|_2}  \leq \frac{ \left|\sum_{t=1}^d \sigma_t^2 z_{it} z_{jt}\right|}{\sigma_{\min}^2( \textbf{B} )\left\| \textbf{z}_i \right \|_2\left\| \textbf{z}_j \right \|_2},
\end{equation}
where $z_{it}$ and $z_{jt}$ denote the $t$-th elements in the vectors $\textbf{z}_i$ and $\textbf{z}_j$, respectively.

Identify the index set $I^{+}$ so that for each $t\in I^+$ it has $z_{it}z_{jt} \geq 0$, also identify the index set $I^{-}$ so that for each $t\in I^-$ it has $z_{it}z_{jt} <0$.
We firstly analyse  the case where $ \sum_{t=1}^d \sigma_t^2 z_{it} z_{jt} \geq 0 $. 
This results in
\begin{align}
\nonumber
    \left|\sum_{t=1}^d \sigma_t^2 z_{it} z_{jt}\right| =\; & \sum_{t\in I^+}  \sigma_t^2   \underbrace{z_{it} z_{jt} }_{\geq 0} + \sum_{t\in I^-}  \sigma_t^2 \underbrace{z_{it} z_{jt}}_{<0} \\
\nonumber
    \leq\; &  \sigma_{\max}^2( \textbf{B} ) \sum_{t\in I^+}  \underbrace{z_{it} z_{jt} }_{\geq 0} +  \sigma_{\min}^2( \textbf{B} ) \sum_{t\in I^-}   \underbrace{z_{it} z_{jt}}_{<0} \\
    \nonumber
    = \; & \left(\sigma_{\max}^2( \textbf{B} ) -\sigma_{\min}^2( \textbf{B} )\right) \sum_{t\in I^+}    z_{it} z_{jt}   + \sigma_{\min}^2( \textbf{B} )\left( \sum_{t\in I^+}    z_{it} z_{jt} + \sum_{t\in I^-}    z_{it} z_{jt} \right) \\
    =\; & \left(\sigma_{\max}^2( \textbf{B} ) -\sigma_{\min}^2( \textbf{B} )\right) \sum_{t\in I^+}    z_{it} z_{jt}  + \sigma_{\min}^2( \textbf{B} )\left|\left\langle\textbf{z}_i, \textbf{z}_j \right \rangle\right|. 
\end{align}
Also, it has
\begin{equation}
\label{eq:zizjI+}
    \left(\sum_{t\in I^+}    z_{it} z_{jt} \right)^2 \leq \left(\sum_{t\in I^+} z_{it}^2\right)\left(\sum_{t\in I^+} z_{jt}^2 \right)\leq \left(\sum_{t=1}^d z_{it}^2\right)\left(\sum_{t=1}^d z_{jt}^2\right) = \left\| \textbf{z}_i \right \|_2^2\left\| \textbf{z}_j \right \|_2^2,
\end{equation}
and therefore $\sum_{t\in I^+}    z_{it} z_{jt} \leq \left\| \textbf{z}_i \right \|_2 \left\| \textbf{z}_j \right \|_2$.
Incorporating these  into Eq. (\ref{express_D}), it has
\begin{align} 
\label{express_D2}
\nonumber
 \left |   \left \langle \textbf{D}_B^{(i)}, \textbf{D}_B^{(j)}\right \rangle  \right | \leq \; & \left(\frac{\sigma_{\max}^2( \textbf{B} )}{\sigma_{\min}^2( \textbf{B} )} -1\right)\frac{\sum_{t\in I^+} z_{it} z_{jt}}{\left\| \textbf{z}_i \right \|_2\left\| \textbf{z}_j \right \|_2} + \frac{\left|\left\langle\textbf{z}_i, \textbf{z}_j \right \rangle\right|}{\left\| \textbf{z}_i \right \|_2\left\| \textbf{z}_j \right \|_2}  \\
 \leq  \; &\left(\frac{\sigma_{\max}^2( \textbf{B} )}{\sigma_{\min}^2( \textbf{B} )} -1\right) + \left| \cos\left(\bm S^{(i)}, \bm S^{(j)} \right) \right|.
\end{align}
When $ \sum_{t=1}^d \sigma_t^2 z_{it} z_{jt} < 0 $, it has 
\begin{align}
\nonumber
    \left|\sum_{t=1}^d \sigma_t^2 z_{it} z_{jt}\right| =\; &   -\sum_{t\in I^-}  \sigma_t^2 \underbrace{z_{it} z_{jt}}_{< 0} - \sum_{t\in I^+}  \sigma_t^2   \underbrace{z_{it} z_{jt} }_{\geq 0} \\
\nonumber
    \leq\; &  -\sigma_{\max}^2( \textbf{B} ) \sum_{t\in I^-}  \underbrace{z_{it} z_{jt} }_{< 0} -  \sigma_{\min}^2( \textbf{B} ) \sum_{t\in I^+}   \underbrace{z_{it} z_{jt}}_{\geq 0} \\
    =\; &- \left(\sigma_{\max}^2( \textbf{B} ) -\sigma_{\min}^2( \textbf{B} )\right) \sum_{t\in I^-}    z_{it} z_{jt}  + \sigma_{\min}^2( \textbf{B} )\left|\left\langle\textbf{z}_i, \textbf{z}_j \right \rangle\right|. 
\end{align}
Similar result that $-\sum_{t\in I^-}    z_{it} z_{jt} \leq \left\| \textbf{z}_i \right \|_2 \left\| \textbf{z}_j \right \|_2$ can be derived in exactly the same way as in Eq. (\ref{eq:zizjI+}). Incorporating these into Eq. (\ref{express_D}), it has
\begin{align} 
\label{express_D3}
\nonumber
 \left |   \left \langle \textbf{D}_B^{(i)}, \textbf{D}_B^{(j)}\right \rangle  \right | \leq \; & \left(\frac{\sigma_{\max}^2( \textbf{B} )}{\sigma_{\min}^2( \textbf{B} )} -1\right)\frac{\left(-\sum_{t\in I^-}  z_{it} z_{jt}\right)}{\left\| \textbf{z}_i \right \|_2\left\| \textbf{z}_j \right \|_2} + \frac{\left|\left\langle\textbf{z}_i, \textbf{z}_j \right \rangle\right|}{\left\| \textbf{z}_i \right \|_2\left\| \textbf{z}_j \right \|_2}  \\
 \leq  \; &\left(\frac{\sigma_{\max}^2( \textbf{B} )}{\sigma_{\min}^2( \textbf{B} )} -1\right) + \left| \cos\left(\bm S^{(i)}, \bm S^{(j)} \right) \right|.
\end{align}
Applying Lemma \ref{singularAB} and Assumption \ref{LI}, it  has
\begin{equation}
     \left |   \left \langle \textbf{D}_B^{(i)}, \textbf{D}_B^{(j)}\right \rangle  \right | \leq  \frac{2\Delta}{1-\Delta} + \frac{\mu_s}{\sqrt{d}},
\end{equation}
 which holds with the same probability as stated in Lemma \ref{singularAB}.
 This completes the proof.
\end{proof}

\subsection{Proof of Lemma \ref{BS_DaDb}}
\begin{proof}
We provide the proof  for $\textbf{D}_A$   and the same applies to $\textbf{D}_B$. 
Given an arbitrary vector $\textbf{w}\in \mathbb{R}^k$, we start from analysing $\left\|\textbf{D}_A^T\textbf{w}\right\|_2$, which plays a key role in the spectral norm of $\textbf{D}_A$.
Given that $\left\|\textbf{D}_A^{(i)}\right\|_2  =1$, it has
\begin{equation}
\left\|\textbf{D}_A^T\textbf{w}\right\|_2^2 = \sum_{i=1}^kw_i^2\left\|\textbf{D}_A^{(i)}\right\|^2_2 + \sum_{i\neq j  }  w_iw_j \left\langle \textbf{D}_A^{(i)}, \textbf{D}_A^{(j)} \right\rangle = \|\textbf{w}\|_2^2 +\sum_{i\neq j  }  w_iw_j \left\langle \textbf{D}_A^{(i)}, \textbf{D}_A^{(j)} \right\rangle.
\end{equation}
Letting $\mu_{D_A} = \max_{ij}\left|\left\langle \textbf{D}_A^{(i)}, \textbf{D}_A^{(j)} \right\rangle  \right|$ and applying triangle inequality,   
it has 
\begin{equation}
\left\|\textbf{D}_A^T\textbf{w}\right\|_2^2 \leq    \|\textbf{w}\|_2^2    + \sum_{i\neq j  }|w_iw_j| \left|\left\langle \textbf{D}_A^{(i)}, \textbf{D}_A^{(j)} \right\rangle  \right|  \leq \|\textbf{w}\|_2^2 + \mu_{D_A} \sum_{i\neq j  }|w_iw_j|  \leq (1-\mu_{D_A})\|\textbf{w}\|_2^2 + \mu_{D_A}\|\textbf{w}\|_1^2.
\end{equation}
As a result, using $\|\textbf{w}\|_1 \leq \sqrt{d}\|\textbf{w}\|_2$, it has
\begin{equation}
\frac{\left\|\textbf{D}_A^T\textbf{w}\right\|_2^2}{\|\textbf{w}\|_2^2}  \leq 1-\mu_{D_A} +   \frac{\mu_{D_A}\|\textbf{w}\|_1^2  }{\|\textbf{w}\|_2^2}  \leq 1  +   \mu_{D_A}(d-1).
\end{equation}
Apply the definition of spectral norm $\| \textbf{D}_A\|_2 =  \max_{\textbf{w}\in \mathbb{R}^k} \frac{\left\|\textbf{D}_A^T\textbf{w}\right\|_2}{\|\textbf{w}\|_2}$, it has
\begin{equation}
 \| \textbf{D}_A\|_2^2  \leq   1+ \mu_{D_A}(d-1).
\end{equation}
Finally, by applying Lemma \ref{MC_DaDb} which gives $\mu_{D_A} = \frac{2\Delta}{1-\Delta} + \frac{\mu_s}{\sqrt{d}}$, we prove for $\textbf{D}_A$
The same  proof applies to $\textbf{D}_B$.
\end{proof}

\subsection{Proof of Lemma \ref{RowLength_YuYv}}

\begin{proof}
We provide the proof  for $\textbf{Y}_V$   and the same applies to $\textbf{Y}_U$.  
Applying  Corollary  \ref{RowLengthFhat} and  Lemma \ref{MC_DaDb}, and the fact that $\left\|\textbf{D}_A^{(i)}\right\|_2=1$, and letting $\mu_{D_A}=\max_{j\neq t} \left|\left\langle\textbf{D}_A^{(j)}, \textbf{D}_A^{(t)}\right\rangle\right| $, it has
\begin{align}
\nonumber
\left\| \textbf{Y}_V^{(i)} \right\|_2^2 =\;&  \sum_{j=1}^k (X_B)_{ij}^2\left\|\textbf{D}_A^{(j )}\right\|_2^2 +\sum_{j\neq t } (X_B)_{ij} (X_B)_{it} \left\langle\textbf{D}_A^{(j)}, \textbf{D}_A^{(t)}\right\rangle\\
\leq \; & s_B M_B ^2  +    s_B(s_B-1) M_B ^2 \mu_{D_A}     <   s_BM_B^2(1+ s_B\mu_{D_A}) .
\end{align}
Therefore, it has
\begin{equation}
\left\| \textbf{Y}_V^{(i)} \right\|_2   \leq   M_B\sqrt{s_B (1+ s_B\mu_{D_A})} \leq  M^{(B)}_{\max}  u_s \sigma_A \sqrt{\frac{ (1+ \Delta )  s_As_B }{mk } \left( 1+ s_B\left(\frac{2\Delta}{1-\Delta} + \frac{\mu_s}{\sqrt{d}}\right)\right)} ,
\end{equation}
which holds with probability at least  $1-ke^{-\frac{C_A\Delta^2  \sigma_A^2n}{k {M_{\max}^{(A)}}^2}}$.    This completes the proof.
\end{proof}

\subsection{Proof of Lemma \ref{SSN_YuYv}}

\begin{proof}
We provide the proof  for $\textbf{Y}_V$   and the same applies to $\textbf{Y}_U$.  
We first analyse the singular value of the second moment matrix of $\textbf{Y}_V^{(i)}$, and define
\begin{equation}
    \textbf{K} = E\left[{\textbf{Y}_V^{(i)}}^T\textbf{Y}_V^{(i)}\right] = \textbf{D}_A^T  E\left[{\textbf{X}_B^{(i)}}^T\textbf{X}_B^{(i)}\right] \textbf{D}_A.
\end{equation}   
For an arbitrary row vector $\textbf{w}\in \mathbb{R}^d$, letting $\textbf{z} = \textbf{w}\textbf{D}_A^T$, it has
    \begin{align}
\nonumber
\textbf{w}\textbf{K} \textbf{w}^T =\;&   \textbf{w}\textbf{D}_A^T  E\left[{\textbf{X}_B^{(i)}}^T\textbf{X}_B^{(i)}\right] \textbf{D}_A\textbf{w}   \\
\nonumber
= \; & \sum_{i=1}^{d} z_i^2  E\left[\chi_i^2M_i^2c_i^2\right] +\sum_{i\neq j} z_iz_j  E\left[\chi_iM_ic_i\chi_jM_jc_j\right]   = \frac{s_B \sigma_B^2 }{k}  \sum_{i=1}^{d} c_i^2z_i^2 \\   
\leq \; & \frac{s_Bc_u^2 \sigma_B^2\|\textbf{z}\|_2^2 }{k}  \leq  \frac{s_Bc_u^2 \sigma_B^2\|\textbf{w}\|_2^2 \|\textbf{D}_A\|_2^2}{k} ,    
\end{align}
where  $c_u$  is an upper bound of $c_i  = \left\|\hat{\textbf{F}}_A^{(i)}\right\|_2 $, and it has $c_u = u_s\sigma_A\sqrt{\frac{(1+ \Delta )  s_A }{mk }} $ according to Corollary \ref{RowLengthFhat}.
Subsequently, applying Lemma \ref{BS_DaDb}, it has
\begin{equation}
    \sigma_{\max}^2(\textbf{K}) \leq \frac{s_Bc_u^2 \sigma_B^2  \|\textbf{D}_A\|_2^2}{k} \leq  \frac{ u_s^2\sigma_A^2\sigma_B^2s_As_B( 1 +\Delta) \left(1+ \mu_D(d-1)\right)}{mk^2} ,
\end{equation}
which holds with probability at least  $1-ke^{-\frac{C_A\Delta^2  \sigma_A^2n}{k {M_{\max}^{(A)}}^2}}$.  

Next, applying Eq. (\ref{SpecRand:eq2}) from Theorem \ref{SpecRand} with $t = s_A\sqrt{\left| I_V \right|}\delta_v$ and applying Lemma \ref{RowLength_YuYv},
 it has 
\begin{align}
&\left\| \textbf{Y}_V^{(I_V)} \right\|_2 \\
\nonumber
\leq \; & \sqrt{ \left| I_V \right| } \sigma_{\max}(\textbf{K})+ t M^{(B)}_{\max}  u_s \sigma_A \sqrt{\frac{ (1+ \Delta ) \left( 1+ s_B\mu_D\right) s_As_B }{mk^2 } }  \\
\nonumber
\leq \; & u_s \sigma_A \sigma_B \sqrt{\frac{ \left| I_V \right| s_As_B( 1 +\Delta) \left(1+ \mu_D(d-1)\right)}{mk^2}} +  t M^{(B)}_{\max}  u_s \sigma_A \sqrt{\frac{ (1+ \Delta ) \left( 1+ s_B\mu_D\right) s_As_B }{mk^2} },  
\end{align}
with probability at least  $ 
p=  1- ke^{-\frac{C_A\Delta^2  \sigma_A^2n}{k {M_{\max}^{(A)}}^2}}-de^{- C_Vt^2} $.
We choose a sufficiently small $t$ such that $t M^{(B)}_{\max}  \sqrt{   1+ s_B\mu_D  } =  \Delta_Y\sigma_B \sqrt{ \left| I_V \right| \left(1+ \mu_D(d-1)\right) }$ with  $0<\Delta_Y<1$, resulting in 
\begin{equation}
    t  = \frac{\Delta_Y\sigma_B}{M^{(B)}_{\max}}\sqrt{\frac{  \left| I_V \right| \left(1+ \mu_D(d-1)\right)}{ 1+ s_B\mu_D}}.
\end{equation}
This enables the following result 
\begin{equation}
    \left\| \textbf{Y}_V^{(I_V)} \right\|_2 \leq   (1+\Delta_Y)   u_s  \sigma_A \sigma_B    \sqrt{\frac{   ( 1 +\Delta)\left(1+ \mu_D(d-1)\right)\left| I_V \right|s_As_B}{mk^2}}, 
\end{equation}
which holds  with probability at least  $ 
p=  1- ke^{-\frac{C_A\Delta^2  \sigma_A^2n}{k {M_{\max}^{(A)}}^2}}-de^{- \frac{ C_V \Delta_Y^2\sigma_B^2\left| I_V \right| \left(1+ \mu_D(d-1)\right)}{ {M_{\max}^{(B)}}^2( 1+ s_B\mu_D)}} $.
This completes the proof.

\end{proof}

\section{Proof of Theorem \ref{main_res}}
\label{app:main_proof1}

According to Lemma \ref{dic_error}, the upper bound of  the inner-product-induced distance between the estimated and original auxiliary dictionary matrices depends on norm-based  quantities like $\left\|\left(\hat{\textbf{X}} _{A}^T \bm\Delta _{A}\right)^{(i)}_{\setminus i} \right\|_2 $,  $\left\|\left(\hat{\textbf{X}} _{B}^T \bm\Delta _{B}\right)^{(i)}_{\setminus i} \right\|_2$, $\left\|\left(\hat{\textbf{X}} _{A}\right)_i^T\left(\hat{\textbf{X}} _{A}\right)_{\setminus i}\right\|_2$, and $\left\|\left(\hat{\textbf{X}} _{B}\right)_i^T\left(\hat{\textbf{X}} _{B}\right)_{\setminus i}\right\|_2$, as well as $\|\bm\Delta_{X_{A/B}}\|_{\infty}$.
We first derive upper bounds on these quantities in Lemmas \ref{coeff_error}, \ref{lemma:quantity1} and \ref{lemma:quantity2}, then conduct further  estimation error analysis for proving Theorem \ref{main_res}.
Specifically, Lemma \ref{coeff_error}  analyses the coefficient estimation error,   yielding an upper bound on $\|\bm\Delta_{X_{A/B}}\|_{\infty}$.
Lemma  \ref{lemma:quantity1} bounds $\left\|\left(\hat{\textbf{X}} _{A}^T \bm\Delta _{A}\right)^{(i)}_{\setminus i} \right\|_2 $ and $\left\|\left(\hat{\textbf{X}} _{B}^T \bm\Delta _{B}\right)^{(i)}_{\setminus i} \right\|_2$, while Lemma \ref{lemma:quantity2} bounds $\left\|\left(\hat{\textbf{X}} _{A}\right)_i^T\left(\hat{\textbf{X}} _{A}\right)_{\setminus i}\right\|_2$  and $\left\|\left(\hat{\textbf{X}} _{B}\right)_i^T\left(\hat{\textbf{X}} _{B}\right)_{\setminus i}\right\|_2$.

\subsection{Supporting Lemma \ref{coeff_error} and Its Proof}
\label{app:support_main_lasso}

\begin{lemma}[Lasso Coefficient Error Bound] \label{coeff_error}
Let $\textbf{y} = \textbf{x}\textbf{D}$ be generated from a
coefficient vector $\mathbf{x}\in\mathbb{R}^{k}$ containing $s$ nonzero entries that have bounded amplitude, i.e.,   $ m \leq \|\textbf{x}\|_{\infty} \leq M$. 
Suppose that the dictionary matrix has unit-norm rows  $\left\|\textbf{D}^{(i)} \right\|_2=1$,  and is pairwise incoherent with parameter $\mu_d >0$ as in Definition \ref{def:incoherence}, where  $\frac{2\mu_ds}{\sqrt{d}} \leq 0.1$.  
Let $\hat{\textbf{D}}\in \mathbb{R}^{k\times d}$ be a dictionary estimate that has unit-norm rows and is aligned with $\textbf{D}$ up to permutation and sign flips.
Suppose the dictionary estimate has bounded error such that $ \max_{i=1}^k   \left\|\hat{\textbf{D}}^{(i)} - \textbf{D}^{(i)} \right\|_2 \leq \frac{\epsilon}{\sqrt{s}}$.
Consider the constrained Lasso problem
\begin{align}
 \nonumber
\tilde{\textbf{x}} =\; &\arg\min_{\textbf{z}\in \mathbb{R}^{k}}   \|\textbf{z}\|_1, \\
\label{eq:Lasso_accuracy}
&\textmd{subject to } \left \|\textbf{y} - \textbf{z}\hat{\textbf{D}} \right\|_2 \leq \sqrt{s}M\epsilon,
\end{align}
and obtain $ \hat{\textbf{x}}$ by thresholding each entry  of $ \tilde{\textbf{x}}$ as
\begin{equation}
\label{eq:thresholding_function}
\hat{x}_i =\left\{
\begin{array}{ll}
 \tilde{x}_i, &  \textmd{if }  \left|\tilde{x}_i \right| \geq
 8.6\sqrt{s}M\epsilon, \\
0,   &       \textmd{otherwise}. 
\end{array}
\right.
\end{equation}
If $ \epsilon \leq \min\left(\frac{1}{40}, \frac{m }{17.2M \sqrt{s}} \right)$, then   $supp(\hat{\textbf{x}}) = supp(\textbf{x}) $  and 
\begin{equation}
\left\|  \hat{\textbf{x}}- \textbf{x} \right\|_{\infty} \leq 8.5\sqrt{s}M\epsilon.
\end{equation}
\end{lemma}

\subsubsection{A Supporting Lemma}

\begin{lemma}[RIC Perturbation Bound] \label{RIC_bound}
Given a matrix $\textbf{D}\in R^{k\times d}$ whose rows  have unit $l_2$-norm, i.e., $\left\|\textbf{D}^{(i)} \right\|_2=1$. 
Let $\delta_s$ denote its restricted isometry constant with  $s\in [k]$. 
Suppose $\textbf{D}$ is perturbed to $\hat{\textbf{D}}\in R^{k\times d}$   without row sign flips,  satisfying $ \max_{i=1}^k   \left \|\hat{\textbf{D}}^{(i)}- \textbf{D}^{(i)} \right\|_2 \leq \epsilon$.  
Let $\hat{\delta}_s$ denote the changed restricted isometry constant. 
Then it  has
\begin{equation}
\hat{\delta}_s  \leq   \delta_s +    2\epsilon\sqrt{2s(1-\delta_s)}- 2s\epsilon^2 .  
\end{equation}
\end{lemma}

\begin{proof}
  For any row vector $\textbf{w}  \in R^{k}$ with at most $s$ nonzero elements,   denote its index set of the nonzero elements by $I= supp(\textbf{w} )$ with $|I|\leq s$. 
  We analyse  the quantity $\|\textbf{w} \hat{\textbf{D}}\|_2 = \left \|\textbf{w}_I\hat{\textbf{D}}^{(I)} \right\|_2$ through examining the singular values of  $\hat{\textbf{D}}^{(I)}$ and applying the definition of restricted isometry constant   in Definition \ref{def:RIC}:
\begin{align}
\nonumber
\sigma_{\max} \left (\hat{\textbf{D}}^{(I)} \right) \leq \; & \sigma_{\max} \left (\textbf{D}^{(I)} \right)  + \sigma_{\max} \left( \hat{\textbf{D}}^{(I)}-\textbf{D}^{(I)}\right)   \leq \sigma_{\max} \left (\textbf{D} \right)    +\left\|  \hat{\textbf{D}}^{(I)} -\textbf{D}^{(I)}   \right \|_2   \\
\leq \;& \sqrt{1+\delta_s}  + \left\|  \hat{\textbf{D}}^{(I)}-\textbf{D}^{(I)}  \right \|_F \leq \sqrt{1+\delta_s}  + \sqrt{s}\max_{i\in I} \left\|  \hat{\textbf{D}}^{(i)}-\textbf{D}^{(i)}  \right \|_2.
\end{align}
Applying Lemma \ref{dist_lemma}, the above singular value can be further bounded by
\begin{equation}
\sigma_{\max} \left (\hat{\textbf{D}}^{(I)} \right) \leq   \sqrt{1+\delta_s}  + \sqrt{2s} \textmd{dist}\left(\textbf{D}, \hat{\textbf{D}}\right)   \leq   \sqrt{1+\delta_s}  + \sqrt{2s}\epsilon.
\end{equation}
Similarly, we can derive that  
\begin{equation}
\sigma_{\min} \left (\hat{\textbf{D}}^{(I)} \right) \geq   \sigma_{\min} \left (\textbf{D}^{(I)} \right)  - \sigma_{\max} \left( \hat{\textbf{D}}^{(I)}-\textbf{D}^{(I)}\right)     \geq  \sqrt{1- \delta_s}  - \sqrt{2s} \epsilon .
\end{equation}

Define an $s\times s$ matrix as the Gram matrix of $\hat{\textbf{D}}^{(I)} $ and denote it by $\textbf{G}_I$.
All of its eigenvalues are within the interval
$\left[\left( \sqrt{1-\delta_s} - \sqrt{2s}\epsilon \right)^2, \left(\sqrt{1+\delta_s}  + \sqrt{2s}\epsilon \right)^2\right]$.  According to Definition \ref{def:RIC},  $\hat{\delta}_s$ is the smallest number such that the following holds
\begin{align}
\nonumber
\left[1-\hat{\delta}_s, 1+ \hat{\delta}_s\right] \subseteq \; &\left[\left( \sqrt{1-\delta_s} - \sqrt{2s}\epsilon \right)^2, \left(\sqrt{1+\delta_s}  + \sqrt{2s}\epsilon \right)^2\right]\\
= \; & \left[1- \left(\underbrace{\delta_s  - 2s\epsilon^2 + 2\epsilon\sqrt{2s(1-\delta_s)}}_a\right), 1+\underbrace{\delta_s  + 2s\epsilon^2 + 2\epsilon\sqrt{2s(1+\delta_s)}}_b\right].
\end{align}
Because, in order for $\left[1-\hat{\delta}_s, 1+ \hat{\delta}_s\right] \subseteq [1-a, 1+b]$ to hold, it should have $\hat{\delta}_s \leq \min(a,b)$, therefore it has
\begin{equation}
\nonumber
\hat{\delta}_s \leq  \delta_s  + 2\sqrt{2s\epsilon^2(1-\delta_s)}- 2s\epsilon^2 .
\end{equation}
\end{proof}

\subsubsection{Main Proof}

\begin{proof}
The dictionary matrix  $\textbf{D}$ is $\mu_d$-incoherent, which means,  $\textmd{ for }i,j \in  [k] $, we have $\left|\left\langle\bm D^{(i)}, \bm D^{(j)} \right\rangle\right| \leq \frac{\mu_d}{\sqrt{d}}$. For any row vector $\textbf{w}  \in R^{k}$ with at most $s$ nonzero elements, let $I$ denote its index set of the nonzero elements. We have 
\begin{align}
\nonumber
\|\textbf{w} \textbf{D}\|_2^2 =\; &  \left\|\sum_{i=1}^k w_i \textbf{D}^{(i)} \right\|_2^2 \leq \sum_{i=1}^k w_i^2 \left\|\textbf{D}^{(i)} \right\|_2^2 + \sum_{i\neq j}|w_iw_j |\left|\left\langle\bm D^{(i)}, \bm D^{(j)} \right\rangle\right| \\
\nonumber
\leq \; &\sum_{i=1}^k w_i^2  + \frac{\mu_d}{\sqrt{d}}\sum_{i\neq j}|w_iw_j | \leq \left\|\textbf{w}\right\|_2^2  + \frac{\mu_d}{\sqrt{d}}\|\textbf{w}\|_1^2 \\
= \;  &\left(1+\frac{\mu_d}{\sqrt{d}} \frac{\left\|\textbf{w}\right\|_1^2}{\left\|\textbf{w}\right\|_2^2}\right) \left\|\textbf{w}\right\|_2^2  = \left(1+\frac{\mu_d}{\sqrt{d}} \frac{\left\|\textbf{w}_I\right\|_1^2}{\left\|\textbf{w}_I\right\|_2^2}\right) \left\|\textbf{w}\right\|_2^2 \leq  \left(1+\frac{\mu_ds}{\sqrt{d}} \right) \left\|\textbf{w}\right\|_2^2,
\end{align}
where  the last inequality results from $\|\textbf{w}_I\|_1\leq \sqrt{s}\|\textbf{w}_I\|_2$. Similarly,  we have
\begin{equation}
\|\textbf{w} \textbf{D}\|_2^2 \geq \sum_{i=1}^k w_i^2 \left\|\textbf{D}^{(i)} \right\|_2^2 - \sum_{i\neq j}|w_iw_j |\left|\left\langle\bm D^{(i)}, \bm D^{(j)} \right\rangle\right|  \geq \left(1-\frac{\mu_ds}{\sqrt{d}} \right) \left\|\textbf{w}\right\|_2^2.
\end{equation}
So, the restricted isometry constant of $\textbf{D}$ satisfies $\delta_s \leq \frac{\mu_ds}{\sqrt{d}} $.  
The assumption $\frac{2\mu_ds}{\sqrt{d}}  \leq 0.1$  indicates $\delta_{2s}  \leq 0.1$.
Denoting the restricted isometry constant of $\hat{\textbf{D}}$ by $\hat{\delta}_{s}$, applying Lemma \ref{RIC_bound} with $2s$ and $ \max_{i=1}^k   \left \|\hat{\textbf{D}}^{(i)}- \textbf{D}^{(i)} \right\|_2 \leq \frac{\epsilon}{\sqrt{s}}$, and assuming $\epsilon \leq \frac{1}{40}$, it has
\begin{equation}
  \hat{\delta}_{2s}    \leq \delta_{2s}  +    4\epsilon\sqrt{ 1-\delta_{2s} } - 4 \epsilon^2  \leq  \delta_{2s}    +  4\epsilon  - 4\epsilon^2  \leq 0.1 + 0.1 -4 \epsilon^2  <0.2,
\end{equation}
which enables the special case of Theorem \ref{Candes}, where $C=8.5$ suffices.
For the observed vector $\textbf{y}$, we have  $\textbf{y} =  \textbf{x}\hat{\textbf{D}} +  \textbf{x}\left(\hat{\textbf{D}}-\textbf{D}\right)$. 
Following the definition of $\textbf{e}$ in Theorem \ref{Candes}, the fact that $\textbf{x}$ contains at most $s$ nonzero elements, and $ \max_{i=1}^k   \left \|\hat{\textbf{D}}^{(i)}- \textbf{D}^{(i)} \right\|_2 \leq \frac{\epsilon}{\sqrt{s}}$, it has 
\begin{equation}
\|\textbf{e}\|_2  = \left\| \textbf{x}\left(\hat{\textbf{D}}-\textbf{D}\right) \right\|_2 \leq  \sum_{i=1}^k |x_i|\left\| \hat{\textbf{D}}^{(i)}-\textbf{D}^{(i)} \right\|_2  \leq   \sqrt{s}M\epsilon.
\end{equation}
Applying Theorem \ref{Candes} with $C=8.5$,  the solution $\tilde{\textbf{x}} = \arg\min_{ \left \|\textbf{y} - \textbf{z}\hat{\textbf{D}} \right\|_2 \leq \sqrt{s}M\epsilon} \; \|\textbf{z}\|_1 $ satisfies
\begin{equation}
\left\|\tilde{\textbf{x}}- \textbf{x}\right\|_{\infty} \leq \left\|\tilde{\textbf{x}}- \textbf{x}\right\|_2 \leq 8.5\sqrt{s}M\epsilon.
\end{equation}

Subsequently, for each zero element of $\textbf{x}$, its Lasso estimation satisfies $|\tilde{x}_i| \leq 8.5\sqrt{s}M\epsilon$.
For each nonzero element of $ \textbf{x}$, its Lasso estimation  satisfies $ \left|\tilde{x}_i\right| \geq   \left|x_i\right| -  \left| \tilde{x}_i -x_i \right | \geq m - 8.5\sqrt{s}M\epsilon \geq 8.7\sqrt{s}M\epsilon$, 
under the assumption $\epsilon\sqrt{s} \leq \frac{m }{17.2M} $.
As a result, after thresholding by Eq. (\ref{eq:thresholding_function}) with  an engineered threshold of $8.6\sqrt{s}M\epsilon$,   the estimations of the zero elements in $ \textbf{x}$  become zero while those estimations  of the nonzero elements in $ \textbf{x}$ remain unchanged. Thus $supp(\hat{\textbf{x}}) = supp(\textbf{x}) $, and meanwhile,
\begin{equation}
\left\|\hat{\textbf{x}}- \textbf{x}\right\|_{\infty} \leq \left\|\tilde{\textbf{x}}- \textbf{x}\right\|_{\infty} \leq 8.5\sqrt{s}M\epsilon.
\end{equation}
To summarise, the above result requires two conditions $\epsilon \leq \frac{1}{40 }$ and $\epsilon\sqrt{s} \leq \frac{m}{17.2M} $ to hold simultaneously, and this is expressed as $\epsilon \leq \min\left(\frac{1}{40 }, \frac{m }{17.2M\sqrt{s } } \right)$.
\end{proof}

\subsection{Supporting Lemma \ref{lemma:quantity1} and Its Proof}
\label{app:support_main1}

\begin{lemma}[On $\left\|\left(\hat{\textbf{X}} _{A}^T \bm\Delta _{A}\right)^{(i)}_{\setminus i} \right\|_2 $ and $\left\|\left(\hat{\textbf{X}} _{B}^T \bm\Delta _{B}\right)^{(i)}_{\setminus i} \right\|_2$] 
\label{lemma:quantity1}
 Suppose Assumptions  \ref{SC}-\ref{LI} hold, also    $supp(\bm\Delta _{A})  \subseteq supp(\textbf{X}_A)$  and $supp(\bm\Delta _{B})  \subseteq supp(\textbf{X}_B)$. 
 There exist universal constants $\tilde{C}_A, \tilde{C}_B>0$ such that the following inequalities hold 
 \begin{align}
\left\|\left(\hat{\textbf{X}} _{A}^T \bm\Delta _{A}\right)^{(i)}_{\setminus i} \right\|_2  \leq \; & \left\|\bm\Delta _{A}\right\|_{\infty}  \sqrt{\frac{  6 n s_A^3}{k^2}}\left(\left\| ( \textbf{X}_A)_i \right\|_2 +  2\left\|  \bm\Delta_A \right\|_{\infty}\sqrt{\frac{ns_A^2}{k}} \right),  \\
\left\|\left(\hat{\textbf{X}} _{B}^T \bm\Delta _{B}\right)^{(i)}_{\setminus i} \right\|_2  \leq \; & \left\|\bm\Delta _{B}\right\|_{\infty}  \sqrt{\frac{  6 m s_B^3}{k^2}}\left(\left\| ( \textbf{X}_B)_i \right\|_2 +  2 \left\|  \bm\Delta_B \right\|_{\infty}\sqrt{\frac{ms_B^2}{k}} \right),
\end{align}
with   probabilities at least $1-ke^{-\frac{\tilde{C}_An}{ks_A}} -  ke^{-\frac{\tilde{C}_A\left\lfloor\frac{s_An}{2k} \right\rfloor}{ks_A}} -2ke^{-\frac{  ns_A}{16k}}$  and $1 -ke^{-\frac{\tilde{C}_Bm}{ks_B}} -  ke^{-\frac{\tilde{C}_B\left\lfloor\frac{s_Bm}{2k} \right\rfloor}{ks_B}} -2ke^{-\frac{  ms_B}{16k}}$, respectively.
\end{lemma}

\subsubsection{A Supporting Lemma}

It is well known that the spectral norm of a submatrix  is always  bounded above by that of the original matrix, i.e., $\left\|\textbf{Y}^{(I)}\right\|_2 \leq \left\|\textbf{Y}\right\|_2$  for any matrix $\textbf{Y} \in \mathbb{R}^{n\times d} $ and any index set $I  \subseteq [n]$.
Lemma \ref{SubSpecNorm} considers the more specialised case, in which  a tighter bound depends on the submatrix size, i.e., $\left\|\textbf{Y}^{(I)}\right\|_2 \leq \mu(|I|) $.
It establishes a universal upper bound   on the spectral norms of a particular family of submatrices.
\begin{lemma}[Submatrix Spectral Norm] \label{SubSpecNorm}
Let $I_i $ denote the support of the $i$-th column of a random support matrix   generated according to  Definition \ref{def:sparse:gen}, i.e., $I_i =   \{\chi_{ji}| \chi_{ji}=1, j\in [n]\} $, $\forall i\in [k]$. 
Suppose that   $\textbf{Y} \in \mathbb{R}^{n\times d} $ satisfies $\left\|\textbf{Y}^{(I)}\right\|_2 \leq \mu(|I|) $,  with probability at least $1-\theta(|I|)$, for every index set $ I\subseteq [n]$.
Then, the spectral norms of all  submatrices $\left\{\textbf{Y}^{(I_i)}\right\}_{i=1}^k$  are uniformly upper bounded by
\begin{equation}
\left\|\textbf{Y}^{(I_i)}\right\|_2 \leq \max_{ l \in \left\{\left\lceil \frac{sn}{2k} \right\rceil,\ldots \left\lfloor\frac{3sn}{2k} \right\rfloor \right\}} \mu( l),
\end{equation}
 with probability at least $1-\max_{ l \in \left\{\left\lceil \frac{sn}{2k} \right\rceil,\ldots \left\lfloor\frac{3sn}{2k} \right\rfloor \right\}}\theta( l) -2ke^{-\frac{ ns}{16k}}$.
\end{lemma}

\begin{proof}
We start from the assumption that the spectral norm of the submatrix $\textbf{Y}^{(I_i)}$ is upper bounded by a quantity depending on the size of $I_i$, i.e.,
\begin{equation}
\label{original_p_res}
p\left(\left\|\textbf{Y}^{(I_i)}\right\|_2 > \mu(|  I_i |)\right)  \leq  \theta(|  I_i |).
\end{equation}
Applying Lemma  \ref{random_nonzero} with $\delta =\frac{1}{2}$,   it has 
\begin{equation}
\label{supp_size_p}
\frac{sn}{2k} \leq  |I_i| = \sum_{i=1}^n \chi_{ij} \leq  \frac{3sn}{2k},
\end{equation}
with probability at least $1-2ke^{-\frac{  ns}{16k}}$. 
Expanding $p\left(\left\|\textbf{Y}^{(I_i)}\right\|_2  > t  \right)$ using Eq. (\ref{supp_size_p}), it has
\begin{align}
\nonumber
 p\left(\left\|\textbf{Y}^{(I_i)}\right\|_2  > t  \right)   = \; &\sum_{l =1}^n  p\left(\left\|\textbf{Y}^{(I_i))}\right\|_2 > t \right)  p\left(|  I_i | = l\right) \\
\nonumber
\leq \; & \sum_{l =\left\lceil \frac{sn}{2k} \right\rceil}^{\left\lfloor\frac{3sn}{2k} \right\rfloor}   p\left(\left\|\textbf{Y}^{(I_i))}\right\|_2 > t \right)   p\left(|  I_i | = l \right)   +   \sum_{l=1}^{ l \leq \left\lceil \frac{sn}{2k} \right\rceil-1}  p  \left(|  I_i | = l \right)  + \sum_{l=\left\lfloor\frac{3sn}{2k} \right\rfloor +1}^{ n  }   p \left(|  I_i | = l \right) \\
 \label{result_p1}
\leq  \;  & \sum_{l =\left\lceil \frac{sn}{2k} \right\rceil}^{\left\lfloor\frac{3sn}{2k} \right\rfloor}   p\left(\left\|\textbf{Y}^{(I_i))}\right\|_2 > t \right)   p\left(|  I_i | =  l \right)    +2ke^{-\frac{  ns }{16k}}.
\end{align}
Letting $t=\max_{ l \in \left\{\left\lceil \frac{sn}{2k} \right\rceil,\ldots \left\lfloor\frac{3sn}{2k} \right\rfloor \right\}} \mu( l)$,   Eq.  (\ref{result_p1}) results in 
\begin{align}
\nonumber
  p\left(\left\|\textbf{Y}^{(I_i)}\right\|_2  > \max_{ l \in \left\{\left\lceil \frac{sn}{2k} \right\rceil,\ldots \left\lfloor\frac{3sn}{2k} \right\rfloor \right\}} \mu( l) \right)   \leq \;&  \sum_{l =\left\lceil \frac{sn}{2k} \right\rceil}^{\left\lfloor\frac{3sn}{2k} \right\rfloor}  p\left(\left\|\textbf{Y}^{(I_i)}\right\|_2  > \mu(l) \right) p\left(|  I_i | = l\right)  +2ke^{-\frac{  ns}{16k}}   \\
  \nonumber
 \leq\;  &  \sum_{l =\left\lceil \frac{sn}{2k} \right\rceil}^{\left\lfloor\frac{3sn}{2k} \right\rfloor}  \theta(l) p\left(|  I_i | = l\right)  +2ke^{-\frac{  ns}{16k}}  \\ 
 \leq\; & \max_{ l \in \left\{\left\lceil \frac{sn}{2k} \right\rceil,\ldots \left\lfloor\frac{3sn}{2k} \right\rfloor \right\}}\theta( l) +2ke^{-\frac{  ns}{16k}}.
\end{align}
This directly leads to the final result, and completes the proof.

\end{proof}

\subsubsection{Main Proof}

\begin{proof}
Define an $n\times n$ diagonal matrix $\textbf{D}_i$, where its $j$-th diagonal element is equal to 1 if   $(X_A)_{ji} \neq 0 $ while   0 otherwise. 
Applying Lemma \ref{random_spars} under the assumption of $supp(\bm\Delta_{A})  \subseteq supp(\textbf{X}_A)$,  the following holds 
\begin{align}
\nonumber
\left\|\left(\hat{\textbf{X}}_{A}^T \bm\Delta_{A}\right)^{(i)}_{\setminus i} \right\|_2  =  \; & \left\|\left(\hat{\textbf{X}}_{A}^T \textbf{D}_i\bm\Delta_{A}\right)^{(i)}_{\setminus i} \right\|_2 \\
\nonumber
\leq\; & \left\| \left(\hat{\textbf{X}}_{A}\right)_i \right\|_2\left\| \left(\textbf{D}_i \bm\Delta_{A}\right)_{\setminus i}  \right\|_2 \leq   \left\| (\hat{\textbf{X}}_{A})_i \right\|_2 \left\|\bm\Delta_{A}^{(supp((\textbf{X}_A)_i))}\right\|_2 \\
\nonumber
\leq \;  & \left(\left\| ( \textbf{X}_A)_i \right\|_2 +  \left\| ( \bm\Delta_{A})_i \right\|_2 \right) \left\|\bm\Delta_{A}^{(supp((\textbf{X}_A)_i))}\right\|_2 \\
\label{eq:result1}
\leq \;  & \left(\left\| ( \textbf{X}_A)_i \right\|_2 +  2\left\|  \bm\Delta_{A} \right\|_{\infty}\sqrt{\frac{ns_A^2}{k}} \right) \left\|\bm\Delta_{A}^{(supp((\textbf{X}_A)_i))}\right\|_2,
\end{align}
with probability at least $1-ke^{-\frac{Cn}{ks_A}}$.
Given an  index set $I\subseteq [n]$, applying Lemma \ref{random_spars},   the following holds 
\begin{equation}
\left\|\bm\Delta_{A}^{(I)}\right\|_2 \leq 2\left\|\bm\Delta_{A}^{(I)}\right\|_{\infty} \sqrt{\frac{|I|s_A^2}{k}} \leq  2\left\|\bm\Delta_{A}\right\|_{\infty} \sqrt{\frac{| I|s_A^2}{k}},
\end{equation}
with probability at least $1-ke^{\frac{C| I|}{ks_A}}$, which results in the following two functions to be used by Lemma  \ref{SubSpecNorm}:
\begin{align}
\mu(|I|)  =\; & 2\left\|\bm\Delta_{A}\right\|_{\infty} \sqrt{\frac{|I|s_A^2}{k}}, \\
\theta(|I|)  =\; & ke^{-\frac{C| I|}{ks_A}}.
\end{align}
Applying   Lemma  \ref{SubSpecNorm} with $I  = supp((\textbf{X}_A)_i)$,   it has
\begin{equation}
p\left(\left\|\bm\Delta_{A}^{(supp((\textbf{X}_A)_i))}\right\|_2  \leq \max_{ l \in \left\{\left\lceil \frac{sn}{2k} \right\rceil,\ldots \left\lfloor\frac{3sn}{2k} \right\rfloor \right\}} \mu( l) = 2 \left\|\bm\Delta_{A}\right\|_{\infty} \sqrt{\frac{ \left\lfloor\frac{3s_An}{2k} \right\rfloor s_A^2}{k}} \right) \geq 1-ke^{-\frac{C\left\lfloor\frac{s_An}{2k} \right\rfloor}{ks_A}} -2ke^{-\frac{  ns_A}{16k}}.
\end{equation}
The above holds when further relaxing the left side to 
\begin{equation}
p\left(\left\|\bm\Delta_{A}^{(supp((\textbf{X}_A)_i))}\right\|_2  \leq   2\left\|\bm\Delta_{A}\right\|_{\infty} \sqrt{\frac{  \frac{3s_An}{2k} \times   s_A^2}{k}} \right) =  p\left(\left\|\bm\Delta_{A}^{(supp((\textbf{X}_A)_i))}\right\|_2  \leq  \left\|\bm\Delta_{A}\right\|_{\infty} \sqrt{\frac{  6 n s_A^3}{k^2}} \right).
\end{equation}
Combining this with Eq. (\ref{eq:result1}), with probability at least $1 -ke^{-\frac{Cn}{ks_A}} - ke^{-\frac{C\left\lfloor\frac{s_An}{2k} \right\rfloor}{ks_A}} -2ke^{-\frac{  ns_A}{16k}}$, the following holds
\begin{equation}
\left\|\left(\hat{\textbf{X}}_{A}^T \bm\Delta_{A}\right)^{(i)}_{\setminus i} \right\|_2  \leq \left\|\bm\Delta_{A}\right\|_{\infty}  \sqrt{\frac{  6 n s_A^3}{k^2}}
\left(\left\| ( \textbf{X}_A)_i \right\|_2 +  2\left\|  \bm\Delta_{A} \right\|_{\infty}\sqrt{\frac{ns_A^2}{k}} \right).
\end{equation}
Replace the constant notation $C$ with $\tilde{C}_A$ and the proof is completed.
\end{proof}

\subsection{Supporting Lemma \ref{lemma:quantity2} and Its Proof}
\label{app:support_main2}

\begin{lemma}[On $\left\|\left(\hat{\textbf{X}}_{A_{t}}\right)_i^T\left(\hat{\textbf{X}}_{A_{t}}\right)_{\setminus i}\right\|_2$ and $\left\|\left(\hat{\textbf{X}}_{B_{t}}\right)_i^T\left(\hat{\textbf{X}}_{B_{t}}\right)_{\setminus i}\right\|_2$] \label{lemma:quantity2}
Suppose Assumptions  \ref{SC}-\ref{LI} hold, also     $supp(\bm\Delta_{A_{t}})  \subseteq supp(\textbf{X}_A)$  and $supp(\bm\Delta_{B_{t}})  \subseteq supp(\textbf{X}_B)$.  
For any $0<\Delta<1$ and $0<\Delta_X <  \frac{l_s(1-\Delta)\min(s_A, s_B)}{u_s(1+\Delta)\max(s_A, s_B)} $,  the following inequalities hold 
\begin{align}
\; & \left\| \left(\hat{\textbf{X}}_{A_{t}}\right)_i^T\left(\hat{\textbf{X}}_{A_{t}}\right)_{\setminus i}\right\|_2  
\leq   \left(\left\|\bm\Delta_{A}\right\|_{\infty} \sqrt{\frac{  6 n s_A^3}{k^2}} +  \sigma_X \sqrt{\frac{1.5s_A}{k}}   \right) 
\left(\left\| ( \textbf{X}_A)_i \right\|_2 +  2\left\|  \bm\Delta_{A} \right\|_{\infty}\sqrt{\frac{ns_A^2}{k}} \right),   \\
\; & \left\| \left(\hat{\textbf{X}}_{B_{t}}\right)_i^T\left(\hat{\textbf{X}}_{B_{t}}\right)_{\setminus i}\right\|_2  \leq \left(\left\|\bm\Delta_{B}\right\|_{\infty} \sqrt{\frac{  6 m s_B^3}{k^2}} +  \sigma_X \sqrt{\frac{1.5s_B}{k}}   \right) 
\left(\left\| ( \textbf{X}_B)_i \right\|_2 +  2\left\|  \bm\Delta_{B} \right\|_{\infty}\sqrt{\frac{ms_B^2}{k}} \right),  
\end{align}
 with probabilities at least $p_A$ and $p_B$, respectively, where
\begin{align}
p_A= \; & 1-   ke^{-\frac{C_B\Delta^2  \sigma_B^2m}{k  {M_{\max}^{(B)}}^2}} - ke^{-\frac{ \hat{C}_Al_s^2 \sigma_A^2 (1- \Delta ) \Delta_X^2 ns_B   }{u_s^2 {M^{(A)}_{\max}}^2   (1+ \Delta ) k^2 }} -2ke^{-\frac{\tilde{C}_An}{ks_A}}-  ke^{-\frac{\tilde{C}_A \left\lfloor\frac{s_An}{2k} \right\rfloor}{ks_A}} -2ke^{-\frac{  ns_A}{16k}},\\
p_B = \; &  1-   ke^{-\frac{C_A\Delta^2  \sigma_A^2n}{k  {M_{\max}^{(A)}}^2}} - ke^{-\frac{ \hat{C}_Bl_s^2 \sigma_B^2 (1- \Delta ) \Delta_X^2 ms_A   }{u_s^2 {M^{(B)}_{\max}}^2   (1+ \Delta ) k^2 }} -2ke^{-\frac{\tilde{C}_Bm}{ks_B}}-  ke^{-\frac{\tilde{C}_B \left\lfloor\frac{s_Bm}{2k} \right\rfloor}{ks_B}} -2ke^{-\frac{  ms_B}{16k}}.
\end{align}

\end{lemma}

\begin{proof}
Since $ \hat{\textbf{X}}_{A} =  \textbf{X}_A- \bm\Delta_{A} $, we have 
\begin{equation} \label{eq:quantity2_res}
\left\| \left(\hat{\textbf{X}}_{A}\right)_i^T\left(\hat{\textbf{X}}_{A}\right)_{\setminus i}\right\|_2 \leq  \left\| \left(\hat{\textbf{X}}_{A}\right)_i^T \left( \textbf{X}_A\right)_{\setminus i} \right\|_2  +   \left\| \left(\hat{\textbf{X}}_{A}\right)_i^T \left( \bm\Delta_{A}\right)_{\setminus i} \right\|_2 .
\end{equation}
Lemma \ref{lemma:quantity1} provides an upper bound for the second term, now  we focus on analysing the first term.
Define $\textbf{D}_i$ the same way as in the proof for Lemma \ref{lemma:quantity1} and follow a similar derivation.  
By applying Lemma \ref{random_spars} under the assumption of  $supp(\bm\Delta_{A})  \subseteq supp(\textbf{X}_A)$,  the following holds 
\begin{align}
\label{eq_result2}
\nonumber
 \left\| \left(\hat{\textbf{X}}_{A}\right)_i^T \left( \textbf{X}_A\right)_{\setminus i} \right\|_2 \leq \; & \left\| \left(\hat{\textbf{X}}_{A}\right)_i \right\|_2\left\| \left(\textbf{D}_i \textbf{X}_A\right)_{\setminus i}  \right\|_2  \leq   \left\| \left(\hat{\textbf{X}}_{A}\right)_i \right\|_2 \left\|\textbf{X}_A^{(supp((\textbf{X}_A)_i))}\right\|_2  \\
 \leq \; & \left(\left\| ( \textbf{X}_A)_i \right\|_2 +  2\left\|  \bm\Delta_{A} \right\|_{\infty}\sqrt{\frac{ns_A^2}{k}} \right)\left\|\textbf{X}_A^{(supp((\textbf{X}_A)_i))}\right\|_2,
\end{align}
with probability at least $1-ke^{-\frac{Cn}{ks_A}}$.

For an arbitrary submatrix $\tilde{\textbf{X}}_A^{(I)}$ containing rows of $\textbf{X}_A$ but with its scaling factor changed to $\frac{1}{\sqrt{|I|m}}$ from $\frac{1}{\sqrt{nm}}$, its maximum singular value  is bounded by   Lemma \ref{singularXAXB}, i.e.,   
\begin{equation}
\label{sig_bound_scale_sub}
 \sigma_{\max}\left(\tilde{\textbf{X}}_A^{(I)}\right) \leq     \frac{u_s\sigma_A\sigma_B\sqrt{(1+\Delta)(1+\Delta_X)} \max(s_A, s_B)}{k},
\end{equation}
with probability at least  $1-   ke^{-\frac{C_B\Delta^2  \sigma_B^2m}{k  {M_{\max}^{(B)}}^2}} - ke^{-\frac{ \hat{C}_Al_s^2 \sigma_A^2 (1- \Delta ) \Delta_X^2 |I|s_B   }{u_s^2 {M^{(A)}_{\max}}^2   (1+ \Delta ) ks_A }} $. 
Going back to the submatrix $\textbf{X}_A^{(I)}$ that we are after, containing rows of $\textbf{X}_A$ with the same scaling factor $\frac{1}{\sqrt{nm}}$, its maximum singular value  is bounded by  the upper bound in Eq. (\ref{sig_bound_scale_sub}) after multiplying the scaling factor of $\sqrt{\frac{|I|}{n}}$. 
This implies the following two functions to be used by Lemma  \ref{SubSpecNorm} for the spectral number of $\textbf{X}_A^{(I)}$ :
\begin{align}
\mu(|I|)  = \;&  u_s\sigma_A\sigma_B\sqrt{(1+\Delta)(1+\Delta_X)} \max(s_A, s_B) \sqrt{\frac{|I|}{nk^2}}, \\
\theta(|I|) =\;& ke^{-\frac{C_B\Delta^2  \sigma_B^2m}{k  {M_{\max}^{(B)}}^2}} + ke^{-\frac{ \hat{C}_Al_s^2 \sigma_A^2 (1- \Delta ) \Delta_X^2 |I|s_B   }{u_s^2 {M^{(A)}_{\max}}^2   (1+ \Delta ) ks_A }} .
\end{align}
Applying   Lemma  \ref{SubSpecNorm},   we have
\begin{align}
\nonumber
& p\left( \left\|\textbf{X}_A^{(supp((\textbf{X}_A)_i))}\right\|_2  \leq    u_s\sigma_A\sigma_B\sqrt{(1+\Delta)(1+\Delta_X)} \max(s_A, s_B) \sqrt{\frac{1.5s_A}{k^3}}  \right) \\
\geq\;&  1- ke^{-\frac{C_B\Delta^2  \sigma_B^2m}{k  {M_{\max}^{(B)}}^2}} - ke^{-\frac{ 0.5\hat{C}_Al_s^2 \sigma_A^2 (1- \Delta ) \Delta_X^2 ns_B   }{u_s^2 {M^{(A)}_{\max}}^2   (1+ \Delta )  k^2 }} -2ke^{-\frac{  ns_A}{16k}}.
\end{align}

Combining it with the result in Eq. (\ref{eq_result2}), it has
\begin{equation}
 \left\| \left(\hat{\textbf{X}}_{A}\right)_i^T \left( \textbf{X}_{A}\right)_{\setminus i} \right\|_2 \leq   \sigma_X \sqrt{\frac{1.5s_A}{k}} \left(\left\| ( \textbf{X}_A)_i \right\|_2 +  2\left\|  \bm\Delta_{A} \right\|_{\infty}\sqrt{\frac{ns_A^2}{k}} \right),
\end{equation}
with probability at least
\begin{equation}
p=   1- ke^{-\frac{C_B\Delta^2  \sigma_B^2m}{k  {M_{\max}^{(B)}}^2}} - ke^{-\frac{ 0.5\hat{C}_Al_s^2 \sigma_A^2 (1- \Delta ) \Delta_X^2 n s_B   }{u_s^2 {M^{(A)}_{\max}}^2   (1+ \Delta ) k^2  }} -2ke^{-\frac{  ns_A}{16k}} -ke^{-\frac{\tilde{C}_An}{ks_A}},
\end{equation}
where the constant notation $C$ is replaced by $\tilde{C}_A$.
Finally, adding result from Lemma \ref{lemma:quantity1}, the quantity in Eq. (\ref{eq:quantity2_res})  is bounded by 
\begin{equation}
 \left\| \left(\hat{\textbf{X}}_{A}\right)_i^T\left(\hat{\textbf{X}}_{A}\right)_{\setminus i}\right\|_2  \leq \left(\left\|\bm\Delta_{A}\right\|_{\infty} \sqrt{\frac{  6 n s_A^3}{k^2}} +  \sigma_X\sqrt{\frac{1.5s_A}{k}}   \right) 
\left(\left\| ( \textbf{X}_A)_i \right\|_2 +  2\left\|  \bm\Delta_{A} \right\|_{\infty}\sqrt{\frac{ns_A^2}{k}} \right),  
\end{equation}
with probability at least
\begin{equation}
p_A= 1-   ke^{-\frac{C_B\Delta^2  \sigma_B^2m}{k  {M_{\max}^{(B)}}^2}} - ke^{-\frac{ 0.5\hat{C}_Al_s^2 \sigma_A^2 (1- \Delta ) \Delta_X^2 n s_B   }{u_s^2 {M^{(A)}_{\max}}^2   (1+ \Delta ) k^2  }} -2ke^{-\frac{\tilde{C}_An}{ks_A}}-  ke^{-\frac{\tilde{C}_A \left\lfloor\frac{s_An}{2k} \right\rfloor}{ks_A}} -2ke^{-\frac{  ns_A}{16k}}.
\end{equation}
This completes the proof.

\end{proof}

\subsection{Main Proof}
\begin{proof}
We provide proof for $\textbf{Y}_U$ case  and the same applies to $\textbf{Y}_V$.
At iteration $t$, the auxiliary problem to solve is $\textbf{Y}_U = \textbf{X}_{A}\textbf{D}_{B_t}$, for which we permute and sign flip the coefficient and dictionary estimation to match the ground truth.
Define $\tilde{\epsilon}_{B_{t-1}}^D $ and $\tilde{\epsilon}_{A_{t-1}}^D $ such that 
\begin{equation}
    \epsilon_{B_{t-1}}^D =\frac{\tilde{\epsilon}_{B_{t-1}}^D }{\sqrt{s_A}},\;   \epsilon_{A_{t-1}}^D =\frac{\tilde{\epsilon}_{A_{t-1}}^D }{\sqrt{s_B}}.
\end{equation}
For each row of the observation matrix $\textbf{Y}_U^{(i)}$, Algorithm \ref{alg:dl}  solves the following Lasso problem:
 \begin{align}
 \label{app:lasso}
 & \min_{\textbf{z}\in \mathbb{R}^{k}}   \|\textbf{z}\|_1, \\
\nonumber
\; &\textmd{subject to } \left \|\textbf{Y}_U^{(i)} - \textbf{z}\hat{\textbf{D}}_{B_{t-1}} \right\|_2 \leq \sqrt{s_A}M_A\tilde{\epsilon}_{B_{t-1}}^D.
\end{align}
Then, it obtains $\hat{ \textbf{X}}_{A_t}^{(i)}$ by thresholding the resulting solution,  keeping elements with amplitude greater than  $8.6\sqrt{s_A}M_A\tilde{\epsilon}_{B_{t-1}}^D$. 
Our proof builds on Lemma \ref{dic_error},   analysing  and upper bounding each contributing term in the dictionary error in Eq. (\ref{eq:dic_error1}).
Defining $G_1= \left\|\left(\hat{\textbf{X}}_{A_t}^T\hat{\textbf{X}}_{A_t}\right)^{-1}\right\|_2 $, $G_2 = \left\|\left(\hat{\textbf{X}}_{A_{t}}^T \bm\Delta_{A_{t}}\right)^{(i)}_{\setminus i} \right\|_2 $, $G_3=\left\|\left(\hat{\textbf{X}}_{A_{t}}\right)_i^T\left(\hat{\textbf{X}}_{A_{t}}\right)_{\setminus i}\right\|_2$,  $G_4 = \left\| \hat{\textbf{X}}_{A_t} \right\|_2  \left\| \bm\Delta_{A_t} \right\|_2$, $T_1 =G_1G_4$, and $T_2 =  G_1G_2+ G_4G_1^2G_3$, Lemma \ref{dic_error} results in
\begin{equation}
\label{dist_bound0}
\textmd{dist}\left(\hat{\textbf{D}}_{B_{t}}^{(i)}, \textbf{D}_{B}^{(i)} \right) \leq  \frac{   T_2 \|\textbf{D}_B\|_2}{1-  T_1 -  T_2  \|\textbf{D}_{B}\|_2 },
\end{equation}

Firstly, we analyse the coefficient error $\left\| \bm\Delta_{A_{t}}\right\|_{\infty} = \left\|\hat{\textbf{X}}_{A_t} - \textbf{X}_{A}\right\|_{\infty} $.
Following the concept of pairwise incoherence in Definition \ref{def:incoherence} and according to Lemma \ref{MC_DaDb},  $\textbf{D}_{B}$ is pairwise incoherent with    $\mu_d = \mu_D\sqrt{d} $ with a  probability at least  $1-ke^{-\frac{C_B\Delta^2  \sigma_B^2m}{k {M_{\max}^{(B)}}^2}}$.
Applying Lemma \ref{coeff_error},   the coefficient estimation satisfies the following
\begin{equation}
\label{lasso_lemma_res1}
supp\left(\hat{\textbf{X}}_{A_t}^{(i)}  \right) = supp\left(   \textbf{X}_{A}^{(i)}\right), \;\forall i\in [n],
\end{equation}
and the estimation error is upper bounded by
\begin{equation}
\label{lasso_lemma_res2}
\left\| \bm\Delta_{A_{t}}\right\|_{\infty} =\max_{i=1}^n\left\| \hat{ \textbf{X}}_{A_t}^{(i)} -  \textbf{X}_{A}^{(i)}\right\|_{\infty} \leq 8.5M_A\sqrt{s_A}\tilde{\epsilon}_{B_{t-1}}^D, 
\end{equation}
under the error requirement of 
\begin{equation}
     \epsilon_{B_{t-1}}^D = \max_{i=1}^k   \left\|\left(\hat{\textbf{D}}^{(B)}\right)^{(i)} - \left(\textbf{D}^{(B)}\right)^{(i)} \right\|_2 \leq \frac{\tilde{\epsilon}_{B_{t-1}}^D}{\sqrt{s_A}},
\end{equation} 
as well  as the two conditions of $\frac{2\mu_ds_A}{\sqrt{d}} \leq 0.1$ (equivalently $2\mu_Ds_A\leq 0.1$) and  $\tilde{\epsilon}_{B_{t-1}}^D \leq \min\left(\frac{1}{40}, \frac{m_A}{17.2M_A\sqrt{s_A}}\right)$.
Rewriting these two conditions, one imposes a sparsity requirement   
\begin{equation}
\label{eq:sparsity_cond1}
    s_A  \leq \frac{0.05}{ \mu_D},
\end{equation}
while the other requires the dictionary error obtained from the last iteration to be sufficiently small such that 
\begin{equation}
\label{dic_error0}
\tilde{\epsilon}_{B_{t-1}}^D  \leq \min\left(\frac{1}{40}, \frac{M_{\min}^{(A)}l_s}{17.2M_{\max}^{(A)}u_s\sqrt{s_A}}\sqrt{\frac{1-\Delta}{1+\Delta}}\right),
\end{equation}
Applying Eq. (\ref{bound_maxA})  to expand $M_A$, Eq. (\ref{lasso_lemma_res2}) becomes
 \begin{equation}
     \label{delta_norm0}
\left\| \bm\Delta_{A_{t}}\right\|_{\infty} \leq  8.5\tilde{\epsilon}_{B_{t-1}}^D  M^{(A)}_{\max}   u_s\sigma_B\sqrt{\frac{(1+ \Delta )  s_As_B }{nk }},
 \end{equation}
under the two conditions in  Eqs. (\ref{eq:sparsity_cond1}) and (\ref{dic_error0}). 
Eq. (\ref{delta_norm0}) holds with a  probability at least   $1- 2ke^{-\frac{C_B\Delta^2  \sigma_B^2m}{k {M_{\max}^{(B)}}^2}}$ which follows from the results of Lemma \ref{MC_DaDb} and Corollary \ref{RowLengthFhat}.

According to Eq. (\ref{lasso_lemma_res1}), we have   $supp\left(\hat{\textbf{X}}_{A_t}^{(i)} -  \textbf{X}_{A}^{(i)}\right) \subseteq supp\left(   \textbf{X}_{A}^{(i)}\right)$, consequently, it has $supp\left(\bm\Delta_{A_{t}}\right) \subseteq supp\left(   \textbf{X}_{A} \right)$.
This enables  the application of  Lemma \ref{random_spars} over $\left\|\bm\Delta_{A_{t}}\right\|_{2} $, leading to $ \left\|\bm\Delta_{A_{t}}\right\|_2 \leq 2 \sqrt{\frac{ns_A^2}{k}} \left\|\bm\Delta_{A_{t}}\right\|_{\infty}$,  which, together with Eq. (\ref{eq:quantity2}), results in
\begin{align}
\nonumber
G_1=\; & \left\|\left(\hat{\textbf{X}}_{A_t}^T\hat{\textbf{X}}_{A_t}\right)^{-1}\right\|_2 \leq\left({\sigma_{\min}\left(\textbf{X}_{A}\right)}^2 -  \frac{4ns_A^2}{k}\left\|\bm\Delta_{A_{t}}  \right\|_{\infty}^2 -4 \sqrt{\frac{ns_A^2}{k}} \left\|\bm\Delta_{A_{t}}\right\|_{\infty}\|\textbf{X}_{A}\|_2 \right)^{-1} \\
\label{XAXA_inv}
=\; &{\sigma_{\min}\left(\textbf{X}_{A}\right)}^{-2} \left(1-  4\left( \frac{ \sqrt{\frac{ns_A^2}{k}}\left\|\bm\Delta_{A_{t}}\right\|_{\infty}}{\sigma_{\min}\left(  \textbf{X}_{A}\right)}\right)^2 - 4  \times  \frac{\sigma_{\max}\left(  \textbf{X}_{A}\right)}{\sigma_{\min}\left(  \textbf{X}_{A}\right)} \times \frac{ \sqrt{\frac{ns_A^2}{k}}\left\|\bm\Delta_{A_{t}}\right\|_{\infty}}{\sigma_{\min}\left(  \textbf{X}_{A}\right)}  \right)^{-1},
\end{align}
and this holds with probability at least $1-k\exp\left(-\frac{\hat{C}_An}{ks_A}\right)$.
Applying Lemma \ref{singularXAXB}, it has
\begin{equation}
\label{temp1}
    \frac{\sigma_{\max}\left(  \textbf{X}_{A}\right)}{\sigma_{\min}\left(  \textbf{X}_{A}\right)} =  \frac{\sigma_{\max}\left(  \textbf{X}_{A}\right)}{\sigma_{\min}\left(  \textbf{X}_{A}\right)} \leq  \frac{u_s\sqrt{(1+\Delta)(1+\Delta_X)} \max(s_A, s_B)}{l_s \sqrt{1-\Delta}\left(\frac{\min(s_A, s_B)}{\max(s_A,s_B)}-\frac{(1+\Delta)u_s\Delta_X}{(1-\Delta)l_s}\right)^{\frac{1}{2}}\min(s_A, s_B)} = \theta_1.
\end{equation}
Applying  Lemma \ref{singularXAXB} and Eq. (\ref{delta_norm0}), it has
\begin{equation}
\label{temp2}
     \frac{ \sqrt{\frac{ns_A^2}{k}}\left\|\bm\Delta_{A_{t}}\right\|_{\infty}}{\sigma_{\min}\left(  \textbf{X}_{A}\right)} \leq\;  \underbrace{\frac{8.5 M^{(A)}_{\max}   u_s \max(s_A, s_B) }{l_s\sigma_A \min(s_A, s_B)\left(\frac{\min(s_A, s_B)}{\max(s_A,s_B)}-\frac{(1+\Delta)u_s\Delta_X}{(1-\Delta)l_s}\right)^{\frac{1}{2}}}\sqrt{\frac{ 1+ \Delta     }{ 1- \Delta    }}}_{\theta_2^{(A)}} \tilde{\epsilon}_{B_{t-1}}^D s_A.
\end{equation}
Incorporating Eqs. (\ref{temp1}) and (\ref{temp2}) into Eq. (\ref{XAXA_inv}), it has 
\begin{equation}
\label{XAXA_inv2}
    G_1 \leq  {\sigma_{\min}\left(\textbf{X}_{A}\right)}^{-2}\left(1 -   \left(2\epsilon_{B_{t-1}}^{D } s_A  \theta_2^{(A) }\right)^2 - 4\tilde{\epsilon}_{B_{t-1}}^Ds_A\theta_1\theta_2^{(A)} \right)^{-1}.
\end{equation}
Under the following error condition  
\begin{equation}
\label{dic_error1}
    \tilde{\epsilon}_{B_{t-1}}^D \leq \min\left(\frac{1}{4 s_A\theta_2^{(A)}}, \frac{1}{16 s_A\theta_1 \theta_2^{(A)}}\right),
\end{equation}
 The quantity $G_1$ can be further bounded by
\begin{equation}
\label{XaXa_bound_final}
G_1 \leq {\sigma_{\min}\left(\textbf{X}_{A}\right)}^{-2} \left( 1- \frac{1}{4} -\frac{1}{4}   \right)^{-1} = 2\sigma_{\min}\left(\textbf{X}_{A}\right)^{-2}   .
\end{equation}
The  bound in Eq. (\ref{XaXa_bound_final}) holds with probability  resulted from Lemma \ref{singularXAXB},  also Lemma \ref{MC_DaDb} and Corollary \ref{RowLengthFhat} to obtain Eq. (\ref{delta_norm0}), as well as Lemma \ref{random_spars} to obtain  Eq. (\ref{XAXA_inv}), which is at least $1-   3ke^{-\frac{C_B\Delta^2  \sigma_B^2m}{k  {M_{\max}^{(B)}}^2}} - ke^{-\frac{ \hat{C}_Al_s^2 \sigma_A^2 (1- \Delta ) \Delta_X^2 ns_B   }{u_s^2 {M^{(A)}_{\max}}^2   (1+ \Delta ) ks_A }}-k\exp\left(-\frac{\hat{C}_An}{ks_A}\right)$.

We continue to analyse  $G_2$ and $G_3$, then $G_4$. 
Defining $G =   \left\| ( \textbf{X}_A)_i \right\|_2 +  2\left\|  \bm\Delta_{A_t} \right\|_{\infty}\sqrt{\frac{ns_A^2}{k}}$, applying Lemma \ref{lemma:quantity1} for $G_2$ and Lemma \ref{lemma:quantity2} for $G_3$,  it has
\begin{align}
\label{bound_G2}
G_2 \leq \; &    G \sqrt{\frac{  6 n s_A^3}{k^2}}  \left\|\bm\Delta_{A_{t}}\right\|_{\infty}, \\
\label{bound_G3}
G_3 \leq \; & G \left(\left\|\bm\Delta_{A}\right\|_{\infty} \sqrt{\frac{  6 n s_A^3}{k^2}} +  \sigma_X \sqrt{\frac{1.5s_A}{k}}   \right),
\end{align}
which simultaneously  hold   with probability at least the following:
\begin{equation}
    p = 1-   ke^{-\frac{C_B\Delta^2  \sigma_B^2m}{k  {M_{\max}^{(B)}}^2}} - ke^{-\frac{ \hat{C}_Al_s^2 \sigma_A^2 (1- \Delta ) \Delta_X^2 ns_B   }{u_s^2 {M^{(A)}_{\max}}^2   (1+ \Delta ) k^2 }} - 3ke^{-\frac{\tilde{C}_An}{ks_A}}-  2ke^{-\frac{\tilde{C}_A \left\lfloor\frac{s_An}{2k} \right\rfloor}{ks_A}} -4ke^{-\frac{  ns_A}{16k}}.
\end{equation}
Assuming the following dictionary error condition is satisfied
\begin{equation}
\label{dic_error2}
\tilde{\epsilon}_{B_{t-1}}^D \leq  \left( 8.5  M^{(A)}_{\max}   u_s\sigma_B\sqrt{(1+ \Delta )s_As_B}\right)^{-1} ,
\end{equation}
Eq. (\ref{delta_norm0}) can be reduced to a simpler  bound as below
 \begin{equation}
 \label{delta_norm}
\left\| \bm\Delta_{A_{t}}\right\|_{\infty}   \leq  \frac{1}{\sqrt{nk}}.
\end{equation}
Applying  Lemma \ref{columnXAXB} with $\delta = \frac{ 0.1u_s^2 \sigma_A^2\sigma_B^2 \Delta s_As_B }{k^2 }$ and  Eq. (\ref{delta_norm}),  it has
 \begin{equation}
 \label{bound_G}
  G \leq   \sqrt{  \frac{  u_s^2 \sigma_A^2\sigma_B^2 (1+1.1\Delta) s_As_B }{k^2 }} +\frac{2s_A}{k} \leq \frac{(u_s  \sigma_A \sigma_B  \sqrt{1+1.1\Delta}+2)\max(s_A,s_B)}{k},
 \end{equation}
 which holds with probability at least  
 \begin{equation}
     p=1 - 2ke^{-\frac{C_B\Delta^2  \sigma_B^2m}{k  {M_{\max}^{(B)}}^2}} - 2e^{- \frac{\frac{1}{2}nk^3\delta^2}{s_A\sigma_A^4 u_s^4\sigma_B^4(1+ \Delta )^2s_B^2 +\frac{1}{3}\left( {M^{(A)}_{\max}}^2k +s_A\sigma_A^2\right)  u_s^2\sigma_B^2(1+ \Delta )  s_B k \delta }}.
 \end{equation}
Incorporating Eqs. (\ref{bound_G})  and (\ref{delta_norm0}) into Eq. (\ref{bound_G2}) to further bound $G_2$, we have
\begin{align}
\nonumber
G_2    \leq \; &  \frac{\sqrt{6}(u_s  \sigma_A \sigma_B  \sqrt{1+1.1\Delta}+2)\max(s_A,s_B)^2}{k^2} \times \sqrt{ns_A}\left\| \bm\Delta_{A_{t}}\right\|_{\infty}  \\
\label{bound_G2_new}
< \; & \underbrace{21(u_s  \sigma_A \sigma_B  \sqrt{1+1.1\Delta}+2) M^{(A)}_{\max}   u_s\sigma_B\sqrt{1+ \Delta}}_{\theta_3^{(A)}} \frac{\max(s_A,s_B)^2}{k^2}\sqrt{\frac{s_A^2s_B}{k}}\tilde{\epsilon}_{B_{t-1}}^D. 
\end{align}
Incorporating Eqs. (\ref{bound_G})  and (\ref{delta_norm}) into Eq. (\ref{bound_G3}), it has 
 \begin{align}
\nonumber
G_3 \leq \; &  \frac{(u_s  \sigma_A \sigma_B  \sqrt{1+1.1\Delta}+2)\max(s_A,s_B)}{k}\left(  \sqrt{\frac{  6  s_A^3}{k^3}} +  \sigma_X   \sqrt{\frac{1.5s_A}{k}} \right) \\
\label{bound_G3_new}
< \; &  \underbrace{1.3(u_s  \sigma_A \sigma_B  \sqrt{1+1.1\Delta}+2)\left(2 + u_s\sigma_A\sigma_B\sqrt{(1+\Delta)(1+\Delta_X)} \right)}_{\theta_4} \frac{\max(s_A,s_B)^2}{k^2} \sqrt{\frac{   s_A }{k }}.
\end{align}
The above results on $G_2$ and $G_3$ simultaneously hold with probability at least
\begin{align}
\nonumber
    p = \;& 1 -  5 ke^{-\frac{C_B\Delta^2  \sigma_B^2m}{k  {M_{\max}^{(B)}}^2}} - ke^{-\frac{ \hat{C}_Al_s^2 \sigma_A^2 (1- \Delta ) \Delta_X^2 ns_B   }{u_s^2 {M^{(A)}_{\max}}^2   (1+ \Delta ) k^2 }} - 3ke^{-\frac{\tilde{C}_An}{k^2}}-  2ke^{-\frac{\tilde{C}_A \left\lfloor\frac{s_An}{2k} \right\rfloor}{ks_A}} -4ke^{-\frac{  ns_A}{16k}} \\
    \; &- 2e^{- \frac{\frac{1}{2}nk^3\delta^2}{s_A\sigma_A^4 u_s^4\sigma_B^4(1+ \Delta )^2s_B^2 +\frac{1}{3}\left( {M^{(A)}_{\max}}^2k +s_A\sigma_A^2\right)  u_s^2\sigma_B^2(1+ \Delta )  s_B k \delta }}.
\end{align}
This is resulted from the probabilities required for Eq. (\ref{bound_G}), Eq. (\ref{delta_norm0}) that requires the same probability as Eq. (\ref{delta_norm}), and also for Eqs. (\ref{bound_G2}) and (\ref{bound_G3}) to hold.
To analyse $G_4$, applying Lemma \ref{singularXAXB}, Lemma \ref{random_spars}, Eq. (\ref{delta_norm}) and Eq. (\ref{delta_norm0}), it has 
 \begin{align}
 \nonumber
G_4 = & \left\| \hat{\textbf{X}}_{A_t} \right\|_2  \left\| \bm\Delta_{A_t} \right\|_2 \le  \left(\left\| \textbf{X}_{A} \right\|_2 + \left\| \bm\Delta_{A_t} \right\|_2 \right) \left\| \bm\Delta_{A_t} \right\|_2 =  \left(\sigma_{\max}\left(\textbf{X}_A \right) + \left\| \bm\Delta_{A_t} \right\|_2 \right) \left\| \bm\Delta_{A_t} \right\|_2   \\
\nonumber
\leq \; & 2\left( \frac{u_s\sigma_A\sigma_B\sqrt{(1+\Delta)(1+\Delta_X)} \max(s_A, s_B)}{k}   + 2 \left\|\bm\Delta_{A_{t}}\right\|_{\infty}\sqrt{\frac{ns_A^2}{k}}  \right) \left\|\bm\Delta_{A_{t}}\right\|_{\infty}\sqrt{\frac{ns_A^2}{k}}    \\
\nonumber
\leq \; & 2\left( \frac{u_s\sigma_A\sigma_B\sqrt{(1+\Delta)(1+\Delta_X)} \max(s_A, s_B)}{k}   +     \frac{2 s_A}{k }   \right)\sqrt{\frac{ns_A^2}{k}} \left\|\bm\Delta_{A_{t}}\right\|_{\infty}  \\
\label{bound_G4}
\leq \;& \frac{2\max(s_A, s_B)}{k}\left(  u_s\sigma_A\sigma_B\sqrt{(1+\Delta)(1+\Delta_X)} +  2   \right)\sqrt{\frac{ns_A^2}{k}}\left\|\bm\Delta_{A_{t}}\right\|_{\infty},  \\
\leq \; & \underbrace{17 \left(  u_s\sigma_A\sigma_B\sqrt{(1+\Delta)(1+\Delta_X)} +  2   \right)\sqrt{ 1+\Delta}  M^{(A)}_{\max}   u_s\sigma_B}_{\theta_5^{(A)}} \frac{\max(s_A, s_B)^2s_A\tilde{\epsilon}_{B_{t-1}}^D   }{k^2},
\end{align}
which holds with probability at least 
\begin{equation}
    p = 1-   3ke^{-\frac{C_B\Delta^2  \sigma_B^2m}{k  {M_{\max}^{(B)}}^2}} - ke^{-\frac{ \hat{C}_Al_s^2 \sigma_A^2 (1- \Delta ) \Delta_X^2 ns_B   }{u_s^2 {M^{(A)}_{\max}}^2   (1+ \Delta ) ks_A }} - ke^{-\frac{\tilde{C}_An}{ks_A}} .
\end{equation}

Now we proceed to analyse the quantities  $T_1 = G_1G_4$ and $ T_2 = G_1G_2+ G_4G_1^2G_3 $ using the above derived bounds for $\{G_i\}_{i=1}^4$ and  Lemma \ref{singularXAXB}.
It has 
\begin{align}
\nonumber
T_1 \leq \; &    \frac{2\sigma_{\min}\left(\textbf{X}_{A}\right)^{-2}\theta_5^{(A)}\max(s_A, s_B)^2s_A\tilde{\epsilon}_{B_{t-1}}^D   }{k^2}, \\
\label{bound_T1_final}
\leq \; & \underbrace{\frac{2\theta_5^{(A)}\max(s_A, s_B)^2}{l_s^2\sigma_A^2\sigma_B^2 (1-\Delta)\left(\frac{\min(s_A, s_B)}{\max(s_A,s_B)}-\frac{(1+\Delta)u_s\Delta_X}{(1-\Delta)l_s}\right)  \min(s_A, s_B)^2}}_{\theta_6^{(A)}}s_A\tilde{\epsilon}_{B_{t-1}}^D,
\end{align}
also
\begin{align}
\nonumber 
 G_1G_2   < \; &    \frac{2\sigma_{\min}\left(\textbf{X}_{A}\right)^{-2}\theta_3^{(A)}\max(s_A,s_B)^2}{k^2}\sqrt{\frac{s_A^2s_B}{k}}\tilde{\epsilon}_{B_{t-1}}^D \\
 \leq\; &\underbrace{\frac{2\theta_3^{(A)}\max(s_A, s_B)^2}{l_s^2\sigma_A^2\sigma_B^2 (1-\Delta)\left(\frac{\min(s_A, s_B)}{\max(s_A,s_B)}-\frac{(1+\Delta)u_s\Delta_X}{(1-\Delta)l_s}\right)  \min(s_A, s_B)^2}}_{\theta_7^{(A)}}\sqrt{\frac{s_A^2s_B}{k}}\tilde{\epsilon}_{B_{t-1}}^D,
\end{align}
and
\begin{align}
\nonumber 
G_4G_1^2G_3   <\; &  \frac{ 4 \sigma_{\min}\left(\textbf{X}_{A}\right)^{-4}\theta_4\theta_5^{(A)} \max(s_A, s_B)^4 }{k^4} \sqrt{\frac{   s_A^3 }{k }}\tilde{\epsilon}_{B_{t-1}}^D\\
\leq\; &  \underbrace{\frac{ 4\theta_4\theta_5^{(A)}\max(s_A,s_B)^4 }{l_s^4\sigma_A^4\sigma_B^4 (1-\Delta)^2\left(\frac{\min(s_A, s_B)}{\max(s_A,s_B)}-\frac{(1+\Delta)u_s\Delta_X}{(1-\Delta)l_s}\right)^2  \min(s_A, s_B)^4}}_{\theta_8^{(A)}} \sqrt{\frac{   s_A^3 }{k }}\tilde{\epsilon}_{B_{t-1}}^D.
\end{align}
Combining the above, we have
\begin{equation}
\label{bound_T2_final}
    T_2 = G_1G_2 +G_4G_1^2G_3 < \left(\theta_7^{(A)}+\theta_8^{(A)}\right)\sqrt{\frac{\max(s_A,s_B)^3}{k}}\tilde{\epsilon}_{B_{t-1}}^D
\end{equation}
Incorporating  Eqs. (\ref{bound_T1_final}) and  (\ref{bound_T2_final}) into Eq. (\ref{dist_bound0}) and applying Lemma \ref{BS_DaDb}, it has 
\begin{align}
\nonumber
\textmd{dist}\left( \hat{\textbf{D}}_{B_{t}}^{(i)}, \textbf{D}_{B_t}^{(i)} \right) \leq \;&   \frac{   T_2 \|\textbf{D}_B\|_2}{1-  T_1 -  T_2  \|\textbf{D}_{B}\|_2 } \\
\nonumber
\leq \; & \frac{\sqrt{1+\mu_D(d-1)}\left(\theta_7^{(A)}+\theta_8^{(A)}\right)\sqrt{\frac{\max(s_A,s_B)^3}{k}}\tilde{\epsilon}_{B_{t-1}}^D}{1-\theta_6^{(A)}s_A\tilde{\epsilon}_{B_{t-1}}^D - \sqrt{1+\mu_D(d-1)}\left(\theta_7^{(A)}+\theta_8^{(A)}\right)\sqrt{\frac{\max(s_A,s_B)^3}{k}}\tilde{\epsilon}_{B_{t-1}}^D}\\
\label{dist_bound_final}
=\; & \frac{ \theta_9^{(A)}\sqrt{\frac{\max(s_A,s_B)^3}{k}}\tilde{\epsilon}_{B_{t-1}}^D}{1-\theta_6^{(A)}s_A\tilde{\epsilon}_{B_{t-1}}^D -  \theta_9^{(A)}\sqrt{\frac{\max(s_A,s_B)^3}{k}}\tilde{\epsilon}_{B_{t-1}}^D}.
\end{align}
When the following  error condition  holds
\begin{equation}
\label{final_dic_error_XaDb}
    \tilde{\epsilon}_{B_{t-1}}^D \leq \min\left(\frac{1}{5\theta_6^{(A)}s_A  }, \frac{ 1}{ 5 \theta_9^{(A)}   \sqrt{\frac{\max(s_A,s_B)^3}{k}} }  \right),
\end{equation}
and  when the following sparsity condition holds for $\eta>0$
\begin{equation}
\label{eq:sparsity_cond2}
    \frac{\max(s_A, s_B)^3}{k} \leq \frac{0.18\eta^2}{  {\theta_9^{(A) }}^2},
\end{equation}
Eq. (\ref{dist_bound_final}) results in $ \textmd{dist}\left(\hat{\textbf{D}}_{B_{t }}^{(i)}, \textbf{D}_{B}^{(i)} \right) \leq   \frac{\sqrt{0.18}\eta \tilde{\epsilon}_{B_{t-1}}^D}{0.6 } = \frac{ \eta \tilde{\epsilon}_{B_{t-1}}^D}{\sqrt{2}} $.
Applying Lemma \ref{dist_lemma}, it then has  
\begin{equation}
\label{final_dic_error_recovery_Db}
  \epsilon_D\left(\hat{\textbf{D}}_{B_t}^{(i)}, \textbf{D}_{B}^{(i)} \right) \leq \sqrt{2}\textmd{dist}\left(\hat{\textbf{D}}_{B_t}^{(i)}, \textbf{D}_{B}^{(i)} \right) \leq     \eta \tilde{\epsilon}_{B_{t-1}}^D .
\end{equation}
The probability for this result to hold is the joint probability for results on $G_1$, $G_2$, $G_3$ and $G_4$ to hold simultaneously, which is at least $p_A$ in Eq. (\ref{eq:main_pa}), obtained by combining the previously computed probabilities, where  $\delta = \frac{ 0.1u_s^2 \sigma_A^2\sigma_B^2 \Delta s_As_B }{k^2 }$.

Defining $ \tilde{\epsilon}_{B_{t}}^D =   \eta\sqrt{s_A} \tilde{\epsilon}_{B_{t-1}}^D$, Eq. (\ref{final_dic_error_recovery_Db}) finally yields  the following recursive error bound 
\begin{equation}
\label{iterative_coeff_error}
   \epsilon_{B_{t}}^D = \max_{i=1}^k\epsilon_D\left(\hat{\textbf{D}}_{B_t}^{(i)}, \textbf{D}_{B}^{(i)} \right) \leq \eta \tilde{\epsilon}_{B_{t-1}}^D = \frac{\eta\sqrt{s_A} \tilde{\epsilon}_{B_{t-1}}^D}{\sqrt{s_A}}  = \frac{\tilde{\epsilon}_{B_{t}}^D}{\sqrt{s_A}}, 
\end{equation}
which enables to apply Lemma \ref{coeff_error} iteratively, and this subsequently  enables to estimate iteratively the dictionary error using Eq. (\ref{final_dic_error_recovery_Db}).
Starting from $\epsilon_{B_{0}}^D = \max_{i=1}^k\epsilon_D\left(\hat{\textbf{D}}_{B_0}^{(i)}, \textbf{D}_{B}^{(i)} \right) = \frac{\tilde{\epsilon}_{B_{0}}^D}{\sqrt{s_A}}$,  applying iteratively Eq. (\ref{iterative_coeff_error}) and $ \tilde{\epsilon}_{B_{t}}^D =   \eta\sqrt{s_A} \tilde{\epsilon}_{B_{t-1}}^D$   results in
\begin{equation}
\label{eq:final_dic_errorB}
 \epsilon_{B_{t}}^D    \leq   (\eta\sqrt{s_A})^t  \frac{\tilde{\epsilon}_{B_{0}}^D}{\sqrt{s_A}} =  (\eta\sqrt{s_A})^t \epsilon_{B_{0}}^D,
\end{equation}
which is incorporated  into Eq. (\ref{delta_norm0}) with $t=T$ to obtain Eq. (\ref{theo:coef_a_full}).

Combining all the error conditions in Eqs. (\ref{dic_error0}), (\ref{dic_error1}), (\ref{dic_error2}) and (\ref{final_dic_error_XaDb}) results in
\begin{equation}
      \tilde{\epsilon}_{B_{0}}^D  \leq    \min\left( \frac{1}{40}, \frac{   \theta_{10}^{(A)}  }{\sqrt{s_A}}, \frac{ \theta_{12}^{(A)}}{s_A},   \frac{1}{ \theta_{11}^{(A)}\sqrt{s_As_B}},  \frac{ 1}{  5\theta_9^{(A)}  \sqrt{\frac{\max(s_A,s_B)^3}{k}} }\right),
\end{equation}
which leads to the  error requirement  in Eq. (\ref{eq:final_dic_condB}) by multiplying $\frac{1}{\sqrt{s_A}}$.
The used problem constants   are given by 
\begin{align}
    \theta_{10}^{(A)} = \; & \frac{M_{\min}^{(A)}l_s}{17.2M_{\max}^{(A)}u_s}\sqrt{\frac{1-\Delta}{1+\Delta}}, \\
    \theta_{11}^{(A)} =\; & 8.5M^{(A)}_{\max}u_s\sigma_B\sqrt{1+ \Delta}, \\
    \theta_{12}^{(A)} = \; &\min\left(\frac{1}{4\theta_2^{(A)}}, \frac{1}{16\theta_1\theta_2^{(A)}}, \frac{1}{5\theta_6^{(A)}}\right).
\end{align}  
Expanding the expression of $\theta_9^{(A) }$, the sparsity condition in Eq. (\ref{eq:sparsity_cond2}) can be re-written as $\max(s_A, s_B)   \leq  \left( \frac{0.18\eta^2 k}{  (1+\mu_D(d-1))\left(\theta_7^{(A)} + \theta_8^{(A) }\right)^2}\right)^{\frac{1}{3}}$.
Re-defining the new constant $\theta_{13}^{(A)} = \frac{0.18}{ \left(\theta_7^{(A)} + \theta_8^{(A) }\right)^2}$, this re-written sparsity condition  and the sparsity condition  in Eq. (\ref{eq:sparsity_cond1}) together result  in
\begin{equation}
     \max(s_A, s_B)   \leq \min\left( \frac{0.05}{ \mu_D}, \left( \frac{\theta_{13}^{(A)}\eta^2 k}{ 1+\mu_D(d-1)  }\right)^{\frac{1}{3}}\right).
\end{equation}

All the above analysis  applies to $\textbf{Y}_V$, leading to results on $ \epsilon_{B_{T}}^D$ and $ \epsilon_{A_{T}}^X$. Finally we analyse the estimation error of the auxiliary latent relation  matrix $\epsilon_{S}\left(\hat{\textbf{S}},  \tilde{\textbf{S}}\right) $, which is the maximum-norm-based  difference between the ground truth $\frac{1}{nm}\tilde{\textbf{S}} =  \textbf{D}_B\bm\Sigma_R^{-1}\textbf{D}_A^T$   and   its estimation $\frac{1}{nm}\hat{\textbf{S}}= \hat{\textbf{D}}_B\bm\Sigma_R^{-1}\hat{\textbf{D}}_A^T$ returned by Algorithm \ref{alg:rl} after $T$ iterations of  coefficient and dictionary estimation. 
Define $\epsilon_T^D = \max\left(\epsilon_{A_{T}}^D, \epsilon_{B_{T}}^D\right) = \max\left((\eta \sqrt{s_B})^T\epsilon_{A_{0}}^D, (\eta \sqrt{s_A})^T\epsilon_{B_{0}}^D\right)$, and by definition and Eqs. (\ref{eq:final_dic_errorB}) and (\ref{eq:final_dic_condB}), it has $0<\epsilon_T^D <1$.
We focus on analysing the following:
\begin{align}
    \frac{1}{nm}\left|\tilde{s}_{ij} - \hat{s}_{ij}\right| = \;&  \left|\textbf{D}_B^{(i)}\bm\Sigma_R^{-1}\textbf{D}_A^{(j)^T} - \hat{\textbf{D}}_B^{(i)}\bm\Sigma_R^{-1}\hat{\textbf{D}}_A^{(j)^T}  \right| \\
    \nonumber
    = \; &  \left|\left(\textbf{D}_B^{(i)}-\hat{\textbf{D}}_B^{(i)}+\hat{\textbf{D}}_B^{(i)}\right)\bm\Sigma_R^{-1}\textbf{D}_A^{(j)^T} - \hat{\textbf{D}}_B^{(i)}\bm\Sigma_R^{-1}\left(\hat{\textbf{D}}_A^{(j)} - \textbf{D}_A^{(j)} + \textbf{D}_A^{(j)}\right)^T \right| \\
    \nonumber
    \leq \; & \left|\left(\textbf{D}_B^{(i)}-\hat{\textbf{D}}_B^{(i)} \right)\bm\Sigma_R^{-1}\textbf{D}_A^{(j)^T} \right| + \left|\hat{\textbf{D}}_B^{(i)}\bm\Sigma_R^{-1}\left(\hat{\textbf{D}}_A^{(j)} - \textbf{D}_A^{(j)}  \right)^T \right| \\
    \nonumber
    \leq \; & \left\|\textbf{D}_B^{(i)}-\hat{\textbf{D}}_B^{(i)} \right\|_2 \left\| \textbf{D}_A^{(j)} \bm\Sigma_R^{-1}\right\|_2 +   \left\|\hat{\textbf{D}}_B^{(i)}\bm\Sigma_R^{-1}  \right\|_2 \left\|\hat{\textbf{D}}_A^{(j)} - \textbf{D}_A^{(j)}  \right\|_2\\
    \nonumber
    \leq \; & \epsilon_T^D \left\|\textbf{D}_A^{(j)} \bm\Sigma_R^{-1}\right\|_2 + \epsilon_T^D \left\|\hat{\textbf{D}}_B^{(i)}\bm\Sigma_R^{-1}  \right\|_2 \\
    \nonumber
    \leq \; &\epsilon_T^D \left\| \textbf{D}_A^{(j)} \bm\Sigma_R^{-1}\right\|_2 + \epsilon_T^D \left\|\left(\textbf{D}_B^{(i)}-\hat{\textbf{D}}_B^{(i)}\right)\bm\Sigma_R^{-1}  \right\|_2  + \epsilon_T^D \left\| \textbf{D}_B^{(i)} \bm\Sigma_R^{-1}  \right\|_2 \\
    \nonumber
    \leq\;& \epsilon_T^D\sigma_{{\min}_R}^{-1} \left\|\textbf{D}_A^{(j)} \right\|_2 + \epsilon_T^D\sigma_{{\min}_R}^{-1}\left\|\textbf{D}_B^{(i)}-\hat{\textbf{D}}_B^{(i)}\right\|_2 +\epsilon_T^D\sigma_{{\min}_R}^{-1} \left\|\textbf{D}_B^{(j)} \right\|_2 \\
    \nonumber
    \leq \; & 2\epsilon_T^D\sigma_{{\min}_R}^{-1}+\left(\epsilon_T^{D}\right)^2\sigma_{{\min}_R}^{-1} < 3\epsilon_T^D\sigma_{{\min}_R}^{-1}.
\end{align}
This completes the proof.

\end{proof}

\section{Proofs for Theorem \ref{main_res2}}
\label{app:main_proof_alg3}

Algorithm \ref{alg:ini}    approximates the dictionary matrix  through two key operations.
The  theoretical properties  of these two operations are established in Lemmas \ref{threshold_YuYv} and \ref{singular_YuYv}, which provide the foundation for proving Theorem \ref{main_res2}.

\subsection{Supporting Lemma \ref{threshold_YuYv} and Its Proof}
The first key operation of   Algorithm \ref{alg:ini}  identifies  unique intersection pairs using Algorithm \ref{alg:unique}.
\citet{Aga17} showed that, when  the unique intersection threshold is set as $\rho_p = \frac{62}{64}$  and  the correlation threshold  $\rho$ satisfies a certain condition,  Algorithm \ref{alg:unique}  returns dictionary estimate with bounded approximation error.
This result is re-stated  as  Lemma \ref{UniqueAtom}  and the required condition  on $\rho$ is formalised in Assumption \ref{threshold_rho}.
It is therefore important to identify suitable choices of  $\rho$ that satisfy Assumption \ref{threshold_rho}.
Lemma \ref{threshold_YuYv} provides one such choice.

\begin{lemma}[Correlation Threshold] 
\label{threshold_YuYv}
Suppose Assumptions  \ref{SC}-\ref{LI} hold. 
Given $0< \Delta <1$, assume that Eq. (\ref{s_AB_condition1}) is satisfied.
Then, the   threshold ranges specified in Eqs. (\ref{eq:threshold_rho_u}) and  (\ref{eq:threshold_rho_v}) satisfy Assumption  \ref{threshold_rho} for $\textbf{Y}_U $ and   $\textbf{Y}_V $, with  probabilities at least $1-ke^{-\frac{C_B\Delta^2  \sigma_B^2m}{k {M_{\max}^{(B)}}^2}}$ and $1-ke^{-\frac{C_A\Delta^2  \sigma_A^2n}{k {M_{\max}^{(A)}}^2}}$, respectively.
\end{lemma}

 \begin{proof}
 We provide the proof  for $\textbf{Y}_V$   and the same applies to $\textbf{Y}_U$.  
Let $\textbf{y}_i$ and $\textbf{y}_j$ denote the $i$th and $j$th rows of $\textbf{Y}_V$: $\textbf{y}_i = \sum_{t=1}^{k}(X_B)_{it}\textbf{D}_A^{(t)}$ and $\textbf{y}_j = \sum_{t=1}^{k}(X_B)_{jt}\textbf{D}_A^{(t)}$. 
When they do not share any common dictionary atom, applying Lemma  \ref{MC_DaDb}, it has
\begin{equation}
\left|\langle \textbf{y}_i, \textbf{y}_j\rangle \right |=     \left| \sum_{t,g=1}^k (X_B)_{it}(X_B)_{jg}\left\langle \textbf{D}_A^{(t)}, \textbf{D}_A^{(g)}\right\rangle \right| 
\leq   s_B^2 M_B^2 \mu_D.
\end{equation}
When  $\textbf{y}_i$ and $\textbf{y}_j$ share one common atom,  it has
\begin{align}
\nonumber
\left|\langle \textbf{y}_i, \textbf{y}_j\rangle  \right | =\; & \left|(X_B)_{ip}(X_B)_{jp}\left\langle \textbf{D}_A^{(p)}, \textbf{D}_A^{(p)}\right\rangle  +  \sum_{ t\neq g }  (X_B)_{it}(X_B)_{jg}\left\langle \textbf{D}_A^{(t)}, \textbf{D}_A^{(g)}\right\rangle \right| \\
\nonumber
\geq \; &   \left| (X_B)_{ip}(X_B)_{jp}\left\langle \textbf{D}_A^{(p)}, \textbf{D}_A^{(p)}\right\rangle \right| -  \left|  \sum_{ t\neq g}  (X_B)_{it}(X_B)_{jg}\left\langle \textbf{D}_A^{(t)}, \textbf{D}_A^{(g)}\right\rangle \right|    >   m_B^2 -  s_B^2 M_B^2  \mu_D .
\end{align}
To guarantee the identification of vector pairs that share one single atom, a sufficient condition is to let 
\begin{equation}
   m_B^2 -  s_B^2 M_B^2 \mu_D \geq s_B^2 M_B^2 \mu_D, 
\end{equation}
i.e.,  $m_B^2   \geq 2s_B^2 M_B^2 \mu_D$. 
Therefore, an eligible  threshold $\rho_v$  for filtering  $\left|\langle \textbf{y}_i, \textbf{y}_j\rangle  \right |$ is within the following range  
\begin{equation}
\label{eq:rho_v_threshold1}
    s_B^2 M_B^2 \mu_D \leq \rho_v \leq m_B^2-s_B^2 M_B^2 \mu_D,
\end{equation}
in order to identify the unique intersection pair.
Incorporating Eqs. (\ref{bound_maxB}) and (\ref{bound_minB}) into $m_B^2   \geq 2s_B^2 M_B^2 \mu_D$, it has
\begin{equation}
       \frac{ {M^{(B)}_{\min}}^2  l_s^2\sigma_A^2 (1- \Delta ) s_A  }{mk } \geq \frac{2{M^{(B)}_{\max}}^2  u_s^2\sigma_A^2 (1+ \Delta ) \mu_D s_As_B^2 }{mk }  ,
\end{equation}
which is equivalent to asking
\begin{equation}
       s_B \leq \frac{ {M^{(B)}_{\min}}  l_s     }{ {M^{(B)}_{\max}}u_s} \sqrt{\frac{ 1- \Delta     }{2 \mu_D(1+ \Delta )}}.
\end{equation}
Also, expanding with Eqs. (\ref{bound_maxB}) and (\ref{bound_minB}), the condition in Eq. (\ref{eq:rho_v_threshold1}) that the threshold $\rho_v$ should satisfy becomes
\begin{equation}
\label{eq:rho_v_threshold2}
    \frac{{M^{(B)}_{\max}}^2  u_s^2\sigma_A^2 (1+ \Delta ) \mu_D s_As_B^2 }{mk } \leq \rho_v \leq  \frac{ {M^{(B)}_{\min}}^2  l_s^2\sigma_A^2 (1- \Delta ) s_A  }{mk }-\frac{{M^{(B)}_{\max}}^2  u_s^2\sigma_A^2 (1+ \Delta ) \mu_D s_As_B^2 }{mk } .
\end{equation}
The probability for the above to hold is the same as that in Corollary \ref{RowLengthFhat} to enable the bounds for $\textbf{F}_A$ and $\textbf{D}_A$,  i.e., $ 
p=  1- ke^{-\frac{C_A\Delta^2  \sigma_A^2n}{k {M_{\max}^{(A)}}^2}}$.

\end{proof}

\subsection{Supporting Lemma \ref{singular_YuYv} and Its Proof}

The second key operation, corresponding to Steps \ref{step:atom1}-\ref{step:atom2} of Algorithm \ref{alg:ini},  estimates the shared  atom  for each identified unique intersection pair  $\left(\textbf{Y}_U^{(i)}, \textbf{Y}_U^{(j)}\right)$ or $\left(\textbf{Y}_V^{(i)}, \textbf{Y}_V^{(j)}\right)$  as the leading singular vector of the following $d\times d$ matrix:
\begin{equation}
\textbf{L}_U\left(\textbf{Y}_U^{(i)}, \textbf{Y}_U^{(j)}, Y_V\right) =   \sum_{t\in N_{\rho_u}
\left(\textbf{Y}_U^{(i)}, \textbf{Y}_U^{(j)}, Y_U \right)} {\textbf{Y}_U^{(t)}}^T\textbf{Y}_U^{(t)},  
\end{equation}
or 
\begin{equation}
\textbf{L}_V\left(\textbf{Y}_V^{(i)}, \textbf{Y}_V^{(j)}, Y_V\right) =   \sum_{t\in N_{\rho_v}
\left(\textbf{Y}_V^{(i)}, \textbf{Y}_V^{(j)}, Y_V \right)} {\textbf{Y}_V^{(t)}}^T\textbf{Y}_V^{(t)}. 
\end{equation}
Lemma \ref{singular_YuYv}  bounds the estimation error of this spectral approximation.

\begin{lemma}[Atom Approximation Error] \label{singular_YuYv}
Suppose Assumptions \ref{SC}-\ref{LI} hold.
Denote the shared atom of a unique intersection pair as $D\left(\textbf{Y}_U^{(i)}, \textbf{Y}_U^{(j)}\right) = \left\{\textbf{D}_B^{(p)}\right\}$ and $D\left(\textbf{Y}_V^{(i)}, \textbf{Y}_V^{(j)}\right) = \left\{\textbf{D}_A^{(p)}\right\}$, and denote the leading singular vectors of $ \textbf{L}_U\left(\textbf{Y}_U^{(i)}, \textbf{Y}_U^{(j)}, Y_U\right)$ and  $   \textbf{L}_V\left(\textbf{Y}_V^{(i)}, \textbf{Y}_V^{(j)}, Y_V\right)$ by $\textbf{u}$ and $\textbf{v}$, respectively.
Given $0< \Delta, \Delta_Y<1$,  and $0<\Delta_N <1-\frac{\max\left(s_A^3, s_B^3\right)}{k}$, the following approximation error bounds hold
\begin{align}
\min_{z\in\{+1,-1\}}\left\|\textbf{D}_B^{(p)} -z\textbf{u}\right\|_2^2 <  \; &   \frac{12M^{(A)}_{\max}u_s^2\sigma_A^2\hat{\gamma}(1+4\hat{\gamma})}{ {M^{(A)}_{\min}}^2l_s^2(1- \Delta )\left( 1-\frac{s_A^3}{k} -\Delta_N\right) } \sqrt{\frac{s_A}{k}}, \\
\min_{z\in\{+1,-1\}}\left\|\textbf{D}_A^{(p)} -z\textbf{v}\right\|_2^2   <  \; &\frac{12M^{(B)}_{\max}u_s^2\sigma_B^2\hat{\gamma}(1+4\hat{\gamma})}{ {M^{(B)}_{\min}}^2l_s^2(1- \Delta )\left( 1-\frac{s_B^3}{k} -\Delta_N\right)  } \sqrt{\frac{s_B}{k}},
\end{align}
with probabilities at least  $p_u$ and $p_v$, respectively, where
\begin{align}
\label{alg1_prob1}
p_u = \; & 1- ke^{-\frac{C_B\Delta^2  \sigma_B^2m}{k {M_{\max}^{(B)}}^2}}-de^{- \frac{ C_U \Delta_Y^2\sigma_A^2\left| I_U \right| \left(1+ \mu_D(d-1)\right)}{ {M_{\max}^{(A)}}^2( 1+ s_A\mu_D)}} -e^{-2\Delta_N^2\left|N_{\rho_u}\right|}, \\
\label{alg1_prob2}
    p_v = \; & 1- ke^{-\frac{C_A\Delta^2  \sigma_A^2n}{k {M_{\max}^{(A)}}^2}}-de^{- \frac{ C_V \Delta_Y^2\sigma_B^2\left| I_V \right| \left(1+ \mu_D(d-1)\right)}{ {M_{\max}^{(B)}}^2( 1+ s_B\mu_D)}}-e^{-2\Delta_N^2\left|N_{\rho_v}\right|}.
\end{align}
\end{lemma}

\begin{proof}
We exemplify the proof for  $\textbf{D}_A$ case. 
For the singular vector $\textbf{v}$,   it has 
\begin{equation}
\label{dic_error_bound}
\min_{z\in\{+1,-1\}}\left\|\textbf{D}_A^{(p)} -z\textbf{v}\right\|_2^2  =  2-2\max_{z\in\{+1,-1\}}z\textbf{D}_A^{(p)}\textbf{v}^T \leq 2- 2\left(\textbf{D}_A^{(p)}\textbf{v}^T\right)^2,
\end{equation}
where the last inequality results from the fact $0 \leq \max_{z\in\{+1,-1\}}z\textbf{D}_A^{(p)}\textbf{v}^T \leq 1$, as $\textbf{v}$ has unit length.
To derive an upper bound for the above, we  lower bound $\left(\textbf{D}_A^{(p)}\textbf{v}^T\right)^2$.

As in Eqs. (\ref{eq:set1}) and  (\ref{eq:set2}), the neighboring set $N_{\rho_v}
\left(\textbf{Y}_V^{(i)}, \textbf{Y}_V^{(j)} \right)$ of the unique intersection pair $\textbf{Y}_V^{(i)}$ and $\textbf{Y}_V^{(j)}$ can be divided into two disjoint sets, including $\tilde{N}_{\rho_v}
\left(\textbf{Y}_V^{(i)}, \textbf{Y}_V^{(j)} \right)$ of which the elements share the same atom $\textbf{D}_A^{(p)}$ as $\textbf{Y}_V^{(i)}$ and $ \textbf{Y}_V^{(j)}$ and its complement $\tilde{N}^{\neg}_{\rho_v}
\left(\textbf{Y}_V^{(i)}, \textbf{Y}_V^{(j)} \right)$.
We simplify these set notations to $N_{\rho_v}$, $\tilde{N}_{\rho_v}$ and $\tilde{N}^{\neg}_{\rho_v}$, and  simplify $\textbf{L}_V\left(\textbf{Y}_V^{(i)}, \textbf{Y}_V^{(j)}, Y_V\right) $ to $\textbf{L}_V$.
Applying the  set division, it has
\begin{align}
\nonumber
\textbf{L}_V\left(\textbf{Y}_V^{(i)}, \textbf{Y}_V^{(j)}\right) = \; & \sum_{t\in N_{\rho_v}
} {\textbf{Y}_V^{(t)}}^T\textbf{Y}_V^{(t)}   = \underset{\textbf{L}_{V_1} }{\underbrace{\sum_{t\in \tilde{N}_{\rho_v}} {\textbf{Y}_V^{(t)}}^T\textbf{Y}_V^{(t)} }}+ \underset{\textbf{L}_{V_2} }{\underbrace{\sum_{t\in \tilde{N}^{\neg}_{\rho_v}} {\textbf{Y}_V^{(t)}}^T\textbf{Y}_V^{(t)} }}, \\
= \; & \underset{\textbf{L}_{V_1} }{\underbrace{{\textbf{Y}_V^{(\tilde{N}_{\rho_v})}}^T\textbf{Y}_V^{(\tilde{N}_{\rho_v})} }}+ \underset{\textbf{L}_{V_2} }{\underbrace{ {\textbf{Y}_V^{(\tilde{N}^{\neg}_{\rho_v})}}^T\textbf{Y}_V^{(\tilde{N}^{\neg}_{\rho_v})} }}.
\end{align}
The row vector $\textbf{Y}_V^{(t)}$ with $t\in\tilde{N}_{\rho_v}$ shares the atom $\textbf{D}_A^{(p)}$ with  $\textbf{Y}_V^{(i)}$ and $\textbf{Y}_V^{(j)}$. 
Thus, we re-express it as
\begin{equation}
    \textbf{Y}_V^{(t)} = (X_{B})_{tp}\textbf{D}_A^{(p)} + \textbf{y}_V^{(t)}.
\end{equation}
Denoting the observation matrix containing $\left\{\textbf{y}_V^{(t)}\right\}_{t\in\tilde{N}_{\rho_v}}$,   by $\tilde{\textbf{Y}}_V^{(I)}$ and applying the  definition of spectral norm, triangle inequality and Cauchy-Schwarz inequality, it has
\begin{align}
\nonumber
\left |\textbf{v}\textbf{L}_{V}\textbf{v}^T\right | \leq \; & \left|\textbf{v}\textbf{L}_{V_1}\textbf{v}^T\right | + \left|\textbf{v}\textbf{L}_{V_2}\textbf{v}^T\right | \\
\nonumber
\leq\; & \left| \sum_{t\in \tilde{N}_{\rho_v}}  (X_{B})^2_{tp} \left(\textbf{D}_A^{(p)}\textbf{v}^T\right)^2 + \sum_{t\in \tilde{N}_{\rho_v}}(X_{B})_{tp}\textbf{v}\left( \textbf{D}_A^{(p)^T}  \textbf{y}_V^{(t)}+ \textbf{y}_V^{(t)^T}\textbf{D}_A^{(p)}\right)\textbf{v}^T + \sum_{t\in \tilde{N}_{\rho_v}} \textbf{v} \textbf{y}_V^{(t)^T} \textbf{y}_V^{(t)} \textbf{v}^T \right |   \\
\label{lv1_lv2_analysis}
& + \left |  \textbf{v} \left(\sum_{t\in \tilde{N}_{\rho_v}^{\neg}} \textbf{Y}_V^{(t)^T} \textbf{Y}_V^{(t)} \right) \textbf{v}^T \right |   \\
\nonumber
\leq & \left| \sum_{t\in \tilde{N}_{\rho_v}}  (X_{B})^2_{tp} \left(\textbf{D}_A^{(p)}\textbf{v}^T\right)^2 \right | + \left| \sum_{t\in \tilde{N}_{\rho_v}}(X_{B})_{tp}\textbf{v}\left(  \textbf{D}_A^{(p)^T}   \textbf{y}_V^{(t)}+ \textbf{y}_V^{(t)^T}\textbf{D}_A^{(p)}\right)\textbf{v}^T  \right| \\ 
\nonumber
& \left\| \tilde{\textbf{Y}}_V^{(I)^T} \tilde{\textbf{Y}}_V^{(I)} \right \|_2  + \left\|   \textbf{Y}_V^{(\tilde{N}^{\neg}_{\rho_v})^T}  \textbf{Y}_V^{(\tilde{N}^{\neg}_{\rho_v})} \right \|_2   \\
\leq \; & \left(\textbf{D}_A^{(p)}\textbf{v}^T\right)^2 \sum_{t\in \tilde{N}_{\rho_v}}  (X_{B})^2_{tp} + 2  \sqrt{\left|\tilde{N}_{\rho_v}\right|}M_B \left \|  \tilde{\textbf{Y}}_V^{(I)}\right\|_2 +\left \|  \tilde{\textbf{Y}}_V^{(I)}\right\|_2^2 + \left \| \textbf{Y}_V^{(\tilde{N}^{\neg}_{\rho_v})}\right\|_2^2.
\end{align}
The above enables the following lower bound for $\left(\textbf{D}_A^{(p)}\textbf{v}^T\right)^2  $, i.e.,
\begin{equation}
\label{vlv_bound1}
\left(\textbf{D}_A^{(p)}\textbf{v}^T\right)^2  \geq\frac{ \left |\textbf{v}\textbf{L}_{V}\textbf{v}^T\right | - 2  \sqrt{\left|\tilde{N}_{\rho_v}\right|}M_B \left \|  \tilde{\textbf{Y}}_V^{I}\right\|_2 - \left \|  \tilde{\textbf{Y}}_V^{I}\right\|_2^2 - \left \| \textbf{Y}_V^{(\tilde{N}^{\neg}_{\rho_v})}\right\|_2^2  }{\sum_{t\in \tilde{N}_{\rho_v}}  (X_{B})^2_{tp}}.
\end{equation}
Since $\textbf{v} $ is the top singular vector of $\textbf{L}_{V}$, following a similar analysis  for Eq. (\ref{lv1_lv2_analysis}), it has
\begin{align}
\nonumber
\left | \textbf{v} \textbf{L}_{V} \textbf{v}^T\right | = \; &\|\textbf{L}_{V}\|_2 \geq \|\textbf{L}_{V_1}\|_2 -\|\textbf{L}_{V_2}\|_2\\
\nonumber
=\; & \left\| \sum_{t\in \tilde{N}_{\rho_v}}  (X_{B})^2_{tp}  \textbf{D}_A^{(p)^T}   \textbf{D}_A^{(p)} + \sum_{t\in \tilde{N}_{\rho_v}}(X_{B})_{tp} \left(  \textbf{D}_A^{(p)^T}   \textbf{y}_V^{(t)}+  \textbf{y}_V^{(t)^T} \textbf{D}_A^{(p)}\right) + \sum_{t\in \tilde{N}_{\rho_v}}  \textbf{y}_V^{(t)^T} \textbf{y}_V^{(t)}  \right \|_2  \\
\nonumber
&- \left \| \textbf{Y}_V^{(\tilde{N}^{\neg}_{\rho_v})}\right\|_2^2\\
\label{vlv_bound2}
 \geq \; & \sum_{t\in \tilde{N}_{\rho_v}}  (X_{B})^2_{tp}   - 2  \sqrt{\left|\tilde{N}_{\rho_v}\right|}m_B \left \|  \tilde{\textbf{Y}}_V^{I} \right\|_2 - \left \|   \tilde{\textbf{Y}}_V^{I} \right\|_2^2 - \left \| \textbf{Y}_V^{(\tilde{N}^{\neg}_{\rho_v})}\right\|_2^2.
\end{align}
Subsequently, incorporating Eq. (\ref{vlv_bound2}) into Eq. (\ref{vlv_bound1}), and considering the fact that $\sum_{t\in \tilde{N}_{\rho_v}}  (X_{B})^2_{tp} \geq \left|\tilde{N}_{\rho_v}\right| m_B^2$, it has
\begin{align}
\nonumber
\left(\textbf{D}_A^{(p)}\textbf{v}^T\right)^2 \geq\; & 1- \frac{  4 \sqrt{\left|\tilde{N}_{\rho_v}\right|}M_B\left \|  \tilde{\textbf{Y}}_V^{I}\right\|_2  +  2\left \|  \tilde{\textbf{Y}}_V^{I}\right\|_2^2  +   2\left \| \textbf{Y}_V^{(\tilde{N}^{\neg}_{\rho_v})}   \right\|_2^2   }{\left|\tilde{N}_{\rho_v}\right| m_B^2} \\
\geq \; &1- \underset{T_1}{\underbrace{\frac{4 M_B\left \|  \tilde{\textbf{Y}}_V^{I}\right\|_2  }{\sqrt{\left|\tilde{N}_{\rho_v}\right|}m^2_B}}} - \underset{T_2}{\underbrace{\frac{   2\left \|  \tilde{\textbf{Y}}_V^{I}\right\|_2^2}{\left|\tilde{N}_{\rho_v}\right|m_B^2} }}- \underset{T_3}{\underbrace{\frac{  2\left \|  \textbf{Y}_V^{(\tilde{N}^{\neg}_{\rho_v})}\right\|_2^2}{\left|\tilde{N}_{\rho_v}\right|  m_B^2}}}.
\end{align}
Combining the above with Eq. (\ref{dic_error_bound}), it has
\begin{equation}
\label{alg2_res1}
     \min_{z\in\{+1,-1\}}\left\|\textbf{D}_A^{(p)} -z\textbf{v}\right\|_2^2 \leq 2(T_1+T_2+T_3).
\end{equation}
Applying Lemma \ref{SSN_YuYv} and Lemma \ref{AtomSharing} with $0<\Delta_N < 1-\frac{s_B^3}{k}$, it has
\begin{align}
\label{alg2_res2}
    T_1 + T_2+ T_3=\; & \frac{4M_B\gamma}{m_B^2}\sqrt{\frac{s_As_B}{mk^2}}+ \frac{2\gamma^2s_As_B}{m_B^2mk^2} +  \frac{2\gamma^2s_As_B\left(\left|N_{\rho_v}\right|-\left|\tilde{N}_{\rho_v}\right| \right)}{m_B^2mk^2\left|\tilde{N}_{\rho_v}\right| } \\
    \nonumber
      = \; & \frac{4M_B\gamma}{m_B^2}\sqrt{\frac{s_As_B}{mk^2}} +  \frac{2\gamma^2s_As_B \left|N_{\rho_v}\right|}{m_B^2mk^2\left|\tilde{N}_{\rho_v}\right| } \leq \frac{4M_B\gamma}{m_B^2}\sqrt{\frac{s_As_B}{mk^2}} + \frac{2\gamma^2s_As_B  }{m_B^2mk^2\left( 1-\frac{s_B^3}{k} -\Delta_N\right) },  
\end{align}
with probability at least $ p_v-e^{-2\Delta_N^2\left|N_{\rho_v}\right|}$.
Expanding on $M_B$ and  $m_B$ using Eqs. (\ref{bound_maxB}) and (\ref{bound_minB}), it has
\begin{align}
\label{alg1_quantity1}
  & \frac{4M_B }{m_B^2}\sqrt{\frac{s_As_B}{mk^2}} =   \frac{4M^{(B)}_{\max}u_s}{ {M^{(B)}_{\min}}^2l_s^2\sigma_A(1- \Delta )} \sqrt{\frac{(1+ \Delta )s_B}{k}} < \frac{6M^{(B)}_{\max}u_s}{ {M^{(B)}_{\min}}^2l_s^2\sigma_A(1- \Delta )} \sqrt{\frac{s_B}{k}} ,  \\
\label{alg1_quantity2}
  & \frac{2s_As_B  }{m_B^2mk^2} =   \frac{2 s_B}{ {M^{(B)}_{\min}}^2l_s^2\sigma_A^2(1- \Delta )k} \leq \frac{2}{ {M^{(B)}_{\min}}^2l_s^2\sigma_A^2(1- \Delta ) } \sqrt{\frac{s_B}{k}}. 
\end{align}
Therefore, by combining Eqs. (\ref{alg2_res1}), (\ref{alg2_res2}), (\ref{alg1_quantity1}) and (\ref{alg1_quantity2}), it has
\begin{align}
\nonumber
     \min_{z\in\{+1,-1\}}\left\|\textbf{D}_A^{(p)} -z\textbf{v}\right\|_2^2 < \; &\frac{12M^{(B)}_{\max}u_s\gamma}{ {M^{(B)}_{\min}}^2l_s^2\sigma_A(1- \Delta )} \sqrt{\frac{s_B}{k}} + \frac{4\gamma^2}{ {M^{(B)}_{\min}}^2l_s^2\sigma_A^2(1- \Delta ) \left( 1-\frac{s_B^3}{k} -\Delta_N\right) } \sqrt{\frac{s_B}{k}} \\
     <\; & \frac{12M^{(B)}_{\max}u_s\sigma_A\gamma  + 4\gamma^2}{ {M^{(B)}_{\min}}^2l_s^2\sigma_A^2(1- \Delta )\left( 1-\frac{s_B^3}{k} -\Delta_N\right)  } \sqrt{\frac{s_B}{k}}.
\end{align}
Incorporating $\gamma  = (1+\Delta_Y)   u_s  \sigma_A \sigma_B    \sqrt{  ( 1 +\Delta)\left(1+ \mu_D(d-1)\right)  }$, and define a new quantity $\hat{\gamma}  = (1+\Delta_Y)    \sqrt{  ( 1 +\Delta)\left(1+ \mu_D(d-1)\right) }$, it has
\begin{equation}
    \min_{z\in\{+1,-1\}}\left\|\textbf{D}_A^{(p)} -z\textbf{v}\right\|_2^2 < \frac{12M^{(B)}_{\max}u_s^2\sigma_B^2\hat{\gamma}(1+4\hat{\gamma})}{ {M^{(B)}_{\min}}^2l_s^2(1- \Delta )\left( 1-\frac{s_B^3}{k} -\Delta_N\right)  } \sqrt{\frac{s_B}{k}},
\end{equation}
This completes the proof.
\end{proof}

\subsection{Main Proof}
\begin{proof}
We provide proof for the $\textbf{Y}_V$ case  and the same applies to $\textbf{Y}_U$.
According to  Lemmas \ref{threshold_YuYv} and \ref{UniqueAtom},   a correlation threshold $\rho_v$ that satisfies  Eq. (\ref{eq:threshold_rho_v})   and  the unique intersection threshold   $\rho_p =\frac{62}{64}$ enable the identification of a unique intersection pair. 
Then,  an initial estimation of a dictionary atom $\hat{\textbf{D}}_{A}^{(i)}$ can be computed from this unique intersection pair, for which Lemma \ref{singular_YuYv} proves that the 
  estimation error of $\hat{\textbf{D}}_{A}^{(i)}$ satisfies
\begin{equation}
\label{eq:ini_error1}
    \epsilon_D  \left(    \hat{\textbf{D}}_{A}^{(i)},  \textbf{D}_{A}^{(i)}\right) <\frac{12M^{(B)}_{\max}u_s^2\sigma_B^2\hat{\gamma}(1+4\hat{\gamma})}{ {M^{(B)}_{\min}}^2l_s^2(1- \Delta )\left( 1-\frac{s_B^3}{k} -\Delta_N\right)  } \sqrt{\frac{s_B}{k}}, 
\end{equation}
with probability at least
\begin{equation}
\label{eq:main_pv2a}
    p_v =  1- 2ke^{-\frac{C_A\Delta^2  \sigma_A^2n}{k {M_{\max}^{(A)}}^2}}-de^{- \frac{ C_V \Delta_Y^2\sigma_B^2\left| I_V \right| \left(1+ \mu_D(d-1)\right)}{ {M_{\max}^{(B)}}^2( 1+ s_B\mu_D)}}-e^{-2\Delta_N^2\left|N_{\rho_v}\right|}-2e^{- 2\left|N_{\rho_v}\right|(\left|N_{\rho_v}\right|-1)\gamma^2},
\end{equation}
for a  constant $\gamma\leq \frac{1}{64}$.
This probability is derived from the joint probability that enables Lemmas \ref{threshold_YuYv},    \ref{singular_YuYv}, and \ref{UniqueAtom}, simultaneously.   
Defining a constant $\theta_{14}^{(B)}=\frac{12M^{(B)}_{\max}u_s^2\sigma_B^2\hat{\gamma}(1+4\hat{\gamma})}{ {M^{(B)}_{\min}}^2l_s^2(1- \Delta )  }  $, it follows straightforwardly that 
\begin{equation}
     \epsilon_{A}^D = \max_{i=1}^k \epsilon_D  \left(    \hat{\textbf{D}}_{A}^{(i)},  \textbf{D}_{A}^{(i)}\right) \leq \frac{\theta_{14}^{(B)}}{   1-\frac{s_B^3}{k} -\Delta_N  } \sqrt{\frac{s_B}{k}}.
\end{equation}
 
\end{proof}

\section{Proof of Theorem \ref{main_res3}}
\label{app:model_gap}

The proof relies on the formulations of the auxiliary factor matrices:  $\textbf{X}_A = \frac{1}{\sqrt{nm}}\textbf{A}\textbf{L}_B^{-1}$, $\textbf{X}_B = \frac{1}{\sqrt{nm}}\textbf{B}\textbf{L}_A^{-1}$, and $\tilde{\textbf{S}} = nm\textbf{L}_B \textbf{S}\textbf{L}_A$, where  $\textbf{L}_A =  \textmd{diag}\left( \left\|\textbf{F}_A^{(1)} \right\|_2^{-1}, \left\|\textbf{F}_A^{(2)} \right\|_2^{-1}, \dots, \left\|\textbf{F}_A^{(k)} \right\|_2^{-1}\right) $, and $\textbf{L}_B =  \textmd{diag}\left( \left\|\textbf{F}_B^{(1)} \right\|_2^{-1}, \left\|\textbf{F}_B^{(2)} \right\|_2^{-1}, \dots, \left\|\textbf{F}_B^{(k)} \right\|_2^{-1}\right)$.
Also it relies on the upper and lower bounds for $\left\|\textbf{F}_A^{(i)} \right\|_2$ and $\left\|\textbf{F}_B^{(i)} \right\|_2$  established  in Lemma \ref{RowLengthF}.
For convenience, we denote the maximum and minimum diagonal entries of a diagonal matrix by $d_{\max}(\cdot)$ and $d_{\min}(\cdot)$, respectively.

\begin{proof}
We prove for the case $(\textbf{A}, \textbf{X}_A)$, and the same applies for $(\textbf{B}, \textbf{X}_B)$.
Since element-wise scaling   does not alter the support of a vector, it has
\begin{equation}
    supp\left(\textbf{X}_A^{(i)}\right) = supp\left(\textbf{A}^{(i)} \textbf{L}_B^{-1} \right)= supp\left(\textbf{A}^{(i)}\right).
\end{equation}
Expanding $\textbf{X}_A^{(i)} = \frac{1}{\sqrt{nm}} \textbf{A}^{(i)} \textbf{L}_B^{-1} $, it has
\begin{equation}
    \frac{1}{nm}\min_{j=1}^k \left\|\textbf{F}_B^{(j)} \right\|_2^2 \leq \frac{\left\|\textbf{X}_A^{(i)}\right\|_2^2}{\left\|\textbf{A}^{(i)}\right\|_2^2} =\frac{\sum_{j=1}^k \frac{1}{nm}A_{ij}^2\left\|\textbf{F}_B^{(j)} \right\|_2^2}{\sum_{j=1}^kA_{ij}^2}\leq \frac{1}{nm}\max_{j=1}^k \left\|\textbf{F}_B^{(j)} \right\|_2^2.
\end{equation}
Applying Lemma \ref{RowLengthF}, it has 
\begin{equation}
  l_s\sigma_B\sqrt{\frac{(1- \Delta )   s_B }{nk }}  \leq  \frac{\left\|\textbf{X}_A^{(i)}\right\|_2 }{\left\|\textbf{A}^{(i)}\right\|_2 } \leq u_s\sigma_B\sqrt{\frac{(1+ \Delta )   s_B }{nk }},
\end{equation}
which holds with probability at least    $1-ke^{-\frac{C_B\Delta^2  \sigma_B^2m}{k  {M_{\max}^{(B)}}^2}}$.
Subsequently, it has
\begin{equation}
  \frac{l_s}{u_s}\sqrt{\frac{1-\Delta}{1+\Delta}}  \leq \frac{\max_{i=1}^n \frac{\left\|\textbf{X}_A^{(i)}\right\|_2 }{\left\|\textbf{A}^{(i)}\right\|_2 } }{\min_{i=1}^n \frac{\left\|\textbf{X}_A^{(i)}\right\|_2 }{\left\|\textbf{A}^{(i)}\right\|_2 } } \leq \frac{u_s}{l_s}\sqrt{\frac{1+\Delta}{1-\Delta}}.
\end{equation}

Next, we analyse the inner product between   two auxiliary coefficient vectors $\textbf{X}_A^{(i)}$ and $\textbf{X}_A^{(j)}$. 
Defining the index sets   $I^{+}=\{t| t\in[k], A_{it}A_{jt}> 0\}  $ and  $I^{-}=\{t| t\in[k], A_{it}A_{jt} < 0\}  $, it has  $nm\textbf{X}_A^{(i)} {\textbf{X}_A^{(j)}}^T    =    \sum_{t\in I^+}  A_{it}A_{jt}\left\|\textbf{F}_B^{(t)} \right\|_2^2 - \sum_{t\in I^-}  |A_{it}A_{jt}|\left\|\textbf{F}_B^{(t)} \right\|_2^2 $.
We bound the inner product by 
\begin{align}
\nonumber
  nm  \textbf{X}_A^{(i)} {\textbf{X}_A^{(j)}}^T      \leq   \; &  \max_{t = 1}^k\left\|\textbf{F}_B^{(t)} \right\|_2^2\sum_{t \in I^+}   A_{it}A_{jt} - \min_{t=1}^k  \left\|\textbf{F}_B^{(t)} \right\|_2^2\sum_{t\in I^-}   |A_{it}A_{jt} | \\
    =\; &   \min_{t =1}^k  \left\|\textbf{F}_B^{(t)} \right\|_2^2 \textbf{A}^{(i)} {\textbf{A}^{(j)}}^T + \left(\max_{t=1}^k  \left\|\textbf{F}_B^{(t)} \right\|_2^2 -\min_{t=1}^k  \left\|\textbf{F}_B^{(t)} \right\|_2^2 \right)\sum_{t\in I^+}   A_{it}A_{jt}, 
    \label{eq:inner_upper_bound}
\end{align}
and 
\begin{align}
\nonumber
  nm  \textbf{X}_A^{(i)} {\textbf{X}_A^{(j)}}^T      \geq   \; &  \min_{t = 1}^k\left\|\textbf{F}_B^{(t)} \right\|_2^2\sum_{t \in I^+}   A_{it}A_{jt} - \max_{t=1}^k  \left\|\textbf{F}_B^{(t)} \right\|_2^2\sum_{t\in I^-}   |A_{it}A_{jt} | \\
    =\; &   \max_{t =1}^k  \left\|\textbf{F}_B^{(t)} \right\|_2^2 \textbf{A}^{(i)} {\textbf{A}^{(j)}}^T - \left(\max_{t=1}^k  \left\|\textbf{F}_B^{(t)} \right\|_2^2 -\min_{t=1}^k  \left\|\textbf{F}_B^{(t)} \right\|_2^2 \right)\sum_{t\in I^+}   A_{it}A_{jt}.
    \label{eq:inner_lower_bound}
\end{align}
Also, we bound $nm \left\|\textbf{X}_A^{(i)} \right\|_2^2$ by 
\begin{equation}
\label{eq:norm-bound}
  \min_{t=1}^k \left\|\textbf{F}_B^{(t)} \right\|_2^2 \left\|\textbf{A}^{(i)} \right\|_2^2     \leq  nm \left\|\textbf{X}_A^{(i)} \right\|_2^2  =  \sum_{t=1}^k  A_{it}^2\left\|\textbf{F}_B^{(t)} \right\|_2^2 \leq \max_{t=1}^k \left\|\textbf{F}_B^{(t)} \right\|_2^2 \left\|\textbf{A}^{(i)} \right\|_2^2 
\end{equation}
Incorporate Eqs. (\ref{eq:inner_upper_bound})-(\ref{eq:norm-bound}) into $\text{cos}\left(\textbf{X}_A^{(i)}, \textbf{X}_A^{(j)}\right) = \frac{\textbf{X}_A^{(i)} {\textbf{X}_A^{(j)}}^T}{\left\|\textbf{X}_A^{(i)} \right\|_2\left\|\textbf{X}_A^{(j)} \right\|_2}$ and apply  Lemma \ref{RowLengthF}.
Define the following quantities 
\begin{equation}
    \epsilon_l = \frac{ u_s^2 -l_s^2 +\left(u_s^2 + l_s^2\right)\Delta}{l_s^2(1-\Delta)},\; \epsilon_u = \frac{ u_s^2 -l_s^2 +\left(u_s^2 + l_s^2\right)\Delta}{u_s^2(1+\Delta)},
\end{equation}
which satisfies $\epsilon_u< \epsilon_l $.
When $ \text{cos}\left(\textbf{A}^{(i)}, \textbf{A}^{(j)}\right) \geq 0$, we have 
\begin{align}
    \nonumber
        \text{cos}\left(\textbf{X}_A^{(i)}, \textbf{X}_A^{(j)}\right)   \leq \; & \frac{\min_{t =1}^k  \left\|\textbf{F}_B^{(t)} \right\|_2^2 }{\min_{t=1}^k  \left\|\textbf{F}_B^{(t)} \right\|_2^2 }    \text{cos}\left(\textbf{A}^{(i)}, \textbf{A}^{(j)}\right)    \\
         \nonumber
       &+\frac{\left(\max_{t=1}^k  \left\|\textbf{F}_B^{(t)} \right\|_2^2 -\min_{t =1}^k  \left\|\textbf{F}_B^{(t)} \right\|_2^2\right )}{\min_{t=1}^k  \left\|\textbf{F}_B^{(t)} \right\|_2^2} \times \frac{\sum_{t\in I^+}   A_{it}A_{jt}}{\left\|\textbf{A}^{(i)} \right\|_2\left\|\textbf{A}^{(j)} \right\|_2}\\
         \label{eq:cosine_shift_bound_pa}
       \leq \; &   \text{cos}\left(\textbf{A}^{(i)}, \textbf{A}^{(j)}\right) + \frac{\left(u_s^2 -l_s^2 +(u_s^2 + l_s^2)\Delta\right){s_AM_{\max}^{(A)}}^2 }{l_s^2(1-\Delta){kM_{\min}^{(A)}}^2} = \text{cos}\left(\textbf{A}^{(i)}, \textbf{A}^{(j)}\right)+\epsilon_l \frac{s_A{M_{\max}^{(A)}}^2}{k{M_{\min}^{(A)}}^2},
\end{align}   
and
\begin{align}
    \nonumber
        \text{cos}\left(\textbf{X}_A^{(i)}, \textbf{X}_A^{(j)}\right)   \geq \; & \frac{\max_{t =1}^k  \left\|\textbf{F}_B^{(t)} \right\|_2^2 }{\max_{t=1}^k  \left\|\textbf{F}_B^{(t)} \right\|_2^2 }    \text{cos}\left(\textbf{A}^{(i)}, \textbf{A}^{(j)}\right) \\
         \nonumber
       & - \frac{\left (\max_{t=1}^k  \left\|\textbf{F}_B^{(t)} \right\|_2^2 -\min_{t =1}^k  \left\|\textbf{F}_B^{(t)} \right\|_2^2\right )}{\min_{t=1}^k  \left\|\textbf{F}_B^{(t)} \right\|_2^2} \times \frac{\sum_{t\in I^+}   A_{it}A_{jt}}{\left\|\textbf{A}^{(i)} \right\|_2\left\|\textbf{A}^{(j)} \right\|_2}\\
       \label{eq:cosine_shift_bound_pb}
       \geq \; &    \text{cos}\left(\textbf{A}^{(i)}, \textbf{A}^{(j)}\right)  - \frac{\left(u_s^2 -l_s^2 +(u_s^2 + l_s^2)\Delta\right){s_AM_{\max}^{(A)}}^2 }{l_s^2(1-\Delta){k M_{\min}^{(A)}}^2} = \text{cos}\left(\textbf{A}^{(i)}, \textbf{A}^{(j)}\right) -\epsilon_l \frac{s_A{M_{\max}^{(A)}}^2}{k {M_{\min}^{(A)}}^2},
\end{align}  
Together,  Eqs. (\ref{eq:cosine_shift_bound_pa}) and (\ref{eq:cosine_shift_bound_pb}) result in 
\begin{equation}
\label{eq:cosine_shift_boundp}
    \left| \text{cos}\left(\textbf{X}_A^{(i)}, \textbf{X}_A^{(j)}\right) -\text{cos}\left(\textbf{A}^{(i)}, \textbf{A}^{(j)}\right) \right| \leq   \frac{\epsilon_ls_A{M_{\max}^{(A)}}^2}{k{M_{\min}^{(A)}}^2}.
\end{equation}
When $ \text{cos}\left(\textbf{A}^{(i)}, \textbf{A}^{(j)}\right) <0  $, we have 
\begin{align}
    \nonumber
        \text{cos}\left(\textbf{X}_A^{(i)}, \textbf{X}_A^{(j)}\right)   \leq \; & \frac{\min_{t =1}^k  \left\|\textbf{F}_B^{(t)} \right\|_2^2 }{\max_{t=1}^k  \left\|\textbf{F}_B^{(t)} \right\|_2^2 }    \text{cos}\left(\textbf{A}^{(i)}, \textbf{A}^{(j)}\right)  + \frac{\left(u_s^2 -l_s^2 +(u_s^2 + l_s^2)\Delta\right){s_A M_{\max}^{(A)}}^2 }{l_s^2(1-\Delta){k M_{\min}^{(A)}}^2}  \\
    \nonumber
         \leq \; &  \frac{l_s^2(1-\Delta)}{u_s^2(1+\Delta)} \text{cos}\left(\textbf{A}^{(i)}, \textbf{A}^{(j)}\right)  +\epsilon_l \frac{s_A{M_{\max}^{(A)}}^2}{k{M_{\min}^{(A)}}^2},  \\
    \nonumber
         =\;& (1-\epsilon_u)\text{cos}\left(\textbf{A}^{(i)}, \textbf{A}^{(j)}\right)  + \epsilon_l \frac{s_A{M_{\max}^{(A)}}^2}{k{M_{\min}^{(A)}}^2}, \\
   \label{eq:cosine_shift_bound_na}
         \leq\; & \text{cos}\left(\textbf{A}^{(i)}, \textbf{A}^{(j)}\right) +   \epsilon_l \left(\frac{s_A{M_{\max}^{(A)}}^2}{k{M_{\min}^{(A)}}^2} - \text{cos}\left(\textbf{A}^{(i)}, \textbf{A}^{(j)}\right)  \right),  
\end{align}   
and
\begin{align}
    \nonumber
        \text{cos}\left(\textbf{X}_A^{(i)}, \textbf{X}_A^{(j)}\right)   \geq \; & \frac{\max_{t =1}^k  \left\|\textbf{F}_B^{(t)} \right\|_2^2 }{\min_{t=1}^k  \left\|\textbf{F}_B^{(t)} \right\|_2^2 }    \text{cos}\left(\textbf{A}^{(i)}, \textbf{A}^{(j)}\right)  - \frac{\left(u_s^2 -l_s^2 +(u_s^2 + l_s^2)\Delta\right){s_AM_{\max}^{(A)}}^2 }{l_s^2(1-\Delta){k M_{\min}^{(A)}}^2} \\
\nonumber
       \geq \; &   \frac{u_s^2(1+\Delta)}{l_s^2(1-\Delta)}   \text{cos}\left(\textbf{A}^{(i)}, \textbf{A}^{(j)}\right)  -  \epsilon_l \frac{s_A{M_{\max}^{(A)}}^2}{k{M_{\min}^{(A)}}^2}, \\
  \nonumber 
         =\;& (1+\epsilon_l)\text{cos}\left(\textbf{A}^{(i)}, \textbf{A}^{(j)}\right)  - \epsilon_l \frac{s_A{M_{\max}^{(A)}}^2}{k{M_{\min}^{(A)}}^2}, \\
         \label{eq:cosine_shift_bound_nb}
          =\; & \text{cos}\left(\textbf{A}^{(i)}, \textbf{A}^{(j)}\right) -  \epsilon_l  \left(\frac{s_A{M_{\max}^{(A)}}^2}{k{M_{\min}^{(A)}}^2} -  \text{cos}\left(\textbf{A}^{(i)}, \textbf{A}^{(j)}\right)  \right).
\end{align}  
Meanwhile, for the case $ \text{cos}\left(\textbf{A}^{(i)}, \textbf{A}^{(j)}\right) <0  $, it has
\begin{equation}
    -  \text{cos}\left(\textbf{A}^{(i)}, \textbf{A}^{(j)}\right)= \frac{\left|\sum_{i=1}^k  A_{ik}A_{jk}\right| }{\sqrt{\sum_{t=1}^kA_{it}}\sqrt{\sum_{t=1}^kA_{jt}}}\leq \frac{s_A{M_{\max}^{(A)}}^2}{k{M_{\min}^{(A)}}^2}.
\end{equation}
The above, together with Eqs. (\ref{eq:cosine_shift_bound_na}) and (\ref{eq:cosine_shift_bound_nb}), results in 
\begin{equation}
\label{eq:cosine_shift_boundn}
    \left| \text{cos}\left(\textbf{X}_A^{(i)}, \textbf{X}_A^{(j)}\right) -\text{cos}\left(\textbf{A}^{(i)}, \textbf{A}^{(j)}\right) \right| \leq   \frac{2\epsilon_ls_A{M_{\max}^{(A)}}^2}{k{M_{\min}^{(A)}}^2}.
\end{equation}
Together, Eqs. (\ref{eq:cosine_shift_boundp}) and (\ref{eq:cosine_shift_boundn})  result in the upper bound on $\delta_{\text{cos}}^{(A)}$ in  Eq. (\ref{eq:cosine_shift_bound}), which holds with probability at least    $1-ke^{-\frac{C_B\Delta^2  \sigma_B^2m}{k  {M_{\max}^{(B)}}^2}}$ following from the high-probability bound on $\left\|\textbf{F}_B^{(j)} \right\|_2^2$.

Now we analyse the  changes of the   largest and smallest singular values of the latent relation matrix after transforming from $\textbf{S}$ to $\tilde{\textbf{S}} =  nm\textbf{L}_B \textbf{S}\textbf{L}_A$.
By definition, it has
\begin{equation}
 \sigma_{\max}\left(\tilde{\textbf{S}}\right)=    \left\|\tilde{\textbf{S}}\right\|_2  \leq  nm\left\|\textbf{L}_B\right\|_2 \left\|\textbf{S}\right\|_2\left\|\textbf{L}_A\right\|_2 = nmd_{\max}(\textbf{L}_A)d_{\max}(\textbf{L}_B)\sigma_{\max}\left(\textbf{S}\right).
\end{equation}
Starting from  $\textbf{S} = \frac{1}{nm}\textbf{L}_B^{-1}\tilde{\textbf{S}}\textbf{L}_A^{-1}$, we have
\begin{equation}
     \left\|\textbf{S}\right\|_2 \leq \frac{1}{nm}  \left\|\textbf{L}_B^{-1}\right\|_2 \left\|\tilde{\textbf{S}}\right\|_2\left\|\textbf{L}_A^{-1}\right\|_2 = \frac{1}{nm} d_{\min}(\textbf{L}_A)^{-1}d_{\min}(\textbf{L}_B)^{-1}\left\|\tilde{\textbf{S}}\right\|_2,
\end{equation}
which results in 
\begin{equation}
       \sigma_{\max}\left(\tilde{\textbf{S}}\right) \geq nm d_{\min}(\textbf{L}_A)d_{\min}(\textbf{L}_B)\sigma_{\max}\left(\textbf{S}\right).
\end{equation}
Since $ \sigma_{\min}(\textbf{S})^{-1} = \left\| \textbf{S}^{\dagger}\right\|_2$, we have 
\begin{equation}
      \left\|\tilde{\textbf{S}}^{\dagger}\right\|_2 = \frac{1}{nm}\left\|\textbf{L}_A^{-1} \textbf{S}^{\dagger}\textbf{L}_B^{-1}\right\|_2 \leq  \frac{1}{nm}\left\|\textbf{L}_A^{-1}\right\|_2 \left\|\textbf{S}^{\dagger}\right\|_2\left\|\textbf{L}_B^{-1}\right\|_2 =  \frac{1}{nm} d_{\min}(\textbf{L}_A)^{-1}d_{\min}(\textbf{L}_B)^{-1}  \sigma_{\min}\left( \textbf{S}\right)^{-1}.
\end{equation} 
Thus
\begin{equation}
    \sigma_{\min}\left(\tilde{\textbf{S}}\right)  \geq nm d_{\min}(\textbf{L}_A)d_{\min}  (\textbf{L}_B) \sigma_{\min}\left( \textbf{S}\right).
\end{equation}
Starting from  $\textbf{S}^{\dagger} = nm \textbf{L}_A\tilde{\textbf{S}}^{\dagger}\textbf{L}_B$, we have
\begin{equation}
    \left\| \textbf{S} ^{\dagger}\right\|_2 \leq nm \left\|\textbf{L}_A \right\|_2 \left\|\tilde{\textbf{S}}^{\dagger}\right\|_2 \left\|\textbf{L}_B\right\|_2 = nm d_{\max}(\textbf{L}_A)d_{\max}  (\textbf{L}_B) \sigma_{\min}\left( \tilde{\textbf{S}}\right)^{-1},
\end{equation}
resulting in
\begin{equation}
    \sigma_{\min}\left(\tilde{\textbf{S}}\right)  \leq nm d_{\max}(\textbf{L}_A)d_{\max}  (\textbf{L}_B) \sigma_{\min}\left( \textbf{S}\right).
\end{equation}
Combining the above, it has
\begin{align}
     nm d_{\min}(\textbf{L}_A)d_{\min}(\textbf{L}_B)\sigma_{\max}\left(\textbf{S}\right) \leq \; & \sigma_{\max}\left(\tilde{\textbf{S}}\right) \leq nmd_{\max}(\textbf{L}_A)d_{\max}(\textbf{L}_B)\sigma_{\max}\left(\textbf{S}\right), \\
      nm d_{\min}(\textbf{L}_A)d_{\min}  (\textbf{L}_B) \sigma_{\min}\left( \textbf{S}\right) \leq \; &\sigma_{\min}\left(\tilde{\textbf{S}}\right) \leq  nm d_{\max}(\textbf{L}_A)d_{\max}  (\textbf{L}_B) \sigma_{\min}\left( \textbf{S}\right).
\end{align}
Subsequently, we have 
\begin{equation}
  \frac{ d_{\min}(\textbf{L}_A)d_{\min}(\textbf{L}_B)\sigma_{\max}\left(\textbf{S}\right) }{ d_{\max}(\textbf{L}_A)d_{\max}(\textbf{L}_B) \sigma_{\min}\left(\textbf{S}\right)}   \leq   \frac{\sigma_{\max}\left(\tilde{\textbf{S}}\right)}{\sigma_{\min}\left(\tilde{\textbf{S}}\right)}  \leq \frac{ d_{\max}(\textbf{L}_A)d_{\max}(\textbf{L}_B) \sigma_{\max}\left(\textbf{S}\right)  }{ d_{\min}(\textbf{L}_A)d_{\min}(\textbf{L}_B) \sigma_{\min}\left(\textbf{S}\right)}   .
\end{equation}
Incorporating $ \textbf{L}_A = \textbf{F}_A^{-1}$ and $ \textbf{L}_B = \textbf{F}_B^{-1}$ to the above, it has 
\begin{equation}
  \frac{ d_{\min}(\textbf{F}_A)d_{\min}(\textbf{F}_B) }{ d_{\max}(\textbf{F}_A)d_{\max}(\textbf{F}_B) }   \leq  \frac{\kappa\left(\tilde{\textbf{S}}\right)}{ \kappa(\textbf{S})}   \leq \frac{ d_{\max}(\textbf{F}_A)d_{\max}(\textbf{F}_B)   }{ d_{\min}(\textbf{F}_A)d_{\min}(\textbf{F}_B) } .
\end{equation}
Applying Lemma \ref{RowLengthF}, it then has 
\begin{equation}
 \frac{l_s}{u_s}\sqrt{\frac{1-\Delta}{1+\Delta}} \leq  \sqrt{\frac{\kappa\left(\tilde{\textbf{S}}\right)}{\kappa\left(\textbf{S}\right)}} \leq \frac{u_s}{l_s}\sqrt{\frac{1+\Delta}{1-\Delta}},
\end{equation}
 with probability at least    $1-ke^{-\frac{C_A\Delta^2  \sigma_A^2n}{k  {M_{\max}^{(A)}}^2}}-ke^{-\frac{C_B\Delta^2  \sigma_B^2m}{k  {M_{\max}^{(B)}}^2}}$, which follows from the high-probability bounds on both $\left\|\textbf{F}_A^{(j)} \right\|_2^2$ and $\left\|\textbf{F}_B^{(j)} \right\|_2^2$.
\end{proof}

\bibliography{ref}
\end{document}